\documentclass{article}

\usepackage[final]{neurips_2025}

\usepackage[utf8]{inputenc} 
\usepackage[T1]{fontenc}    
\usepackage{hyperref}       
\usepackage{url}            
\usepackage{booktabs}       
\usepackage{amsfonts}       
\usepackage{nicefrac}       
\usepackage{microtype}      
\usepackage{xcolor}         

\usepackage{amsmath}
\usepackage{amssymb}
\usepackage{mathtools}
\usepackage{amsthm}

\usepackage{amsmath,amsfonts,bm}

\def\eqref#1{equation~\ref{#1}}
\def\Eqref#1{Equation~\ref{#1}}

\def\1{\bm{1}}

\DeclareMathAlphabet{\mathsfit}{\encodingdefault}{\sfdefault}{m}{sl}
\SetMathAlphabet{\mathsfit}{bold}{\encodingdefault}{\sfdefault}{bx}{n}

\def\sD{{\mathbb{D}}}

\DeclareMathOperator*{\argmax}{arg\,max}
\DeclareMathOperator*{\argmin}{arg\,min}

\usepackage{amsfonts, optidef}
\usepackage{bbm}

\theoremstyle{plain}
\newtheorem{theorem}{Theorem}[section]
\newtheorem{proposition}[theorem]{Proposition}
\newtheorem{lemma}[theorem]{Lemma}
\newtheorem{corollary}[theorem]{Corollary}
\theoremstyle{definition}
\newtheorem{definition}[theorem]{Definition}
\newtheorem{assumption}[theorem]{Assumption}
\theoremstyle{remark}

\usepackage[utf8]{inputenc} 
\usepackage[T1]{fontenc}    
\usepackage{enumitem}

\usepackage{xcolor}         
\definecolor{goodcolor}{RGB}{9, 160, 63}
\definecolor{goodorange}{RGB}{255, 127, 0}
\definecolor{goodpurple}{RGB}{228, 26, 28}
\definecolor{goodblue}{RGB}{55, 126, 184}
\definecolor{goodgreen}{RGB}{77, 175, 74}
\definecolor{goodred}{RGB}{220, 20, 60}

\usepackage{url}            
\usepackage{booktabs}       
\usepackage{nicefrac}       
\usepackage{microtype}      
\newcommand\myshade{90}
\hypersetup{
    colorlinks=true,
    linkcolor=goodcolor!\myshade!black,
    citecolor=goodcolor!\myshade!black,    
    urlcolor=goodcolor!\myshade!black,
}

\usepackage{graphicx}
\usepackage{bm}
\usepackage{soul}
\usepackage{array} 
\usepackage{multirow}
\usepackage{booktabs}
\usepackage{makecell}

\usepackage{caption}
\usepackage{subcaption}
\usepackage{tabularx}

\usepackage{algorithm}
\usepackage{algpseudocode}
\usepackage{wrapfig}

\newcommand{\maxucb}[1]{\colorbox{goodblue!20}{#1}}
\newcommand{\ucb}[1]{\colorbox{goodred!20}{#1}}
\newcommand{\OURALGO}[0]{MaxUCB}

\newcommand{\tabrepo}[1][]{TabRepo\ifx\relax#1\relax\else[#1]\fi}
\newcommand{\yahpogym}[1][]{YaHPOGym\ifx\relax#1\relax\else[#1]\fi}
\newcommand{\tabreporaw}[1][]{TabRepoRaw\ifx\relax#1\relax\else[#1]\fi}
\newcommand{\hebo}[1][]{Reshuffling\ifx\relax#1\relax\else[#1]\fi}
\newcommand{\numberOfBenchmarks}{four}

\newcommand{\QuantileBayesUCB}{\textit{Quantile~Bayes~UCB}}
\newcommand{\ERUCBS}{\textit{ER-UCB-S}}
\newcommand{\RisingBandits}{\textit{Rising~Bandits}}
\newcommand{\QoMaxSDA}{\textit{QoMax-SDA}}
\newcommand{\MaxMedian}{\textit{Max-Median}}
\newcommand{\UCB}{\textit{UCB}}
\newcommand{\SuccessiveHalving}{\textit{Successive~Halving}}
\newcommand{\randomsearch}{\textit{random~search}}
\newcommand{\SMAC}{\textit{SMAC}}
\newcommand{\combinedsearch}{\textit{combined search}}
\newcommand{\decomposedCASH}{\textit{two-level CASH}}
\newcommand{\ourProposition}{Lemma~\ref{theorem:lemma}}

\title{Put CASH on Bandits: A Max K-Armed Problem for \\
Automated Machine Learning}

\author{%
  Amir Rezaei Balef, Claire Vernade and Katharina Eggensperger \\
  Department of Computer Science, University of Tübingen \\
  \texttt{\{amir.rezaei-balef, claire.vernade, katharina.eggensperger\}@uni-tuebingen.de} \\
}

\begin{document}

\maketitle

\begin{abstract}
The Combined Algorithm Selection and Hyperparameter optimization (CASH) is a challenging resource allocation problem in the field of AutoML. We propose \OURALGO{}, a max $k$-armed bandit method to trade off exploring different model classes and conducting hyperparameter optimization. \OURALGO{} is specifically designed for the light-tailed and bounded reward distributions arising in this setting and, thus, provides an efficient alternative compared to classic max $k$-armed bandit methods assuming heavy-tailed reward distributions. We theoretically and empirically evaluate our method on four standard AutoML benchmarks demonstrating superior performance over prior approaches. 
We make our code and data available at \url{https://github.com/amirbalef/CASH_with_Bandits}.
\end{abstract}
\section{Introduction} 
\label{Sec:Introduction}
The performance of machine learning (ML) solutions is highly sensitive to the choice of algorithms and their hyperparameter configurations which can make finding an effective solution a challenging task. AutoML aims to reduce this complexity and make ML more accessible by automating these critical choices~\citep{hutter-book19a,baratchi-air24a}. 
For example, Hyperparameter optimization (HPO) methods focus on finding well-performing hyperparameter settings given a resource constraint, such as an iteration count or a time limit. 
However, in practice, it is often unclear which ML model class would perform best on a given dataset \citep{bischl2025}. The problem of jointly searching the model class and the appropriate hyperparameters has been coined CASH, Combined Algorithm Selection and Hyperparameter optimization~\citep{thornton-kdd13a}. As a prime example, on tabular data, a ubiquitous data modality~\citep{van2024tabular}, 
the state-of-the-art ML landscape covers classic ML methods, ensembles of gradient-boosted decision trees and modern deep learning approaches~\citep{kadra-neurips21a,gorishniy-neurips21a,gorishniy-iclr24a,mcelfresh-neuripsdbt23a,kohli2024towards,hollmann-iclr23a,holzmueller-neurips24a}.

\begin{figure}[htbp]
\centering
\includegraphics[height=3cm]{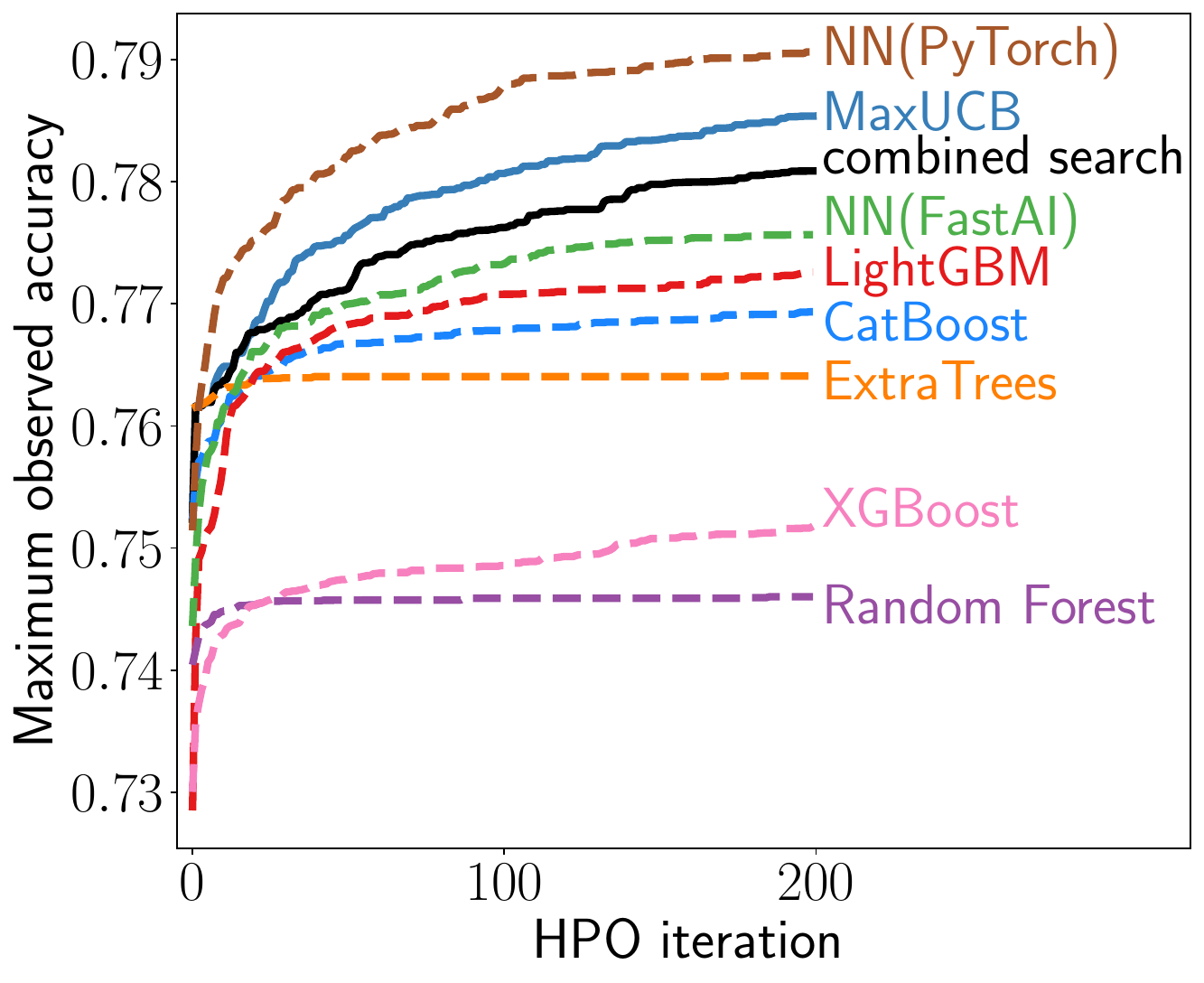}
\includegraphics[height=3.05cm]{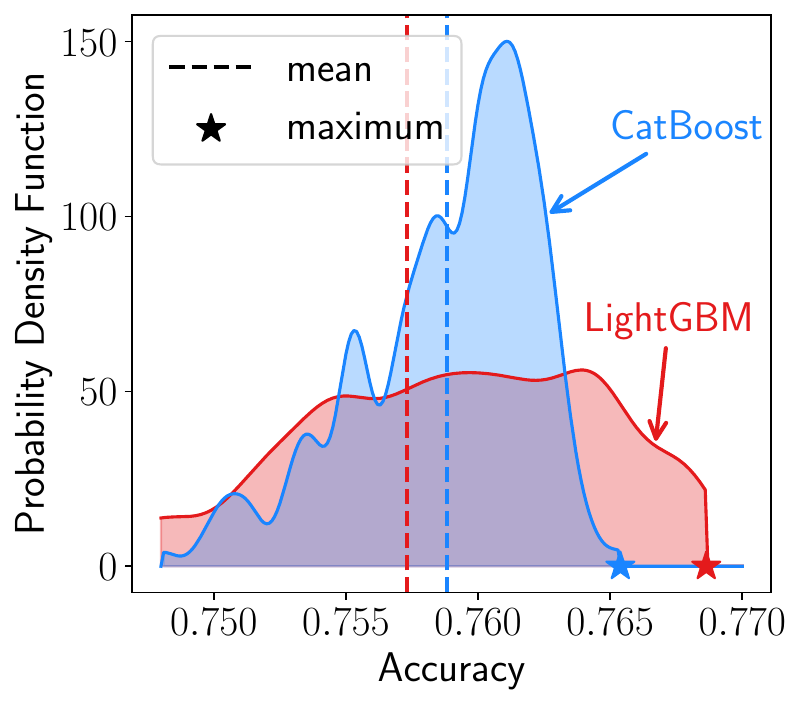}
\includegraphics[clip, trim=0.0cm -0.7cm 0.0cm 0cm,,height=3cm]{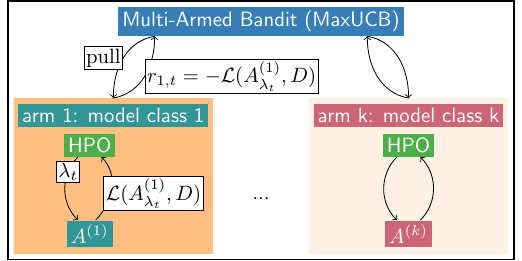}
\caption{(Left) \OURALGO{} (blue line) \textit{outperforms} \combinedsearch{} (black line) to identify the best-performing model class (brown line). (Middle) The irregular distribution of the empirical performance of model classes is left-skewed, and \textit{a higher mean may not correspond to a higher maximum}.
(Right) \OURALGO{} selects for which model to run one iteration of HPO during \textit{two-level optimization}.}
\label{fig:HPO_runs}
\end{figure}

A popular approach to address the CASH problem is to use categorical and conditional hyperparameters to run HPO directly on the combined hierarchical search space of models and hyperparameters. AutoML systems use this approach, which we call  \combinedsearch{}, to search well-performing ML pipelines~\citep{thornton-kdd13a,feurer-nips15a, komer-automl14a,kotthoff2017auto,feurer-jmlr22a}, but HPO remains inefficient in high-dimensional and hierarchical search spaces. 
A naive solution to address the scalability limitation is to run HPO independently for the smaller search spaces of each ML model class and then compare the found solutions. However, this solution often exceeds available computational resources and does not scale well with an increasing number of ML models.
Figure~\ref{fig:HPO_runs} (left) illustrates the difference between searching each space individually (colored dashed lines) and \combinedsearch{} (black line) on an exemplary dataset.

To leverage the efficiency of HPO in low dimensions, we use a Multi-Armed Bandit (MAB) method \citep{lattimore2020bandit} to dynamically allocate our budget. As shown in Figure~\ref{fig:HPO_runs} (right), each time the bandit strategy pulls an arm, it runs one iteration of HPO to evaluate a new configuration, resulting in a loss. The negative of the loss is used as reward feedback for the bandit algorithm. This approach is known as \decomposedCASH{} (\textit{decomposed CASH}) ~\citep{hoffman2014correlation, liu2020admm}.

While most classical MAB problems aim to maximize the average rewards over time, the goal of the bandit algorithm for \decomposedCASH{} should be to maximize the maximum reward observed over time: as illustrated in Figure~\ref{fig:HPO_runs} (middle), tuning the LightGBM (red) will eventually outperform CatBoost. This goal aligns with Max $K$-Armed Bandit (MKB) problems \citep{carpentier2014extreme,achab2017max,baudry2022efficient}, often referred to as Extreme Bandits.

We precisely address this open question through a thorough statistical analysis of the empirical reward distributions of HPO tasks. 
\textbf{Our main contribution is a state-of-the-art algorithm for \decomposedCASH{}  based on a novel extreme bandit algorithm we call \OURALGO{}} (Algorithm \ref{alg:pseudocode_MaxUCB}). We demonstrate the performance of our method on \numberOfBenchmarks{} benchmarks, which highlights the relevance of our assumptions for a wide variety of CASH problems. We analyze the theoretical performance of \OURALGO{} (Theorem \ref{theorem:number_suboptimal_MaxUCB}) through regret bounds that also justify our novel choice of exploration bonus for the type of distributions relevant to the CASH problem. Importantly, our objective is rather to propose a practical algorithm with good empirical performance on CASH rather than a general multi-purpose bandit algorithm, so our guarantees hold under carefully crafted assumptions that resolve previously open questions \citep{nishihara2016no}. 

\section{Solving CASH using Bandits}
\label{Sec:Problem_Formulation}
The CASH problem for supervised learning tasks is defined as follows~\citep{thornton-kdd13a}. Given a dataset $\sD = \{ D_{train}, D_{valid}\}$ of a supervised learning task, let $\mathcal{A} = \{A^{(1)}, ..., A^{(K)} \}$ be the set of $K$ candidate ML algorithms, where each algorithm $A^{(i)}$ has its own hyperparameter search space $\pmb{\Lambda}^{(i)}$. The goal is to search the joint algorithm and hyperparameter configuration space to find the optimal algorithm $A^{(i^*)}$ and its optimal hyperparameter configuration $\pmb{\lambda}^{*}$ that minimizes a loss metric $\mathcal{L}$, e.g., the validation error\footnote{We note that $\mathcal{L}$ can also be the result of k-fold cross-validation or other evaluation protocols measuring the expected performance of a model on unseen data~\citep{raschka-arxiv20a}.}. Formally,
\begin{equation}
A^{(i^*)}_{\pmb{\lambda}^{*}} \in \argmin_{ A^{(i)} \in \mathcal{A}, \pmb{\lambda} \in \pmb{\Lambda}^{(i)}  } \mathcal{L}(A^{(i)}_{\pmb{\lambda}}  ,\sD).
\end{equation}
For our approach, we study the decomposed variant \citep{hoffman2014correlation, liu2020admm} and address the following two-level optimization problem depicted in Figure~\ref{fig:HPO_runs} (right): at the upper level, we aim to find the overall best-performing ML model $A^{(i^*)}$ by selecting model  $A^{(i)} \in \mathcal{A}$ iteratively, and at the lower level, we aim to find the best-performing configuration $\pmb{\lambda^{*}}\in\pmb{\Lambda}^{(i)}$ for the selected model $A^{(i)}$. Formally,

\begin{align}
A^{(i^*)} \in \argmin_{A^{(i)} \in \mathcal{A}} \mathcal{L}(A^{(i)}_{\pmb{\lambda}^*}, \sD), \quad \text{s.t.} \quad \pmb{\lambda}^{*} \in \argmin_{\pmb{\lambda} \in \pmb{\Lambda^{(i)}}} \mathcal{L}(A^{(i)}_{\pmb{\lambda}}, \sD). \label{eq:decomposed_optimzation}
\end{align}

The right-hand side of \Eqref{eq:decomposed_optimzation} in the lower level can be efficiently addressed by existing iterative HPO methods such as Bayesian optimization (BO)~\citep{jones-jgo98a,garnett-book22a}, which has been demonstrated to perform well in practical settings~\citep{snoek-nips12a,chen-arxiv18a,cowenrivers-jair22a}. BO fits a surrogate model and uses an acquisition function to find a promising configuration to evaluate next. 
On the upper level, or left-hand side of \Eqref{eq:decomposed_optimzation}, the challenge is to carefully allocate the budget $T$ of HPO runs to the $K$ models in a manner that trades off exploration of the hyperparameter space of all models and exploitation (optimization) of the most promising model. As already noted in previous work, this is a typical MAB problem~\citep{cicirello-aaai05a, streeter2006asymptotically,nishihara2016no,metelli2022stochastic}.

At time $t$, the bandit algorithm chooses model $I_t \in \mathcal{A}$, and we denote $\pmb{\lambda}_t$ the configuration proposed by the HPO method in the lower level. As a reward $r_{i,t}$, we feed back to the bandit algorithm an evaluation of the negative loss:
\begin{equation}
\label{eq:rewards}
r_{i,t} =  -\mathcal{L}( A^{(i)}_{\pmb{\lambda}_t}  ,\sD). 
\end{equation}
In general, and as opposed to standard MAB, this reward process is not i.i.d. conditionally on the arm choices because the loss of the models depends on the progress of HPO on each model class (arm) as well as additional loss evaluation noise. To be able to design a tractable bandit algorithm, it is crucial to find an appropriate way to model this process to build controllable estimators. We focus on this aspect in the next section. 

To complete the bandit model of this problem, we need to choose a regret metric that defines the oracle objective we compare to, and indeed aligns with \Eqref{eq:decomposed_optimzation}. For HPO, the regret should target max-value objectives~\citep{jamieson2016non, nishihara2016no}. 
\begin{equation}
R(T) = \max_{k\leq K}\mathop{\mathbbm{E}[\max r_{k,t}]}_{t \leq T} - \mathop{\mathbbm{E}[\max r_{I_t,t}]}_{t \leq T}. 
\label{eq:extremebandit_regret}
\end{equation}
This regret describes the gap between the highest-possible oracle reward that could be obtained by pulling only the arm with the highest performance (left part) and the actual observed rewards obtained by applying our strategy (right part). Notably, the expectation is necessary to account for the inherent stochasticity of both the HPO procedure (e.g., random search or Bayesian optimization) and the ML models themselves (e.g., random initialization and training variability).

Instead of using MKB algorithms to directly address~\Eqref{eq:extremebandit_regret}, related prior works focus on alternative methods.
For example,~\citet{hu2021cascaded} proposed and analyzed the \textit{Extreme-Region Upper Confidence Bounds (ER-UCB)} algorithm maximizing the extreme region of the feedback distribution, assuming Gaussian rewards.
More recently, \citet{balef2024towards} have shown that existing MKB algorithms underperform when applied to \decomposedCASH{}, and they proposed methods for maximizing the quantile values instead of the maximum value.

As another alternative method, \citet{li2020efficient} framed the CASH problem as a Best Arm Identification (BAI) task and introduced the \textit{Rising Bandits} algorithm~\citep{li2020efficient}.
This MAB method assumes that the reward function for each arm increases with each pull, following a rested bandit model with non-decreasing payoffs~\citep{heidari2016tight} (which has been shown to have linear regret when the reward increment per pull exceeds a threshold~\citep{metelli2022stochastic}). \textit{Rising Bandits} can be used for our setting using the maximum observed performance of the HPO history as the reward. However, this algorithm assumes deterministic rewards and increasing concave reward functions. To weaken this assumption,~\citet{li2020efficient} introduced a hyperparameter to increase initial exploration.  \citet{mussi2024best} further weakened this assumption by assuming that the moving average of the rewards is an increasing concave function.

Generally, in BAI approaches, the goal is to identify the arm with the highest mean reward. Here, rewards are the result of HPO runs, so this objective does not align well with \Eqref{eq:decomposed_optimzation}: there is no reason to measure the quality of a model on average over a random subset of hyperparameters chosen by HPO. This approximation made by prior work is justified by the complexity of this modeling problem and the existence of solid foundations on BAI to build upon. But in this work, we propose a fully data-driven model that better fits the true CASH objective and results in better empirical results. 

\section{Data Analysis of HPO Tasks}
\label{sec:dataanalysis}
The reward process in \Eqref{eq:rewards} is complex as it depends on the model and on the chosen HPO algorithm. So, as discussed, it is necessary to model it. Rather than choosing a convenient parametric family, we conduct a thorough analysis of typical sequences of losses obtained on real benchmarks.  For each ML model class (corresponding to arms in our setup), we run $T=200$ iterations of HPO, with $32$ repetitions (each using a different seed) on a varying number of datasets on \numberOfBenchmarks{} AutoML benchmarks (see Appendix~\ref{app:experimentalsetup} and ~\ref{app:reward_distribution_analysis}).

We first analyze the \textit{survival function} of the reward distributions. Recall that for a random variable $X\sim d$, the survival function is defined by:   
$
x \mapsto G(x)=P_{X\sim d}(X \geq x).
$
Figure \ref{fig:HPO_ecdf_arms_distributions} shows the average empirical survival function of all observed performances (normalized between $0$ and $1$ for each task) for each arm. We rank model classes (i.e., arms) based on their best performance per dataset and report results on two benchmarks (see Appendix~\ref{app:details_on_assumptions} for more details and results on all benchmarks).
We observe that the reward distributions are 
\begin{enumerate}
    \item \textbf{bounded}. The reward, which measures the performance of a model, is determined by a score metric, e.g., accuracy. The extreme values vary for each arm and depend on the capability of the model class and the complexity of the task. Therefore, even if we run HPO indefinitely, achieving an infinite reward is impossible.
    Consequently, each arm has a bounded support with different maximum values, and at least a single optimal arm exists.
    $\rightarrow$ \textbf{We can define a sub-optimality gap (Definition ~\ref{def:suboptimality_gap}).} 
    \item \textbf{short-tailed, left-skewed, and the right tail is not heavy}.
    The rewards are concentrated near the maximal value per model class, and extreme events are not outliers. As the HPO method's performance reaches a certain level, further optimization often yields only small gains as optima tend to have flat regions~\citep{pushak-acm22a}.
    Therefore, many configurations perform similarly well, resulting in a skewed distribution.\footnote{This has also been observed for related tasks, e.g., neural architecture search on CIFAR-10~\citep{su2021prioritized}.}
    $\rightarrow$ \textbf{We expect to observe many extreme rewards.}
    \item \textbf{nearly stationary}. This means that the optimal arm does not change over time. This is clear for \tabreporaw{}. We observe significant changes in the distributions for \yahpogym{}, but most of the sub-optimal arms remain sub-optimal over time.
    $\rightarrow$ \textbf{Ranking of well-performing arms does not change over time.}
\end{enumerate}

\begin{figure}[htb]
\centering
\setlength{\tabcolsep}{0pt}
\renewcommand{\arraystretch}{0} 
\begin{tabular}{cccc}
\yahpogym{}
&\tabreporaw{}
&\yahpogym{}
&\tabreporaw{}\\
\includegraphics[clip, trim=0.0cm 0cm 0cm 0cm, height=2.05cm]{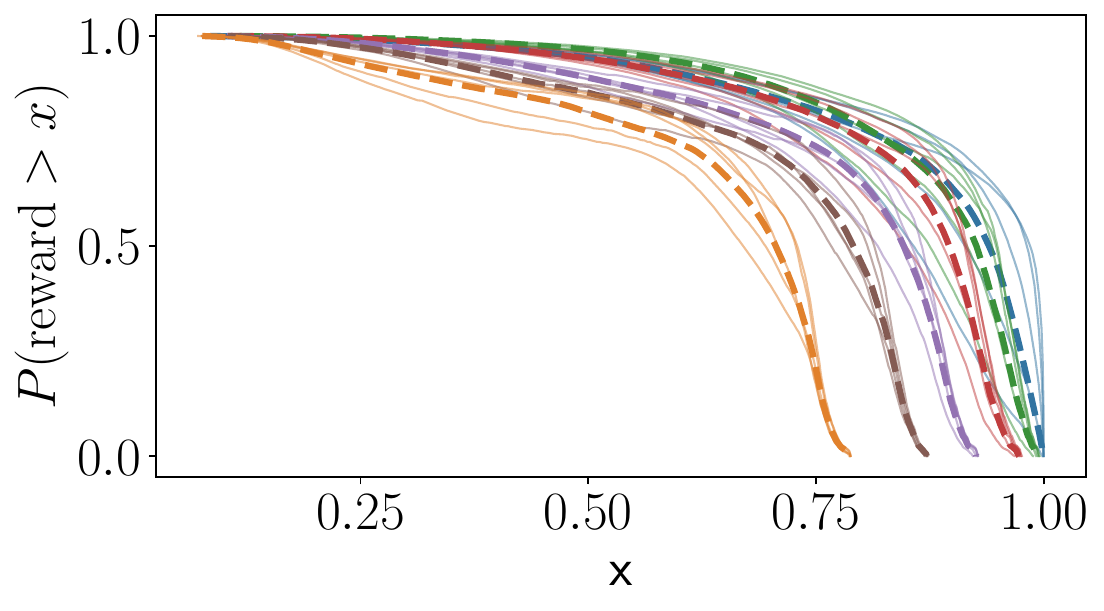}
&\includegraphics[clip, trim=2.5cm 0cm 60cm 0cm, height=2.05cm]{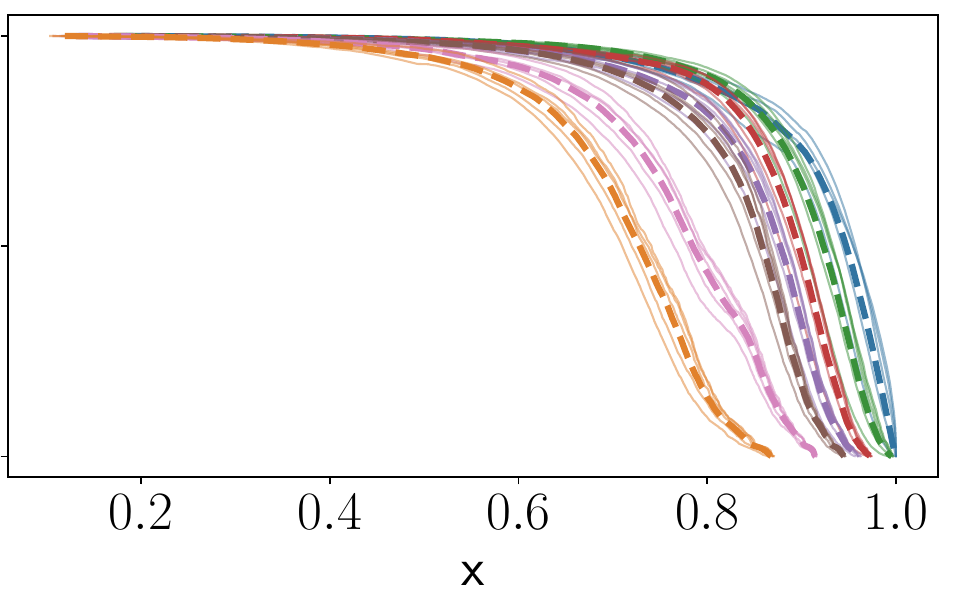}
&\includegraphics[clip, trim=0.0cm 0cm 0cm 0cm, height=2.05cm]{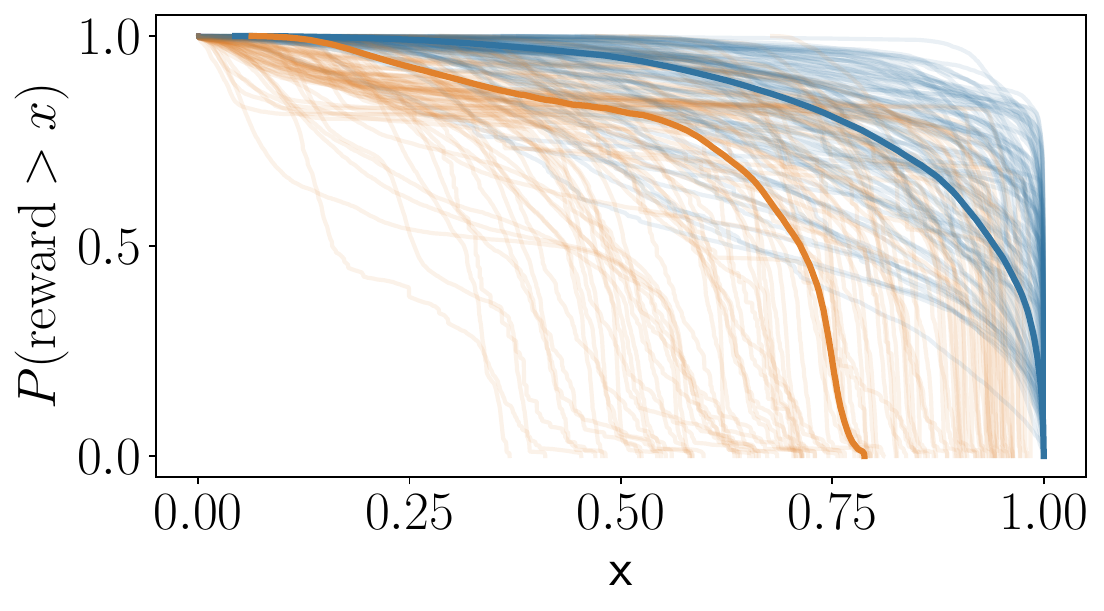}
&\includegraphics[clip, trim=2.5cm 0cm 0cm 0cm, height=2.05cm]{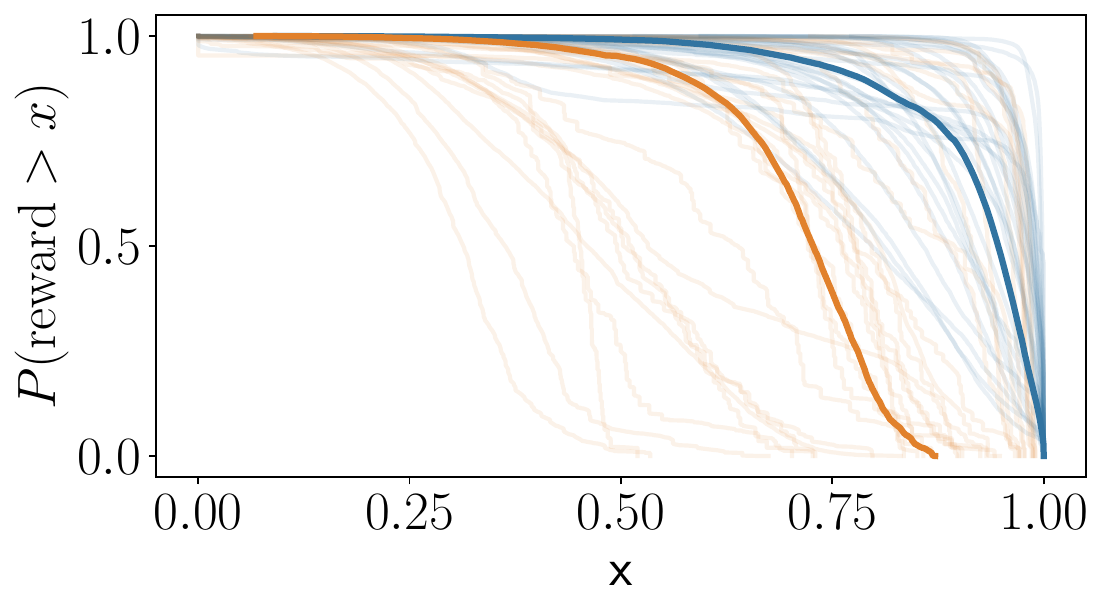} 
\\
\multicolumn{4}{c}{
\includegraphics[clip, trim=18.5cm 3.5cm 0cm 2cm, height=1.0cm]{fig_reward_dist__TabRepoRaw_HPO_eCDF_shifts.pdf}
}
\end{tabular} 
\caption{(Left) The average empirical survival function of the reward distribution per arm ranked per dataset.
Thin lines correspond to segments of the reward sequence and show the distribution change over time.
(Right) The average empirical survival function per dataset for the best and worst arm. Thin lines correspond to individual datasets.
\label{fig:HPO_ecdf_arms_distributions}
}
\end{figure}

Prior research has shown that existing MKB algorithms can underperform if their assumptions about the distributions are too weak or misaligned~\citep{nishihara2016no, balef2024towards}. This underperformance is largely due to our second observation, which significantly differs from the common assumptions of the existing algorithms.

\textbf{Preliminaries.}
Based on these observations, we can formulate the following definition and assumption on which we develop and analyze our algorithm. 
\begin{definition}
\label{def:suboptimality_gap}
The \emph{suboptimality gap} $\Delta_i \ge 0$ for arm \(i\) is defined as:
\[
\Delta_i = \mathbbm{E}\left[\max_{t \leq T} r_{i^*,t}\right] - \mathbbm{E}\left[\max_{t \leq T} r_{i,t}\right]
\]
where \(r_{i^*,t}\) and \(r_{i,t}\) are the rewards observed from optimal arm $i^*$ and arm \(i\) at time \(t\), respectively.
\end{definition}

\begin{assumption} We assume that the i.i.d. random variable \(X\), representing the rewards, follows a bounded distribution with support in \([a, b]\) and continuous survival function \(G\).
\label{theorem:assumption}
\end{assumption}
\begin{lemma}\label{theorem:lemma}
Suppose Assumption~\ref{theorem:assumption} holds. Then, there exists $L,U \geq 0$ such that the survival function \( G \) can be bounded near $b$ by two linear functions:
\begin{align}
\forall  \epsilon \in (0,b-a), \quad L \epsilon \leq  G( b - \epsilon) \leq U\epsilon.
\end{align}
\end{lemma}

\ourProposition{} (proof in Appendix~\ref{app:our_proposition_proof}) provides a way to characterize the shape of any bounded distribution near its maximum value through distribution-dependent constants, $L$ and $U$ (see Appendix~\ref{app:examples_validation_assumption} for more details and a visualization). Intuitively, it quantifies the behavior of the distribution of the ML model performance in a given hyperparameter space near the optimum. A large value of $L$  indicates a steep drop in the survival function near the maximum, while a small value of $U$  leads to a more gradual decay of the survival function and conversely (see Figure \ref{fig:L_U_distributions}).

\vspace{0.3cm}
\begin{minipage}[t]{0.48\textwidth}
\centering
\setlength{\tabcolsep}{0pt}
\renewcommand{\arraystretch}{0} 
\begin{tabular}{c c c}
\small\yahpogym{}
&\small\tabreporaw{}
& {} \\
\includegraphics[clip, trim=0.0cm 0cm 0cm 0cm, height=2.8cm]{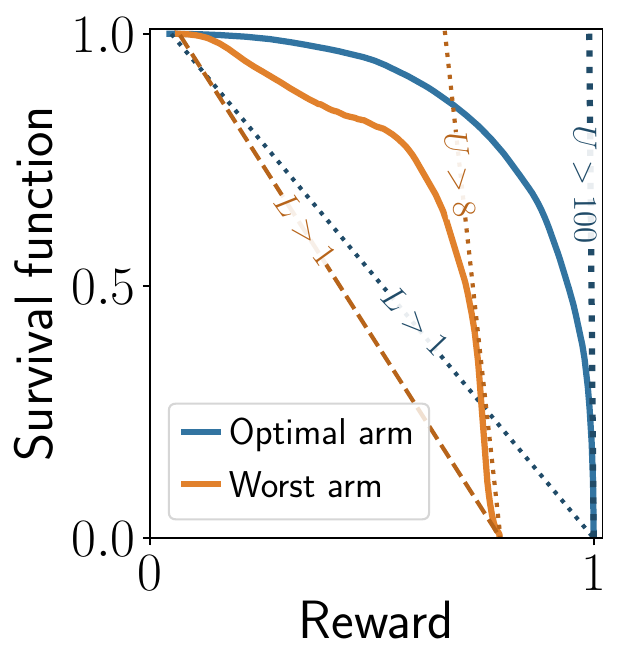}
&\includegraphics[clip, trim=2.3cm 0cm 0cm 0cm, height=2.8cm]{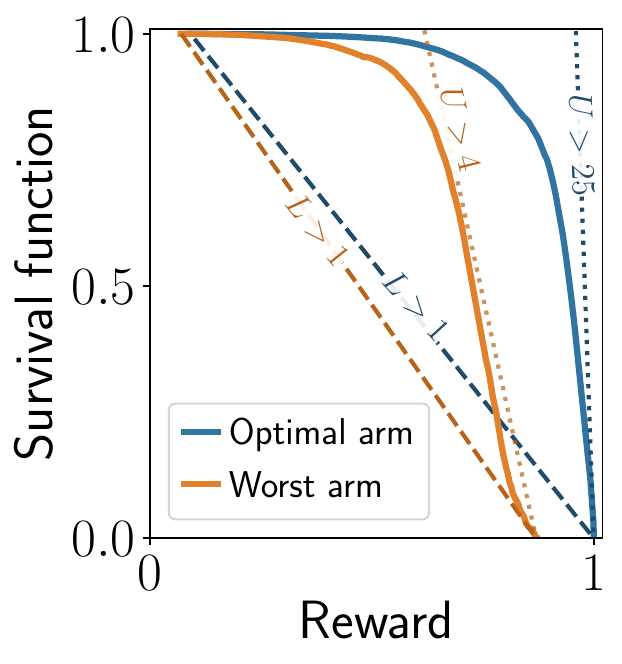}
&\includegraphics[clip, trim=2.3cm 0cm 0cm 0cm, height=2.8cm]{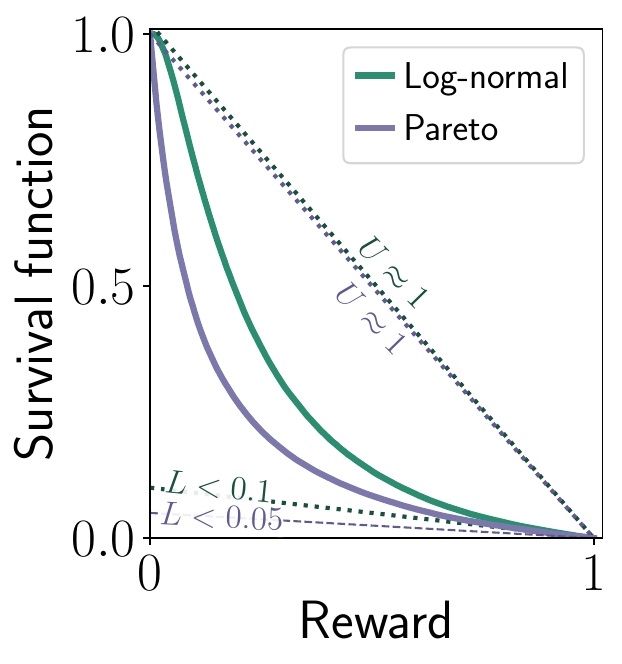}
\end{tabular} 

\captionof{figure}{(Left, Middle) L and U for the average survival functions in Figure~\ref{fig:HPO_ecdf_arms_distributions}. (Right) We highlight the difference to right-skewed Log-normal and Pareto distributions.}
\label{fig:L_U_distributions}
\end{minipage} \hfill
\begin{minipage}[t]{0.48\textwidth}
\centering
\setlength{\tabcolsep}{0pt}
\renewcommand{\arraystretch}{0} 
\begin{tabular}{cc}
\small\yahpogym{}
& \small\tabreporaw{} \\
\includegraphics[clip, trim=0.0cm 0cm 0cm 0cm, height=2.cm]{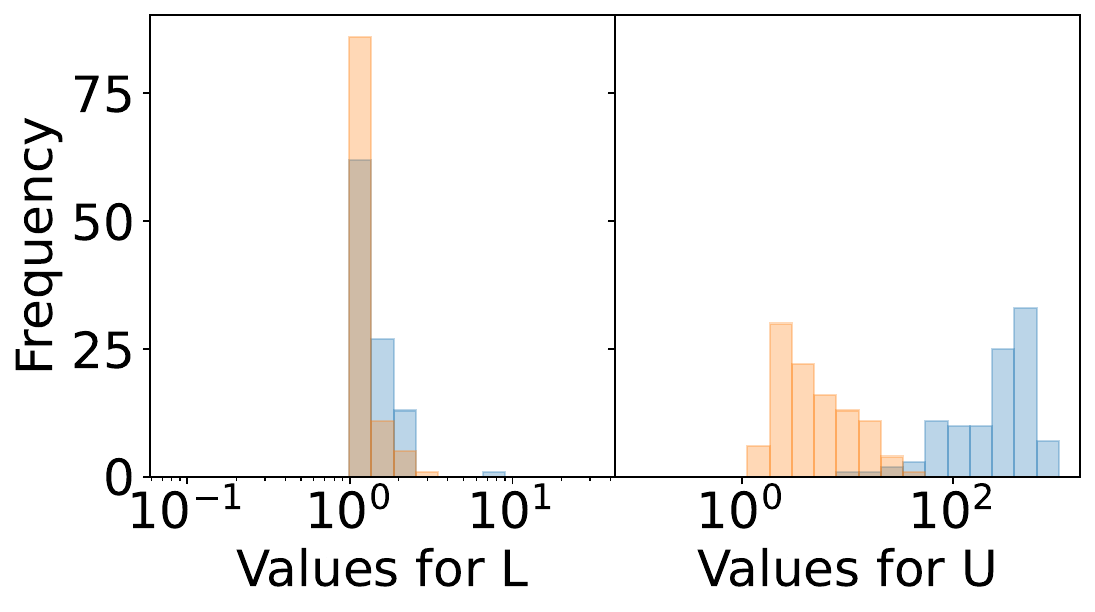}
&\includegraphics[clip, trim=1.2cm 0cm 0cm 0cm, height=2cm]{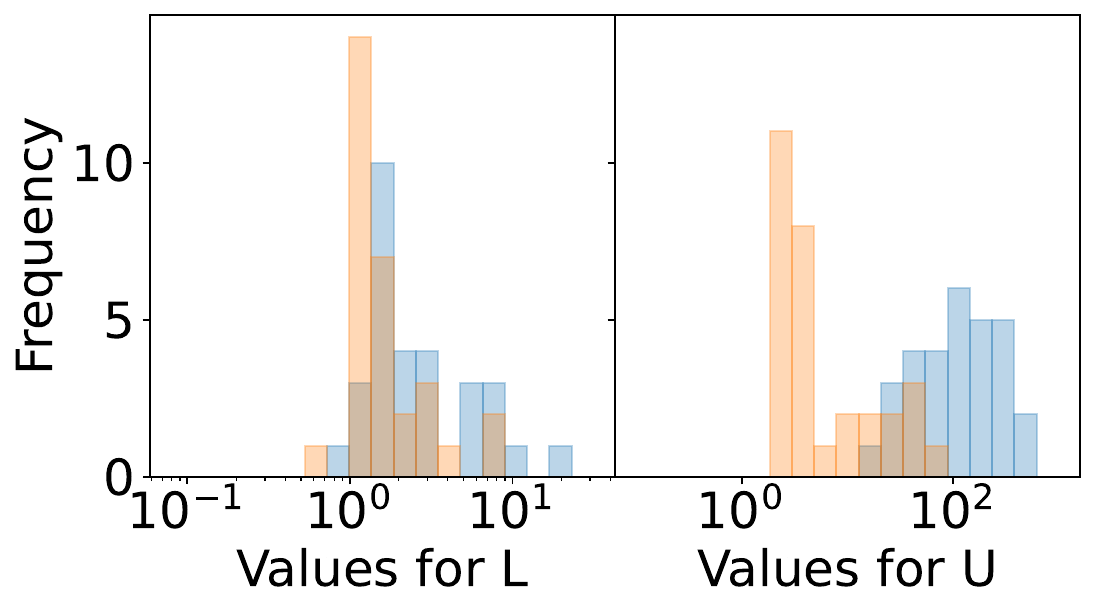} \\
\multicolumn{2}{c}{\includegraphics[clip, trim=0cm 0cm 0cm 0cm, height=0.5cm]{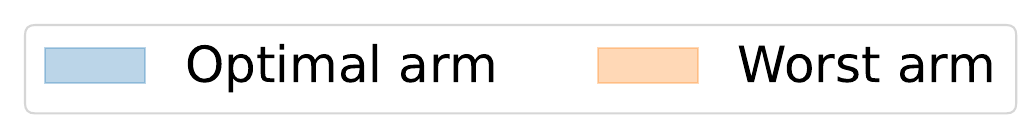}}\\
\end{tabular} 
\captionof{figure}{Histogram of $L$ and $U$ values across individual datasets.}
\label{fig:L_U_histograms}
\end{minipage}
\vspace{0.3cm}

We study the values of $L$ and $U$ in \ourProposition{} on our empirical data and show the distribution of $L$ and $U$ in Figure \ref{fig:L_U_histograms} (see Appendix~\ref{app:empirical_validation_assumption} for more analyses). Next, we focus on the uniqueness of our assumptions before discussing how this impacts our MKB algorithm's performance in Section~\ref{sec:method}.

\textbf{Common Max $K$-armed Bandit assumptions are misaligned with CASH.} 
Existing MKB algorithms can be classified into two main categories.
First, distribution-free approaches do not leverage any assumption on the type of reward distributions~\citep{streeter2006simple,bhatt2022extreme, baudry2022efficient}. Secondly, parametric and semi-parametric approaches typically assume that the reward distributions follow heavy-tailed distributions, typically following a second-order Pareto assumption whose parameters can be estimated using extreme value theory~\citep{carpentier2014extreme,achab2017max}. 

Our empirical analysis shows that existing MKB algorithms' assumptions about reward distributions are either overly broad or misaligned with the CASH problem.
We identify two key differences between our assumptions and those made by current MKB algorithms. 
First, the sub-optimality gap makes the regret definition (\Eqref{eq:extremebandit_regret}) for bounded distributions meaningful. Without the existence of a nonzero gap, the regret definition fails since any policy consistently selecting a single arm can achieve zero regret, as \citet{nishihara2016no} pointed out with Bernoulli distributions.
Second, our reward distributions differ from those used in existing MKB algorithms. \ourProposition{} characterizes the shape of the distribution, with a higher $L$ ensuring more mass near extreme values, making extreme values easier to estimate. In our case, values for $L$ are mostly higher than $1$ while for heavy-tailed distributions, commonly used as the basis for MKB algorithms, are close to $0$ (see the rightmost plot in Figure~\ref{fig:L_U_distributions}). 
In general, our analysis shows that considering the right range of $L$ values unlocks the problem raised by \citet{nishihara2016no}, who focused on constructing counterexamples through the distributions that have unrealistic values for $L$. Thus, under our assumptions, their negative result\footnote{Theorem $11$ in \citep{nishihara2016no}: ``no policy can be guaranteed to perform asymptotically as well as an oracle that plays the single best arm for a given time horizon.'', which means any policy needs to explore all arms for budget T} does not apply here.

\section{\OURALGO{}}
\label{sec:method}
Based on \ourProposition{} and the regret definition (\Eqref{eq:extremebandit_regret}), we introduce \OURALGO{} in Algorithm~\ref{alg:pseudocode_MaxUCB} for $K$ arms with a limited time horizon of $T$ iterations.\footnote{\label{footnote:time_complexity}\OURALGO{} needs to store the number of pulls and the maximum reward for each arm, resulting in a memory requirement of $\mathcal{O}(K)$. The time complexity is $\mathcal{O}(KT)$.}

 \textbf{Description of \OURALGO{}.} Our algorithm balances exploration and exploitation according to the standard optimism principle at the heart of Upper Confidence Bound (UCB) bandit methods \citep{auer2002using, lattimore2020bandit}. The main novelty we introduce is in the computation of a distribution-adapted exploration bonus for \OURALGO{}. 

Our exploration bonus $(\frac{\alpha\log(t)}{n})^2$ deviates from typical UCB literature due to faster concentration of maximum values in bounded distributions. This is because the probability of bad events (violating confidence intervals for the expectation of the max, see Equation \ref{eq:probability_error_two} in Appendix~\ref{app:theorem1_proof}) can be written as:
\begin{align}
P(\text{Bad events}) \leq O\left(e^{-n \sqrt{C(n)}} + n C(n)\right),
\label{eq:convergence}
\end{align}
where distribution-dependent constants are hidden for clarity. 
Then, setting \( C(n) = 1/n^2 \) minimizes the probability of bad events: the first term becomes independent of \( n \), while the second term decreases with \( n \).  Notably, this faster concentration can only be obtained for the reasonably well-behaved distributions we consider following the study of the previous section and it is not a general property of the maxima; more details can be found in Appendix \ref{app:theorem1_proof_extension}.
Furthermore, \OURALGO{} uses $\alpha\geq0$ to control the exploration-exploitation trade-off; a higher $\alpha$ leads to more exploration.

\begin{algorithm}[htbp]
\scriptsize
\caption{\OURALGO}
\label{alg:pseudocode_MaxUCB}
\begin{algorithmic}[1]
\Require  \maxucb{$\alpha$(exploration parameter)}, $T$(time horizon), $K$(arms)
\For {each arm $i \leq K$}
 \State Pull arm $i$, set $n_i \gets 1$, observe reward $r_{i,1}$  \Comment{Evaluating default configuration}
\EndFor

\For {$t=K + 1$ to $T$}
    \For {each arm $i \leq K$}
    \State \parbox[t]{0.4\linewidth}{Update policy \maxucb{$U_i =\max{(r_{i,1},...,r_{i,n_i})} +  (\frac{ \alpha \log(t)}{n_i})^2 $ }}  \Comment{differs from classical UCB, where \ucb{$U_i = \bar{r}_i + \sqrt{\tfrac{\alpha \log(t)}{n_i}}$}}
    \EndFor
    \State Select arm $I_t = \argmax\limits_{i\leq K} U_i$ , \quad $n_{I_t} \gets n_{I_t} +  1$ , then observe reward $r_{I_t, n_{I_t}}$
\EndFor
\end{algorithmic}
\end{algorithm}

\textbf{Analysis of \OURALGO{}.} We first show a regret decomposition result specific to max $K$-armed bandits that directly relates the regret definition in \eqref{eq:extremebandit_regret} with the number of suboptimal trials.

\begin{proposition} (Regret Upper Bound) the regret upper bound up to time $T$ is related to the number of times sub-optimal arms are pulled:
\begin{equation}
R(T) \leq \frac{\max\limits_{i \leq K} b_i} {T}  \sum\limits_{i\neq i^* }^{K} N_i(T)
\end{equation}
where $N_i(T)=  \mathbbm{E}(\sum\nolimits_{t=1}^{T} \mathbbm{1}{\{I_t=i\}})$ is the number of sub-optimal pulls of arm $i$, and $b_i$ is the upper bound on the support of the rewards of arm $i$.
\label{proposition:Regret_ExUCB}
\end{proposition}

The proof is provided in Appendix~\ref{app:proposition1_proof} and relies on standard tools in the extreme bandit literature \citep{baudry2022efficient}. From this result, it is clear that we can now obtain an upper bound on the regret by controlling the number of suboptimal arm pulls $(N_i(T))_{i\neq i^\star}$ individually. Our main theoretical result below proves such an upper bound for Algorithm~\ref{alg:pseudocode_MaxUCB}.
\begin{theorem} 
For any suboptimal arm $i\neq i^\star$, the number of suboptimal draws $N_i(T)$ performed by \OURALGO{}~(\textbf{Algorithm \ref{alg:pseudocode_MaxUCB}}) up to time $T$ is bounded by

\begin{align}
  N_i(T) \leq \frac{ T^{ 1 - 2 L_{i^*} \alpha  \sqrt{\Delta_i} }}{1 - 2 L_{i^*} \alpha  \sqrt{\Delta_i}} + 2 \alpha \sqrt{U_i T} \log(T). 
\label{eq:number_suboptimal_MaxUCB}
\end{align}
\label{theorem:number_suboptimal_MaxUCB}
\end{theorem}

The result of Theorem \ref{theorem:number_suboptimal_MaxUCB} (proof in Appendix~\ref{app:theorem1_proof}) highlights that \OURALGO{} primarily leverages two key properties: the sub-optimality gap $\Delta_i$ and the shape of the distribution, as defined in \ourProposition{}. Specifically, the performance improves with a larger sub-optimality gap $\Delta_i$ and higher values of $L_{i^*}$ ($L$ for the optimal arm), which means that samples drawn from the distribution of the optimal arm are likely to be close to the extreme values. Additionally, smaller values of $U_i$ ($U$ for a sub-optimal arm $i$), which means it is less likely to draw samples close to the extreme values, reduce the number selecting sub-optimal arm $i$, thus enhancing overall performance. 

For our task, as is shown in Figure \ref{fig:L_U_distributions}, the values for $L_{i^*}$ are higher than $1$ and $U_i$ less than $10$ in most cases, yielding high empirical performance. We compare the number of pulls observed in our experiments with the expected values based on our theoretical analysis in Appendix~\ref{app:sec:thoery_vs_reality}, showing that our analysis is not loose.
However, it is essential to note that, in general, finding the optimal arm is challenging if $\Delta_i$ is close to zero or $L_{i^*}$ is very small. We assess the performance of \OURALGO{} on synthetic tasks in Appendix~\ref{app:sec:extreme_bandits_experiments}, showing that the performance of \OURALGO{} deteriorates on tasks that do not satisfy our assumptions.

Finally, we provide a regret upper bound that combines the decomposition in Proposition~\ref{proposition:Regret_ExUCB} with the individual upper bounds above. This requires finding a parameter $\alpha$ that resolves a trade-off between the arms and minimizes the total upper bound, as shown in Corollary~\ref{corollary:choosing_alpha}, whose proof is immediate. 

\begin{corollary} If $L_{i^*}$, $\min\limits_{i \neq i^*}(\Delta_i)$, and $T$ are known in advance, then the total regret $R(T)$ can be bounded as follows by choosing the exploration parameter $\alpha$ appropriately:

\begin{equation}
R(T) \leq \mathcal{O}\bigg( \frac{ K \log T} { \sqrt{T}}  \max\limits_{i \leq K} b_i \bigg), \quad \text{when } \alpha = \frac{1}{4 L_{i^*} \sqrt{\min\limits_{i \neq i^*}\Delta_i} } \bigg(1 - \frac{2\log(\log(T))}{\log(T)}\bigg).
\label{eq:corollary_alpha}
\end{equation}
\label{corollary:choosing_alpha}
\end{corollary}

This final result helps to understand the role of $\alpha$ and serves as guidelines to choose it in practice.
Specifically, \Eqref{eq:corollary_alpha} shows that either a small value of $L_{i^*}$ or a small sub-optimality gap requires a higher value for $\alpha$. Intuitively, a small $L_i^*$ means that the max value of the best arm needs more samples to be nearly reached, and a small suboptimality gap means that the two best arms are close and so hard to distinguish. Indeed, these problem-dependent quantities are unknown to the practitioner, so a direct approach to calculate $\alpha$ is not feasible. Therefore, evaluating performance robustness under "loose tuning" of  $\alpha$  is essential.

\section{Performance on AutoML tasks}
\label{Sec:Numerical_Experiments}
Finally, we examine the empirical performance of \OURALGO{} in an AutoML setting via reporting average ranking and the number of wins, ties, and losses across tasks for each benchmark (details in Appendix~\ref{app:metric_calculation}). 
We first focus on the impact of the hyperparameter $\alpha$ and then compare our approach to others. Specifically, we show that our two-level approach performs better than single-level HPO on the joint space, and \OURALGO{} outperforms other state-of-the-art bandit methods.  To begin, we will provide a brief overview of the experimental setup used across all experiments.
\textbf{Experimental setup.}
We use \numberOfBenchmarks{} AutoML benchmarks\footnote{\label{footnote:implementation}We used the AutoML Toolkit (AMLTK)~\citep{Bergman2024}.
We ran HPO on a compute cluster equipped with Intel Xeon Gold $6\,240$ CPUs, requiring $20\,000$  CPU hours. We conducted the remaining experiments on a local machine with Intel Core i7-1370P, requiring an additional $32$ CPU hours. 
}, all implementing CASH for tabular supervised learning differing in the considered ML models, HPO method, and datasets, which we detail in Table~\ref{tab:automltasks}.
For \tabrepo{} and \hebo{}, we use available pre-computed HPO trajectories, whereas for \tabreporaw{} and \yahpogym{}, we use SMAC~\citep{hutter-lion11a,lindauer-jmlr22a} as a Bayesian optimization method to run HPO ourselves.

\textbf{Competitive Baselines.} We compare against \QuantileBayesUCB{}~\citep{balef2024towards}, \ERUCBS{}~\citep{hu2021cascaded} and \RisingBandits{}~\citep{li2020efficient} which have been developed for the decomposed CASH task. We consider extreme bandits (\QoMaxSDA{}~\citep{baudry2022efficient}, \MaxMedian{}~\citep{bhatt2022extreme}) and classic \UCB{} as general bandit methods. We use default hyperparameter settings for all methods.
As \combinedsearch{} baselines, we consider  Bayesian optimization (\SMAC{}) and \randomsearch{}.\footnote{
Only available for \tabreporaw{} and \yahpogym{}, where we computed HPO trajectories ourselves.}
Additionally, we report the performance of the best (oracle) arm.

\begin{wrapfigure}{l}{0.5\textwidth}
\begin{tabular}{c@{\hskip 0mm}c}
    \small\yahpogym{}
    & \small\tabreporaw{} \\
    \includegraphics[clip, trim=0.25cm 0cm  0.3cm 0cm,height=3.5cm]{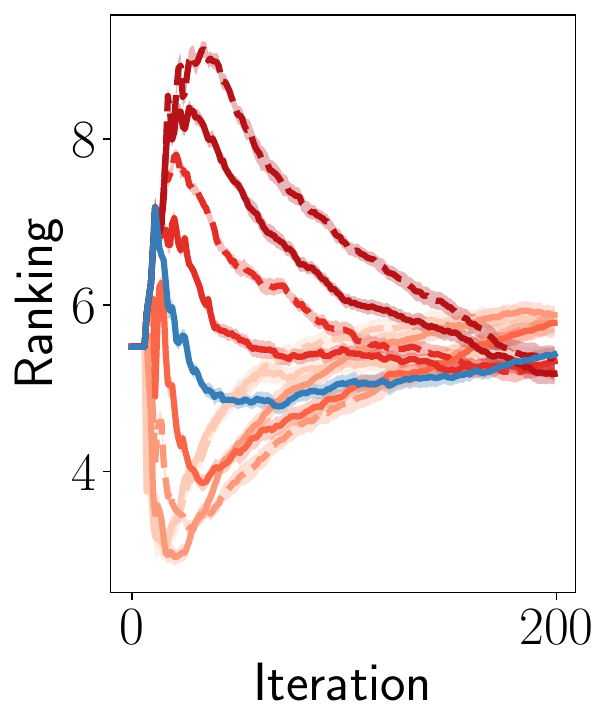}
    & \includegraphics[clip, trim=0.3cm 0cm  0cm 0cm,height=3.5cm]{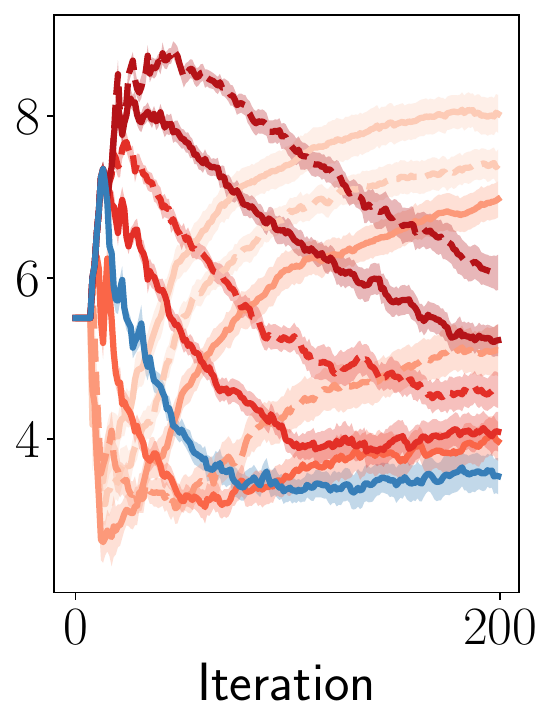} \\
    \multicolumn{2}{c}{\includegraphics[width=0.95\linewidth]{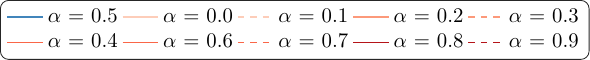}} \\
\end{tabular}
\caption{The sensitivity of \OURALGO{} to hyperparameter $\alpha$, lower is better.}
\label{fig:sensitivity_to_alpha} 
\end{wrapfigure}

\textbf{How sensitive is \OURALGO{} to the choice of $\alpha$?}
Figure~\ref{fig:sensitivity_to_alpha}, shows that the choice of  $\alpha$ impacts performance. Lower values of  $\alpha$ (light red lines) lead to good performance with small budgets, whereas higher values (dark red lines) achieve stronger final performance when sufficient time is available. An $\alpha$ around $0.5$ yields a balanced trade-off, offering robust anytime performance across tasks (see Appendix~\ref{app:ablation_study} for detailed results). Another insight from this study is that the right choice of $\alpha$ may depend on characteristics of the datasets, such as the support of the losses, and could be meta-learned.

\textbf{How does the two-level approach compare against \combinedsearch{}?}
We first compare the average rank over time in Figure~\ref{fig:average_rank} using \combinedsearch{} (black lines; \randomsearch{} and \SMAC{} if available) to the two-level approach. We observe that all methods outperform \randomsearch{} (dotted line) and that while SMAC quickly catches up, some bandit methods continuously achieve a better (lower) rank. Additionally, we observe that most MAB algorithms (except \RisingBandits{} and \QoMaxSDA{}) lead to superior performance in the early stages ($T = 50$). This demonstrates that the decomposition is particularly useful when the number of iterations is limited.  Additionally, Table~\ref{tab:sign_test} shows that the difference in final performance between \combinedsearch{} and the two-level approach is significant.

\textbf{How does \OURALGO{} compare against other bandit methods?}
Figure~\ref{fig:average_rank} shows a substantial difference in the ranking.
In the beginning, many methods perform competitively, but \OURALGO{} yields the best anytime and final performance.
Classical \UCB{} (red) and extreme bandits (\QoMaxSDA{}; brown) perform worst. The \MaxMedian{} algorithm (purple) shows strong initial performance, but its effectiveness declines with more iterations. While  
\MaxMedian{} identifies and avoids the worst arm; it sometimes struggles to select the optimal arm, resulting in non-robust performance.

Next, we look at methods that were originally designed for AutoML. Both \ERUCBS{} (pink) and \QuantileBayesUCB{} (orange) focus on estimating the higher region of the reward distribution rather than the extreme values. \ERUCBS{} assumes a Gaussian reward distribution and is consistently outperformed by \QuantileBayesUCB{}, a distribution-free algorithm.
\RisingBandits{} (green) underperforms initially 
due to its costly initialization
but reaches a competitive final performance. 
This is especially pronounced for the \yahpogym{} benchmark, where \RisingBandits{} outperforms \OURALGO{} with respect to normalized average performance (see Figure~\ref{app:fig:average_norm_loss} in Appendix~\ref{app:detailed_experiments}). This benchmark contains datasets where the optimal arm changes over time. Since \RisingBandits{} models non-stationary rewards, it performs better for these instances.\footnote{A burn-in phase, i.e., pulling each arm for a few rounds at the beginning without observing the rewards, yields a competitive solution as we assess in Appendix~\ref {app:sec:burning_extUCB}.}
 Finally, \OURALGO{} and \QuantileBayesUCB{} are the only ones that significantly outperform \combinedsearch{} in Table~\ref{tab:sign_test}. And looking at 
 \tabreporaw{} and \hebo{}, as depicted in Figure~\ref{fig:average_rank}, demonstrates that \OURALGO{} is a robust method for CASH problems. 

\begin{figure*}[tb]
    \begin{tabular}{c c c c c}
    \scriptsize\tabrepo{[RS]}
    &\hspace{-1.5em}\scriptsize\tabreporaw{[SMAC]}
    & \hspace{-1em}\scriptsize\yahpogym{[SMAC]}
    &\hspace{-1.5em}\scriptsize\hebo{[HEBO]}
    & \\
    \hspace{-1.5em} \includegraphics[clip, trim=0.2cm 0cm 0.3cm 0.25cm,height=3.9cm]{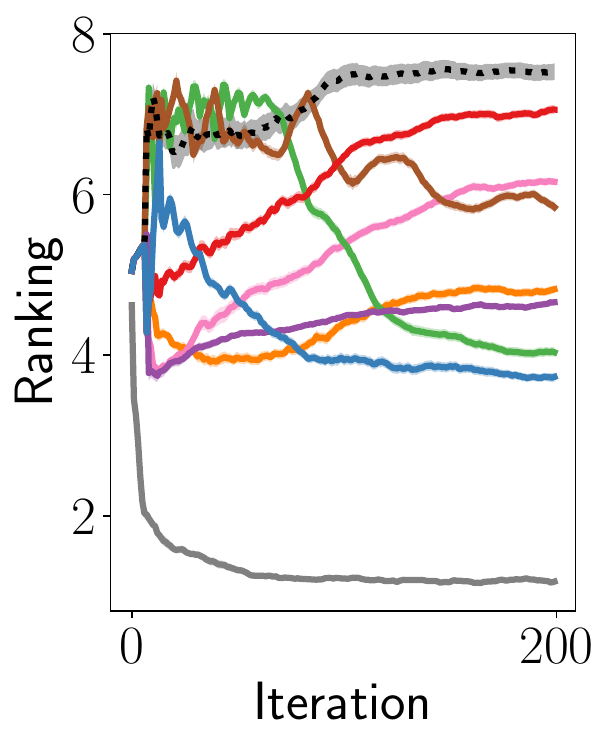}
    & \hspace{-1em} \includegraphics[clip, trim=0.3cm 0cm 0cm 0.25cm, height=3.8cm]{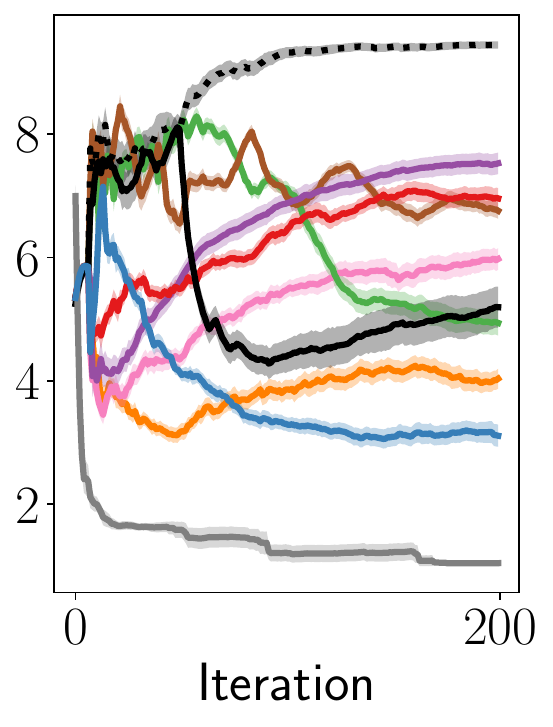}
    &  \hspace{-1em}\includegraphics[clip, trim=0.3cm 0cm 0cm 0.25cm, height=3.85cm]{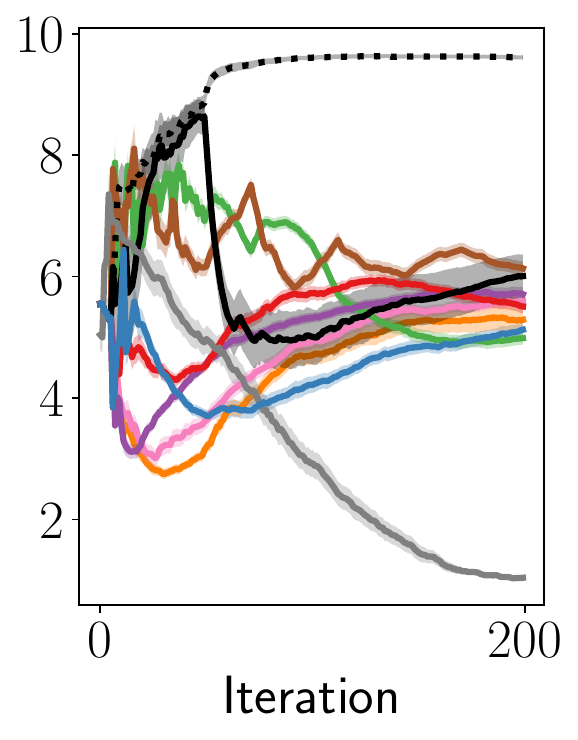}
    & \hspace{-1.5em} \includegraphics[clip, trim=0.3cm 0cm 0cm 0.25cm, height=3.8cm]{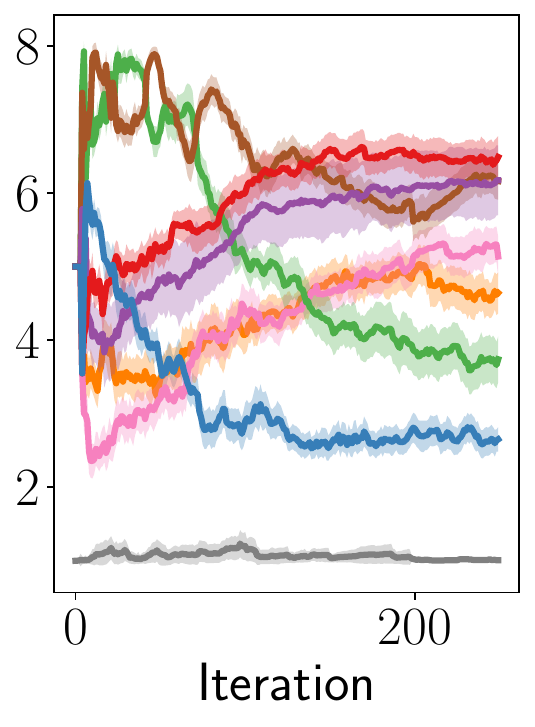}
    &\hspace{-1.4em} \includegraphics[clip, trim=0.3cm -4.0cm 0.0cm 0.0cm, height=3.5cm]{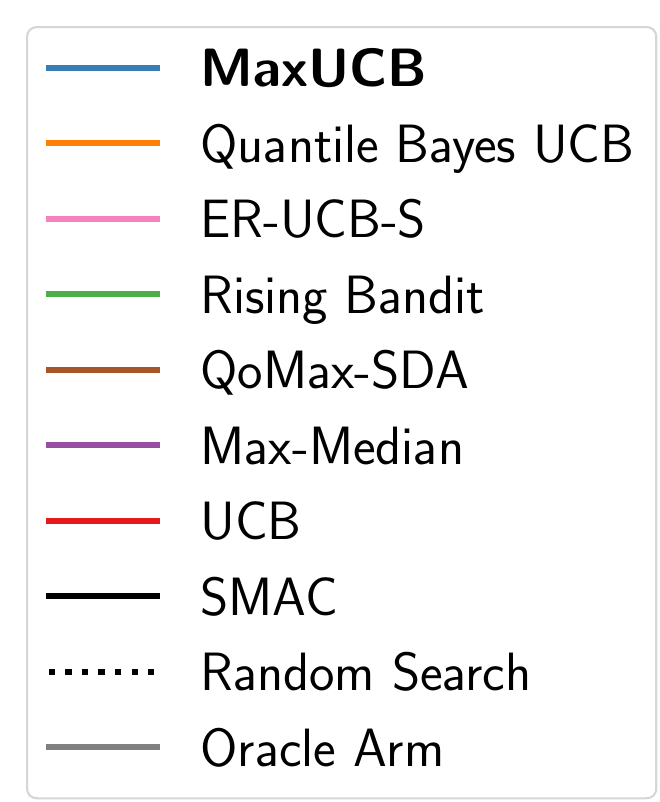}
    \end{tabular}
\caption{Average rank of algorithms on different benchmarks, lower is better. \SMAC{} and \randomsearch{} perform \combinedsearch{} across the joint space.}
\label{fig:average_rank}
\end{figure*}

\begin{table*}[tb]
\centering
\scriptsize
\begin{tabular}{@{\hskip 0mm}l@{\hskip 1mm}lccccccc@{\hskip 0mm}}
Benchmark &  & \textbf{MaxUCB } &  \makecell{Quantile\\Bayes UCB} & ER-UCB-S & Rising Bandit & QoMax-SDA & Max-Median & UCB\\
\midrule 
\multicolumn{1}{l}{\multirow{2}{*}{\makecell{TabRepo\\{[RS]}}}} & p-value  & $\mathbf{\underline{0.00000}}$ & $\mathbf{\underline{0.00000}}$ & $\mathbf{\underline{0.00000}}$ & $\mathbf{\underline{0.00000}}$ & $\mathbf{\underline{0.00000}}$ & $\mathbf{\underline{0.00006}}$ & $\mathbf{\underline{0.00000}}$ \\
 &  w/t/l  & $186$/$4$/$10$ & $179$/$6$/$15$ & $135$/$5$/$60$ & $185$/$4$/$11$ & $172$/$5$/$23$ & $126$/$3$/$71$ & $135$/$5$/$60$\\
\midrule 
\multicolumn{1}{l}{\multirow{2}{*}{\makecell{TabRepoRaw\\{[SMAC]}}}} & p-value  & $\mathbf{\underline{0.00072}}$ & $\mathbf{\underline{0.00261}}$ & $0.95063$ & $0.42777$ & $0.99194$ & $0.99984$ & $0.99194$ \\
 &  w/t/l  & $24$/$0$/$6$ & $23$/$0$/$7$ & $11$/$0$/$19$ & $16$/$0$/$14$ & $9$/$0$/$21$ & $6$/$0$/$24$ & $9$/$0$/$21$\\
\midrule 
\multicolumn{1}{l}{\multirow{2}{*}{\makecell{YaHPOGym\\{[SMAC]}}}} & p-value  & $\mathbf{0.00880}$ & $\mathbf{\underline{0.00503}}$ & $0.31038$ & $\mathbf{\underline{0.00074}}$ & $0.08372$ & $0.50000$ & $\mathbf{0.02412}$ \\
 &  w/t/l  & $64$/$0$/$39$ & $65$/$0$/$38$ & $54$/$1$/$48$ & $68$/$0$/$35$ & $59$/$0$/$44$ & $52$/$0$/$51$ & $62$/$0$/$41$\\
\bottomrule
\end{tabular}
\caption{P-values from a sign test assessing whether bandit methods outperform \combinedsearch{}. P-values below $\alpha = 0.05$ are underlined, while those below $\alpha = 0.05$ after multiple comparison correction (adjusting $\alpha$ by \#comparisons) are boldfaced, indicating that the two-level approach is superior to \combinedsearch{}. 
Additionally, we report the number of wins, ties, and losses (w/t/l).}
\vspace{-0.4cm}
\label{tab:sign_test}
\end{table*}

\section{Conclusions, Discussions and Future Work} 
\label{Sec:Conclusion}

This paper addresses the CASH problem, proposing \OURALGO{}, an MKB method. Our data-driven analysis answers an open question on the applicability of extreme bandits to CASH. We provide a novel theoretical analysis and show state-of-the-art performance on several important benchmarks.
\textbf{Limitations.} Though our method can be applied beyond AutoML in principle, it is finely tuned for this setting with bounded and skewed distributions and for maximal value optimization.
Our analysis relies on stationary distributions, which might not always be accurate, especially at the beginning of the HPO run, so a short burn-in phase may be needed to reach this regime. Our approach may not be distributionally optimal, as optimality in bounded extreme bandits remains an open question, and establishing lower bounds is left for future work.
Lastly, we provide a default value for our hyperparameter $\alpha$ that might need adjusting for other applications.

\textbf{Impact on AutoML systems.} Our approach complements prior work on AutoML systems and increases their flexibility. First, our approach allows to choose any HPO method for each model at the lower level and thus it may integrate recent progress in HPO methods for some ML model classes, e.g., multi-fidelity or meta-learned methods ~\citep{gyorgy2011efficient,li2018hyperband,falkner-icml18a,muller2023pfns4bo,chen2022towards}. Second, 
some AutoML systems~\citep{swearingen2017atm,li2023volcanoml} decompose the search space into smaller subspaces to scale to a distributed setting and use Bandit methods to select promising subspaces~\citep{levine2017rotting,li2020efficient}. 
While we focus on applying bandits to select promising ML models, our methods could also be applied in this setting. Finally, beyond CASH, \OURALGO{} is well-suited for sub-supernet selection in Neural Architecture Search (NAS) \citep{hu2022generalizing}, showing similar reward distributions \citep{ly2024analyzing} (see Appendix~\ref{app:sec:supernet_experiments}). 

\textbf{Choosing $\alpha$ adaptively.}  Figure~\ref{fig:sensitivity_to_alpha} suggests that one could try to tune $\alpha$ online. However, it is known that, in theory, without additional information, data-adaptive parameters cannot be found at a reasonable exploration cost in bandit optimization settings~\citep{Locatelli2018AdaptivityTS}. In AutoML systems, though, supplementary data, like estimates of the sub-optimality gap, reward distribution shape, and HPO convergence rate, can help adjust $\alpha$ adaptively.

\textbf{Future Directions.}  
To extend our extreme bandit setting, one could further refine the reward modeling by incorporating the non-stationarity, especially in AutoML for data streams, where optimal models shift with data distributions~\citep{VermaASMLAS}.
Incorporating cost-aware optimization is another promising direction, as computational resources and time, rather than iteration counts, often define budgets in AutoML; this would require estimating model training times and factoring them into the decision process. 
Addressing the growing complexity of heterogeneous ML tasks, such as those involving pre-training, fine-tuning, or prompt engineering, may benefit from a hierarchical approach that allocates resources effectively across diverse~\citep{balefposterior, balef2025context}. Additionally, exploiting structural similarities among algorithms and their hyperparameters could reduce exploration costs by sharing information across arms in the CASH problem, further enhancing efficiency.

\begin{ack}
The authors are funded by the Deutsche Forschungsgemeinschaft (DFG, German Research Foundation) under Germany’s Excellence Strategy – EXC number 2064/1 – Project number 390727645. Additionally, C. Vernade acknowledges funding from the DFG under the project 468806714 of the Emmy Noether Programme.  The authors also thank the International Max Planck Research School for Intelligent Systems (IMPRS-IS).
\end{ack}

\bibliographystyle{unsrtnat}
\bibliography{strings,references,lib,myproc,proc}

\begin{thebibliography}{67}
\providecommand{\natexlab}[1]{#1}
\providecommand{\url}[1]{\texttt{#1}}
\expandafter\ifx\csname urlstyle\endcsname\relax
  \providecommand{\doi}[1]{doi: #1}\else
  \providecommand{\doi}{doi: \begingroup \urlstyle{rm}\Url}\fi

\bibitem[Hutter et~al.(2019)Hutter, Kotthoff, and Vanschoren]{hutter-book19a}
F.~Hutter, L.~Kotthoff, and J.~Vanschoren, editors.
\newblock \emph{Automated Machine Learning: Methods, Systems, Challenges}.
\newblock Springer, 2019.
\newblock Available for free at \url{http://automl.org/book}.

\bibitem[Baratchi et~al.(2024)Baratchi, Wang, Limmer, {van Rijn}, Hoos, Thomas, and Olhofer]{baratchi-air24a}
M.~Baratchi, C.~Wang, S.~Limmer, J.~{van Rijn}, H.~Hoos, B.~Thomas, and M.~Olhofer.
\newblock Automated machine learning: past, present and future.
\newblock \emph{Artificial Intelligence Review}, 57, 2024.

\bibitem[Bischl et~al.(2025)Bischl, Casalicchio, Das, Feurer, Fischer, Gijsbers, Mukherjee, Müller, Németh, Oala, Purucker, Ravi, {van Rijn}, Singh, Vanschoren, {van der Velde}, and Wever]{bischl2025}
Bernd Bischl, Giuseppe Casalicchio, Taniya Das, Matthias Feurer, Sebastian Fischer, Pieter Gijsbers, Subhaditya Mukherjee, Andreas~C. Müller, László Németh, Luis Oala, Lennart Purucker, Sahithya Ravi, Jan~N. {van Rijn}, Prabhant Singh, Joaquin Vanschoren, Jos {van der Velde}, and Marcel Wever.
\newblock Openml: Insights from 10 years and more than a thousand papers.
\newblock \emph{Patterns}, 6\penalty0 (7):\penalty0 101317, 2025.
\newblock ISSN 2666-3899.
\newblock \doi{https://doi.org/10.1016/j.patter.2025.101317}.

\bibitem[Thornton et~al.(2013)Thornton, Hutter, Hoos, and Leyton-Brown]{thornton-kdd13a}
C.~Thornton, F.~Hutter, H.~Hoos, and K.~Leyton-Brown.
\newblock {A}uto-{WEKA}: combined selection and {H}yperparameter {O}ptimization of classification algorithms.
\newblock In I.~Dhillon, Y.~Koren, R.~Ghani, T.~Senator, P.~Bradley, R.~Parekh, J.~He, R.~Grossman, and R.~Uthurusamy, editors, \emph{The 19th {ACM} {SIGKDD} International Conference on Knowledge Discovery and Data Mining ({KDD}'13)}, pages 847--855. ACM Press, 2013.

\bibitem[van Breugel and van~der Schaar(2024)]{van2024tabular}
Boris van Breugel and Mihaela van~der Schaar.
\newblock Why tabular foundation models should be a research priority.
\newblock In R.~Salakhutdinov, Z.~Kolter, K.~Heller, A.~Weller, N.~Oliver, J.~Scarlett, and F.~Berkenkamp, editors, \emph{Proceedings of the 41st International Conference on Machine Learning ({ICML}'24)}, volume 235 of \emph{Proceedings of Machine Learning Research}. PMLR, 2024.

\bibitem[Kadra et~al.(2021)Kadra, Lindauer, Hutter, and Grabocka]{kadra-neurips21a}
A.~Kadra, M.~Lindauer, F.~Hutter, and J.~Grabocka.
\newblock Well-tuned simple nets excel on tabular datasets.
\newblock In M.~Ranzato, A.~Beygelzimer, K.~Nguyen, P.~Liang, J.~Vaughan, and Y.~Dauphin, editors, \emph{Proceedings of the 35th International Conference on Advances in Neural Information Processing Systems ({N}eur{IPS}'21)}. Curran Associates, 2021.

\bibitem[Gorishniy et~al.(2021)Gorishniy, Rubachev, Khrulkov, and Babenko]{gorishniy-neurips21a}
Y.~Gorishniy, I.~Rubachev, V.~Khrulkov, and A.~Babenko.
\newblock Revisiting deep learning models for tabular data.
\newblock In M.~Ranzato, A.~Beygelzimer, K.~Nguyen, P.~Liang, J.~Vaughan, and Y.~Dauphin, editors, \emph{Proceedings of the 35th International Conference on Advances in Neural Information Processing Systems ({N}eur{IPS}'21)}. Curran Associates, 2021.

\bibitem[Gorishniy et~al.(2024)Gorishniy, Rubachev, Kartashev, Shlenskii, Kotelnikov, and Babenko]{gorishniy-iclr24a}
Y.~Gorishniy, I.~Rubachev, N.~Kartashev, D.~Shlenskii, A.~Kotelnikov, and A.~Babenko.
\newblock {TabR}: Tabular deep learning meets nearest neighbors.
\newblock In \emph{International Conference on Learning Representations ({ICLR}'24)}, 2024.
\newblock Published online: \url{iclr.cc}.

\bibitem[McElfresh et~al.(2023)McElfresh, Khandagale, Valverde, Prasad, Feuer, Hegde, Ramakrishnan, Goldblum, and White]{mcelfresh-neuripsdbt23a}
D.~McElfresh, S.~Khandagale, J.~Valverde, V.~Prasad, B.~Feuer, C.~Hegde, G.~Ramakrishnan, M.~Goldblum, and C.~White.
\newblock When do neural nets outperform boosted trees on tabular data?
\newblock In E.~Denton, J.~Ha, and J.~Vanschoren, editors, \emph{Proceedings of the Neural Information Processing Systems Track on Datasets and Benchmarks}, 2023.

\bibitem[Kohli et~al.(2024)Kohli, Feurer, Eggensperger, Bischl, and Hutter]{kohli2024towards}
Ravin Kohli, Matthias Feurer, Katharina Eggensperger, Bernd Bischl, and Frank Hutter.
\newblock Towards quantifying the effect of datasets for benchmarking: A look at tabular machine learning.
\newblock \emph{Data-centric Machine Learning Research (DMLR) Workshop at ICLR}, 2024.

\bibitem[Hollmann et~al.(2023)Hollmann, M{\"u}ller, Eggensperger, and Hutter]{hollmann-iclr23a}
N.~Hollmann, S.~M{\"u}ller, K.~Eggensperger, and F.~Hutter.
\newblock Tab{PFN}: A transformer that solves small tabular classification problems in a second.
\newblock In \emph{International Conference on Learning Representations ({ICLR}'23)}, 2023.
\newblock Published online: \url{iclr.cc}.

\bibitem[Holzmüller et~al.(2024)Holzmüller, Grinsztajn, and Steinwart]{holzmueller-neurips24a}
D.~Holzmüller, L.~Grinsztajn, and I.~Steinwart.
\newblock Better by default: Strong pre-tuned mlps and boosted trees on tabular data.
\newblock In A.~Globerson, L.~Mackey, D.~Belgrave, A.~Fan, U.~Paquet, J.~Tomczak, and C.~Zhang, editors, \emph{Proceedings of the 37th International Conference on Advances in Neural Information Processing Systems ({N}eur{IPS}'24)}, 2024.

\bibitem[Feurer et~al.(2015)Feurer, Klein, Eggensperger, Springenberg, Blum, and Hutter]{feurer-nips15a}
M.~Feurer, A.~Klein, K.~Eggensperger, J.~Springenberg, M.~Blum, and F.~Hutter.
\newblock Efficient and robust automated machine learning.
\newblock In C.~Cortes, N.~Lawrence, D.~Lee, M.~Sugiyama, and R.~Garnett, editors, \emph{Proceedings of the 29th International Conference on Advances in Neural Information Processing Systems ({N}eur{IPS}'15)}, pages 2962--2970. Curran Associates, 2015.

\bibitem[Komer et~al.(2014)Komer, Bergstra, and Eliasmith]{komer-automl14a}
B.~Komer, J.~Bergstra, and C.~Eliasmith.
\newblock Hyperopt-sklearn: Automatic hyperparameter configuration for scikit-learn.
\newblock In F.~Hutter, R.~Caruana, R.~Bardenet, M.~Bilenko, I.~Guyon, B.~Kégl, and H.~Larochelle, editors, \emph{{ICML} workshop on Automated Machine Learning (Auto{ML} workshop 2014)}, 2014.

\bibitem[Kotthoff et~al.(2017)Kotthoff, Thornton, Hoos, Hutter, and Leyton-Brown]{kotthoff2017auto}
Lars Kotthoff, Chris Thornton, Holger~H Hoos, Frank Hutter, and Kevin Leyton-Brown.
\newblock {Auto-WEKA 2.0}: Automatic model selection and hyperparameter optimization in {WEKA}.
\newblock \emph{Journal of Machine Learning Research}, 18\penalty0 (25):\penalty0 1--5, 2017.

\bibitem[Feurer et~al.(2022)Feurer, Eggensperger, Falkner, Lindauer, and Hutter]{feurer-jmlr22a}
M.~Feurer, K.~Eggensperger, S.~Falkner, M.~Lindauer, and F.~Hutter.
\newblock {Auto-Sklearn} 2.0: Hands-free automl via meta-learning.
\newblock \emph{Journal of Machine Learning Research}, 23\penalty0 (261):\penalty0 1--61, 2022.

\bibitem[Lattimore and Szepesv{\'a}ri(2020)]{lattimore2020bandit}
Tor Lattimore and Csaba Szepesv{\'a}ri.
\newblock \emph{Bandit Algorithms}.
\newblock Cambridge University Press, 2020.

\bibitem[Hoffman et~al.(2014)Hoffman, Shahriari, and Freitas]{hoffman2014correlation}
Matthew Hoffman, Bobak Shahriari, and Nando Freitas.
\newblock On correlation and budget constraints in model-based bandit optimization with application to automatic machine learning.
\newblock In \emph{Artificial Intelligence and Statistics}, pages 365--374. PMLR, 2014.

\bibitem[Liu et~al.(2019)Liu, Ram, Vijaykeerthy, Bouneffouf, Bramble, Samulowitz, Wang, Conn, and Gray]{liu2020admm}
Sijia Liu, Parikshit Ram, Deepak Vijaykeerthy, Djallel Bouneffouf, Gregory Bramble, Horst Samulowitz, Dakuo Wang, Andrew~R. Conn, and Alexander~G. Gray.
\newblock An {ADMM} based framework for {AutoML} pipeline configuration.
\newblock In \emph{AAAI Conference on Artificial Intelligence}, 2019.

\bibitem[Carpentier and Valko(2014)]{carpentier2014extreme}
Alexandra Carpentier and Michal Valko.
\newblock Extreme bandits.
\newblock In Z.~Ghahramani, M.~Welling, C.~Cortes, N.~Lawrence, and K.~Weinberger, editors, \emph{Proceedings of the 28th International Conference on Advances in Neural Information Processing Systems ({N}eur{IPS}'14)}. Curran Associates, 2014.

\bibitem[Achab et~al.(2017)Achab, Clémençon, Garivier, Aurélien, Sabourin, and Vernade]{achab2017max}
M.~Achab, S.~Clémençon, A.~Garivier, A.~Aurélien, A.~Sabourin, and C.~Vernade.
\newblock \emph{{Max K-Armed Bandit}: On the {ExtremeHunter} Algorithm and Beyond}, page 389–404.
\newblock Springer International Publishing, 2017.

\bibitem[Baudry et~al.(2022)Baudry, Russac, and Kaufmann]{baudry2022efficient}
D.~Baudry, Y.~Russac, and E.~Kaufmann.
\newblock Efficient algorithms for extreme bandits.
\newblock In \emph{International Conference on Artificial Intelligence and Statistics}, 2022.

\bibitem[Nishihara et~al.(2016)Nishihara, Lopez-Paz, and Bottou]{nishihara2016no}
Robert Nishihara, David Lopez-Paz, and L{\'e}on Bottou.
\newblock No regret bound for extreme bandits.
\newblock In \emph{Artificial Intelligence and Statistics}, pages 259--267. PMLR, 2016.

\bibitem[Raschka(2018)]{raschka-arxiv20a}
S.~Raschka.
\newblock Model evaluation, model selection, and algorithm selection in machine learning.
\newblock \emph{ArXiv}, abs/1811.12808, 2018.

\bibitem[Jones et~al.(1998)Jones, Schonlau, and Welch]{jones-jgo98a}
D.~Jones, M.~Schonlau, and W.~Welch.
\newblock Efficient global optimization of expensive black box functions.
\newblock \emph{Journal of Global Optimization}, 13:\penalty0 455--492, 1998.

\bibitem[Garnett(2022)]{garnett-book22a}
R.~Garnett.
\newblock \emph{{Bayesian Optimization}}.
\newblock Cambridge University Press, 2022.
\newblock Available for free at \url{https://bayesoptbook.com/}.

\bibitem[Snoek et~al.(2012)Snoek, Larochelle, and Adams]{snoek-nips12a}
J.~Snoek, H.~Larochelle, and R.~Adams.
\newblock Practical {B}ayesian optimization of machine learning algorithms.
\newblock In P.~Bartlett, F.~Pereira, C.~Burges, L.~Bottou, and K.~Weinberger, editors, \emph{Proceedings of the 26th International Conference on Advances in Neural Information Processing Systems ({N}eur{IPS}'12)}, pages 2960--2968. Curran Associates, 2012.

\bibitem[Chen et~al.(2018)Chen, Huang, Wang, Antonoglou, Schrittwieser, Silver, and de~Freitas]{chen-arxiv18a}
Y.~Chen, A.~Huang, Z.~Wang, I.~Antonoglou, J.~Schrittwieser, D.~Silver, and N.~de~Freitas.
\newblock Bayesian optimization in alphago.
\newblock \emph{arXiv:1812.06855 [cs.LG]}, 2018.

\bibitem[Cowen-Rivers et~al.(2022)Cowen-Rivers, Lyu, Tutunov, Wang, Grosnit, Griffiths, Maraval, Jianye, Wang, Peters, and Ammar]{cowenrivers-jair22a}
A.~Cowen-Rivers, W.~Lyu, R.~Tutunov, Z.~Wang, A.~Grosnit, R.~Griffiths, A.~Maraval, H.~Jianye, J.~Wang, J.~Peters, and H.~Ammar.
\newblock {HEBO}: Pushing the limits of sample-efficient hyper-parameter optimisation.
\newblock \emph{Journal of Artificial Intelligence Research}, 74:\penalty0 1269--1349, 2022.

\bibitem[Cicirello and Smith(2005)]{cicirello-aaai05a}
V.~Cicirello and S.~Smith.
\newblock The max {K}-armed bandit: a new model of exploration applied to search heuristic selection.
\newblock In M.~Veloso and S.~Kambhampati, editors, \emph{Proceedings of the 20th National Conference on Artificial Intelligence ({AAAI}'05)}, page 1355–1361. {AAAI} Press, 2005.

\bibitem[Streeter and Smith(2006{\natexlab{a}})]{streeter2006asymptotically}
M.~Streeter and S.~Smith.
\newblock An asymptotically optimal algorithm for the {Max k-Armed Bandit} problem.
\newblock In \emph{AAAI Conference on Artificial Intelligence}, 2006{\natexlab{a}}.

\bibitem[Metelli et~al.(2022)Metelli, Trovo, Pirola, and Restelli]{metelli2022stochastic}
Alberto~Maria Metelli, Francesco Trovo, Matteo Pirola, and Marcello Restelli.
\newblock Stochastic rising bandits.
\newblock In K.~Chaudhuri, S.~Jegelka, L.~Song, C.~Szepesvári, G.~Niu, and S.~Sabato, editors, \emph{Proceedings of the 39th International Conference on Machine Learning ({ICML}'22)}, volume 162 of \emph{Proceedings of Machine Learning Research}. PMLR, 2022.

\bibitem[Jamieson and Talwalkar(2016)]{jamieson2016non}
Kevin Jamieson and Ameet Talwalkar.
\newblock Non-stochastic best arm identification and hyperparameter optimization.
\newblock In \emph{Artificial intelligence and statistics}, pages 240--248. PMLR, 2016.

\bibitem[Hu et~al.(2021)Hu, Liu, and Yu]{hu2021cascaded}
Y.~Hu, X.~Liu, and S.~Liand~Y. Yu.
\newblock Cascaded algorithm selection with extreme-region {UCB} bandit.
\newblock \emph{IEEE Transactions on Pattern Analysis and Machine Intelligence}, 44\penalty0 (10):\penalty0 6782--6794, 2021.

\bibitem[Balef et~al.(2024)Balef, Vernade, and Eggensperger]{balef2024towards}
Amir~Rezaei Balef, Claire Vernade, and Katharina Eggensperger.
\newblock Towards bandit-based optimization for automated machine learning.
\newblock In \emph{5th Workshop on practical ML for limited/low resource settings}, 2024.
\newblock URL \url{https://openreview.net/forum?id=S5da3rzyuk}.

\bibitem[Li et~al.(2020)Li, Jiang, Gao, Shao, Zhang, and Cui]{li2020efficient}
Y.~Li, J.~Jiang, J.~Gao, Y.~Shao, C.~Zhang, and B.~Cui.
\newblock Efficient automatic {CASH} via rising bandits.
\newblock In F.~Rossi, V.~Conitzer, and F.~Sha, editors, \emph{Proceedings of the Thirty-Fourth Conference on Artificial Intelligence ({AAAI}'20)}, pages 4763--4771. Association for the Advancement of Artificial Intelligence, {AAAI} Press, 2020.

\bibitem[Heidari et~al.(2016)Heidari, Kearns, and Roth]{heidari2016tight}
H.~Heidari, M.~Kearns, and A.~Roth.
\newblock Tight policy regret bounds for improving and decaying bandits.
\newblock In S.~Kambhampati, editor, \emph{Proceedings of the 25th International Joint Conference on Artificial Intelligence ({IJCAI}'16)}, pages 1562--1570, 2016.

\bibitem[Mussi et~al.(2024)Mussi, Montenegro, Trov{\'o}, Restelli, and Metelli]{mussi2024best}
Marco Mussi, Alessandro Montenegro, Francesco Trov{\'o}, Marcello Restelli, and Alberto~Maria Metelli.
\newblock Best arm identification for stochastic rising bandits.
\newblock In R.~Salakhutdinov, Z.~Kolter, K.~Heller, A.~Weller, N.~Oliver, J.~Scarlett, and F.~Berkenkamp, editors, \emph{Proceedings of the 41st International Conference on Machine Learning ({ICML}'24)}, volume 235 of \emph{Proceedings of Machine Learning Research}. PMLR, 2024.

\bibitem[Pushak and Hoos(2022)]{pushak-acm22a}
Y.~Pushak and H.~Hoos.
\newblock Automl loss landscapes.
\newblock \emph{ACM Transactions on Evolutionary Learning and Optimization}, 2\penalty0 (3):\penalty0 1--30, 2022.

\bibitem[Su et~al.(2021)Su, Huang, Li, You, Wang, Qian, Zhang, and Xu]{su2021prioritized}
Xiu Su, Tao Huang, Yanxi Li, Shan You, Fei Wang, Chen Qian, Changshui Zhang, and Chang Xu.
\newblock Prioritized architecture sampling with monto-carlo tree search.
\newblock \emph{2021 IEEE/CVF Conference on Computer Vision and Pattern Recognition (CVPR)}, pages 10963--10972, 2021.

\bibitem[Streeter and Smith(2006{\natexlab{b}})]{streeter2006simple}
Matthew~J. Streeter and Stephen~F. Smith.
\newblock A simple distribution-free approach to the {Max k-Armed Bandit} problem.
\newblock In \emph{International Conference on Principles and Practice of Constraint Programming}, 2006{\natexlab{b}}.

\bibitem[Bhatt et~al.(2022)Bhatt, Li, and Samorodnitsky]{bhatt2022extreme}
S.~Bhatt, P.~Li, and G.~Samorodnitsky.
\newblock Extreme bandits using robust statistics.
\newblock \emph{IEEE Transactions on Information Theory}, 69\penalty0 (3):\penalty0 1761--1776, 2022.

\bibitem[Auer(2002)]{auer2002using}
Peter Auer.
\newblock Using confidence bounds for exploitation-exploration trade-offs.
\newblock \emph{Journal of Machine Learning Research}, 3\penalty0 (Nov):\penalty0 397--422, 2002.

\bibitem[Bergman et~al.(2024)Bergman, Feurer, Bahram, Balef, Purucker, Segel, Lindauer, Hutter, and Eggensperger]{Bergman2024}
Edward Bergman, Matthias Feurer, Aron Bahram, Amir~Rezaei Balef, Lennart Purucker, Sarah Segel, Marius Lindauer, Frank Hutter, and Katharina Eggensperger.
\newblock {AMLTK}: A {M}odular {A}utoml {T}oolkit in {P}ython.
\newblock \emph{Journal of Open Source Software}, 9\penalty0 (100):\penalty0 6367, 2024.
\newblock \doi{10.21105/joss.06367}.
\newblock URL \url{https://doi.org/10.21105/joss.06367}.

\bibitem[Hutter et~al.(2011)Hutter, Hoos, and Leyton-Brown]{hutter-lion11a}
F.~Hutter, H.~Hoos, and K.~Leyton-Brown.
\newblock Sequential model-based optimization for general algorithm configuration.
\newblock In C.~Coello, editor, \emph{Proceedings of the Fifth International Conference on Learning and Intelligent Optimization ({LION}'11)}, volume 6683 of \emph{Lecture Notes in Computer Science}, pages 507--523. Springer, 2011.

\bibitem[Lindauer et~al.(2022)Lindauer, Eggensperger, Feurer, Biedenkapp, Deng, Benjamins, Ruhkopf, Sass, and Hutter]{lindauer-jmlr22a}
M.~Lindauer, K.~Eggensperger, M.~Feurer, A.~Biedenkapp, D.~Deng, C.~Benjamins, T.~Ruhkopf, R.~Sass, and F.~Hutter.
\newblock {SMAC3}: A versatile bayesian optimization package for {H}yperparameter {O}ptimization.
\newblock \emph{Journal of Machine Learning Research}, 23\penalty0 (54):\penalty0 1--9, 2022.

\bibitem[Gy{\"o}rgy and Kocsis(2011)]{gyorgy2011efficient}
Andr{\'a}s Gy{\"o}rgy and Levente Kocsis.
\newblock Efficient multi-start strategies for local search algorithms.
\newblock \emph{Journal of Artificial Intelligence Research}, 41:\penalty0 407--444, 2011.

\bibitem[Li et~al.(2018)Li, Jamieson, DeSalvo, Rostamizadeh, and Talwalkar]{li2018hyperband}
Lisha Li, Kevin Jamieson, Giulia DeSalvo, Afshin Rostamizadeh, and Ameet Talwalkar.
\newblock {Hyperband}: A novel bandit-based approach to hyperparameter optimization.
\newblock \emph{Journal of Machine Learning Research}, 18\penalty0 (185):\penalty0 1--52, 2018.

\bibitem[Falkner et~al.(2018)Falkner, Klein, and Hutter]{falkner-icml18a}
S.~Falkner, A.~Klein, and F.~Hutter.
\newblock {BOHB}: Robust and efficient {H}yperparameter {O}ptimization at scale.
\newblock In J.~Dy and A.~Krause, editors, \emph{Proceedings of the 35th International Conference on Machine Learning ({ICML}'18)}, volume~80, pages 1437--1446. Proceedings of Machine Learning Research, 2018.

\bibitem[M{\"u}ller et~al.(2023)M{\"u}ller, Feurer, Hollmann, and Hutter]{muller2023pfns4bo}
Samuel~G. M{\"u}ller, Matthias Feurer, Noah Hollmann, and Frank Hutter.
\newblock {PFNs4BO}: In-context learning for bayesian optimization.
\newblock In A.~Krause, E.~Brunskill, K.~Cho, B.~Engelhardt, S.~Sabato, and J.~Scarlett, editors, \emph{Proceedings of the 40th International Conference on Machine Learning ({ICML}'23)}, volume 202 of \emph{Proceedings of Machine Learning Research}. PMLR, 2023.

\bibitem[Chen et~al.(2022)Chen, Song, Lee, Wang, Zhang, Dohan, Kawakami, Kochanski, Doucet, Ranzato, et~al.]{chen2022towards}
Yutian Chen, Xingyou Song, Chansoo Lee, Zi~Wang, Richard Zhang, David Dohan, Kazuya Kawakami, Greg Kochanski, Arnaud Doucet, Marc'aurelio Ranzato, et~al.
\newblock Towards learning universal hyperparameter optimizers with transformers.
\newblock In S.~Koyejo, S.~Mohamed, A.~Agarwal, D.~Belgrave, K.~Cho, and A.~Oh, editors, \emph{Proceedings of the 36th International Conference on Advances in Neural Information Processing Systems ({N}eur{IPS}'22)}, pages 32053--32068. Curran Associates, 2022.

\bibitem[Swearingen et~al.(2017)Swearingen, Drevo, Cyphers, Cuesta-Infante, Ross, and Veeramachaneni]{swearingen2017atm}
Thomas Swearingen, Will Drevo, Bennett Cyphers, Alfredo Cuesta-Infante, Arun Ross, and Kalyan Veeramachaneni.
\newblock {ATM}: A distributed, collaborative, scalable system for automated machine learning.
\newblock In \emph{2017 IEEE international conference on big data (big data)}, pages 151--162. IEEE, 2017.

\bibitem[Li et~al.(2023)Li, Shen, Zhang, Zhang, and Cui]{li2023volcanoml}
Yang Li, Yu~Shen, Wentao Zhang, Ce~Zhang, and Bin Cui.
\newblock {VolcanoML}: speeding up end-to-end {AutoML} via scalable search space decomposition.
\newblock \emph{The VLDB Journal}, 32\penalty0 (2):\penalty0 389--413, 2023.

\bibitem[Levine et~al.(2017)Levine, Crammer, and Mannor]{levine2017rotting}
Nir Levine, Koby Crammer, and Shie Mannor.
\newblock Rotting bandits.
\newblock In I.~Guyon, U.~von Luxburg, S.~Bengio, H.~Wallach, R.~Fergus, S.~Vishwanathan, and R.~Garnett, editors, \emph{Proceedings of the 31st International Conference on Advances in Neural Information Processing Systems ({N}eur{IPS}'17)}. Curran Associates, 2017.

\bibitem[Hu et~al.(2022)Hu, Wang, Lanqing, Li, Hsieh, and Feng]{hu2022generalizing}
Shoukang Hu, Ruochen Wang, HONG Lanqing, Zhenguo Li, Cho-Jui Hsieh, and Jiashi Feng.
\newblock Generalizing few-shot {NAS} with gradient matching.
\newblock In \emph{International Conference on Learning Representations ({ICLR}'22)}, 2022.

\bibitem[Ly-Manson et~al.(2024)Ly-Manson, Leonardon, El~Bey, Hacene, and Mauch]{ly2024analyzing}
Timot{\'e}e Ly-Manson, Mathieu Leonardon, Abdeldjalil~Aissa El~Bey, Ghouti~Boukli Hacene, and Lukas Mauch.
\newblock Analyzing few-shot neural architecture search in a metric-driven framework.
\newblock In M.~Lindauer, K.~Eggensperger, R.~Garnett, J.~Vanschoren, and J.~Gardner, editors, \emph{Proceedings of the Third International Conference on Automated Machine Learning {(AutoML'24)}}. Proceedings of Machine Learning Research, 2024.

\bibitem[Locatelli and Carpentier(2018)]{Locatelli2018AdaptivityTS}
Andrea Locatelli and Alexandra Carpentier.
\newblock Adaptivity to smoothness in {X}-armed bandits.
\newblock In \emph{Annual Conference Computational Learning Theory}, 2018.

\bibitem[Verma et~al.(2024)Verma, Bifet, Pfahringer, and Bahri]{VermaASMLAS}
Nilesh Verma, Albert Bifet, Bernhard Pfahringer, and Maroua Bahri.
\newblock {ASML}: A scalable and eﬀicient {AutoML} solution for data streams.
\newblock In M.~Lindauer, K.~Eggensperger, R.~Garnett, J.~Vanschoren, and J.~Gardner, editors, \emph{Proceedings of the Third International Conference on Automated Machine Learning {(AutoML'24)}}. Proceedings of Machine Learning Research, 2024.

\bibitem[Balef and Eggensperger(2025{\natexlab{a}})]{balefposterior}
Amir~Rezaei Balef and Katharina Eggensperger.
\newblock Posterior sampling using prior-data fitted networks for optimizing complex automl pipelines.
\newblock In \emph{Eighteenth European Workshop on Reinforcement Learning}, 2025{\natexlab{a}}.

\bibitem[Balef and Eggensperger(2025{\natexlab{b}})]{balef2025context}
Amir~Rezaei Balef and Katharina Eggensperger.
\newblock In-context decision making for optimizing complex automl pipelines.
\newblock In \emph{ECAI 2025}, pages 4758--4765. IOS Press, 2025{\natexlab{b}}.

\bibitem[Pfisterer et~al.(2022)Pfisterer, Schneider, Moosbauer, Binder, and Bischl]{pfisterer-automl22a}
F.~Pfisterer, L.~Schneider, J.~Moosbauer, M.~Binder, and B.~Bischl.
\newblock {YAHPO Gym} -- an efficient multi-objective multi-fidelity benchmark for hyperparameter optimization.
\newblock In I.~Guyon, M.~Lindauer, M.~van~der Schaar, F.~Hutter, and R.~Garnett, editors, \emph{Proceedings of the First International Conference on Automated Machine Learning}. Proceedings of Machine Learning Research, 2022.

\bibitem[Salinas and Erickson(2024)]{salinas2024tabrepo}
David Salinas and Nick Erickson.
\newblock {TabRepo}: A large scale repository of tabular model evaluations and its {Auto{ML}} applications.
\newblock In M.~Lindauer, K.~Eggensperger, R.~Garnett, J.~Vanschoren, and J.~Gardner, editors, \emph{Proceedings of the Third International Conference on Automated Machine Learning {(AutoML'24)}}. Proceedings of Machine Learning Research, 2024.

\bibitem[Nagler et~al.(2024)Nagler, Schneider, Bischl, and Feurer]{nagler2024reshuffling}
Thomas Nagler, Lennart Schneider, B.~Bischl, and Matthias Feurer.
\newblock Reshuffling resampling splits can improve generalization of hyperparameter optimization.
\newblock \emph{ArXiv}, abs/2405.15393, 2024.

\bibitem[Karnin et~al.(2013)Karnin, Koren, and Somekh]{karnin2013almost}
Zohar Karnin, Tomer Koren, and Oren Somekh.
\newblock Almost optimal exploration in multi-armed bandits.
\newblock In \emph{International Conference on Machine Learning}, pages 1238--1246. PMLR, 2013.

\bibitem[Kikkawa and Ohno(2024)]{kikkawa2024materials}
Nobuaki Kikkawa and Hiroshi Ohno.
\newblock Materials discovery using {Max K-Armed Bandit}.
\newblock \emph{Journal of Machine Learning Research}, 25\penalty0 (100):\penalty0 1--40, 2024.

\bibitem[Demšar(2006)]{demsar-06a}
J.~Demšar.
\newblock Statistical comparisons of classifiers over multiple data sets.
\newblock \emph{Journal of Machine Learning Research}, 7:\penalty0 1--30, 2006.

\bibitem[Auer et~al.(2002)Auer, Cesa-Bianchi, Freund, and Schapire]{auer2002nonstochastic}
Peter Auer, Nicolo Cesa-Bianchi, Yoav Freund, and Robert~E Schapire.
\newblock The nonstochastic multiarmed bandit problem.
\newblock \emph{SIAM Journal on Computing}, 32\penalty0 (1):\penalty0 48--77, 2002.

\end{thebibliography}
\newpage
\appendix 
\counterwithin{figure}{section}
\counterwithin{table}{section}
\counterwithin{algorithm}{section}
\onecolumn
\vspace*{1cm}
\section*{\centering Table of Contents for the Appendices}
\vspace{1cm}
\begin{itemize}
\item \textbf{\hyperref[app:preliminiaries]{Appendix A: Preliminaries}} \dotfill \pageref{app:preliminiaries}
\item \textbf{\hyperref[app:proofs]{Appendix B: Proofs}} \dotfill \pageref{app:proofs}
\begin{itemize}
    \item \hyperref[app:our_proposition_proof]{B.1 Proof of Lemma \ref{theorem:lemma}} \dotfill \pageref{app:our_proposition_proof}
    \item \hyperref[app:proposition1_proof]{B.2 Proof of Proposition \ref{proposition:Regret_ExUCB}} \dotfill \pageref{app:proposition1_proof}
    \item \hyperref[app:theorem1_proof]{B.3 Proof of Theorem \ref{theorem:number_suboptimal_MaxUCB}} \dotfill \pageref{app:theorem1_proof}
     \item \hyperref[app:theorem1_proof_extension]{B.4 Extension of Proof of Theorem \ref{theorem:number_suboptimal_MaxUCB}} \dotfill \pageref{app:theorem1_proof_extension}
\end{itemize}
\item \textbf{\hyperref[app:details_on_assumptions]{Appendix C: More Details on Reward Distribution}} \dotfill \pageref{app:details_on_assumptions}
    \begin{itemize}
        \item \hyperref[app:reward_distribution_analysis]{C.1 Reward Distribution Analysis} \dotfill \pageref{app:reward_distribution_analysis}
        \item \hyperref[app:examples_validation_assumption]{C.2 More Details on Lemma \ref{theorem:lemma}} \dotfill \pageref{app:examples_validation_assumption}
        \item \hyperref[app:empirical_validation_assumption]{C.3 Empirical Validation of Lemma \ref{theorem:lemma}} \dotfill \pageref{app:empirical_validation_assumption}
    \end{itemize}
\item \textbf{\hyperref[app:details_on_experiments]{Appendix D: More Details on the Experiments}} \dotfill \pageref{app:details_on_experiments}
    \begin{itemize}
        \item \hyperref[app:metric_calculation]{D.1 Metric Calculation} \dotfill \pageref{app:metric_calculation}
        \item \hyperref[app:experimentalsetup]{D.2 Experimental Setup} \dotfill \pageref{app:experimentalsetup}
        \item \hyperref[app:baselines_hyperparameters]{D.3 Baselines and Their Hyperparameters} \dotfill \pageref{app:baselines_hyperparameters}
        \item \hyperref[app:ablation_study]{D.4 More Results on the Sensitivity Analysis of Hyperparameter $\alpha$} \dotfill \pageref{app:ablation_study}
        \item \hyperref[app:detailed_experiments]{D.5 More Results for the Empirical Evaluation} \dotfill \pageref{app:detailed_experiments}
        \item \hyperref[app:more_baselines_experiments]{D.6 More Baselines for the Empirical Evaluation} \dotfill \pageref{app:more_baselines_experiments}
    \end{itemize}
\item \textbf{\hyperref[app:algorithm_behaviour]{Appendix E: More Details on the Empirical Behaviour of MaxUCB}} \dotfill \pageref{app:algorithm_behaviour}
    \begin{itemize}
        \item \hyperref[app:number_pulls_arms]{E.1 The Number of Times Each Arm is Pulled} \dotfill \pageref{app:number_pulls_arms}
        \item \hyperref[app:sec:thoery_vs_reality]{E.2 From Theory to Practice} \dotfill \pageref{app:sec:thoery_vs_reality}
        \item \hyperref[app:sec:burning_extUCB]{E.3 Addressing Non-Stationary Rewards} \dotfill \pageref{app:sec:burning_extUCB}
        \item \hyperref[app:sec:extreme_bandits_experiments]{E.4 Toy Examples from the Extreme Bandit's Literature} \dotfill \pageref{app:sec:extreme_bandits_experiments}
        \item \hyperref[app:sec:supernet_experiments]{E.5 Supernet Selection in Few-Shot Neural Architecture Search} \dotfill \pageref{app:sec:supernet_experiments}
    \end{itemize}
\item \textbf{\hyperref[app:checklist]{Appendix F: NeurIPS Paper Checklist}} \dotfill \pageref{app:checklist}
\end{itemize}

\newpage

\section{Preliminiaries}
\label{app:preliminiaries}

\begin{lemma}
\label{theorem:lemma_max}
Let \(X_1, \ldots, X_n\) be \(n\) samples independently drawn from distribution \(d\), and let \(G(x) = P(X > x)\) be the survival function. We have:
\begin{align}
    P\left( \max_{1 \leq t \leq n} X_t \leq x \right) &\leq e^{-n G(x)}, \notag \\
    P\left( \max_{1 \leq t \leq n} X_t > x \right) &\leq n G(x).
    \label{eq:lemma_max}
\end{align}
\end{lemma}

\begin{proof}
Let \(F(x) = P(X \leq x)\) be the cumulative distribution function, so \(G(x) = 1 - F(x) = P(X > x)\). First, consider the probability that the maximum of the \(n\) samples is less than or equal to \(x\):
\begin{align}
    P\left( \max_{1 \leq t \leq n} X_t \leq x \right) &= \prod_{i=1}^{n} P(X_i \leq x) = (F(x))^n = (1 - G(x))^n \leq e^{-n G(x)},
\end{align}
using the inequality \((1-x)^n \leq e^{-nx}\). Next, consider the probability that the maximum of the \(n\) samples is greater than \(x\):
\begin{align}
    P\left( \max_{1 \leq t \leq n} X_t > x \right) &\leq \sum_{i=1}^{n} P( X > x) = n G(x).
\end{align}
\end{proof}

\section{Proofs}
\label{app:proofs}

\subsection{Proof of \ourProposition{}}
\label{app:our_proposition_proof}

\textbf{\ourProposition{}.} Suppose Assumption~\ref{theorem:assumption} holds. Then, there exists $L,U \geq 0$ such that the survival function \( G \) can be bounded near $b$ by two linear functions. 
\begin{align}
\forall  \epsilon \in (0,b-a), \quad L \epsilon \leq   G( b - \epsilon) \leq U\epsilon
\end{align}

\begin{proof}
By applying the Mean Value Theorem (MVT) to the survival function $G$ over an interval $[b - \epsilon,b]$, there exists a point $c \in (b - \epsilon,b)$ such that:
\begin{align}
G(b)- G(b - \epsilon) =  G^\prime(c) (b - (b - \epsilon)) = G^\prime(c) \epsilon
\end{align}
Since $G^\prime(x) = -f(x)$ where $f(x)$ the probability density function (PDF) and $G(b) = 0$ we have:
\begin{align}
 G(b - \epsilon) =  f(c) \epsilon
\end{align}
Let $L$ and $U$ be the minimum and maximum values of the PDF $f(c) = \frac{G(b - \epsilon)}{\epsilon}$ over $\epsilon \in (0,b-a)$.
\begin{align}
  L \epsilon \leq G(b - \epsilon) \leq  U \epsilon
\end{align}

Notably, in cases where we are interested in the survival function near some $b_1<b$ and $a_1>a$ i.e., over $(b_1 - \epsilon,b_1)$ applying the MVT again, there exists  $c \in (b_1 - \epsilon,b_1)$  such that:
\begin{align}
G(b_1)- G(b_1 - \epsilon) =  G^\prime(c)  \epsilon = -f(c) \epsilon
\end{align}
which rearranges to:
\begin{align}
 G(b_1 - \epsilon) =  f(c) \epsilon +  G(b_1)
\end{align}
For any  $\epsilon \in (\delta,b_1-a_1)$ with $\delta>0$ we can bound  $G(b_1 - \epsilon)$ as:
\begin{align}
   L \epsilon \leq L \epsilon +  G(b_1)  \leq G(b_1 - \epsilon) \leq   ( f(c) + \frac{G(b_1)}{\delta}) \epsilon  \leq  U \epsilon
\end{align}
\end{proof}

\subsection{Proof of Proposition \ref{proposition:Regret_ExUCB}}
\label{app:proposition1_proof}
\textbf{Proposition~\ref{proposition:Regret_ExUCB}.}  (Upper Regret Bound) the upper regret bound up to time $T$ is related to the number of times sub-optimal arms are pulled,
\begin{equation}
R(T) \leq \frac{\max\limits_{i \leq K} b_i} {T}  \sum\limits_{i\neq i^* }^{K} N_i(T)
\end{equation}
Where $N_i(T)=  \mathbbm{E}(\sum\limits_{t=1}^{T} \mathbbm{1}{\{I_t=i\}})$ is the number of sub-optimal pulls of arm $i$, arm $i^*$ is the optimal arm and $b_i$ is the upper bound on the reward of arm $i$ as given by Assumption \ref{theorem:assumption}, i.e., the reward of arm $i$ lies within the interval $[a_i,b_i]$.
\begin{proof}
This proof is inspired by Assumption 1 of \citet{baudry2022efficient}. First, we need to determine an upper bound for the difference in the highest observed reward for arm $i$ when it has been pulled for $N_{i}(T)$ times compared to when it has been pulled for $T$ times.

\begin{align}
    & \mathbbm{E} \left[ \max_{t \leq T} r_{i,t} \right] - \mathbbm{E} \left[ \max_{t \leq N_i(T)} r_{i,t} \right] = \mathbbm{E} \left[ \mathbbm{1}\left\{ \max_{N_i(T) + 1 \leq t \leq T} r_{i,t} = \max_{t \leq T} r_{i,t} \right\} 
    \max_{N_i(T) + 1 \leq t \leq T} r_{i,t} \right] \notag \\
    & \leq \mathbbm{E} \Bigg[ 
        \mathbbm{1}\left\{ \max_{N_i(T) + 1 \leq t \leq T} r_{i,t} = \max_{t \leq T} r_{i,t} \right\} 
        \underbrace{\mathbbm{1}\left\{ \max_{N_i(T) + 1 \leq t \leq T} r_{i,t} \leq B \right\} 
        \max_{N_i(T) + 1 \leq t \leq T} r_{i,t}}_{\leq B} 
    \Bigg] \notag \\
    & \quad + \mathbbm{E} \Bigg[ 
        \mathbbm{1}\left\{ \max_{N_i(T) + 1 \leq t \leq T} r_{i,t} = \max_{t \leq T} r_{i,t} \right\} 
        \mathbbm{1}\left\{ \max_{N_i(T) + 1 \leq t \leq T} r_{i,t} > B \right\} 
        \max_{N_i(T) + 1 \leq t \leq T} r_{i,t} 
    \Bigg] \notag \\
    & \leq P\left(\max_{N_i(T) + 1 \leq t \leq T} r_{i,t} = \max_{t \leq T} r_{i,t}\right) B \notag \\
    & \quad + \mathbbm{E} \Bigg[ 
        \mathbbm{1}\left\{ \max_{N_i(T) + 1 \leq t \leq T} r_{i,t} = \max_{t \leq T} r_{i,t} \right\} 
        \mathbbm{1}\left\{ \max_{N_i(T) + 1 \leq t \leq T} r_{i,t} > B \right\} 
        \max_{N_i(T) + 1 \leq t \leq T} r_{i,t} 
    \Bigg].
\end{align}

Since always \(\max\limits_{N_i(T) + 1 \leq t \leq T} r_{i,t} \leq b_i \) by choosing \( B = b_i \) we ensure that \(\mathbbm{1}\left\{ \max\limits_{N_i(T) + 1 \leq t \leq T} r_{i,t} > B \right\} = 0\), leading to:

\begin{align}
    \mathbbm{E} \left[ \max_{t \leq T} r_{i,t} \right] - \mathbbm{E} \left[ \max_{t \leq N_i(T)} r_{i,t} \right] 
    & \leq P\left(\max_{N_i(T) + 1 \leq t \leq T} r_{i,t} = \max_{t \leq T} r_{i,t}\right) B \notag \\
    & \leq \left(1 - \frac{N_i(T)}{T}\right) B = \left(1 - \frac{N_i(T)}{T}\right) b_i.
    \label{eq:num_pulls_regret}
\end{align}

Using this and according to the regret definition, we obtain the following:
\begin{align}
R(T) &= \mathop{\mathbbm{E}[\max r_{i^*,t}]}_{t \leq T} - \mathop{\mathbbm{E}[\max r_{I_t,t}]}_{t \leq T}\notag \\ 
&\leq \mathop{\mathbbm{E} \left[ \max_{t \leq T} r_{i^*,t} \right]} - \max_{i \leq K} \mathbbm{E} \left[ \max_{t \leq N_i(T)} r_{i,t} \right] \notag \\ 
&= \min_{i \leq K} (\mathop{\mathbbm{E} \left[ \max_{t \leq T} r_{i^*,t} \right]} - \mathbbm{E} \left[ \max_{t \leq N_i(T)} r_{i,t} \right]) \notag \\ 
&= \min_{i \leq K} ( \Delta_i + \underset{\Eqref{eq:num_pulls_regret}}{\underbrace{{\mathop{\mathbbm{E} \left[ \max_{t \leq T} r_{i,t} \right]} - \mathbbm{E} \left[ \max_{t \leq N_i(T)} r_{i,t} \right]}}})\notag \\ 
&  \leq \min_{i \leq K} ( \Delta_i + (1- \frac{N_i(T)}{T}) b_i )\notag \\ 
& \leq \min_{i \leq K} ( \Delta_i + (1- \frac{N_i(T)}{T}) \max\limits_{i \leq K} b_i )
\label{eq:regret_proposition1_proof}
\end{align}
Based on Definition \ref{def:suboptimality_gap}, the suboptimality gap for the optimal arm $i^*$ is zero  ($\Delta_{i^*} = 0$ ). Additionally, the total number of pulls for the optimal arm $N_{i^*}(T)$ can be calculated as the difference between the total number of pulls across all arms $T$ and the pulls for all suboptimal arms ($i \neq i^*$), i.e. $N_{i^*}(T) =  T - \sum\limits_{i\neq i^* }^{K} N_i(T)$. We upper bound the min in \eqref{eq:regret_proposition1_proof} by the specific value for the optimal arm $i^*$:
\begin{align}
R(T) \leq \frac{\max\limits_{i \leq K} b_i} {T}  \sum\limits_{  i\neq i^* }^{K} N_i(T)
\end{align}

\end{proof}

\subsection{Proof of Theorem \ref{theorem:number_suboptimal_MaxUCB}}
\label{app:theorem1_proof}
\textbf{Theorem~\ref{theorem:number_suboptimal_MaxUCB}.} For any suboptimal arm $i\neq i^\star$, the number of suboptimal draws $N_i(T)$ performed by \textbf{Algorithm \ref{alg:pseudocode_MaxUCB}} up to time $T$ is bounded by
\begin{align}
  N_i(T) \leq \frac{ T^{ 1 - 2 L_{i^*} \alpha  \sqrt{\Delta_i} }}{1 - 2 L_{i^*} \alpha  \sqrt{\Delta_i}} + 2 \alpha \sqrt{U_i T} \log(T) 
\end{align}

\begin{proof}
In order to find the upper bound for the number of sub-optimal pulls of arm $i$ for the algorithm~\ref{alg:pseudocode_MaxUCB}, without loss of generality, we assume that arm $1$ is the optimal arm, i.e. $i^* = 1$. Let $\Delta_i=\mathbbm{E}[\max\limits_{t \leq T} r_{1,t}]- \mathop{\mathbbm{E}[\max\limits_{t \leq T} r_{i,t}]}$ be the suboptimality gap. Our goal is to determine an upper bound on $N_i(T)$, the number of times the sub-optimal arm $i$ has been pulled up to time $T$. First, we identify the event that the algorithm pulls the sub-optimal arm $i$ at time $t$:
\begin{align}
     S &= \{ \max{(r_{1,1},...,r_{1,n_1(t)})} +  C_t(n_1(t))  \leq \max{(r_{i,1},...,r_{i,n_i(t)})} +  C_t(n_i(t)) \} \notag \\
       &= \{ \max\limits_{ 1 \leq n \leq n_1(t)} r_{1,n} +  C_t(n_1(t))  \leq \max\limits_{ 1 \leq n \leq n_i(t)} r_{i,n} +  C_t(n_i(t) \}
    \label{eq:probability1}
\end{align}

where $S$ is the event of selecting a sub-optimal arm $i$ with a padding function $C_t(n)$.  The exploration bonus $C_t(n)$ is a function that is designed to account for the exploration-exploitation trade-off and typically depends on $t$ and the number of times each arm has been pulled $n$. Let $n_1(t)$ and $n_i(t)$  represent the number of times the optimal arm $1$ and an sub-optimal arm $i$ have been pulled, respectively, where $n_1(t) \leq t$ and $n_i(t)\leq t$. We want to express $S$ in a union of events that covers all possible scenarios leading to $S$. Thus, we split $S$ into two complementary conditions as follows:

\begin{align}
    S \subseteq \{  \max\limits_{ 1 \leq n \leq n_1(t)} r_{1,n} + C_t(n_1(t)) \leq  x \}   \cup  \{ \max\limits_{ 1 \leq n \leq n_i(t)} r_{i,n} +  C_t(n_i(t)) >  x \} 
    \label{eq:probability2}
\end{align}
Where $x$ is a threshold value, we take $x = \mathbbm{E}[ \max\limits_{t \leq T} r_{i,t}]$,
\begin{align}
    S &\subseteq \left\{ \max_{1 \leq n \leq n_1(t)} r_{1,n} + C_t(n_1(t)) \leq \mathbbm{E}[\max_{t \leq T} r_{i,t}] \right\}  \notag \\
    &\quad \cup \left\{ \max_{1 \leq n \leq n_i(t)} r_{i,n} + C_t(n_i(t)) >\mathbbm{E}[\max_{t \leq T} r_{i,t}] \right\}  \notag \\
    &= \left\{ \max_{1 \leq n \leq n_1(t)}  r_{1,n}   \leq \mathop{\mathbbm{E}[\max_{t \leq T} r_{1,t}]} - \Delta_i - C_t(n_1(t))\right\}  \notag \\
    &\quad \cup \left\{ \max_{1 \leq n \leq n_i(t)} r_{i,n} > \mathop{\mathbbm{E}[\max_{t \leq T} r_{i,t}]} - C_t(n_i(t)) \right\}
    \label{eq:error_probability}
\end{align}

Thus, the event $S$ can be contained within the union of two bad events:
\begin{itemize}
    \item Underestimating the upper confidence bound of extreme values for the optimal arm $1$
    \item Overestimating the upper confidence bound of extreme values for the sub-optimal arm $i$
\end{itemize}

Now we use Lemma \ref{theorem:lemma_max} to calculate the probability of \ref{eq:error_probability}:
\begin{align}
    P(S) &\leq P\left( \left\{ \max_{1 \leq n \leq n_1(t)} r_{1,n} \leq \mathbbm{E}[\max_{t \leq T} r_{1,t}] - C_t(n_1(t)) - \Delta_i \right\} \right) \notag \\
    &\quad +  P\left( \left\{ \max_{1 \leq n \leq n_i(t)} r_{i,n} >\mathbbm{E}[\max_{t \leq T} r_{i,t}] - C_t(n_i(t)) \right\} \right)  \notag \\
    &\leq e^{-n_1(t) G_{1}\left( \mathbbm{E}[\max\limits_{t \leq T} r_{1,t}] - C_t(n_1(t)) - \Delta_i \right)} \notag \\
    &\quad + n_i(t) G_i\left( \mathbbm{E}[\max_{t \leq T}  r_{i,t}]- C_t(n_i(t)) \right).
\end{align}

Now, by applying \ourProposition, we can simplify the analysis by eliminating the complexities associated with survival functions $G_{1}$ and $G_i$.
\begin{align}
     P(S) \leq e^{-n_1(t) L_1 ( C_t(n_1(t)) + \Delta_i )} +    n_i(t) U_i C_t(n_i(t))
     \label{eq:probability_error_two}
\end{align}

For the first term of the right-hand side, by the Arithmetic Mean-Geometric Mean (AM-GM) inequality (\(a + b \geq 2\sqrt{ab}\)) of \eqref{eq:probability_error_two}, we have:
\begin{align}
    e^{-n_1(t) L_1 (C_t(n_1(t)) + \Delta_i)} \leq e^{-2 L_1 n_1(t) \sqrt{C_t(n_1(t)) \Delta_i}}
    \label{eq:probability_error_left_hand}
\end{align}

In this stage, we want to find a proper padding function $C_t(n)$, which controls the right-hand side of \eqref{eq:probability_error_two} and  \eqref{eq:probability_error_left_hand}. By choosing $C_t(n)=  (\frac{ \alpha \log(t)}{n})^2$, we have:
\begin{align}
    e^{-n_1(t) L_1 ( C_t(n_1(t)) + \Delta_i )} &\leq e^{-2 L_1 n_1(t) \sqrt{C_t(n_1(t)) \Delta_i}} =  e^{- 2 L_1 \alpha \log(t) \sqrt{\Delta_i }} = t^{ - 2 L_1 \alpha  \sqrt{\Delta_i} }  \\
    n_i(t) U_i C_t(n_i(t))  &\leq \frac{\alpha^2 U_i \log^2(t)}{n_i(t)}
\end{align}

This selection of the function of the exploration bonus results in two significant advantages. First, it provides an upper bound for the right-hand side of \eqref{eq:probability_error_two} that remains independent of $n_1$. Furthermore, \Eqref{eq:probability_error_left_hand} shows a decreasing trend as the number of pulls for the sub-optimal arm $i$ increases.\footnote{We note that this choice is not based on the inherent property of maximum values. In general one can use $C_t(n)=  (\frac{ \alpha \log(t)}{n})^m$ for $m > 1$ with the optimal $m$ depending on the setting. In Appendix~\ref{app:theorem1_proof_extension} we show how $m$ affects the regret.}

Now, we assume that the sub-optimal arm $i$ has been played for $l_i$ times, so $n_i(t) \geq l_i$. We want to calculate the number of sub-optimal pulls of arm $i$ up to time $T$:
\begin{align}
    N_i(T)  &\leq   l_i + \sum_{t= l_i}^{T} P(S) \leq   l_i +  \sum_{t=l_i}^{T}  t^{ - 2 L_1 \alpha  \sqrt{\Delta_i} }  + \sum_{t=l_i}^{T} \frac{\alpha^2 U_i \log^2(t)}{l_i}   \notag \\
     &\leq l_i + \frac{ T^{ 1 - 2 L_1 \alpha  \sqrt{\Delta_i} }}{1 - 2 L_1 \alpha  \sqrt{\Delta_i}} +  \frac{\alpha^2 U_i }{l_i} T  \log^2(T) 
\end{align}

By choosing $l_i = \alpha \sqrt{U_i T}  \log(T)  $, we have:

\begin{align}
    N_i(T) &\leq \frac{ T^{ 1 - 2 L_1 \alpha  \sqrt{\Delta_i} }}{1 - 2 L_1 \alpha  \sqrt{\Delta_i}} + 2 \alpha \sqrt{U_i T} \log(T) 
\end{align}

\end{proof}

\subsection{Extension of Proof of Theorem \ref{theorem:number_suboptimal_MaxUCB} }
\label{app:theorem1_proof_extension}
As we discuss in Section~\ref{sec:dataanalysis}, the reward distribution in our setting is left-skewed.  We now show that under the assumption that the survival function decays rapidly near the maximum i.e., $G_i(\mathbbm{E}[\max_{t \leq T} r_{i,t}]) =  \mathcal{O}(\frac{1}{T^2})$, a tighter bound can be derived. This assumption means that the reward distribution has a very light-tail near its upper extreme, which \textit{may} often hold for our left-skewed distributions. Furthermore, we generalize our algorithm by assuming   $C_t = (\frac{\alpha\log(t)}{n_i})^m$  as a exploration bonus function, where $m \geq 1$ is a hyperparameter. 

\begin{proof}
We begin with \Eqref{eq:probability2} and we take  $x = \mathbbm{E}[ \max\limits_{t \leq T} r_{i,t}] + c \Delta_i$, where $ c$ is an arbitrary variable $c\in [0, 1]$. We have: 

\begin{align}
    S &\subseteq \left\{ \max_{1 \leq n \leq n_1(t)} r_{1,n} + C_t(n_1(t)) \leq \mathbbm{E}[\max_{t \leq T} r_{i,t}]  + c \Delta_i \right\}\notag \\ 
    & \cup \left\{ \max_{1 \leq n \leq n_i(t)} r_{i,n} + C_t(n_i(t)) >\mathbbm{E}[\max_{t \leq T} r_{i,t}]  + c \Delta_i \right\}  \notag \\
    &= \left\{ \max_{1 \leq n \leq n_1(t)}  r_{1,n}   \leq \mathop{\mathbbm{E}[\max_{t \leq T} r_{1,t}]} - (1-c)\Delta_i - C_t(n_1(t)) \right\} \notag \\ 
    & \cup \left\{ \max_{1 \leq n \leq n_i(t)} r_{i,n} > \mathop{\mathbbm{E}[\max_{t \leq T} r_{i,t}]} - C_t(n_i(t))  + c \Delta_i \right\} \notag \\
    &= \underset{S_1}{ \underbrace{\left\{ \max_{1 \leq n \leq n_1(t)}  r_{1,n}   \leq \mathop{\mathbbm{E}[\max_{t \leq T} r_{1,t}]} - (1-c)\Delta_i - C_t(n_1(t)) \right\}}} \\
    &\cup  \underset{S_2}{ \underbrace{\left\{ \max_{1 \leq n \leq n_i(t)} r_{i,n} > \mathop{\mathbbm{E}[\max_{t \leq T} r_{i,t}]} - C_t(n_i(t))  + c \Delta_i, \quad   C_t(n_i(t))  \leq c \Delta_i \right\}}}\\
    &\cup \underset{S_3}{\underbrace{\left\{ \max_{1 \leq n \leq n_i(t)} r_{i,n} > \mathop{\mathbbm{E}[\max_{t \leq T} r_{i,t}]} - C_t(n_i(t))  + c \Delta_i, \quad C_t(n_i(t))  > c \Delta_i \right\}}}
    \label{eq:error_probability_extend}
\end{align}

First, we calculate the probability of event $S_2$:

\begin{align}
&P(S_2)\leq P(\left\{ \max_{1 \leq n \leq n_i(t)} r_{i,n} > \mathop{\mathbbm{E}[\max_{t \leq T} r_{i,t}]} \right\}) \leq n_i G_i(\mathbbm{E}[\max_{t \leq T} r_{i,t}]) \leq T G_i(\mathbbm{E}[\max_{t \leq T} r_{i,t}).
\end{align}

 By calculating the number of sub-optimal pulls of arm $i$ up to time $T$, we know the third part of the event ($S_3$) can happen at most $C_T^{-1}( c\Delta_i) = (\frac{\alpha \log(T)}{c \Delta_i})^{\frac{1}{m}}$ times:
\begin{align}
    N_i(T)  &\leq    (\frac{\alpha \log(T)}{c \Delta_i})^{\frac{1}{m}} + \sum_{t= 1}^{T} P(S_1)+ \sum_{t= 1}^{T} P(S_2) \\  &\leq   (\frac{\alpha \log(T)}{c \Delta_i})^{\frac{1}{m}} +T^2 G_i(\mathbbm{E}[\max_{t \leq T} r_{i,t}])  + \sum_{t= 1}^{T} P(S_1) \\
    &\leq   (\frac{\alpha \log(T)}{c \Delta_i})^{\frac{1}{m}} +M  + \sum_{t= 1}^{T} P(S_1)
\end{align}
Where $M$ is a constant as we assume $G_i(\mathbbm{E}[\max_{t \leq T} r_{i,t}]) =  \mathcal{O}(\frac{1}{T^2})$. Finally, we need to find an upper bound for $P(S_1)$.  We need to differentiate between two situations, when $m=1$ and when $m>1$

\textbf{For $m=1$.} We set $c=1$ and then we have:
\begin{align}
P(S_1)  &\leq e^{- n_1 G_1(\mathop{\mathbbm{E}[\max_{t \leq T} r_{1,t}]} - (1-c)\Delta_i - C_t(n_1(t))) } \leq \\  & e^{- L_1 ( C_t(n_1(t)) + (1-c) \Delta_i )} \leq   e^{- \alpha L_1 \log(T)} \leq T^{- \alpha L_1}
\end{align}

Finally, we have:
\begin{align}
    N_i(T)  &\leq M  + \frac{\alpha \log(T)}{\Delta_i} + \frac{T^{1-\alpha L_1}}{1-\alpha L_1}
\end{align}

 With $ \alpha > \frac{1}{L_1} $:

\begin{align}
    N_i(T)  = \mathcal{O}(\frac{\log(T)}{L_1\Delta_i}) 
\end{align}

\textbf{For $m>1$.} We know $n( (\frac{a}{n})^m + b ) \geq  a  b^{\frac{m -1}{m}} [m(m -1)^{\frac{1}{m} - 1}  ] $. We have:

\begin{align}
P(S_1) &\leq e^{-n_1(t) L_1 ( C_t(n_1(t)) + (1-c) \Delta_i )} \leq   e^{- \alpha L_1 \log(T) ((1-c)\Delta_i)^{\frac{m -1}{m}} [m(m -1)^{\frac{1}{m} - 1}  ]} \\ & \leq T^{- \alpha L_1 ((1-c)\Delta_i)^{\frac{m -1}{m}} [m(m -1)^{\frac{1}{m} - 1}  ]}
\end{align}

For simplicity, we set $c = \frac{1}{2}$. Then:

\begin{align}
    N_i(T)  &\leq    (\frac{\alpha \log(T)}{c \Delta_i})^{\frac{1}{m}} + M + \sum_{t= 1}^{T}  P(S_2) \\  &\leq   (\frac{2\alpha \log(T)}{ \Delta_i})^{\frac{1}{m}} + M  +  \frac{T^{ 1 - \alpha L_1 (\frac{\Delta_i}{2})^{\frac{m -1}{m}} [m(m -1)^{\frac{1}{m} - 1}  ]}}{1 - \alpha L_1 (\frac{\Delta_i}{2})^{\frac{m -1}{m}} [m(m -1)^{\frac{1}{m} - 1}  ]} 
\end{align}

Finally, we have:

\begin{align}
    N_i(T) =  (\frac{2 \alpha \log(T)}{\Delta_i})^{\frac{1}{m}} + \mathcal{O}(T^{ 1 - \alpha L_1 (\frac{\Delta_i}{2})^{\frac{m -1}{m}} [m(m -1)^{\frac{1}{m} - 1}  ]}).
\end{align}

And by choosing $\alpha$ as below 
\begin{align}
\alpha = \mathcal{O}\left(\frac{1}{L_1  (\Delta_i)^{\frac{m -1}{m}} } \right), 
\label{eq:optimal_alpha}
\end{align}
We have:
\begin{align}
    N_i(T) =   \mathcal{O}\left(\frac{\log(T)}{L_1 \Delta_i^{{\frac{2m -1}{m}}}}\right)^{\frac{1}{m}}.
    \label{eq:N_i_with_m}
\end{align}

\end{proof}
We would like to emphasize again that this result only holds when $G_i(\mathbbm{E}[\max_{t \leq T} r_{i,t}])$ is sufficiently small. By using a weaker assumption, controlling $G_i(\mathbbm{E}[\max_{t \leq T} r_{i,t}]) $ is necessary. In Theorem \ref{app:theorem1_proof}, we address this for our method by considering $C_t(n_i(t)) \geq G_i(\mathbbm{E}[\max_{t \leq T} r_{i,t}])$ to ensure proper control.

Notably, in general, $\Delta_i$ depends on $T$ (see Definition \ref{def:suboptimality_gap}). Meaning that $\frac{\log(T)}{\Delta_i}$ does not necessarily lead to a logarithmic regret. However, in some special scenarios where the reward distributions have different supports, we can ensure that $\Delta_i = b_1 - b_i$ as $T$ approaches infinity, and a logarithmic regret is achievable.

Furthermore, \Eqref{eq:N_i_with_m} shows that the parameter $m$ controls $N_i(T)$, the number of times a suboptimal arm is pulled asymptotically. When $T$ is very large, a higher value of $m$ (along with the optimal $\alpha$) improves performance. However, it makes the algorithm more sensitive to the choice of $\alpha$. As shown in \Eqref{eq:optimal_alpha}, with increasing $m$, $\Delta_i$ has a greater influence on the optimal $\alpha$. Noting that $\Delta_i$ varies across arms and is typically unknown for an unseen task, finding the optimal $\alpha$ is not feasible in practice. Therefore, there is a trade-off between performance and sensitivity. In the CASH setting, $ m=2$ performs empirically well, while exhibiting low sensitivity to $\alpha$. Furthermore, as shown in Appendix~\ref{app:ablation_study}, we can find a range of $\alpha$ for which the \OURALGO{} works well across different CASH tasks.

\clearpage
\section{More Details on Reward Distribution}
\label{app:details_on_assumptions}

\subsection{Reward distribution analysis}
\label{app:reward_distribution_analysis}
In addition to the analysis in the main paper in Figure~\ref{fig:HPO_ecdf_arms_distributions} in Section~\ref{sec:dataanalysis}, we collected the observed rewards (the output of HPO) for all arms on each model class. We calculate each dataset's empirical survival function $G$ and provide the reward distribution analysis for all benchmark tasks. 
Notably, the shift in distribution (indicated by the thin lines) is low for all tasks, not contradicting the i.i.d. assumptions. For our method, we design our algorithm based on analyzing the distribution of raw rewards (in contrast to the distribution of maximum values over time). 

Over time, the maximum value of samples generated from an i.i.d. distribution has an increasing trend, i.e., the extreme values get better. Notably, this is not contradictory with the \RisingBandits{} strategy~\cite{liu2020admm}, which assumes that the maximum observed value over time is not decreasing (and then analyses the trend of this maximum observed value as a (non i.i.d) reward).

\begin{figure}[hb!]
\centering
\begin{tabular}{c c c}

&\tabrepo{}
&\\
\includegraphics[width=0.3\linewidth]{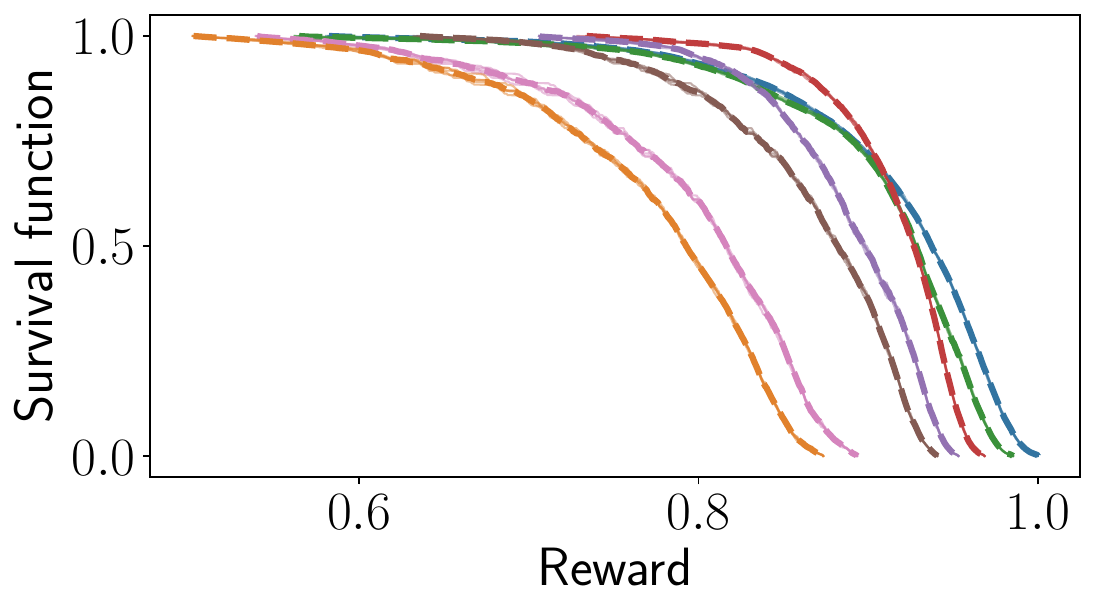}
&\includegraphics[width=0.3\linewidth]{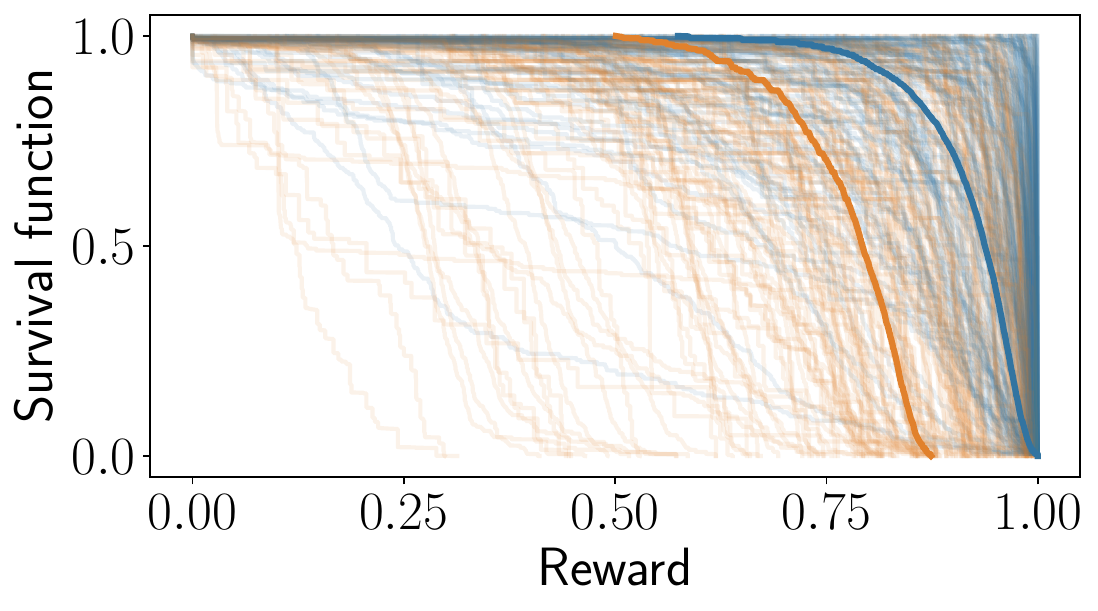}
&\includegraphics[width=0.3\linewidth]{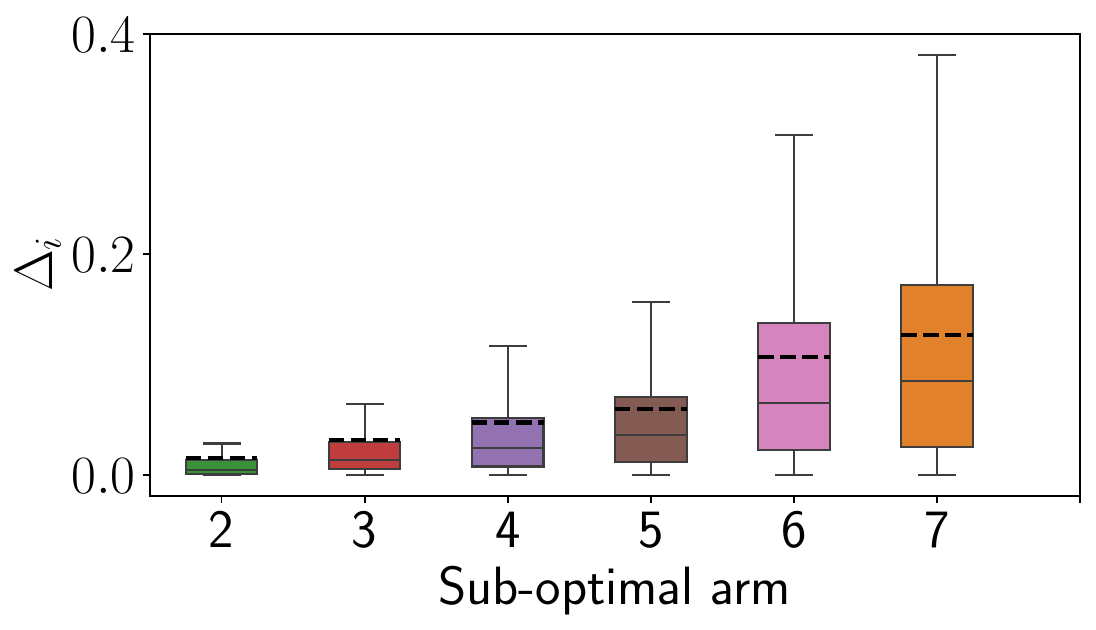}\\

&\tabreporaw{}
&\\
\includegraphics[width=0.3\linewidth]{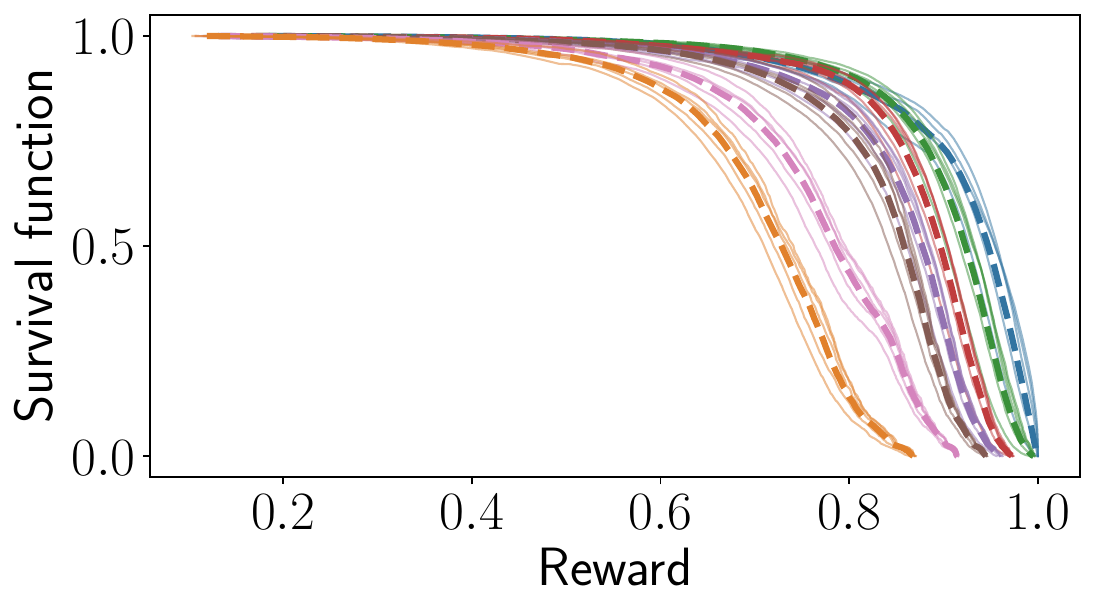}
&\includegraphics[width=0.3\linewidth]{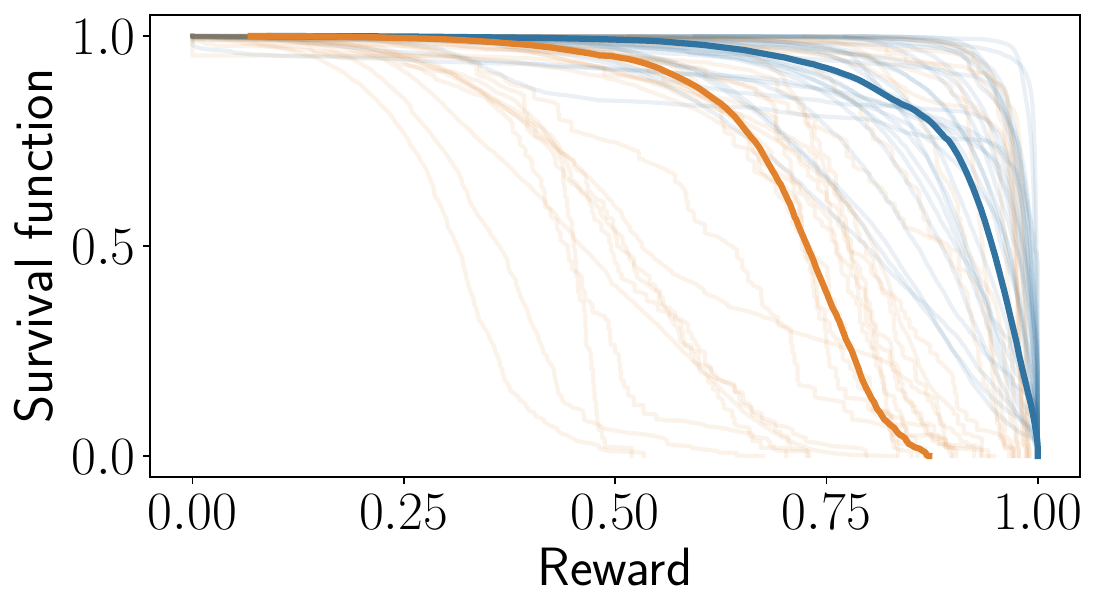}
&\includegraphics[width=0.3\linewidth]{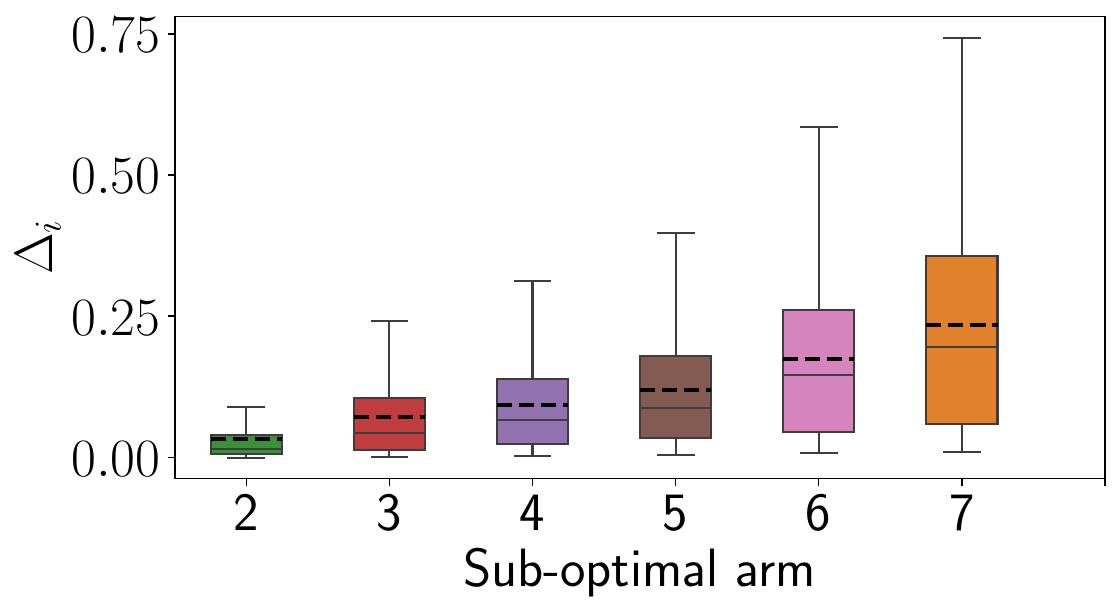}\\

&\yahpogym{}
&\\
\includegraphics[width=0.3\linewidth]{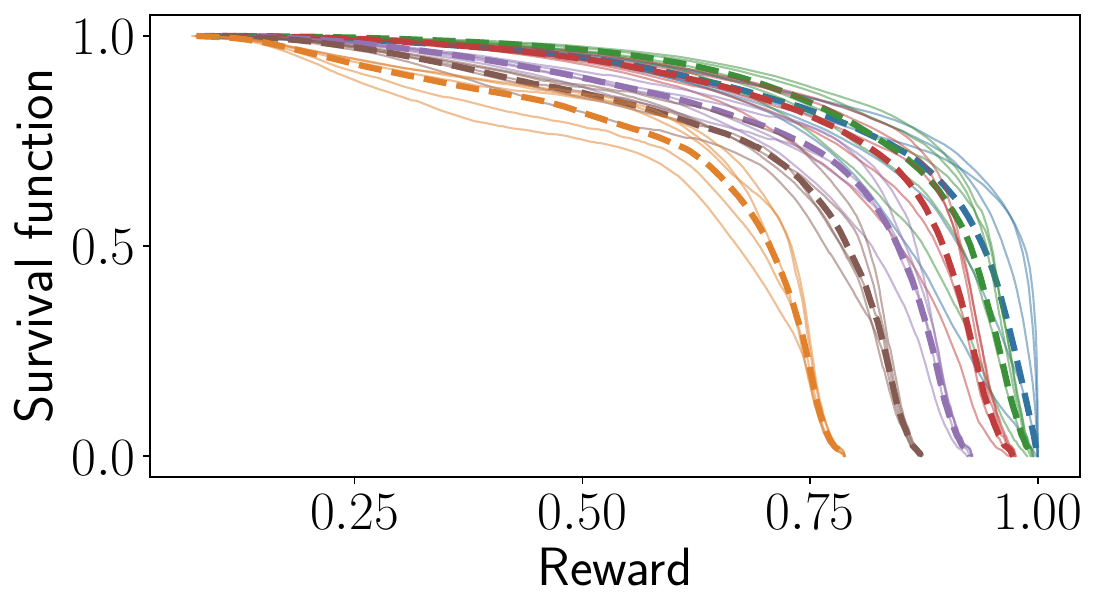}
&\includegraphics[width=0.3\linewidth]{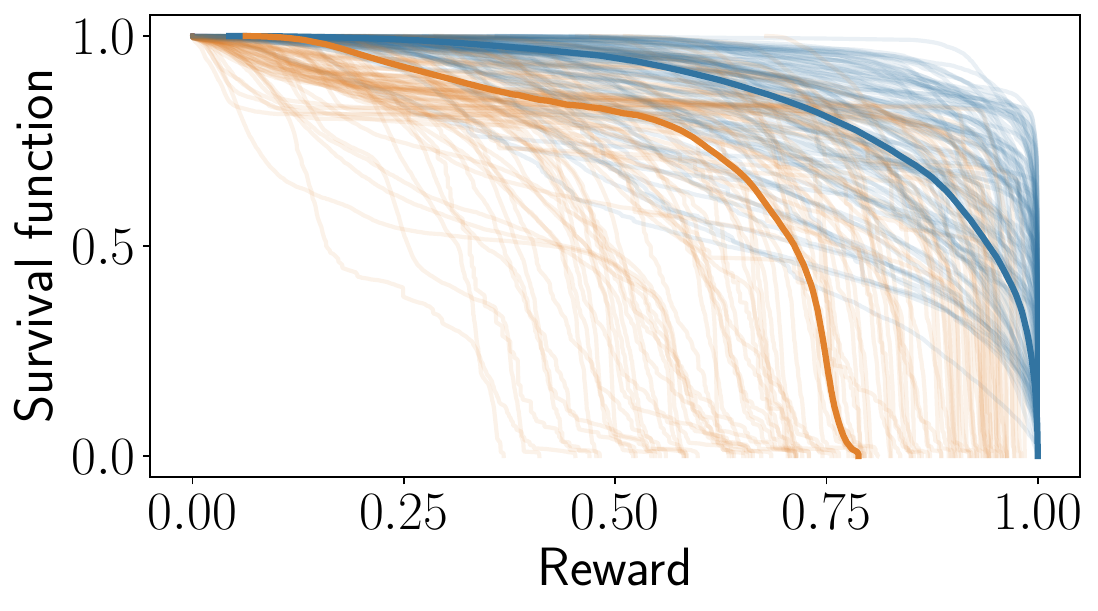}
&\includegraphics[width=0.3\linewidth]{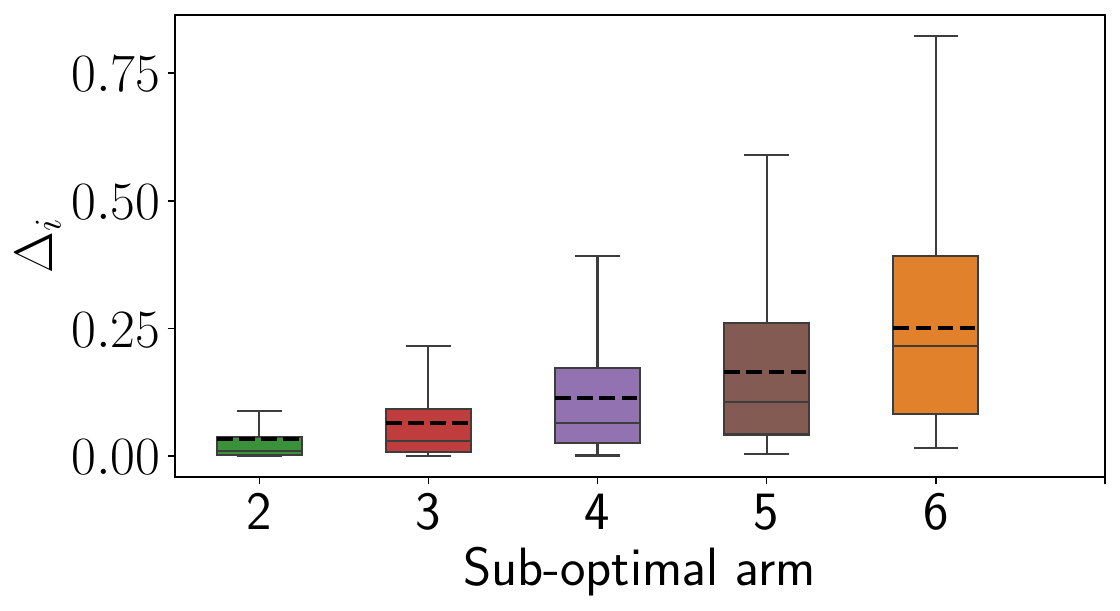}\\

&\hebo{}
&\\
\includegraphics[width=0.3\linewidth]{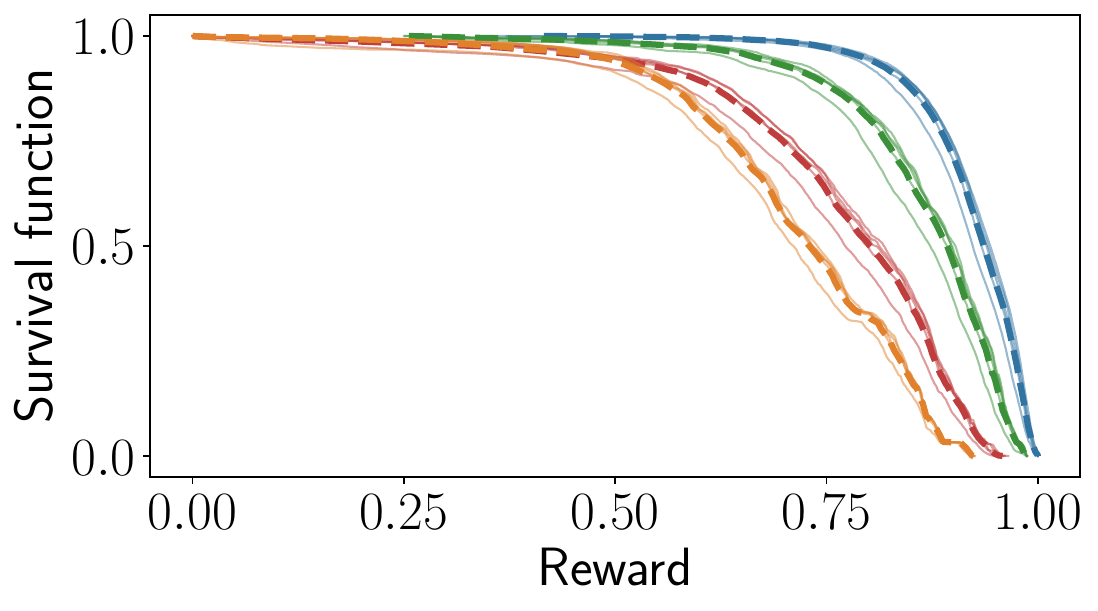}
&\includegraphics[width=0.3\linewidth]{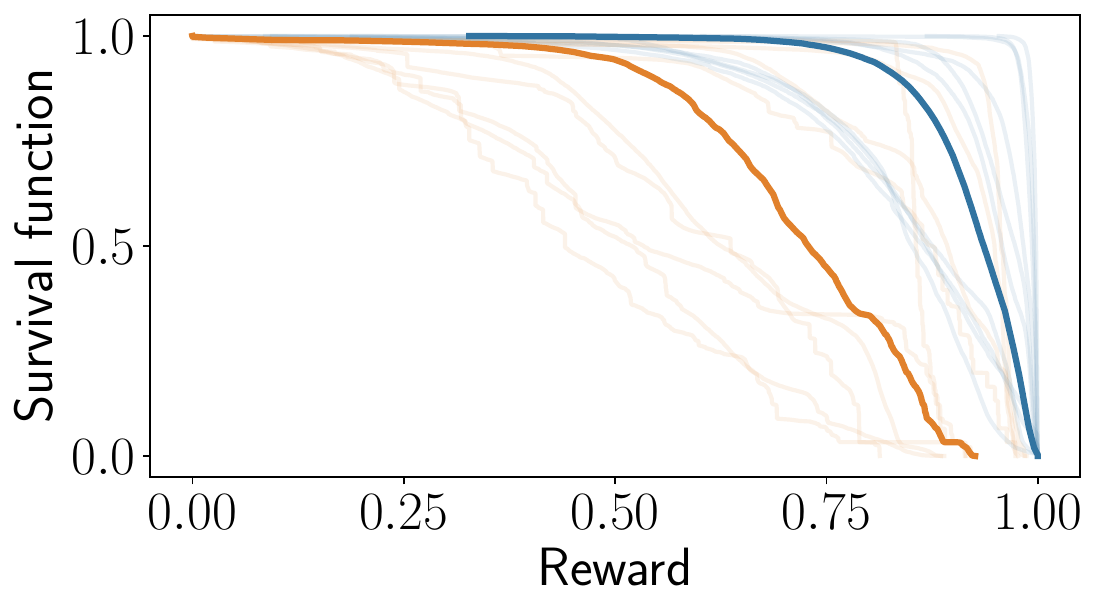}
&\includegraphics[width=0.3\linewidth]{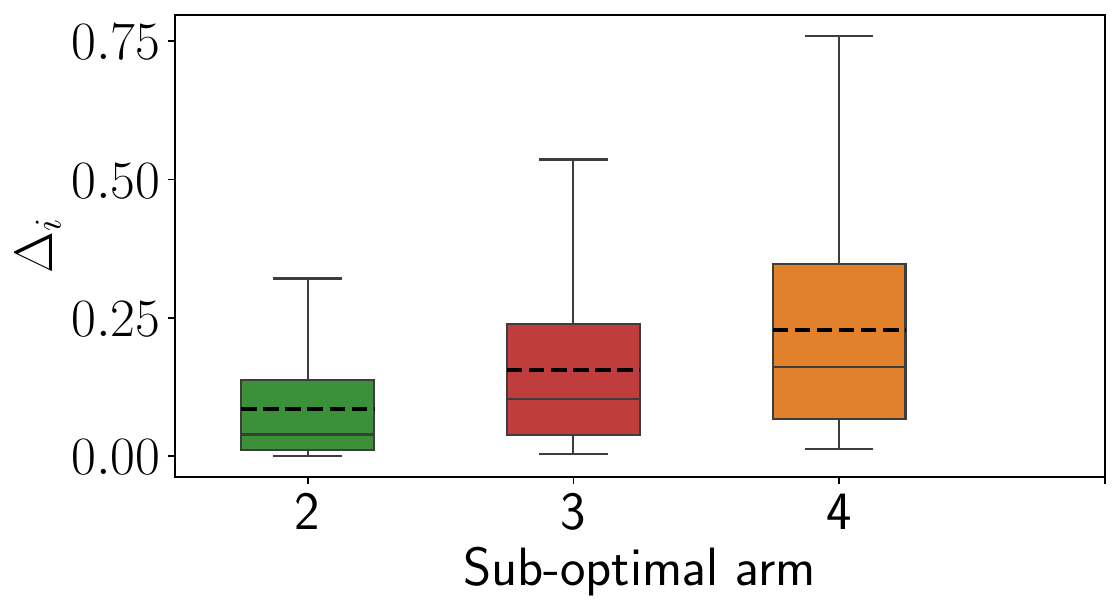}\\

\multicolumn{3}{c}{\includegraphics[width=1.0\linewidth]{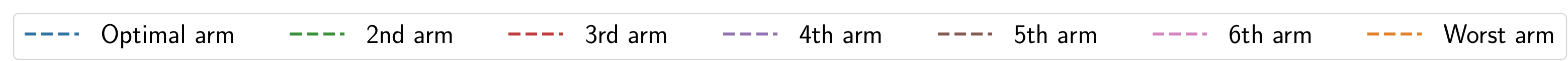}}
 \end{tabular}
\caption{(Left) The average empirical survival function of rewards (observed performances) per arm ranked per dataset. 
We divided the reward sequence into five segments over the budget (time horizon) to show the distribution change over time. Thin lines correspond to empirical survival functions for different segments, visualizing the change over time. (Middle) The average empirical survival function per dataset for the best and worst arm with thin lines corresponding to individual datasets. (Right) The sub-optimality gap $\Delta_i$.}
\label{app:fig:HPO_ecdf_reward_distributions}
\end{figure}

\clearpage

\subsection{More Details on  \ourProposition}
\label{app:examples_validation_assumption}

$L$ and $U$ are lower and upper bounds for the tangent line approximation of the survival function $G$ near the maximum value, indicating the shape of the distribution. We provide three examples to demonstrate this numerically.

\textbf{Toy example}: Assume two simple survival functions $G_1(x) = 1 - x^2$ (left skewed, blue curve) and $G_2(x) = (1-x)^2$ (right skewed, orange curve) with support $[0,1]$. We calculate $G(1- \epsilon)/\epsilon$ over some values of $\epsilon$ in the Figure \ref{app:fig:toy_survival_example}. To compute L and U, we determine the minimum and maximum values of $G(1- \epsilon)/\epsilon$ over the range $0<\epsilon<1$. For clarity, we restricted $\epsilon$ to iterate only over the set $\{0.1, 0.3, 0.5, 0.7, 0.9\}$. It implies that the calculated values of $L$ and $U$ are valid for $\epsilon \in [0.1, 0.9]$.

\begin{figure}[hb!]
\centering
\begin{tabular}{c c}

\begin{minipage}{0.30\linewidth}
    \vspace{-5pt}
    \renewcommand{\arraystretch}{1.5}
    \begin{tabular}{|c|c|c|}
        \hline
        \textbf{$\epsilon$} & \textbf{$\frac{G_1(1-\epsilon)}{\epsilon}$} & \textbf{$\frac{G_2(1-\epsilon)}{\epsilon}$} \\ 
        \hline
        0.10 & 1.90 & 0.10 \\ 
        0.30 & 1.70 & 0.30 \\ 
        0.50 & 1.50 & 0.50 \\ 
        0.70 & 1.30 & 0.70 \\ 
        0.90 & 1.10 & 0.90 \\ 
        \hline
        L &  1.10 & 0.10\\
        \hline
        U &  1.90 & 0.90\\
        \hline
    \end{tabular}
\end{minipage}
&\begin{minipage}{0.6\linewidth}
    \centering
    \includegraphics[width=\linewidth]{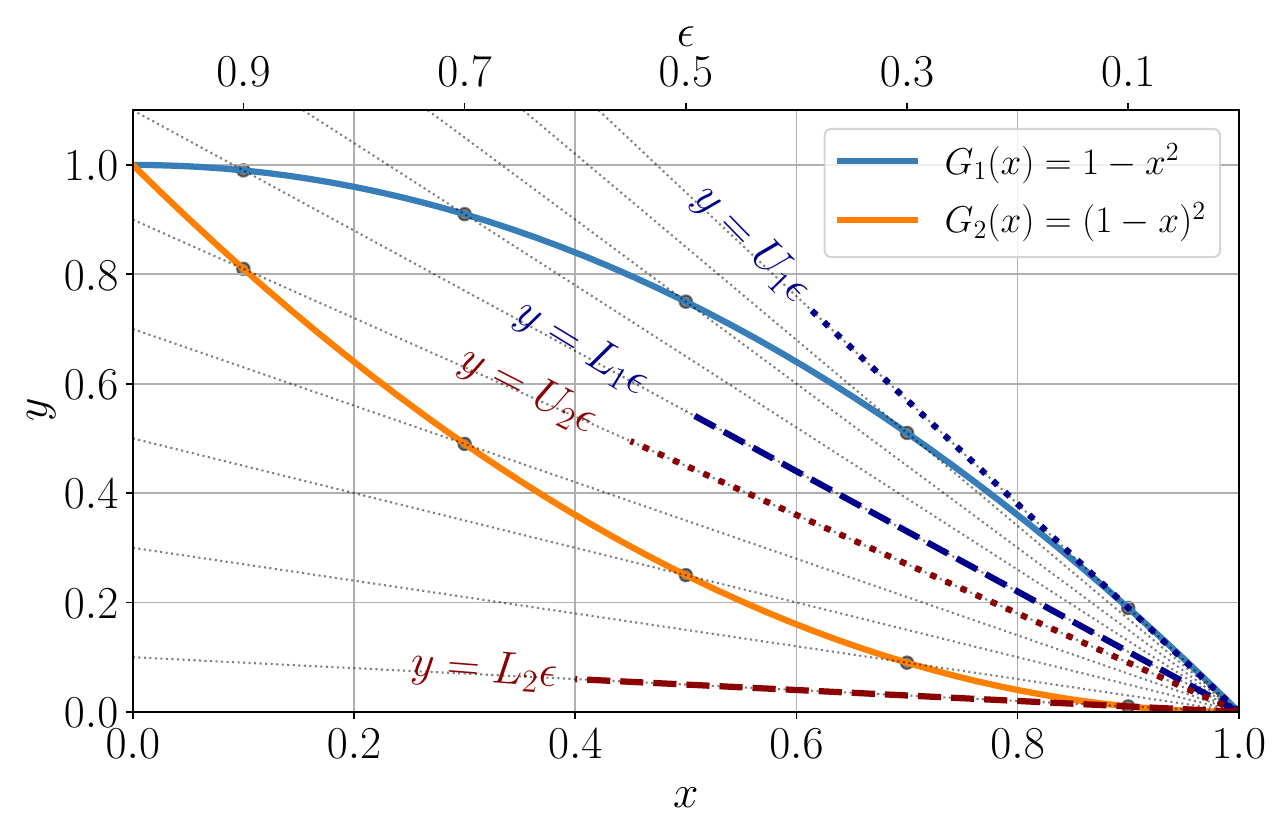}
\end{minipage}
\end{tabular}
\caption{(Left) Determining $L$ and $U$ for $G_1$ and $G_2$. We calculate $\frac{G(1-\epsilon)}{\epsilon}$ over different ranges of 
 $\epsilon$. The minimum and maximum values obtained from this ratio are assigned as $L$ and $U$, respectively. 
 (Right)  Showing two survival functions $G_1$ and $G_2$ along with their linear line approximations (gray lines). These tangent lines illustrate how $L$ and $U$  effectively bound the survival function $G$ near its maximum value.}
\label{app:fig:toy_survival_example}
\end{figure}

\textbf{Truncated uniform distribution}: Assume that we have a truncated uniform distribution with support $[a,b]$. We know $G(x)= \frac{b-x}{b-a}$ for $x \in (a,b)$. For every $\epsilon \in [a,b]$ we have $\frac{G(b - \epsilon) }{ \epsilon} = \frac{1}{b-a}$, which means $L=U= \frac{1}{b-a}$.

\textbf{Truncated Gaussian distribution}: There is no closed-form solution to formulate $L$ and $U$ based on the parameters of the truncated Gaussian distribution. Thus, we show the results of simulations to estimate $L$ and $U$ for truncated Gaussian within [0, 1] with various values $\mu$ and $\sigma$, averaging over $1000$ runs in Figure \ref{app:fig:L_U_Gaussian}.

\begin{figure}[hb!]
\centering
\begin{tabular}{c c}

\begin{minipage}{0.40\linewidth}
    \vspace{-5pt}
    \renewcommand{\arraystretch}{1.5}
    \small
        \begin{tabular}{|c|c|c|c|}
        \hline
        {$\mu$} & {$\sigma$} &{$L$} & {$U$} \\ 
        \hline
        0.25 & 0.5 & $0.58 \pm 0.06$ & $1.70 \pm 0.48$ \\ 
        0.50 &  0.5 & $0.85 \pm 0.07$ & $1.53 \pm 0.34$ \\ 
        0.75 &  0.5 & $1.01 \pm 0.01$ & $1.64 \pm 0.25$ \\ 
        0.25 & 0.2 & $0.34 \pm 0.07$ & $1.54 \pm 0.28$ \\ 
        0.50 & 0.2 & $0.44 \pm 0.08$ & $1.36 \pm 0.04$ \\ 
        0.75 & 0.2 & $1.01 \pm 0.00$ & $1.95 \pm 0.04$ \\ 
        \hline
    \end{tabular}
\end{minipage}
&\begin{minipage}{0.5\linewidth}
    \centering
    \includegraphics[width=\linewidth]{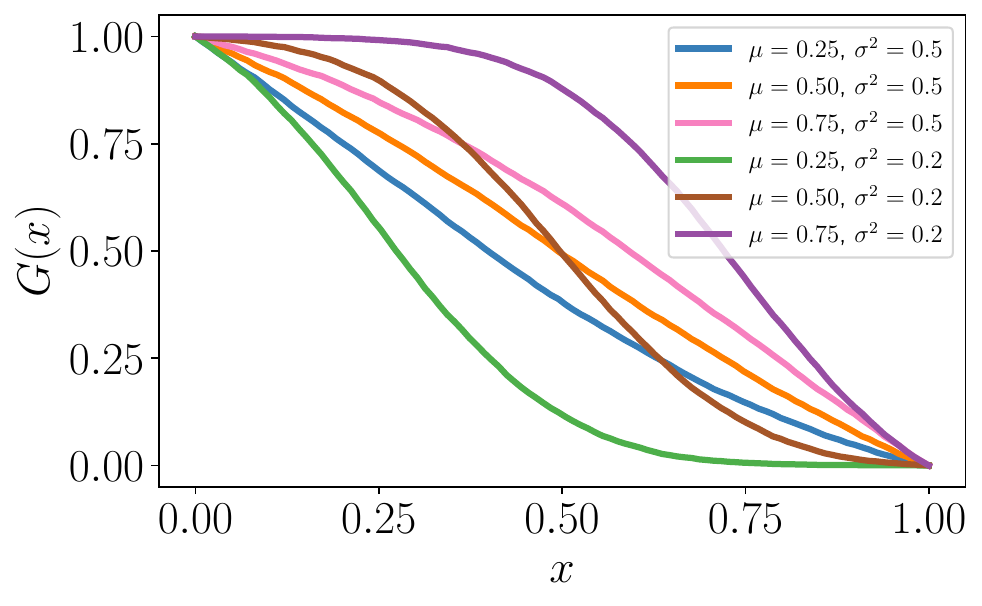}
\end{minipage}
\end{tabular}
\caption{(Left) Determining L and U for truncated Gaussian within [0, 1] with different values for $\mu$ and $\sigma$. Averaged over $1\,000$ runs. (Right) Showing survival function of truncated Gaussian distribution with different values for $\mu$ and $\sigma$.}
\label{app:fig:L_U_Gaussian}
\end{figure}

\clearpage
\subsection{Empirical Validation of \ourProposition}
\label{app:empirical_validation_assumption}
To study the empirical values for $L$ and $U$ in \ourProposition, we leverage the calculated empirical survival function $G$ for each dataset in Appendix \ref{app:reward_distribution_analysis}. Specifically, we evaluate $\frac{G(b - \epsilon)}{\epsilon}$ over the range $G^{-1}(0.99) < \epsilon < G^{-1}(0.01)$, where $G^{-1}(x)$ denotes the inverse of the survival function $G(x)$. Focusing on this range allows us to achieve a more robust estimation. Additionally, for \tabrepo{} dataset, we exclude 23 datasets containing an arm with a standard deviation smaller than $0.001$, further enhancing the robustness of our analysis. In Figures \ref{app:fig:L_U_histogram_tabrepo}, \ref{app:fig:L_U_histogram_tabreporaw}, \ref{app:fig:L_U_histogram_YaHPOGym}, and \ref{app:fig:L_U_histogram_Reshuffling}, the evaluated values of $\frac{G(b - \epsilon)}{\epsilon}$ for different benchmarks are shown. Additionally, the values for $L$ and $U$, corresponding to $\min\limits_{\epsilon}(\frac{G(b - \epsilon)}{\epsilon})$ and $\max\limits_{\epsilon}(\frac{G(b - \epsilon)}{\epsilon})$, respectively, are presented. Finally, the histograms of these two variables are also included.

\begin{figure}[]
\centering
\textbf{\tabrepo{}}\\
\vspace{0.5em} 
\includegraphics[clip, trim=0.0cm 0cm 0cm 0cm, width=0.9\textwidth]{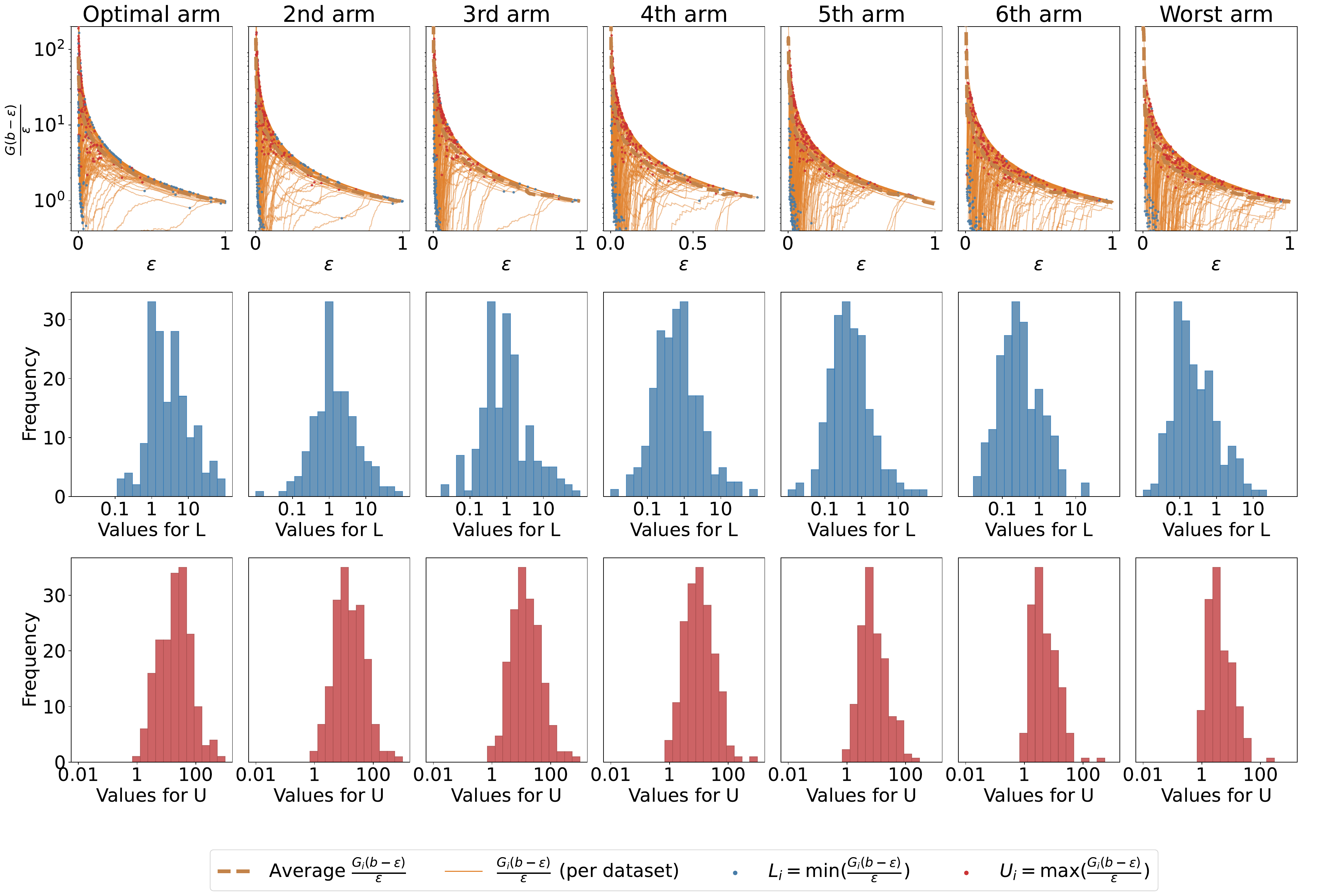}
\caption{ Arms are ordered by sub-optimality gap. (Top) Thin orange lines represent $\frac{G(b - \epsilon)}{\epsilon}$, while the blue and red points correspond to $L$ and $U$ for our empirical reward distributions (see \ourProposition{} for details). (Middle) Histogram of values for $L$. (Bottom) Histogram of values for $U$. }
\label{app:fig:L_U_histogram_tabrepo}
\end{figure}

\begin{figure}[]
\centering
\textbf{\tabreporaw{}}\\
\vspace{0.5em} 
\includegraphics[clip, trim=0.0cm 0cm 0cm 0cm, width=0.9\textwidth]{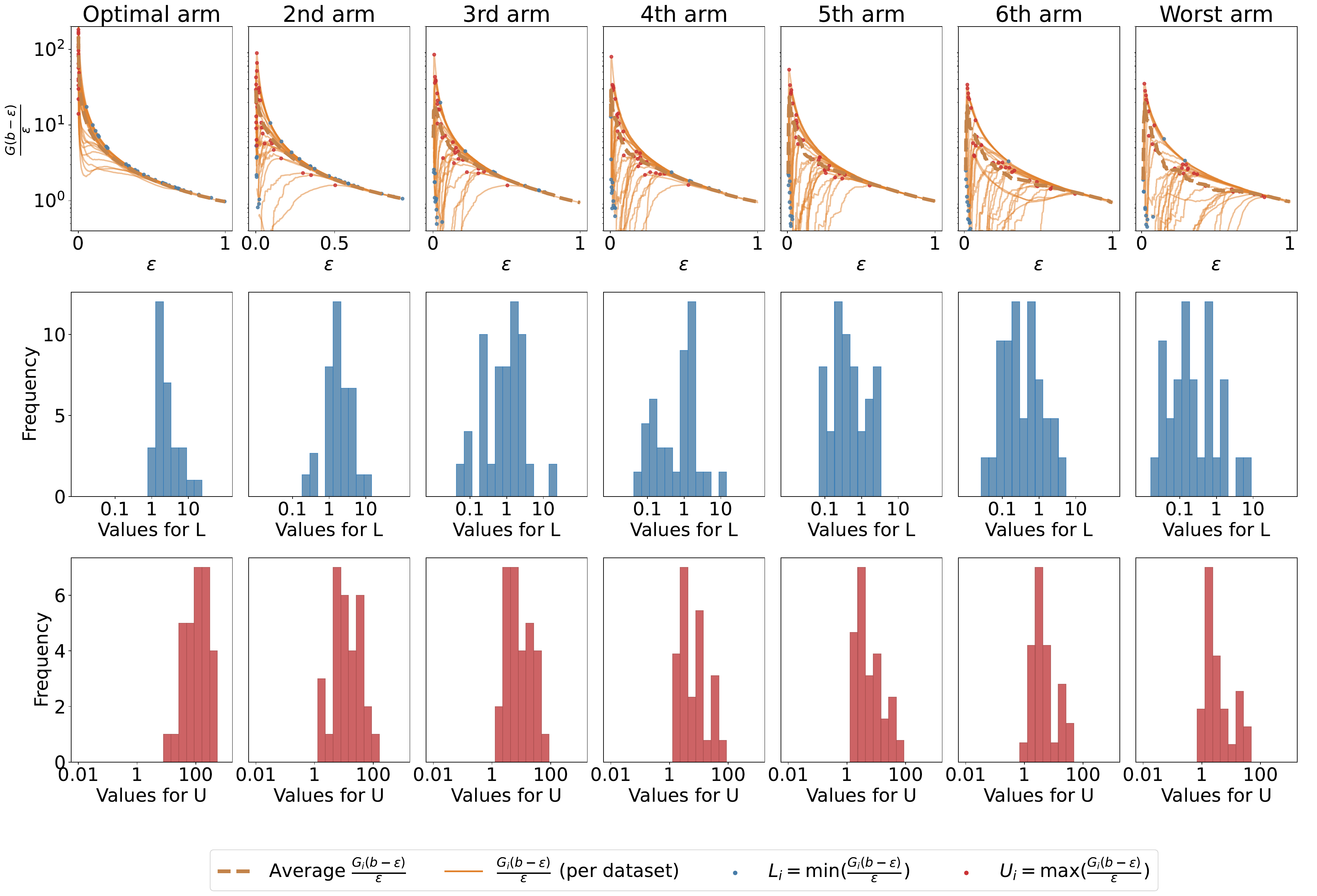}
\caption{ Arms are ordered by sub-optimality gap. (Top) Thin orange lines represent $\frac{G(b - \epsilon)}{\epsilon}$, while the blue and red points correspond to $L$ and $U$ for our empirical reward distributions (see \ourProposition{} for details). (Middle) Histogram of values for $L$. (Bottom) Histogram of values for $U$. }
\label{app:fig:L_U_histogram_tabreporaw}
\end{figure}

\begin{figure}[]
\centering
\textbf{\yahpogym{}}\\
\vspace{0.5em} 
\includegraphics[clip, trim=0.0cm 0cm 0cm 0cm, width=0.9\textwidth]{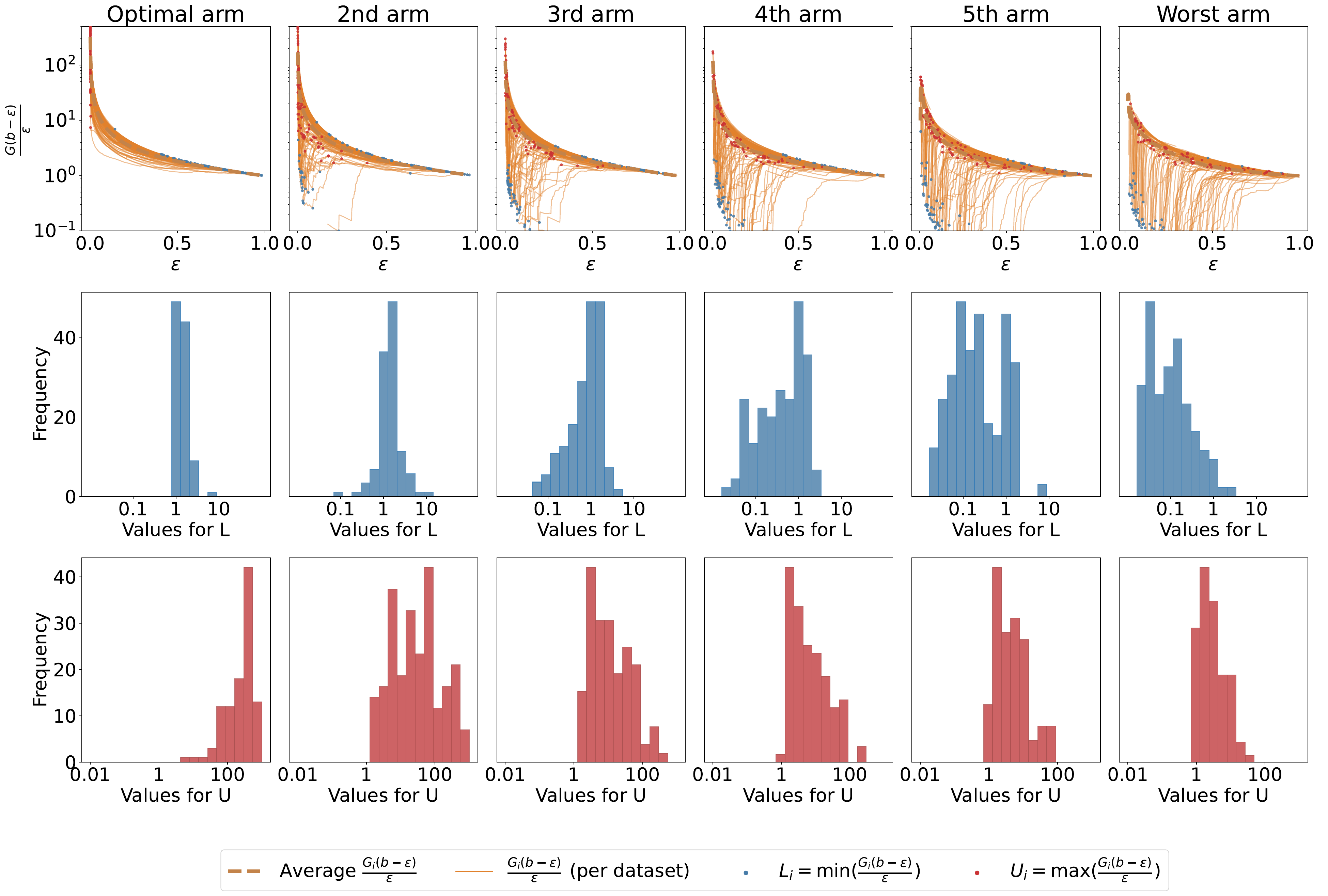}
\caption{ Arms are ordered by sub-optimality gap. (Top) Thin orange lines represent $\frac{G(b - \epsilon)}{\epsilon}$, while the blue and red points correspond to $L$ and $U$ for our empirical reward distributions (see \ourProposition{} for details). (Middle) Histogram of values for $L$. (Bottom) Histogram of values for $U$. }
\label{app:fig:L_U_histogram_YaHPOGym}
\end{figure}

\begin{figure}[]
\centering
\textbf{\hebo{}}\\
\vspace{0.5em} 
\includegraphics[clip, trim=0.0cm 3.3cm 0cm 0cm, width=0.9\textwidth]{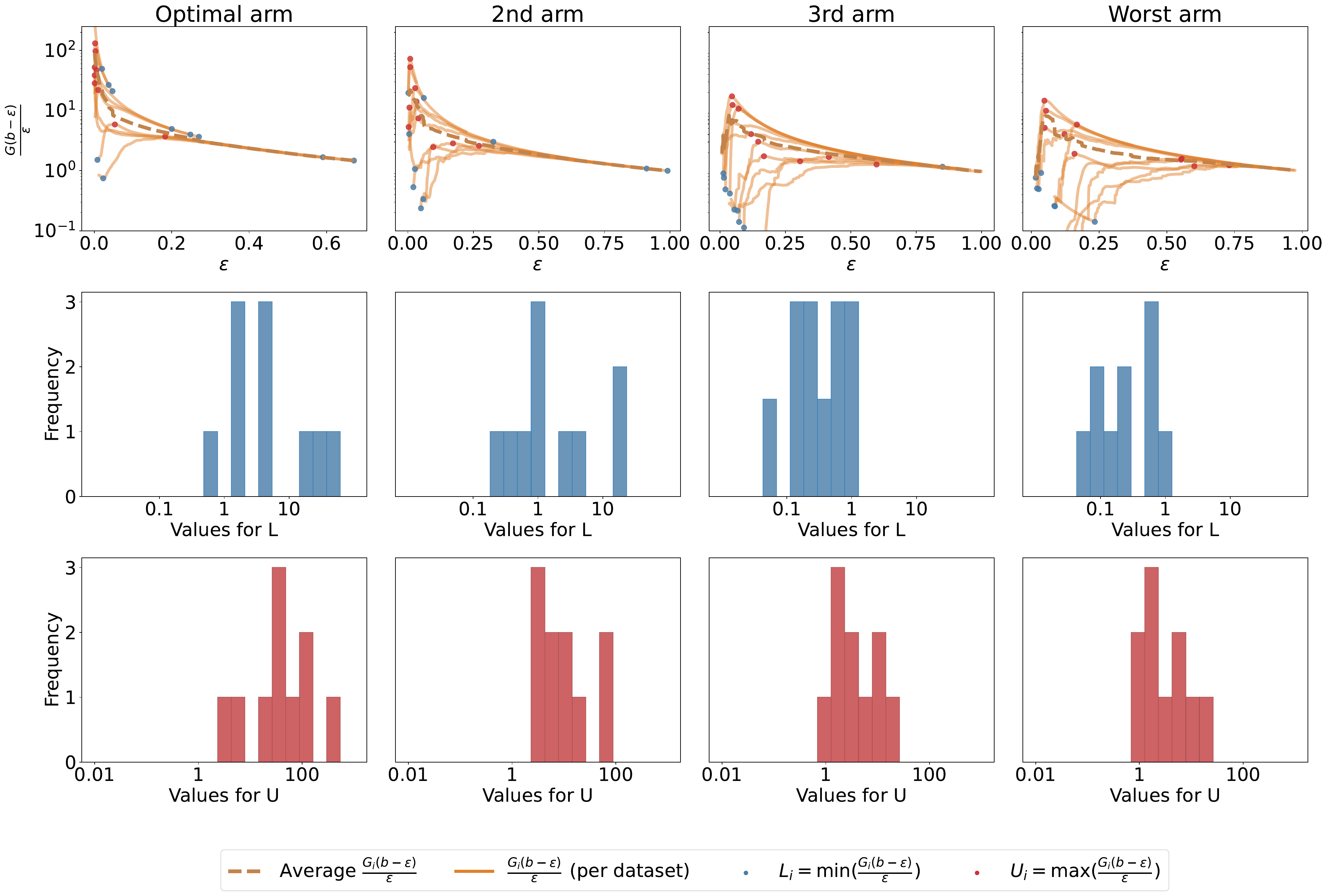}
\caption{ Arms are ordered by sub-optimality gap. (Top) Thin orange lines represent $\frac{G(b - \epsilon)}{\epsilon}$, while the blue and red points correspond to $L$ and $U$ for our empirical reward distributions (see \ourProposition{} for details). (Middle) Histogram of values for $L$. (Bottom) Histogram of values for $U$. }
\label{app:fig:L_U_histogram_Reshuffling}
\end{figure}

\clearpage
\section{More Details on the Experiments}
\label{app:details_on_experiments}

\subsection{Metric calculation}
\label{app:metric_calculation}
\textbf{Average ranking calculation.} We use bootstrapping with Monte Carlo sampling to calculate the average ranking plot with confidence intervals. For each time step and each task in every dataset, we resample the performance of each repetition (with replacement) and compute the average performance. We then rank the algorithms based on these averaged performances and repeat this process for all tasks. Finally, we average the rankings across tasks. This entire procedure was repeated 1000 times to estimate the confidence interval.

\textbf{Number of wins, ties, and losses.} To determine the number of wins, ties, and losses for each task in every dataset, we first compute the average performance of each algorithm over all repetitions at the final time step. We then perform pairwise comparisons of these averaged performances among all algorithms versus \combinedsearch{}. To account for negligible differences that are not statistically significant, we consider two performances to be tied if they are sufficiently close. Specifically, we use NumPy's \texttt{isclose} function to compare the averaged performances, treating values within a default tolerance of $1 \times 10^{-8}$ as equal.

\subsection{Experimental Setup}
\label{app:experimentalsetup}
Here, we provide details on our experimental setups as described in Table~\ref{tab:automltasks}. 

\begin{table*}[tbp]
\centering
\footnotesize
\begin{tabular}{@{\hskip 0mm}l@{\hskip 1mm}ccll@{\hskip 1mm}c@{\hskip 1mm}l@{\hskip 0mm}}
     name & \#models & \#tasks & type & HPO meth. (rep.) & budget & reference \\
     \midrule
     \yahpogym{} & 6 & 103 & surrogate & \SMAC{} ($32$) & 200 & \citep{pfisterer-automl22a} \\
     \tabrepo{} & 7 & 200 & tabular & \randomsearch{} ($32$) & 200 & \citep{salinas2024tabrepo} \\
     \tabreporaw{} & 7 & 30 & raw & \SMAC{} ($32$) & 200 & - \\
     \hebo{} & 4 & 10 & raw & HEBO ($30$) & 250 & \citep{nagler2024reshuffling} \\
     \bottomrule
\end{tabular}
\caption{Overview of AutoML tasks. For \tabrepo{} and \hebo{}, we use pre-computed HPO trajectories. 
\tabreporaw{} resembles the same model space as \tabrepo{}, but instead of \randomsearch{}, we run HPO ourselves. Similarly, we run HPO across provided surrogate HPO benchmark tasks \yahpogym{}. We use SMAC~\citep{lindauer-jmlr22a,hutter-lion11a} implementing Bayesian optimization using Random forests for both tasks.}
\label{tab:automltasks}
\end{table*}

\textbf{\yahpogym{}}~\citep{pfisterer-automl22a}, a surrogate benchmark, covers $6$ ML models (details in Table \ref{app:tab:yahpogym_spaces}) on $103$ datasets and uses a regression model (surrogate model) to predict performances for queried hyperparameter settings. We use Bayesian optimization as implemented by SMAC~\citep{lindauer-jmlr22a} using Random Forests to conduct HPO. Additionally, we compare our two-level approach to \combinedsearch{} using SMAC, SMAC without initial design (SMAC-no-init), and Random Search. We run $32$ repetitions and use a budget of $200$ iterations for each evaluation.

\begin{table}[htbp]
\caption{Hyperparameter spaces for ML models in \yahpogym{}.}
\label{app:tab:yahpogym_spaces}
\centering
\small
\begin{tabular}{lllll}
\toprule
ML model &  Hyperparameter & Type & Range & Info \\
\midrule
\multirow{2}{*}{-} & trainsize & continuous & [0.03, 1] &  =0.525 (fixed) \\ 
    & imputation & categorical & \{mean, median, hist\} &  =mean (fixed)\\ 
\midrule
\multirow{2}{*}{Glmnet} & alpha & continuous & [0, 1] & \\
    & s & continuous & [0.001, 1097] & log \\ 
\midrule
\multirow{4}{*}{Rpart} & cp & continuous & [0.001, 1] & log \\ 
    & maxdepth & integer & [1, 30] &  \\ 
    & minbucket & integer & [1, 100] &  \\
    & minsplit & integer & [1, 100] &  \\ 
\midrule
\multirow{5}{*}{SVM} & kernel & categorical & \{linear, polynomial, radial\} &  \\ 
    & cost & continuous & [4.5e-05, 2.2e4] & log \\ 
    & gamma & continuous & [4.5e-05, 2.2e4] & log, kernel \\ 
    & tolerance & continuous & [4.5e-05, 2] & log \\ 
    & degree & integer & [2, 5] &  kernel \\ 
\midrule
\multirow{5}{*}{AKNN} & k & integer & [1, 50] &  \\ 
    & distance & categorical & \{l2, cosine, ip\} &  \\ 
    & M & integer & [18, 50] &  \\ 
    & ef & integer & [7, 403] & log \\ 
    & ef\_construction & integer & [7, 403] & log \\ 
\midrule
\multirow{7}{*}{Ranger} & num.trees & integer & [1, 2000] &  \\ 
    & sample.fraction & continuous & [0.1, 1] &  \\ 
    & mtry.power & integer & [0, 1] &  \\ 
    & respect.unordered.factors & categorical & \{ignore, order, partition\} &  \\ 
    & min.node.size & integer & [1, 100] &  \\ 
    & splitrule & categorical & \{gini, extratrees\} &  \\ 
    & num.random.splits & integer & [1, 100] & splitrule \\ 
\midrule
\multirow{13}{*}{XGBoost} & booster & categorical & \{gblinear, gbtree, dart\} &  \\ 
    & nrounds & integer & [7, 2980] & log \\ 
    & eta & continuous & [0.001, 1] & log, booster \\ 
    & gamma & continuous & [4.5e-05, 7.4] & log, booster\\ 
    & lambda & continuous & [0.001, 1097] & log \\ 
    & alpha & continuous & [0.001, 1097] & log \\ 
    & subsample & continuous & [0.1, 1] &  \\ 
    & max\_depth & integer & [1, 15] &  booster \\ 
    & min\_child\_weight & continuous & [2.72, 148.4] & log,  booster \\ 
    & colsample\_bytree & continuous & [0.01, 1] &   booster \\ 
    & colsample\_bylevel & continuous & [0.01, 1] &  booster \\ 
    & rate\_drop & continuous & [0, 1] &  booster \\ 
    & skip\_drop & continuous & [0, 1] &  booster \\ 
\bottomrule
\end{tabular}
\end{table}

\textbf{\tabrepo{}}~\citep{salinas2024tabrepo} consists of pre-evaluated performance scores for $200$ iterations of random search for $7$ ML models (details in Table \ref{app:tab:tabrepo_spaces})  on $200$ datasets (context name: $D244\_F3\_C1530\_200$).
We run $32$ repetitions and use a budget of $200$ iterations for each task.

\begin{table}[htbp]
\caption{Hyperparameter spaces for ML models in \tabrepo{} and \tabreporaw{}.}
\label{app:tab:tabrepo_spaces}
\small
\centering
\begin{tabular}{llllll}
\toprule
ML model &  Hyperparameter & Type & Range & Info  & Default value \\
\midrule
\multirow{7}{*}{NN(PyTorch)} & learning rate & continuous & [1e-4, 3e-2] & log  & 3e-4 \\
    & weight decay & continuous & [1e-12, 0.1] & log  & 1e-6 \\
    & dropout prob & continuous & [0, 0.4] &  & 0.1 \\
    & use batchnorm & categorical & False, True  &  & \\  
    & num layers & integer & [1, 5]  &  & 2\\  
    & hidden size & integer & [8, 256]  &  & 128\\  
    & activation & categorical & relu, elu &  & \\  
\midrule
\multirow{9}{*}{NN(FastAI)}  & learning rate & continuous & [5e-4, 1e-1] & log  & 1e-2 \\
    & \multirow[t]{3}{*}{layers}  & \multirow[t]{3}{*}{categorical}   & [200], [400], [200, 100],  &  & \\  
    &   &    &  [400, 200], [800, 400], &  & \\  
    &   &  &   [200, 100, 50], [400, 200, 100] &  & \\  
    & emb drop & continuous & [0.0, 0.7] &   & 0.1 \\
    & ps & continuous &  [0.0, 0.7] &   & 0.1 \\
    & bs  & categorical & 256, 128, 512, 1024, 2048  &  & \\  
    & epochs & integer & [20, 50]  &  & 30\\    
\midrule
\multirow{6}{*}{CatBoost} & learning rate & continuous & [5e-3 ,0.1] & log  & 0.05 \\
    & depth & integer & [4, 8]  &  & 6\\  
    & l2 leaf reg & continuous & [1, 5] &  & 3\\
    & max ctr complexity & integer & [1, 5]  &  & 4\\  
    & one hot max size & categorical & 2, 3, 5, 10 &  & \\  
    & grow policy & categorical & SymmetricTree, Depthwise &  & \\  
\midrule
\multirow{5}{*}{LightGBM} & learning rate & continuous & [5e-3 ,0.1] & log  & 0.05 \\
    & feature fraction & continuous & [0.4, 1.0] &   & 1.0\\
    & min data in leaf & integer & [2, 60]  &  & 20\\  
    & num leaves & integer & [16, 255]  &  & 31\\  
    & extra trees & categorical & False, True &  & \\  
\midrule
\multirow{5}{*}{XGBoost} & learning rate & continuous & [5e-3 ,0.1] & log  & 0.1 \\
    & max depth & integer & [4, 10]  &  & 6\\  
    & min child weight & continuous & [0.5, 1.5] &   & 1.0\\
    & colsample bytree & continuous & [0.5, 1.0] &   & 1.0\\
    & enable categorical & categorical & False, True &  & \\  
\midrule
\multirow{3}{*}{Extra-trees} &  max leaf nodes& integer & [5000, 50000]  &  & \\ 
    & min samples leaf & categorical & 1, 2, 3, 4, 5, 10, 20, 40, 80 &  & \\  
    & max features & categorical & sqrt, log2, 0.5, 0.75, 1.0 &  & \\  
\midrule
\multirow{3}{*}{Random-forest} &  max leaf nodes& integer & [5000, 50000]  &  & \\ 
    & min samples leaf & categorical & 1, 2, 3, 4, 5, 10, 20, 40, 80 &  & \\  
    & max features & categorical & sqrt, log2, 0.5, 0.75, 1.0 &  & \\  
\bottomrule
\end{tabular}
\end{table}

\textbf{\tabreporaw{}} which uses the search space from \tabrepo{} (details in Table \ref{app:tab:tabrepo_spaces})  and allows HPO to evaluate all configurations. For constructing \tabrepo{}~\citep{salinas2024tabrepo}, each configuration was evaluated with a one-hour time limit and $8$-fold cross-validation. To reduce computational requirements for \tabreporaw{}, we reduced this to $5$ minutes and $4$-fold cross-validation, and we provide it for $30$ datasets (context name: $D244\_F3\_C1530\_30$). We use Bayesian optimization as implemented by SMAC~\citep{lindauer-jmlr22a} using Random Forests to conduct HPO. Additionally, we compare our two-level approach to \combinedsearch{} using SMAC, Random Search. We run $32$ repetitions and use a budget of $200$ iterations for the mentioned task.

To enable a fair comparison, we always evaluate the default configuration for each model first and then allow SMAC to run an initial design of $50-\#arms$ configurations in the upper level and $\frac{50}{\#arms-1}$ in the lower level.

\textbf{\hebo{}}~\citep{nagler2024reshuffling} which uses Heteroscedastic and Evolutionary Bayesian Optimization solver (HEBO)~\citep{cowenrivers-jair22a} for HPO. This benchmark includes HPO runs for $4$ ML models (details in Table \ref{app:tab:hebo_spaces})  across $10$ datasets, with $10$ repetitions and $3$ different validation split ratios within a budget of $250$ iterations.  Although the benchmark does not support HPO over the entire search space, it offers a valuable opportunity to compare the performance of bandit methods in a realistic setting.

\begin{table}[htbp]
\caption{Hyperparameter spaces for ML models in \hebo{}.}
\label{app:tab:hebo_spaces}
\centering
\begin{tabular}{llllll}
\toprule
ML model &  Hyperparameter & Type & Range & Info  \\
\midrule
\multirow{6}{*}{Funnel-Shaped MLP} & learning rate & continuous & [1e-4, 1e-1] & log  \\
    & num layers & integer & [1, 5]  & \\ 
    & max units & categorical & 64, 128, 256, 512  & \\ 
    & batch size & categorical. & {16, 32, ..., max\_batch\_size}  & \\ 
    & momentum & continuous. & [0.1, 0.99] &  \\
    & alpha & continuous. &  [1e-6, 1e-1] & log \\
\midrule
\multirow{2}{*}{Elastic Net} & C & continuous & [1e-6, 10e4] & log  \\
     & l1 ratio & continuous & [0.0, 1.0] &   \\
\midrule
\multirow{4}{*}{XGBoost} & max depth & integer & [2, 12]  &  log \\  
    & alpha & continuous & [1e-8, 1.5] &  log\\
    & lambda & continuous & [1e-8, 1.0] &  log\\
    & eta & continuous & [0.01, 0.3] &  log\\
\midrule
\multirow{3}{*}{CatBoost} & learning rate & continuous & [0.01 ,0.3] & log  \\
    & depth & integer & [2, 12]  & \\  
    & l2 leaf reg & continuous & [0.5, 30] & \\
\bottomrule
\end{tabular}
\end{table}

\clearpage
\subsection{Baselines and their hyperparameters}
\label{app:baselines_hyperparameters}
We use several bandit algorithms as baselines. Table~\ref{app:tab:bandit_parameters} summarizes the hyperparameters and their values.

\begin{table}[htbp]
\caption{Hyperparameters of Bandit Baselines.}
\label{app:tab:bandit_parameters}
\centering
\small
\begin{tabular}{lllll}
   \toprule
    Algorithm &  Hyperparameter & Value & Reference\\
    \midrule
        \OURALGO{} & $\alpha$ & 0.5  &  Ours\\
    \midrule
    \multirow{3}{*}{\textit{\makecell[l]{Quantile\\Bayes UCB}}} & $\alpha$ & 1.0 & \multirow{5}{*}{\cite{balef2024towards}} \\
          & $\beta$ & 0.2  \\
          & $\tau$ & 0.95  \\
    \cmidrule(lr){1-3}
    \multirow{2}{*}{\textit{Quantile UCB}} & $\alpha$ & 0.5 \\
         & $\tau$ & 0.95  \\
    \midrule
    \multirow{2}{*}{\ERUCBS{}} & $\beta$ & 0.6 &  \multirow{6}{*}{\cite{hu2021cascaded}} \\
         & $\theta$ & 0.01  \\
         & $\gamma$ & 20.0  \\
    \cmidrule(lr){1-3}
    \multirow{3}{*}{\textit{ER-UCB-N}} & $\alpha$ & 1.0 &  \\
         & $\theta$ & 0.01  \\
         & $\gamma$ & 20.0  \\
    \midrule
    \multirow{2}{*}{\RisingBandits} & $C$ & 7 &  \multirow{2}{*}{\cite{li2020efficient}}  \\
         & $T$ & Time horizon \\
    \midrule
    \MaxMedian & $\epsilon$ & $1/(t)$, $t$ is iteration &  \cite{bhatt2022extreme} \\
    \midrule
    \QoMaxSDA & $q$ & 0.5 &  \multirow{6}{*}{\cite{baudry2022efficient}} \\
    & $\gamma$ & 2/3 \\
    \cmidrule(lr){1-3}
    \textit{QoMax-ETC} & $q$ & 0.5  \\
    & $b_T$ & 4 \\
    & $n_T$ & 3 \\
    & $T$ & Time horizon \\
    \midrule
    \UCB & $\alpha$ & 0.5 & \cite{auer2002using} \\
    \midrule
    \textit{ThresholdAscent} & $\delta$ & 0.1  &  \multirow{3}{*}{\cite{streeter2006simple}}   \\
    & $s$ & 20 \\
    & $T$ & Time horizon \\
    \midrule
    \multirow{2}{*}{\SuccessiveHalving} & $\eta$ & 2.0 & \multirow{2}{*}{\cite{karnin2013almost}}   \\
         & $T$ & Time horizon \\
    \midrule
    \textit{R-SR} & $\epsilon$ & 0.25 & \multirow{6}{*}{\cite{mussi2024best}}  \\
    & $T$ & Time horizon \\
    \cmidrule(lr){1-3}
    \textit{R-UCBE} & $\alpha$ & 57.12  \\
    & $\epsilon$ & 0.25 \\
    & $\sigma$ & 0.05 \\
    & $T$ & Time horizon \\
    \midrule
    \textit{\makecell[l]{MaxSearch\\(Gaussian)}} & $c$ & 1.0 &  \multirow{2}{*}{\cite{kikkawa2024materials}}   \\
    \cmidrule(lr){1-3}
    \textit{\makecell[l]{MaxSearch\\(SubGaussian)}} & $c$ & 0.27 \\
    \bottomrule
\end{tabular}
\end{table}

\clearpage
\subsection{More Results on the Sensitivity Analysis of Hyperparameter \texorpdfstring{$\alpha$}{Lg}}
\label{app:aplation_study}
In addition to the results shown in Figure~\ref{fig:sensitivity_to_alpha} in Section~\ref{Sec:Numerical_Experiments}, we provide additional results here. We evaluated the performance of \OURALGO{} for $\alpha \in [0, 2.9]$ with step size $0.1$. We plot the performance over the number of iterations for different values of $\alpha$ in Figures~\ref{app:fig:ablation_TabRepo}, \ref{app:fig:ablation_TabRepoRaw}, \ref{app:fig:ablation_yahpo_gym}, 
 \ref{app:fig:ablation_hebo}. Green indicates better performance, showing the impact of $\alpha$ at different stages of the optimization procedure. Furthermore, we provide a comparison between \OURALGO{} with different values of $\alpha$ with \combinedsearch{}(\SMAC{}) in Table \ref{app:ablation_table}. For the experiments in the main paper, we choose $\alpha = 0.5$ as a robust choice over all datasets. Notably, $\alpha=0.5$ is selected based on the assumption that the reward distribution support is $[0,1]$. For other supports, we recommend scaling $\alpha$ according to the range of the support.
 
\label{app:ablation_study}
\begin{figure}[htbp]
    \centering
    \begin{subfigure}{0.48\textwidth}
        \includegraphics[height=6cm]{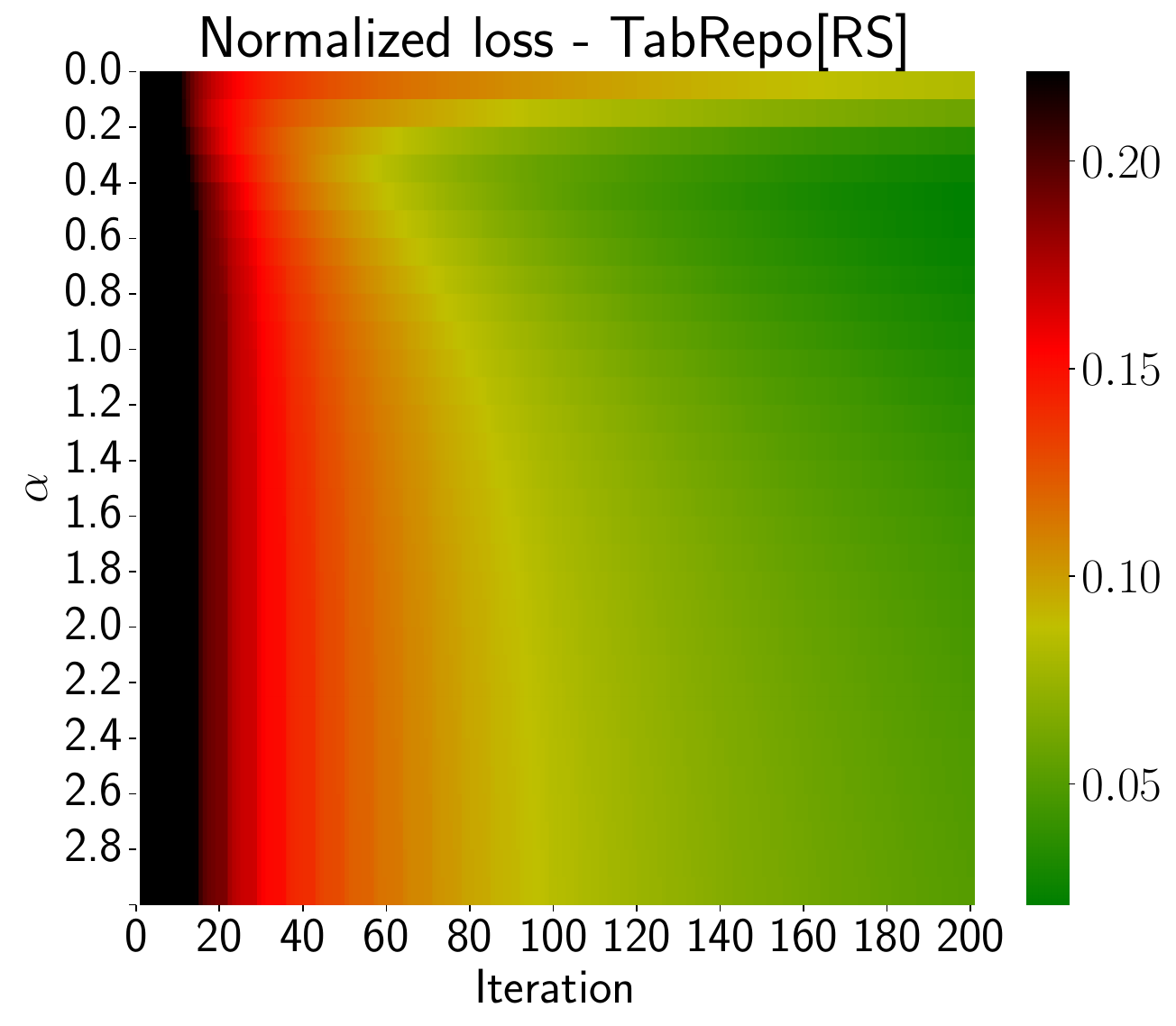}
    \end{subfigure}
    \begin{subfigure}{0.48\textwidth}
        \includegraphics[height=6cm]{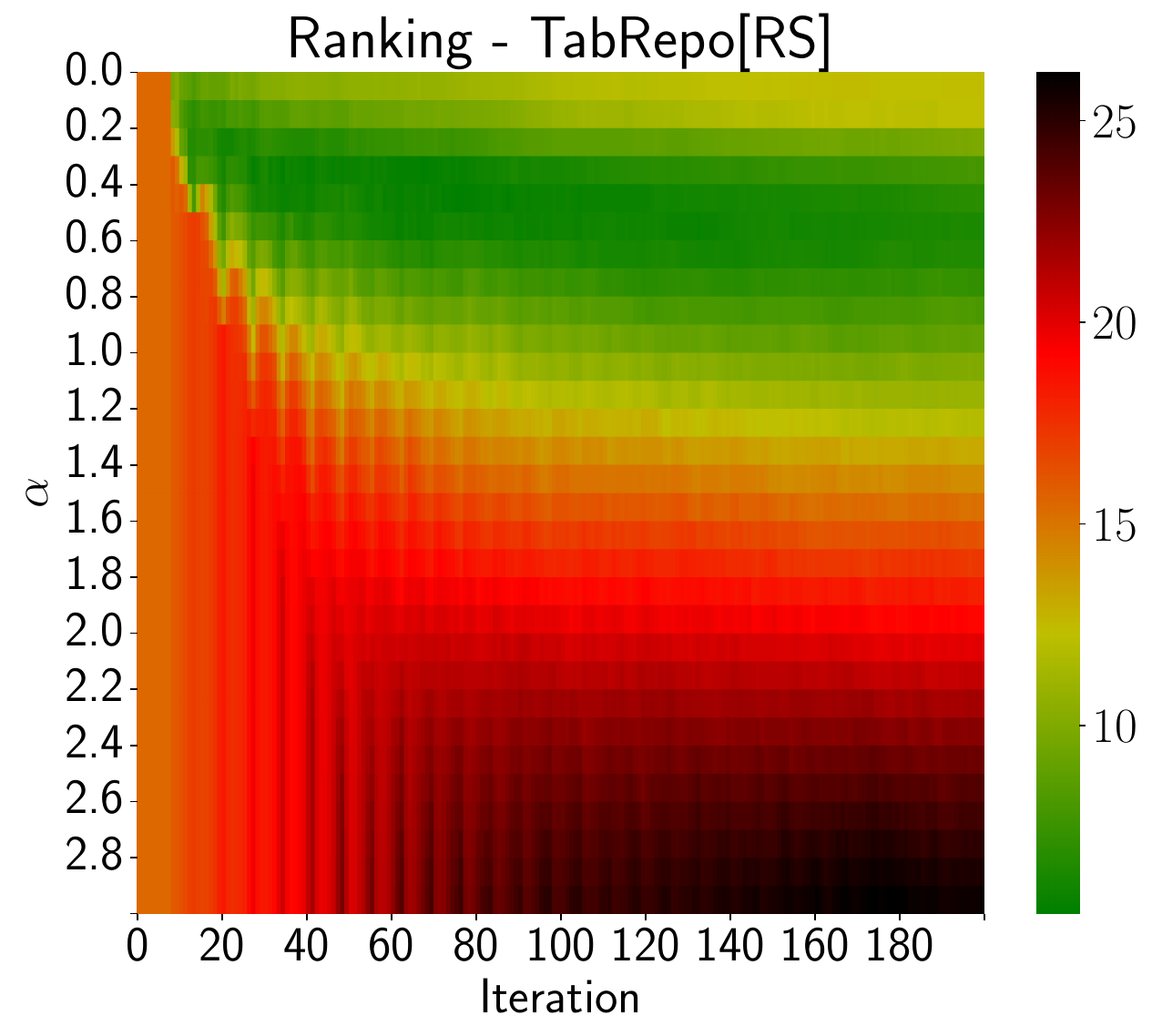}
    \end{subfigure}
    \caption{ Heatmap showing the performance of our algorithm with different values of $\alpha$ for \tabrepo{} dataset.  }
    \label{app:fig:ablation_TabRepo}
\end{figure}

\begin{figure}[htbp]
    \centering
    \begin{subfigure}{0.48\textwidth}
        \includegraphics[height=6cm]{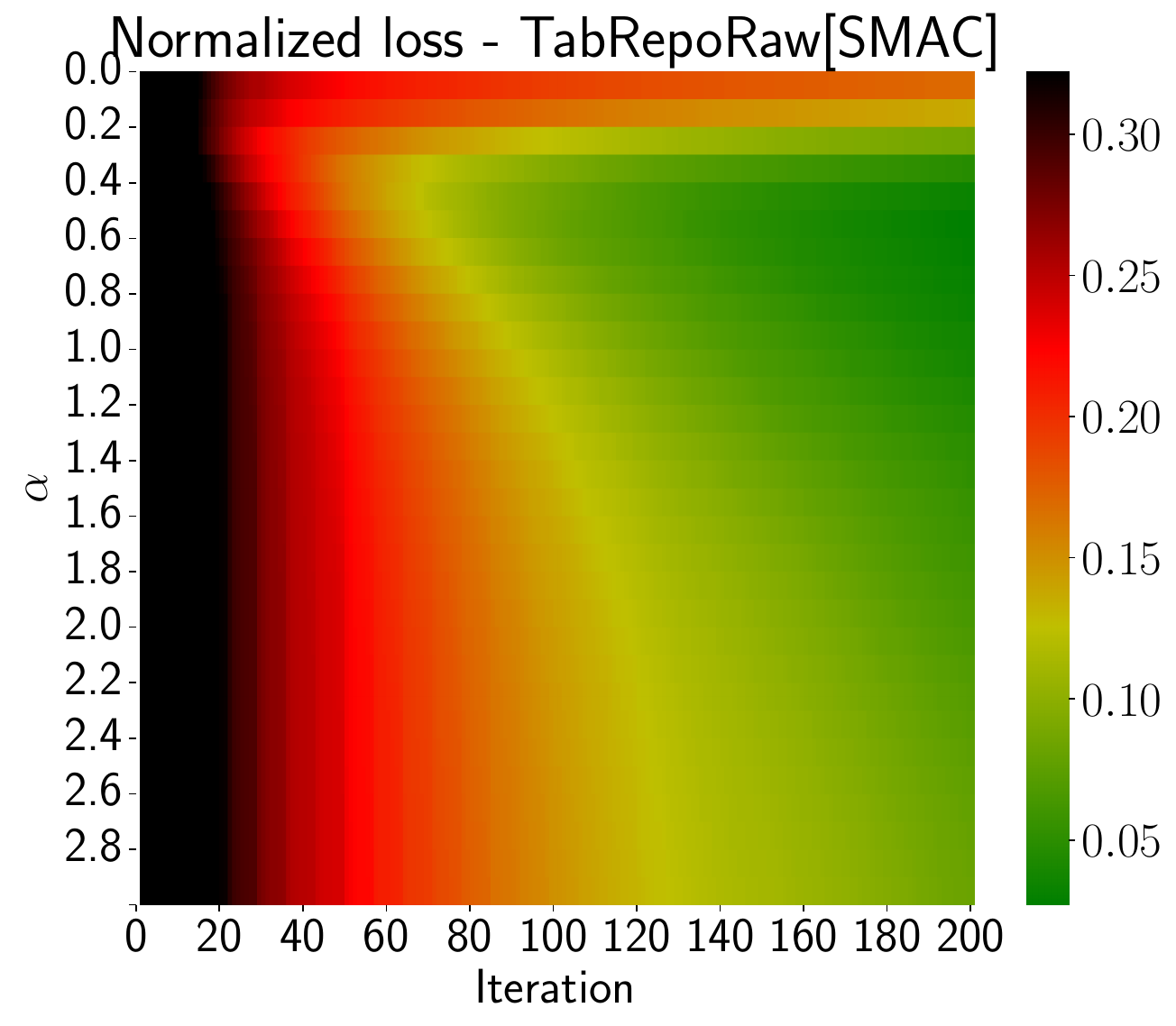}
    \end{subfigure}
    \begin{subfigure}{0.48\textwidth}
        \includegraphics[height=6cm]{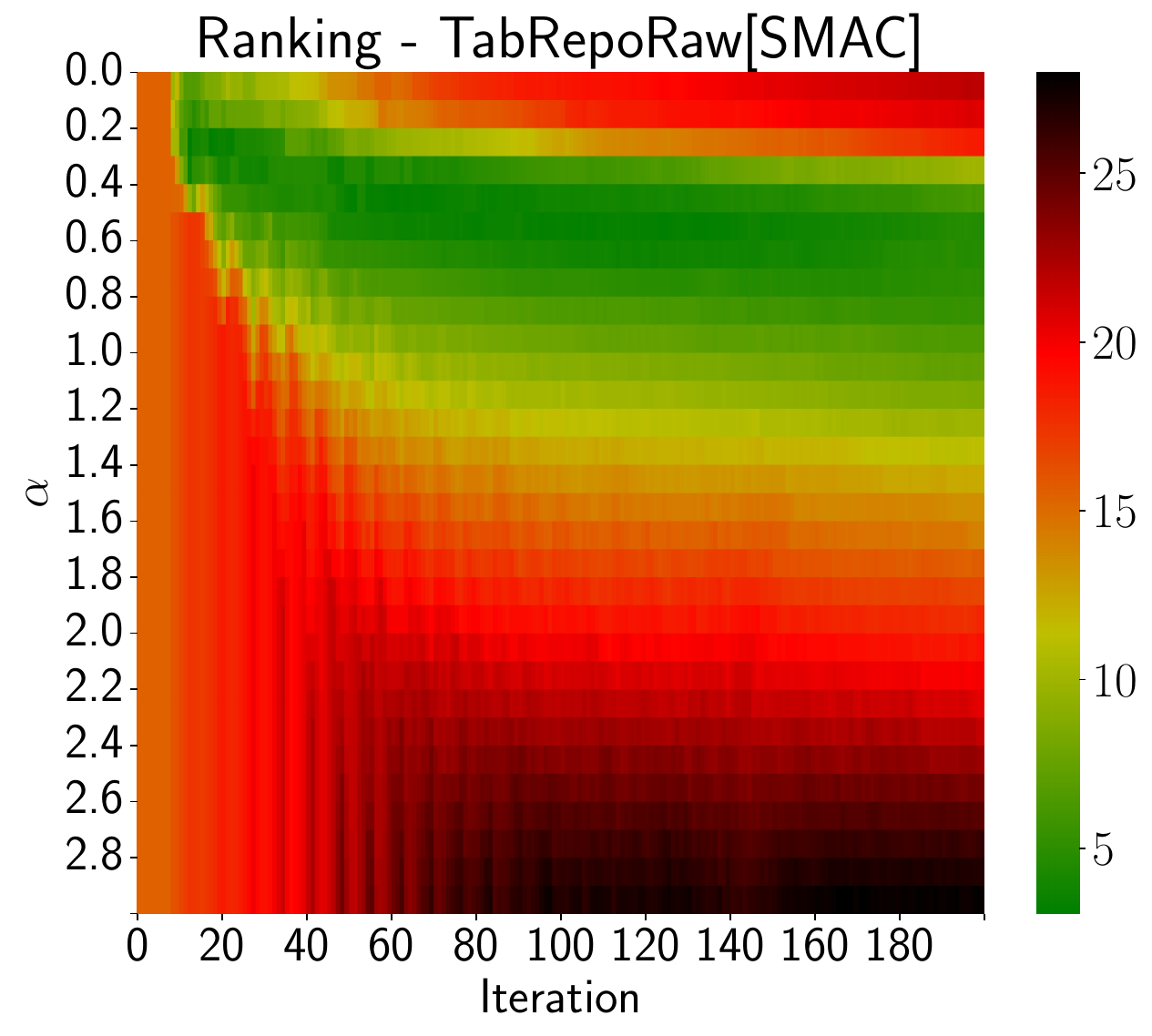}
    \end{subfigure}
    \caption{ Heatmap showing the performance of our algorithm with different values of $\alpha$ for \tabreporaw{} dataset.  }
    \label{app:fig:ablation_TabRepoRaw}
\end{figure}

\begin{figure}[htbp]
    \centering
    \begin{subfigure}{0.48\textwidth}
        \includegraphics[height=6cm]{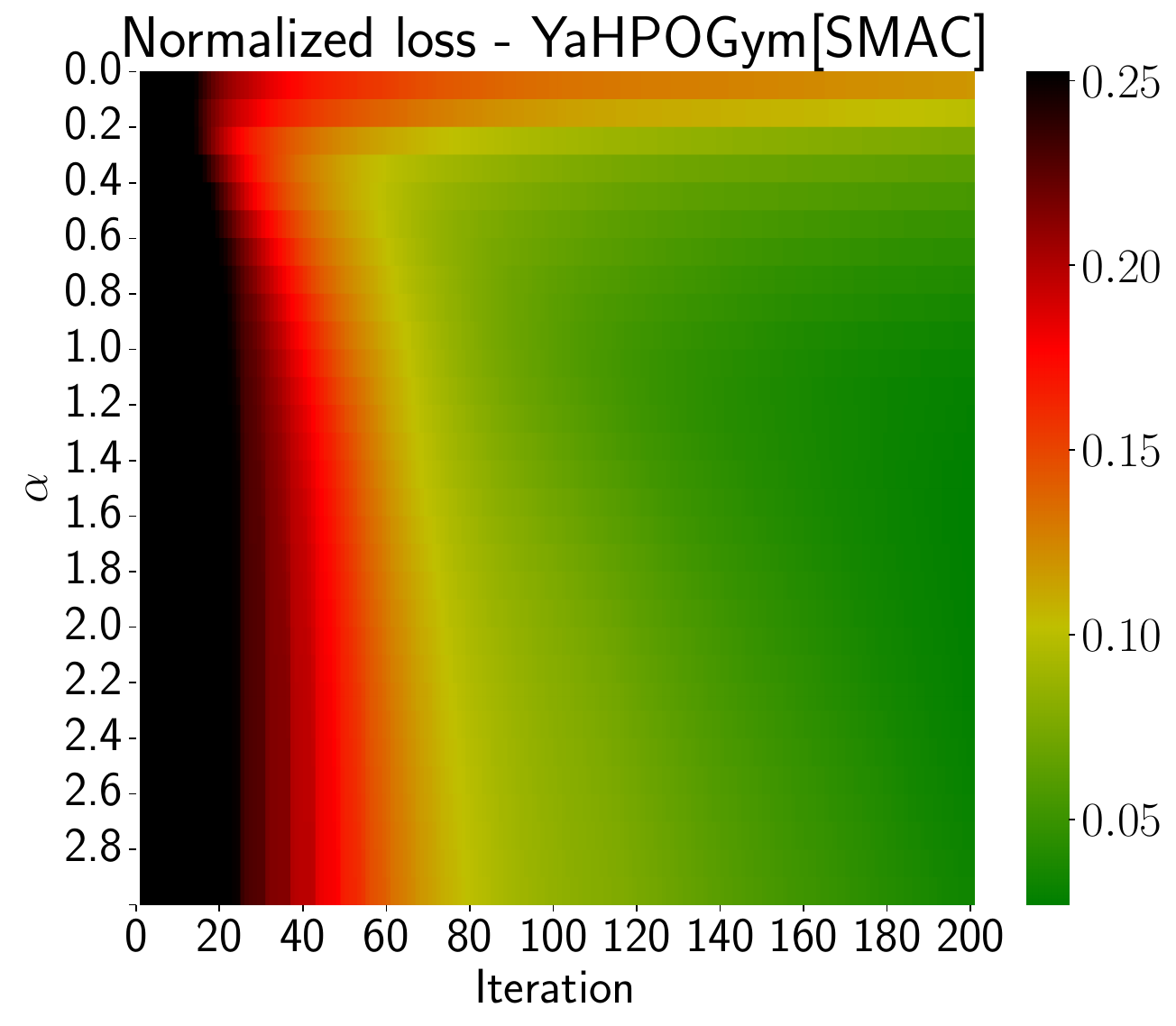}
    \end{subfigure}%
    \begin{subfigure}{0.48\textwidth}
        \includegraphics[height=6cm]{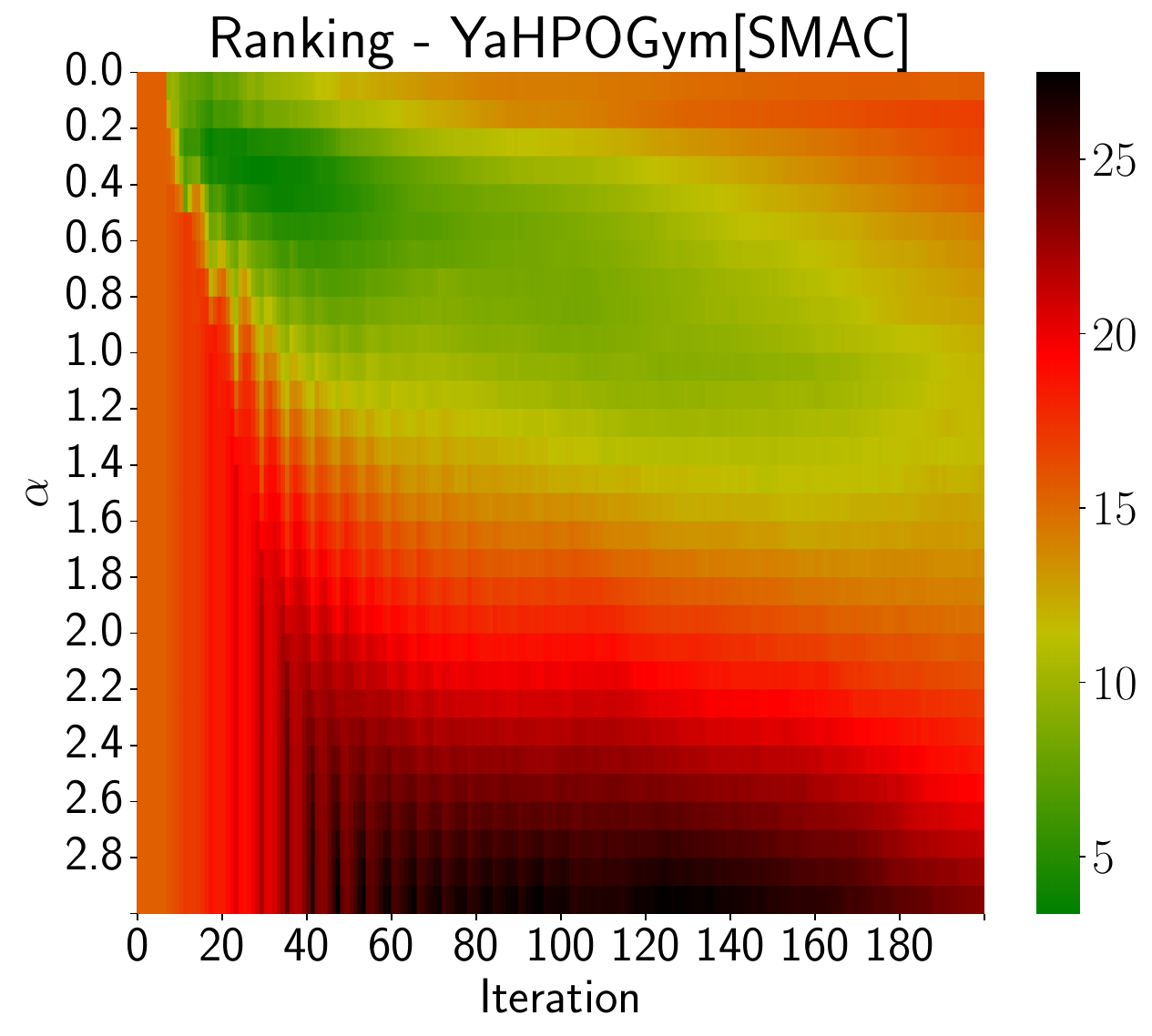}
    \end{subfigure}
    \caption{Heatmap showing the performance of our algorithm with different values of $\alpha$ for \yahpogym{} dataset. }
    \label{app:fig:ablation_yahpo_gym}
\end{figure}

\begin{figure}[htbp]
    \centering
    \begin{subfigure}{0.48\textwidth}
        \includegraphics[height=6cm]{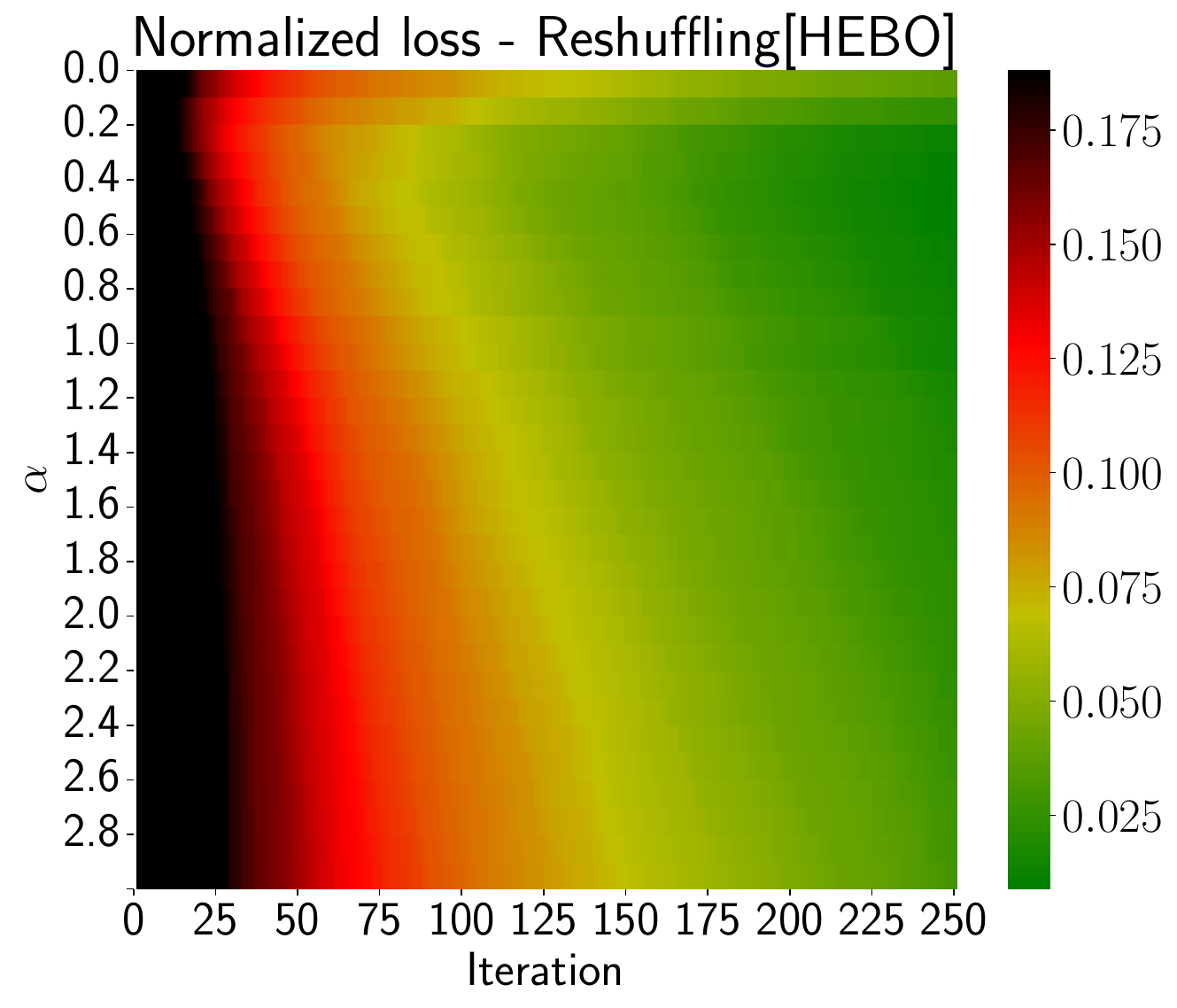}
    \end{subfigure}%
    \begin{subfigure}{0.48\textwidth}
        \includegraphics[height=6cm]{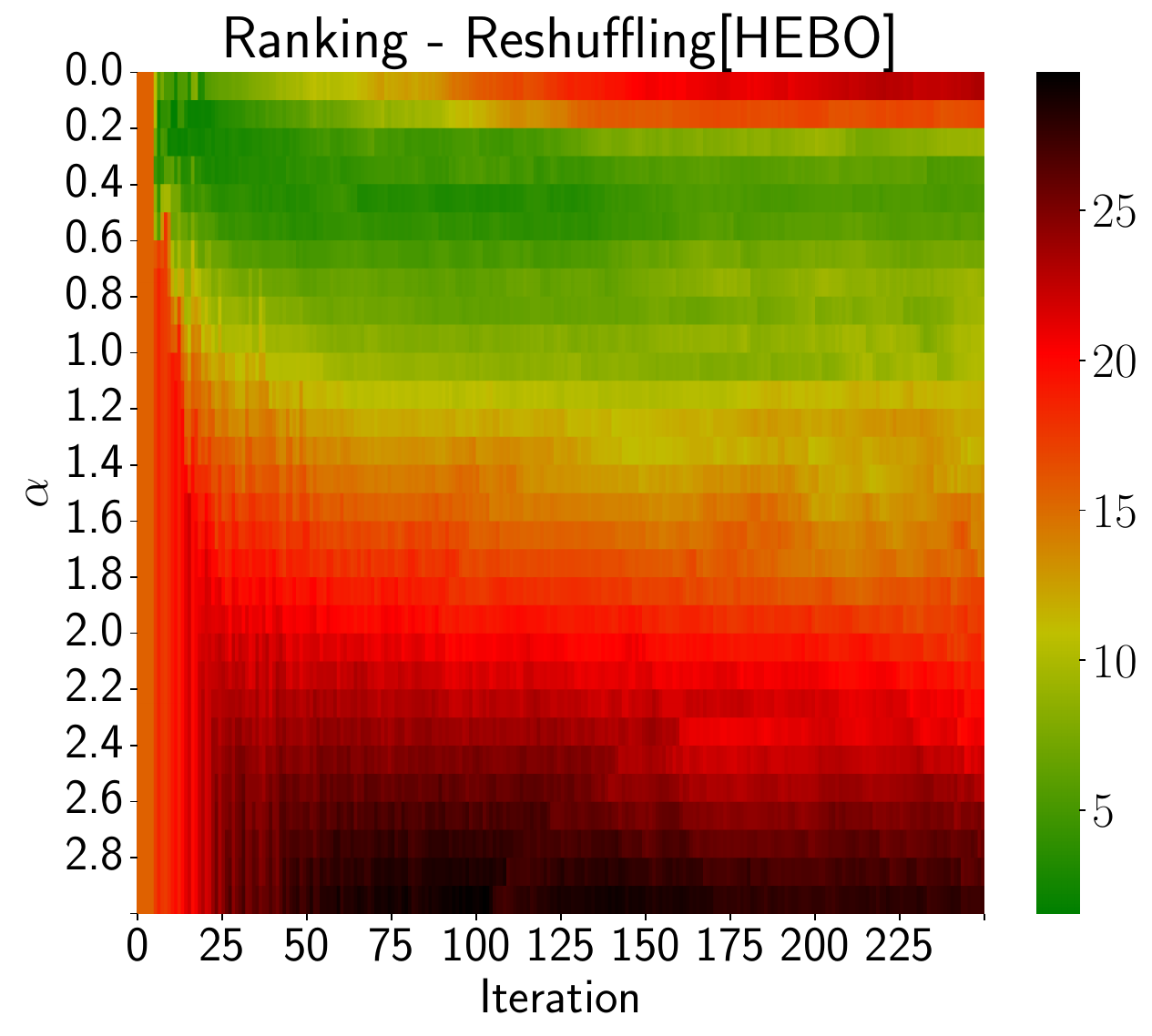}
    \end{subfigure}
    \caption{ Heatmap showing the performance of our algorithm with different values of $\alpha$ for \hebo{} dataset. }
    \label{app:fig:ablation_hebo}
\end{figure}

\begin{table}[htbp]
\centering
\caption{Comparing \OURALGO{}  with \combinedsearch{}(\SMAC{}) for different values of $\alpha$ and time steps. P-values from a sign test assessing whether bandit methods outperform \combinedsearch{}. P-values below $\alpha = 0.05$ are underlined, while those below $\alpha = 0.05$ after multiple comparison correction (adjusting $\alpha$ by the number of comparisons) are boldfaced and underlined indicating that the two-level approach is superior to \combinedsearch{}. 
Additionally, we report the normalized loss and the number of wins, ties, and losses (w/t/l) of bandit methods.}
\resizebox{\textwidth}{!}{%
\begin{tabular}{lllllllllllll}
& Time &  & $\alpha$=0.0 & $\alpha$=0.1 & $\alpha$=0.2 & $\alpha$=0.3 & $\alpha$=0.4 & $\alpha$=0.5 & $\alpha$=0.6 & $\alpha$=0.7 & $\alpha$=0.8 & $\alpha$=0.9\\
\midrule 
\multirow{12}{*}{\rotatebox[origin=c]{90}{TabRepoRaw[SMAC]}}\\
\cmidrule{2-13} 
& \multicolumn{1}{l}{\multirow{3}{*}{50}} & p-value  & $0.1002$ & $\mathbf{0.0214}$ & $\mathbf{\underline{0.0000}}$ & $\mathbf{\underline{0.0000}}$ & $\mathbf{\underline{0.0000}}$ & $\mathbf{\underline{0.0000}}$ & $\mathbf{\underline{0.0000}}$ & $\mathbf{\underline{0.0000}}$ & $\mathbf{\underline{0.0000}}$ & $\mathbf{\underline{0.0000}}$ \\
& &  w/t/l  & $19$/$0$/$11$ & $21$/$0$/$9$ & $27$/$0$/$3$ & $29$/$0$/$1$ & $29$/$0$/$1$ & $29$/$0$/$1$ & $29$/$0$/$1$ & $28$/$0$/$2$ & $27$/$0$/$3$ & $27$/$0$/$3$ \\
& &  loss & $0.2555$ & $0.2446$ & $0.2172$ & $0.1946$ & $0.1965$ & $0.2134$ & $0.2203$ & $0.2229$ & $0.2249$ & $0.2296$\\
\cmidrule{2-13} 
& \multicolumn{1}{l}{\multirow{3}{*}{100}} & p-value  & $0.9974$ & $0.9919$ & $0.8998$ & $0.1808$ & $\mathbf{0.0081}$ & $\mathbf{0.0214}$ & $0.1002$ & $0.5722$ & $0.8192$ & $0.9919$ \\
& &  w/t/l  & $8$/$0$/$22$ & $9$/$0$/$21$ & $12$/$0$/$18$ & $18$/$0$/$12$ & $22$/$0$/$8$ & $21$/$0$/$9$ & $19$/$0$/$11$ & $15$/$0$/$15$ & $13$/$0$/$17$ & $9$/$0$/$21$ \\
& &  loss & $0.2138$ & $0.2053$ & $0.1832$ & $0.1346$ & $0.1283$ & $0.1113$ & $0.1164$ & $0.1208$ & $0.1279$ & $0.1512$\\
\cmidrule{2-13} 
& \multicolumn{1}{l}{\multirow{3}{*}{150}} & p-value  & $0.9993$ & $0.9919$ & $0.9506$ & $0.2923$ & $\mathbf{\underline{0.0007}}$ & $\mathbf{\underline{0.0007}}$ & $\mathbf{\underline{0.0026}}$ & $\mathbf{0.0081}$ & $0.1002$ & $0.7077$ \\
& &  w/t/l  & $7$/$0$/$23$ & $9$/$0$/$21$ & $11$/$0$/$19$ & $17$/$0$/$13$ & $24$/$0$/$6$ & $24$/$0$/$6$ & $23$/$0$/$7$ & $22$/$0$/$8$ & $19$/$0$/$11$ & $14$/$0$/$16$ \\
& &  loss & $0.1994$ & $0.1911$ & $0.1671$ & $0.1020$ & $0.0898$ & $0.0752$ & $0.0775$ & $0.0849$ & $0.0887$ & $0.0973$\\
\cmidrule{2-13} 
& \multicolumn{1}{l}{\multirow{3}{*}{200}} & p-value  & $0.9998$ & $0.9993$ & $0.9786$ & $0.1808$ & $\mathbf{0.0214}$ & $\mathbf{\underline{0.0007}}$ & $\mathbf{\underline{0.0026}}$ & $\mathbf{\underline{0.0026}}$ & $\mathbf{0.0081}$ & $\mathbf{0.0081}$ \\
& &  w/t/l  & $6$/$0$/$24$ & $7$/$0$/$23$ & $10$/$0$/$20$ & $18$/$0$/$12$ & $21$/$0$/$9$ & $24$/$0$/$6$ & $23$/$0$/$7$ & $23$/$0$/$7$ & $22$/$0$/$8$ & $22$/$0$/$8$ \\
& &  loss & $0.1864$ & $0.1790$ & $0.1489$ & $0.0686$ & $0.0651$ & $0.0563$ & $0.0622$ & $0.0698$ & $0.0703$ & $0.0751$\\
\bottomrule
\multirow{12}{*}{\rotatebox[origin=c]{90}{YaHPOGym[SMAC]}}\\
\cmidrule{2-13} 
& \multicolumn{1}{l}{\multirow{3}{*}{50}}& p-value  & $\mathbf{\underline{0.0000}}$ & $\mathbf{\underline{0.0000}}$ & $\mathbf{\underline{0.0000}}$ & $\mathbf{\underline{0.0000}}$ & $\mathbf{\underline{0.0000}}$ & $\mathbf{\underline{0.0000}}$ & $\mathbf{\underline{0.0000}}$ & $\mathbf{\underline{0.0000}}$ & $\mathbf{\underline{0.0000}}$ & $\mathbf{\underline{0.0000}}$ \\
& &  w/t/l  & $74$/$1$/$28$ & $83$/$1$/$19$ & $95$/$1$/$7$ & $102$/$1$/$0$ & $102$/$1$/$0$ & $102$/$1$/$0$ & $102$/$1$/$0$ & $102$/$1$/$0$ & $102$/$1$/$0$ & $101$/$1$/$1$ \\
& &  loss & $0.1853$ & $0.1532$ & $0.1071$ & $0.0930$ & $0.0942$ & $0.0978$ & $0.1047$ & $0.1101$ & $0.1151$ & $0.1190$\\
\cmidrule{2-13} 
& \multicolumn{1}{l}{\multirow{3}{*}{100}}& p-value  & $0.3833$ & $0.3833$ & $0.3833$ & $\mathbf{0.0459}$ & $\mathbf{0.0112}$ & $\mathbf{0.0112}$ & $\mathbf{0.0088}$ & $0.0572$ & $0.2153$ & $0.4220$ \\
& &  w/t/l  & $53$/$1$/$49$ & $53$/$1$/$49$ & $53$/$1$/$49$ & $60$/$1$/$42$ & $63$/$1$/$39$ & $63$/$1$/$39$ & $64$/$0$/$39$ & $60$/$0$/$43$ & $56$/$0$/$47$ & $53$/$0$/$50$ \\
& &  loss & $0.1494$ & $0.1177$ & $0.0876$ & $0.0813$ & $0.0782$ & $0.0713$ & $0.0668$ & $0.0702$ & $0.0718$ & $0.0749$\\
\cmidrule{2-13} 
& \multicolumn{1}{l}{\multirow{3}{*}{150}}& p-value  & $0.3833$ & $0.6167$ & $0.8135$ & $0.0686$ & $\mathbf{\underline{0.0036}}$ & $\mathbf{0.0065}$ & $\mathbf{\underline{0.0028}}$ & $\mathbf{\underline{0.0005}}$ & $\mathbf{\underline{0.0028}}$ & $\mathbf{0.0148}$ \\
& &  w/t/l  & $53$/$1$/$49$ & $50$/$1$/$52$ & $47$/$1$/$55$ & $59$/$1$/$43$ & $65$/$1$/$37$ & $64$/$1$/$38$ & $66$/$0$/$37$ & $68$/$1$/$34$ & $66$/$0$/$37$ & $63$/$0$/$40$ \\
& &  loss & $0.1378$ & $0.1013$ & $0.0758$ & $0.0722$ & $0.0716$ & $0.0551$ & $0.0518$ & $0.0507$ & $0.0520$ & $0.0500$\\
\cmidrule{2-13} 
& \multicolumn{1}{l}{\multirow{3}{*}{200}}& p-value  & $0.5394$ & $0.8135$ & $0.6896$ & $0.2442$ & $\mathbf{0.0185}$ & $\mathbf{0.0088}$ & $\mathbf{\underline{0.0015}}$ & $\mathbf{\underline{0.0002}}$ & $\mathbf{\underline{0.0000}}$ & $\mathbf{\underline{0.0001}}$ \\
& &  w/t/l  & $51$/$1$/$51$ & $47$/$1$/$55$ & $49$/$1$/$53$ & $55$/$1$/$47$ & $62$/$1$/$40$ & $64$/$0$/$39$ & $67$/$0$/$36$ & $70$/$0$/$33$ & $75$/$0$/$28$ & $71$/$0$/$32$ \\
& &  loss & $0.1190$ & $0.0856$ & $0.0697$ & $0.0660$ & $0.0614$ & $0.0457$ & $0.0443$ & $0.0418$ & $0.0423$ & $0.0421$\\
\bottomrule
\end{tabular}
}
\label{app:ablation_table}
\end{table}

\clearpage
\subsection{More Results for the Empirical Evaluation}
\label{app:detailed_experiments}
In addition to the analysis in the main paper in Figure~\ref{fig:average_rank} in Section~\ref{Sec:Numerical_Experiments}, we report the averaged normalized loss over time in Figure~\ref{app:fig:average_norm_loss}, the average ranking in Figure~\ref{app:fig:average_ranking}, the normalized loss per task in Figure~\ref{app:fig:heatmap_norm_loss}, the ranking per task in Figure~\ref{app:fig:heatmap_ranking} and critical distance plots as described by~\cite{demsar-06a} in Figure~\ref{app:fig:res_autorank}. Additionally, we report results for a \textit{Random Policy} (yellow) that selects arms to pull at random.

\begin{figure*}[htbp]
    \begin{tabular}{c c c c c}
    \scriptsize\tabrepo{[RS]}
    &\hspace{-1.0em}\scriptsize\tabreporaw{[SMAC]}
    & \hspace{-1em}\scriptsize\yahpogym{[SMAC]}
    &\hspace{-1.5em}\scriptsize\hebo{[HEBO]}
    & \\
    \hspace{-1.5em}\includegraphics[clip, trim=0.2cm 0cm 0.3cm 0.25cm,height=3.8cm]{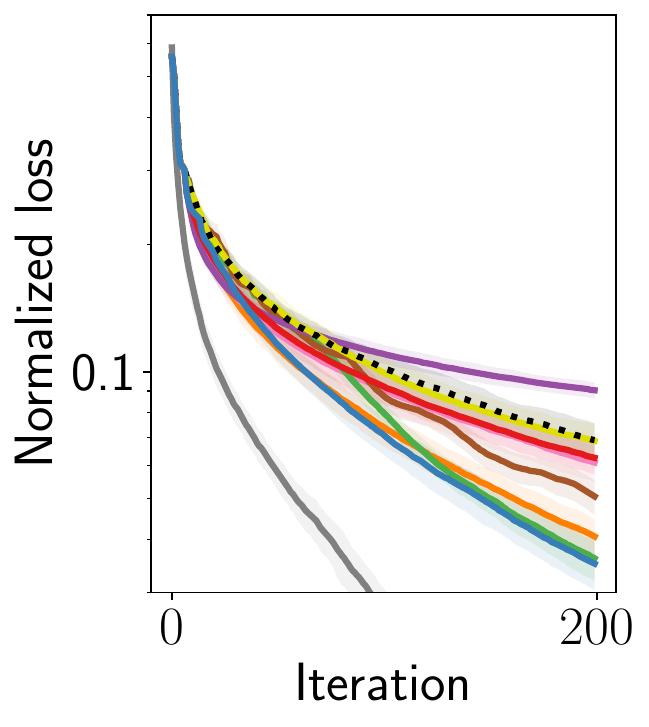}
    & \hspace{-1em} \includegraphics[clip, trim=0.3cm 0cm 0cm 0.25cm, height=3.8cm]{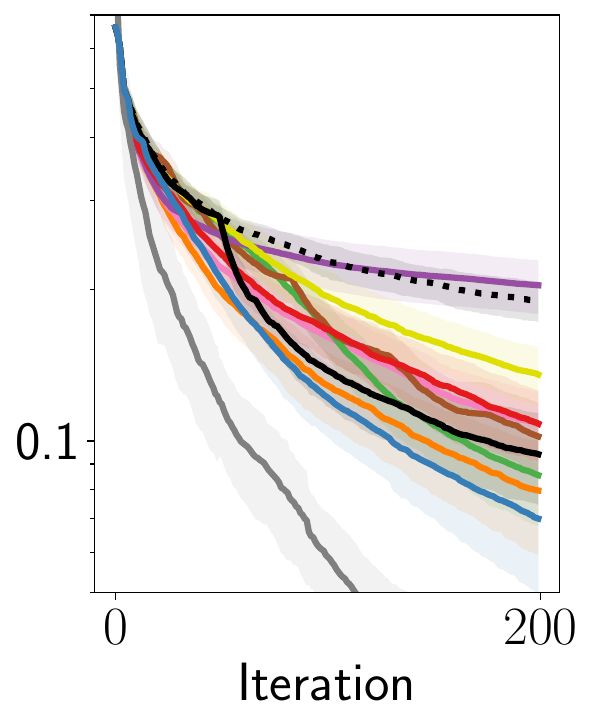}
    &  \hspace{-1em}\includegraphics[clip, trim=0.3cm 0cm 0cm 0.25cm, height=3.8cm]{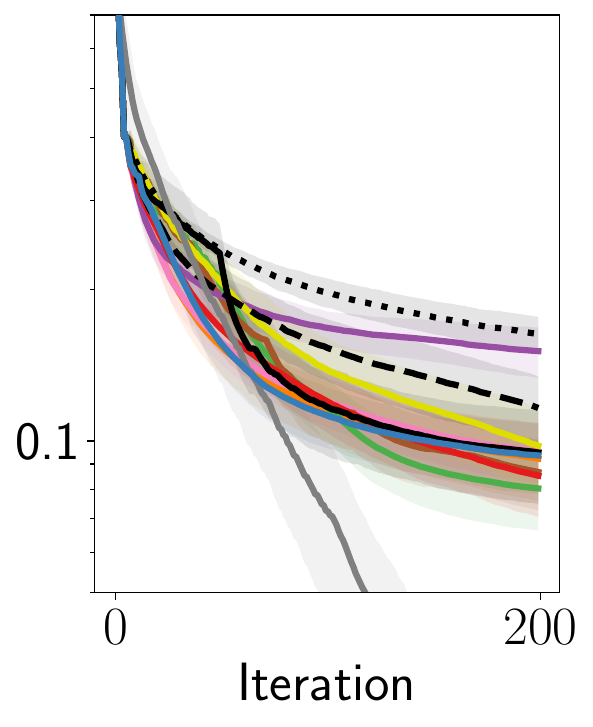}
    & \hspace{-1.5em} \includegraphics[clip, trim=0.3cm 0cm 0cm 0.25cm, height=3.8cm]{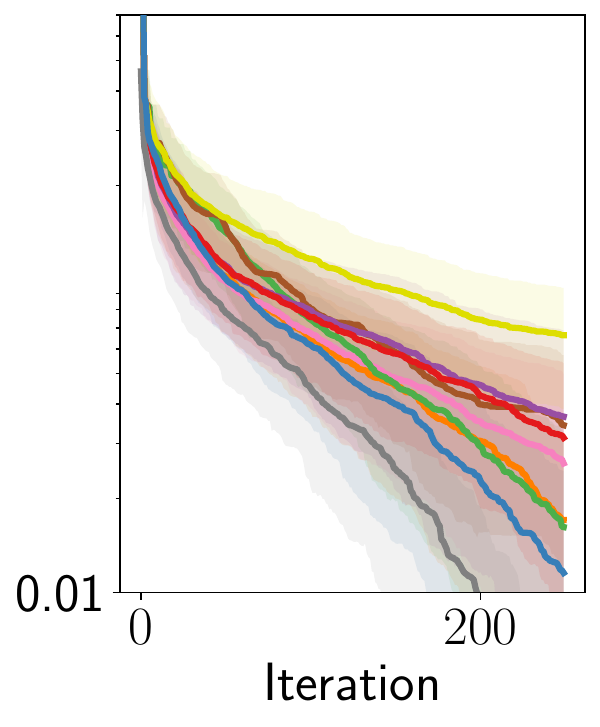}
    &\hspace{-1.4em} \includegraphics[clip, trim=0.3cm -4.5cm 0.0cm 0.0cm, height=3.8cm]{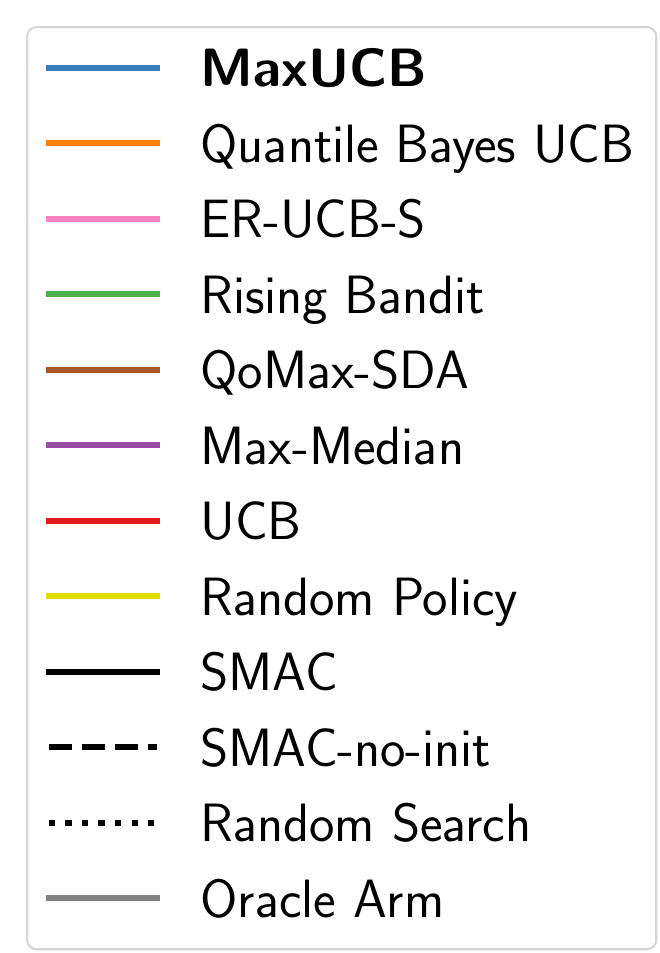}
    \end{tabular}
\caption{Average normalized loss of algorithms on different benchmarks, lower is better. \SMAC{} and \randomsearch{} perform \combinedsearch{} across the joint space.}
\label{app:fig:average_norm_loss}
\end{figure*}

\begin{figure*}[htbp]
    \begin{tabular}{c c c c c}
    \scriptsize\tabrepo{[RS]}
    &\hspace{-0.5em}\scriptsize\tabreporaw{[SMAC]}
    & \hspace{-1em}\scriptsize\yahpogym{[SMAC]}
    &\hspace{-1.5em}\scriptsize\hebo{[HEBO]}
    & \\
    \includegraphics[clip, trim=0.2cm 0cm 0.3cm 0.25cm,height=3.8cm]{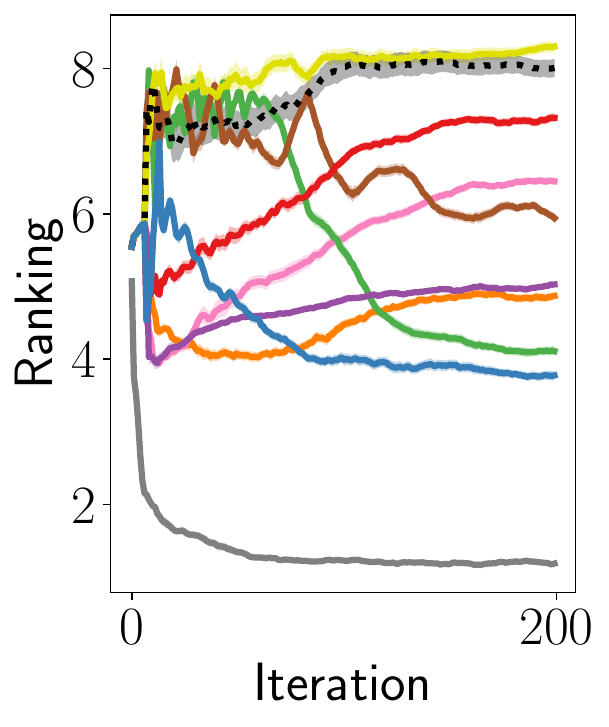}
    & \hspace{-1em} \includegraphics[clip, trim=0.3cm 0cm 0cm 0.25cm, height=3.8cm]{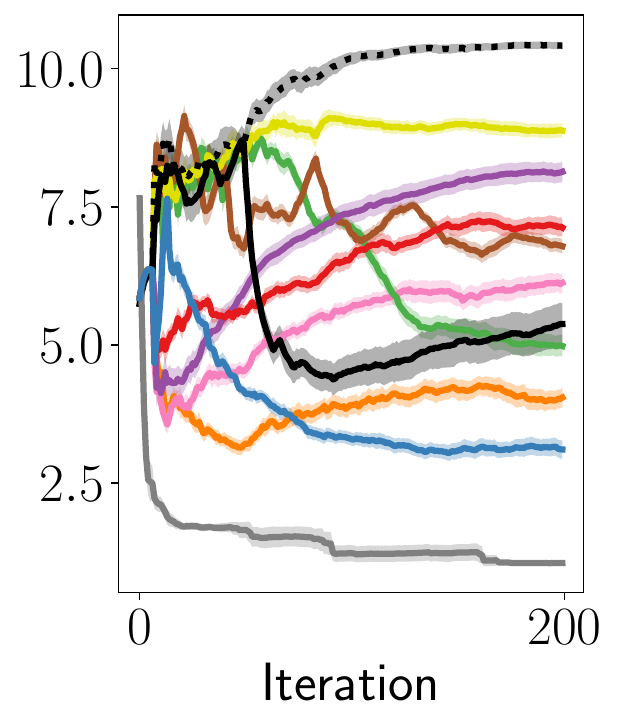}
    &  \hspace{-1em}\includegraphics[clip, trim=0.3cm 0cm 0cm 0.25cm, height=3.8cm]{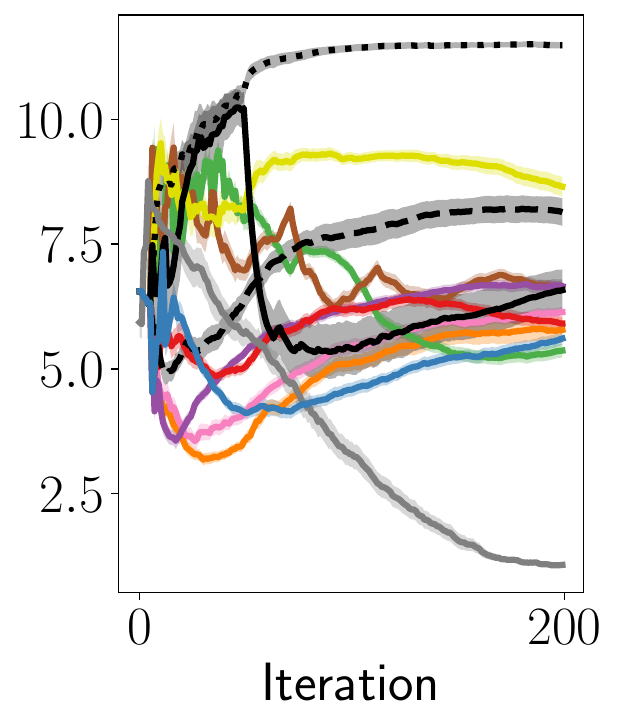}
    & \hspace{-1.5em} \includegraphics[clip, trim=0.3cm 0cm 0cm 0.25cm, height=3.8cm]{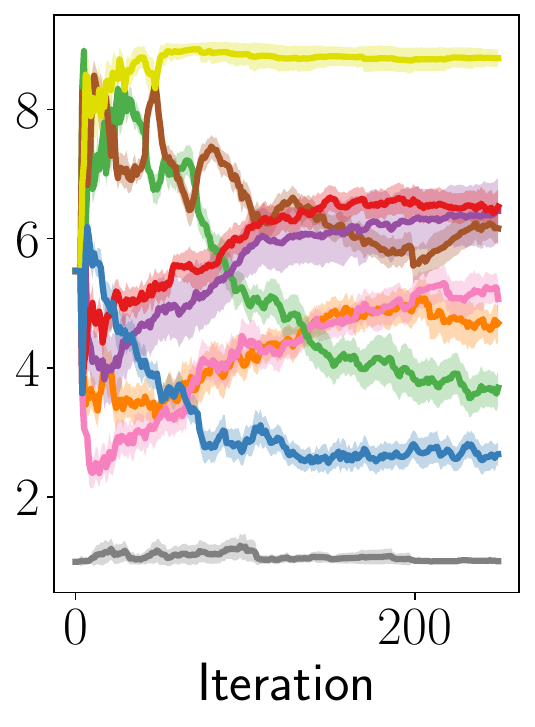}
    &\hspace{-1.4em} \includegraphics[clip, trim=0.3cm -4.5cm 0.0cm 0.0cm, height=3.8cm]{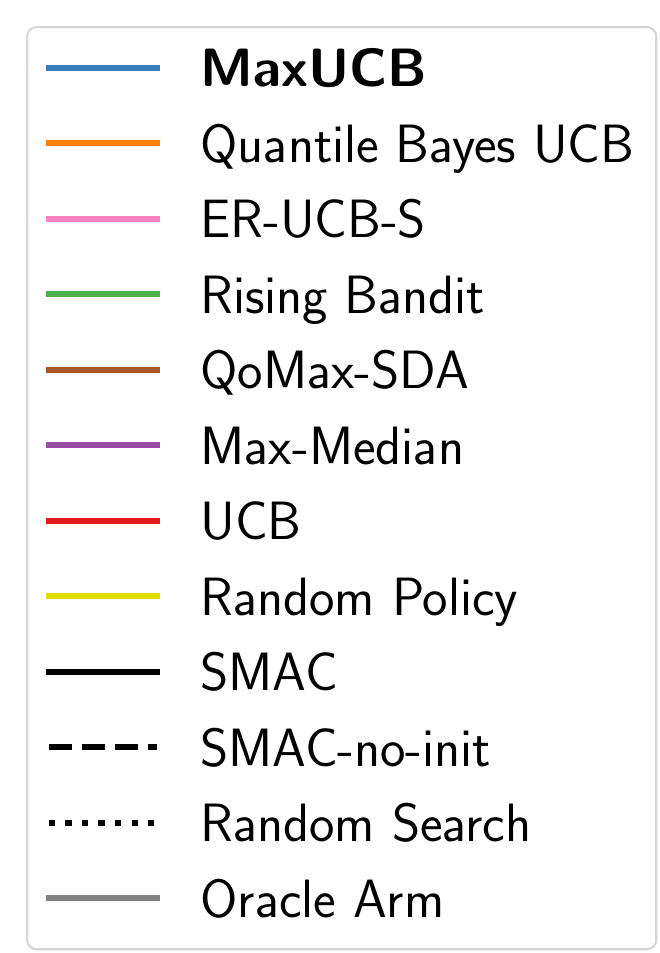}
    \end{tabular}
\caption{Average ranking of algorithms on different benchmarks, lower is better. \SMAC{} and \randomsearch{} perform \combinedsearch{} across the joint space.}
\label{app:fig:average_ranking}
\end{figure*}

\begin{figure}[h!]
\includegraphics[clip, trim=0.0cm 0cm 0.0cm 0.0cm,height=3cm]{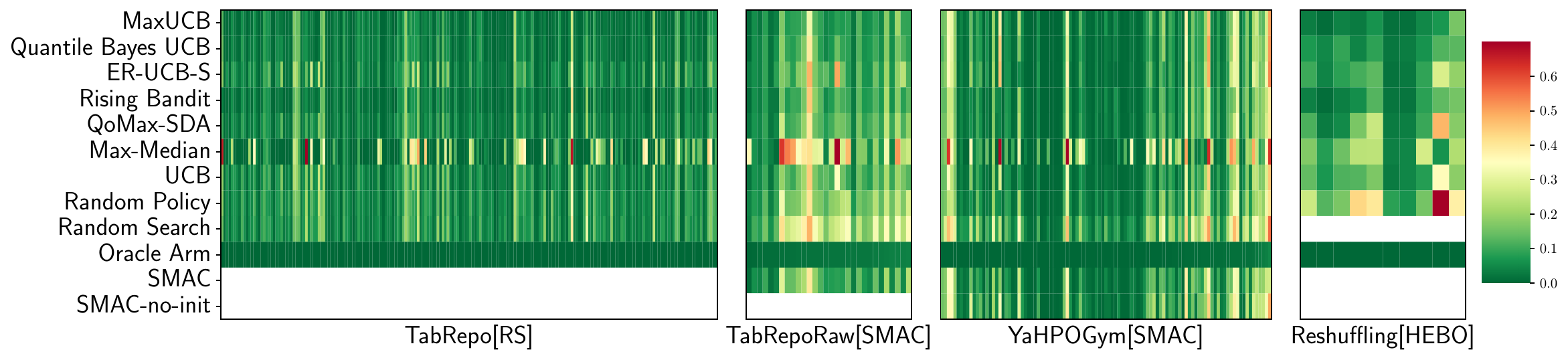}
\caption{Heatmap showing the normalized loss of algorithms per task of each benchmark, sorted by the oracle arm performance, lower is better. \SMAC{} and \randomsearch{} perform \combinedsearch{} across the joint space.}
\label{app:fig:heatmap_norm_loss}
\end{figure}

\begin{figure}[tbp]
\includegraphics[clip, trim=0.0cm 0cm 0.0cm 0.0cm,height=3cm]{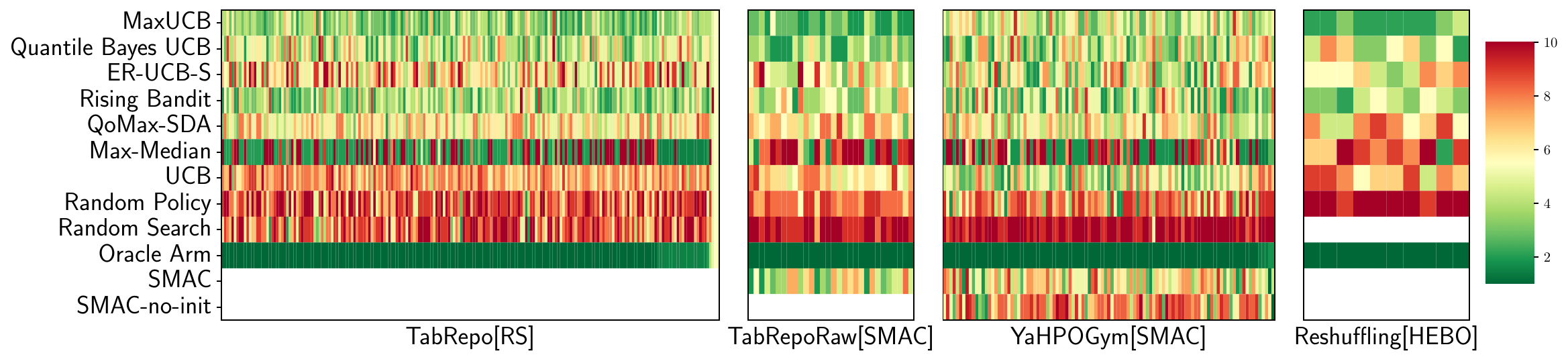}
\caption{Heatmap showing the ranking of algorithms per task of each benchmark, sorted by the oracle arm performance, lower is better. \SMAC{} and \randomsearch{} perform \combinedsearch{} across the joint space.}
\label{app:fig:heatmap_ranking}
\end{figure}

\begin{figure*}[tbp]
\centering
\begin{tabular}{c|c}
\tabrepo{}
&\tabreporaw{}\\
\includegraphics[width=0.40\linewidth]{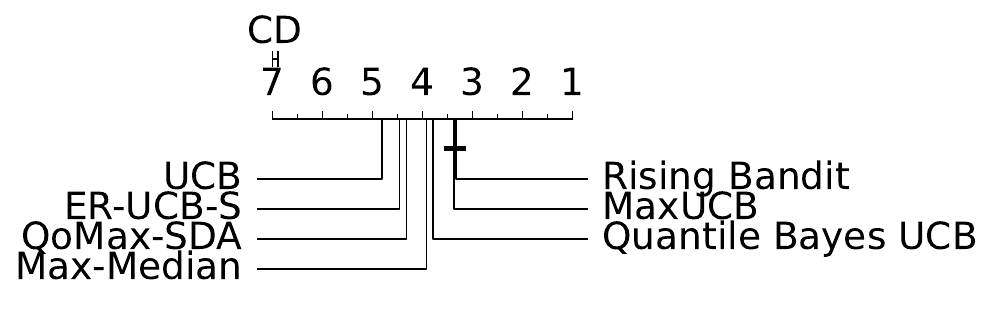}
&\includegraphics[width=0.40\linewidth]{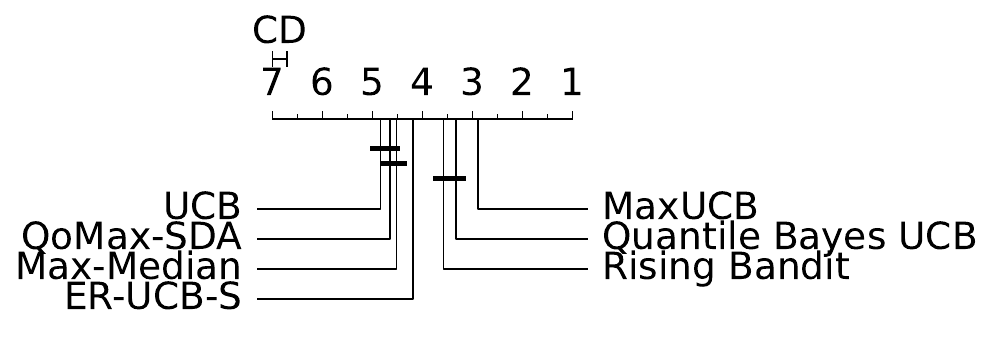}\\
\includegraphics[width=0.40\linewidth]{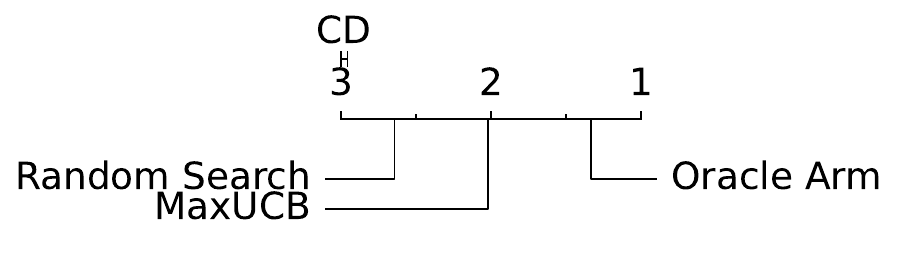}
&\includegraphics[width=0.40\linewidth]{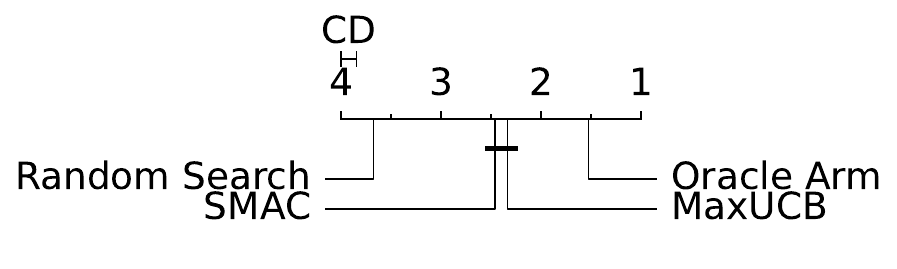}\\
\midrule
\yahpogym{}
&\hebo{}\\
\includegraphics[width=0.40\linewidth]{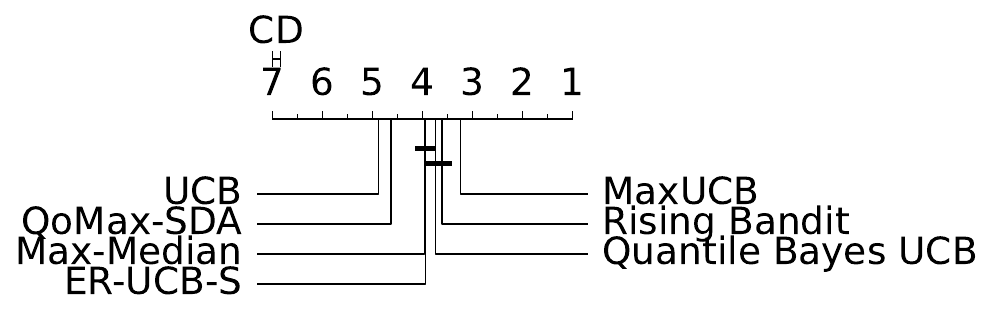}
&\includegraphics[width=0.40\linewidth]{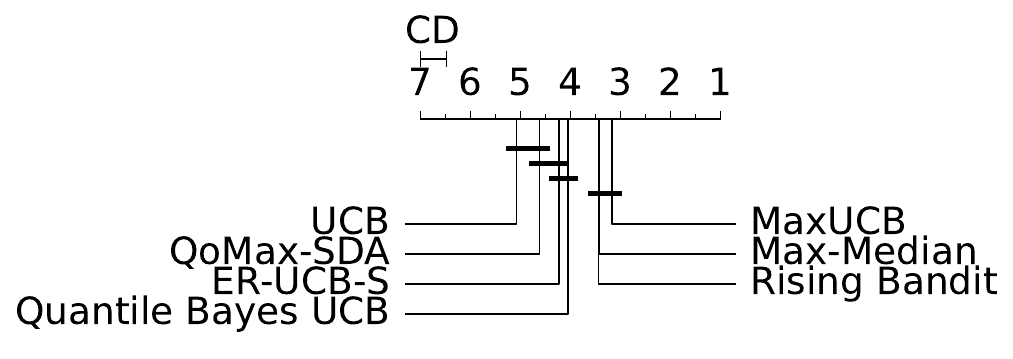}\\
\includegraphics[width=0.40\linewidth]{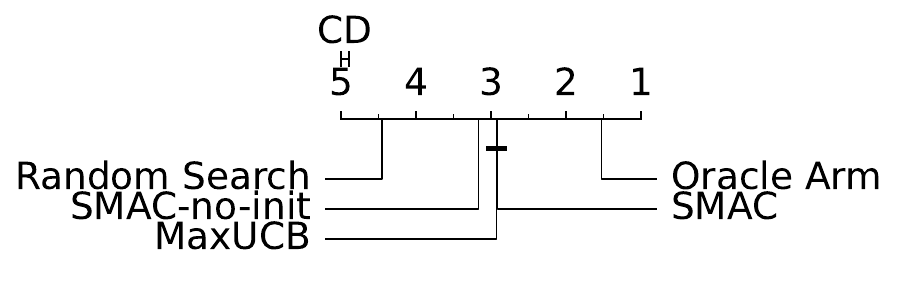}
&
\end{tabular}
\caption{Diagrams to compare the performance (ranking) of different algorithms using the Critical
Distance (CD). For each benchmark on top, we compare bandit methods, and on the bottom, we compare \OURALGO{} against \combinedsearch{} and the oracle arm.}
\label{app:fig:res_autorank}
\end{figure*}

\clearpage
\subsection{More Baselines for the Empirical Evaluation}
\label{app:more_baselines_experiments}
\textbf{Max $K$-armed Bandit Baselines.} We compare \OURALGO{} against \textit{MaxSearch Gaussian}~\citep{kikkawa2024materials},  \textit{MaxSearch SubGaussian}~\citep{kikkawa2024materials},  \textit{QoMax-ETC}~\citep{baudry2022efficient},  \QoMaxSDA{}~\citep{baudry2022efficient}, \MaxMedian{}~\citep{bhatt2022extreme} and \textit{Threshold Ascent}~\citep{streeter2006simple}. We report the averaged normalized loss over time in Figure~\ref{app:fig:average_norm_loss_part_1}, the average ranking in Figure~\ref{app:fig:average_ranking_part_1}.  As shown, our algorithm outperforms all extreme bandit algorithms in these benchmarks.

\begin{figure*}[htbp]
    \begin{tabular}{c c c c c}
    \scriptsize\tabrepo{[RS]}
    &\hspace{-0.5em}\scriptsize\tabreporaw{[SMAC]}
    & \hspace{-1em}\scriptsize\yahpogym{[SMAC]}
    &\hspace{-1.5em}\scriptsize\hebo{[HEBO]}
    & \\
    \includegraphics[clip, trim=0.2cm 0cm 0.3cm 0.25cm,height=3.8cm]{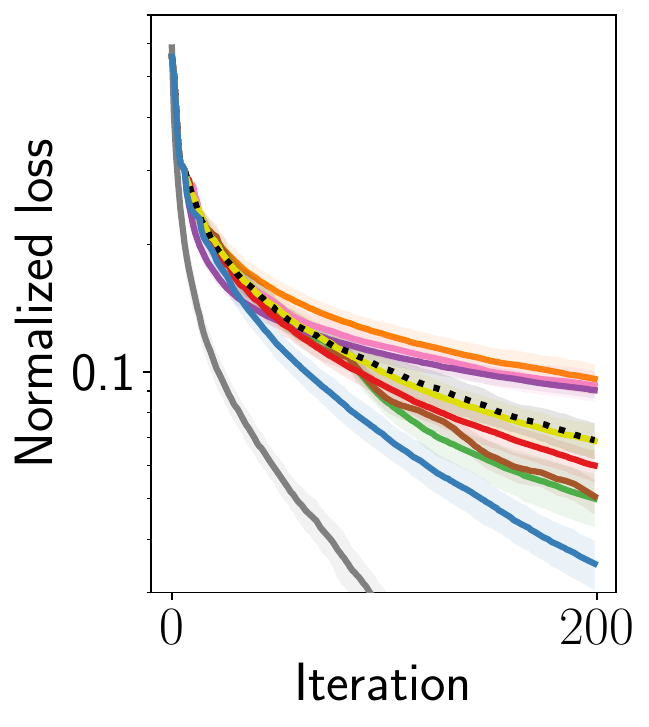}
    & \hspace{-1em} \includegraphics[clip, trim=0.3cm 0cm 0cm 0.25cm, height=3.8cm]{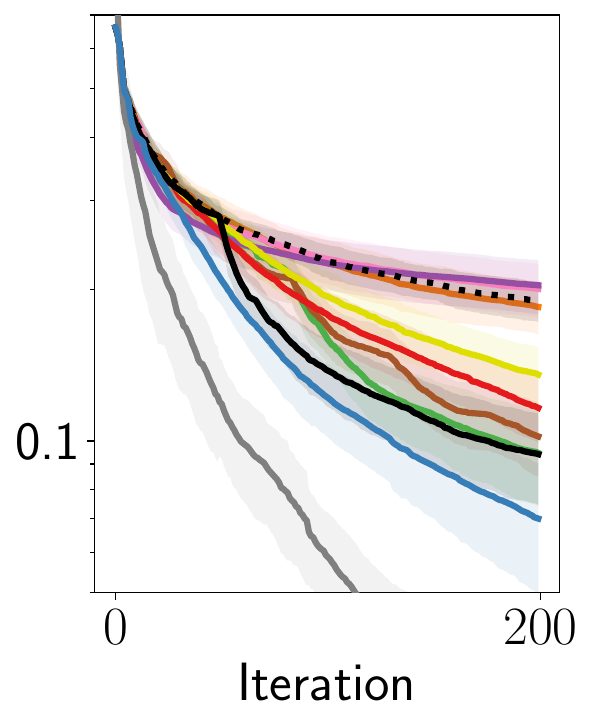}
    &  \hspace{-1em}\includegraphics[clip, trim=0.3cm 0cm 0cm 0.25cm, height=3.8cm]{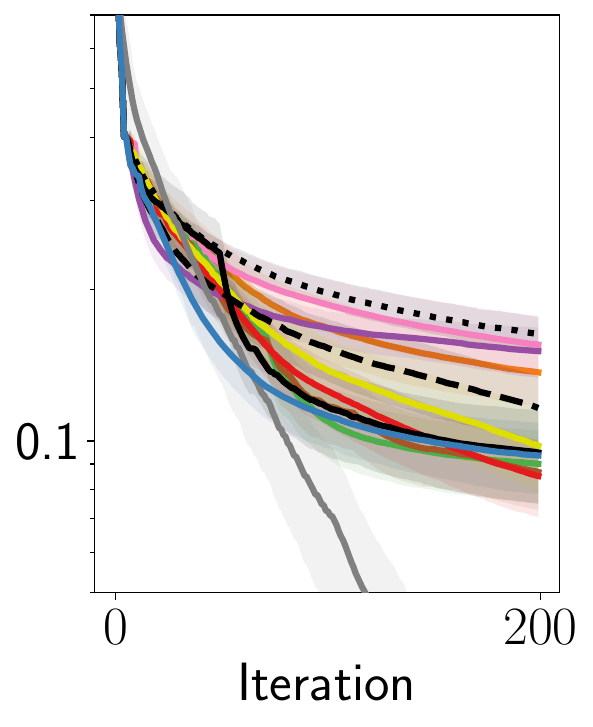}
    & \hspace{-1.5em} \includegraphics[clip, trim=0.3cm 0cm 0cm 0.25cm, height=3.8cm]{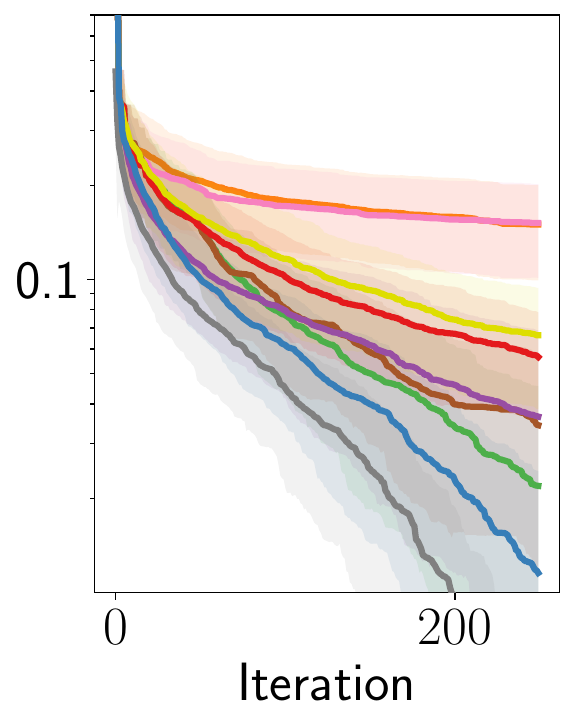}
    &\hspace{-1.4em} \includegraphics[clip, trim=0.3cm -4.5cm 0.0cm 0.0cm, height=3.8cm]{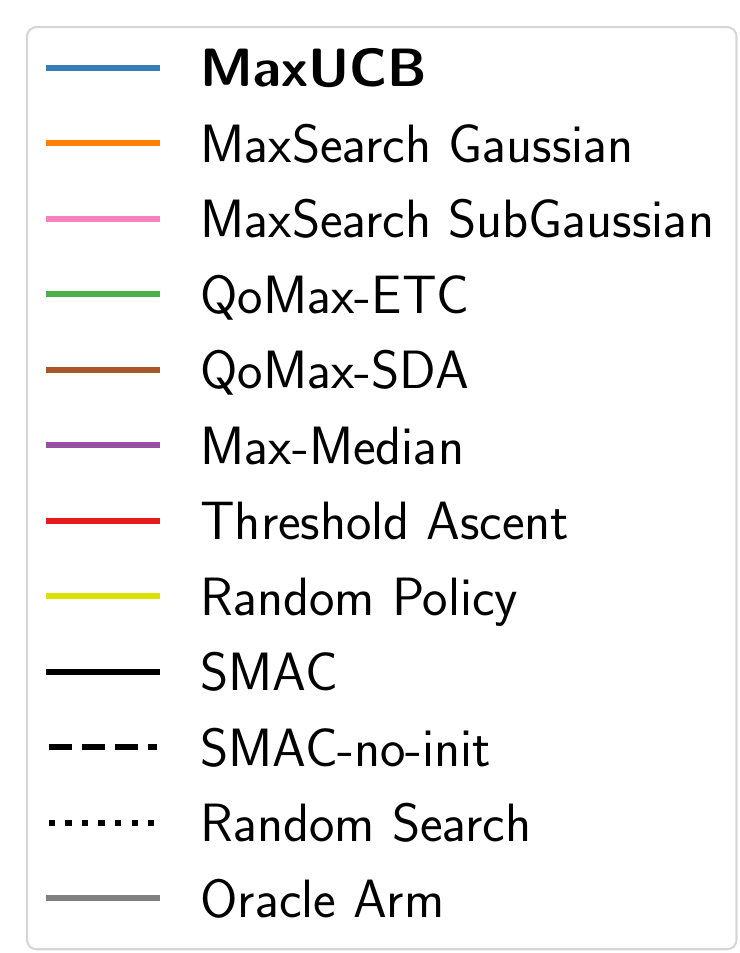}
    \end{tabular}
\caption{Average normalized loss of MKB algorithms on different benchmarks, lower is better. \SMAC{} and \randomsearch{} perform \combinedsearch{} across the joint space.}
\label{app:fig:average_norm_loss_part_1}
\end{figure*}

\begin{figure*}[htbp]
    \begin{tabular}{c c c c c}
    \scriptsize\tabrepo{[RS]}
    &\hspace{-0.5em}\scriptsize\tabreporaw{[SMAC]}
    & \hspace{-1em}\scriptsize\yahpogym{[SMAC]}
    &\hspace{-1.5em}\scriptsize\hebo{[HEBO]}
    & \\
    \includegraphics[clip, trim=0.2cm 0cm 0.3cm 0.25cm,height=3.8cm]{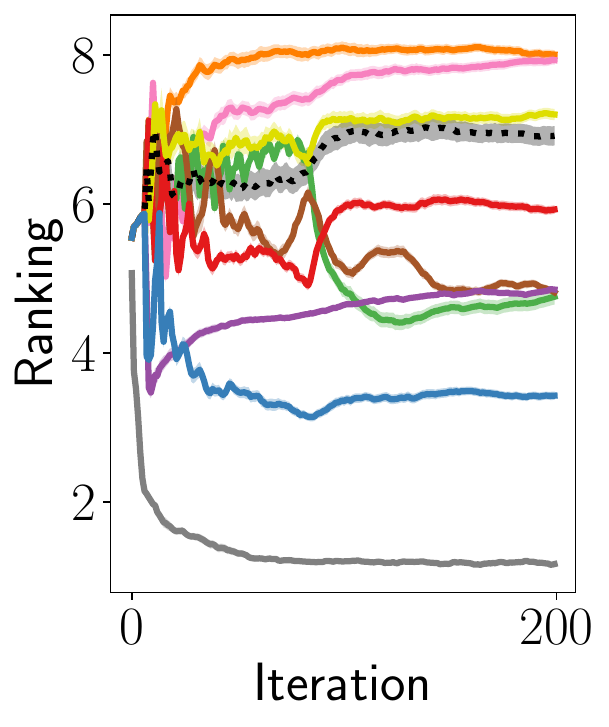}
    & \hspace{-1em} \includegraphics[clip, trim=0.3cm 0cm 0cm 0.25cm, height=3.8cm]{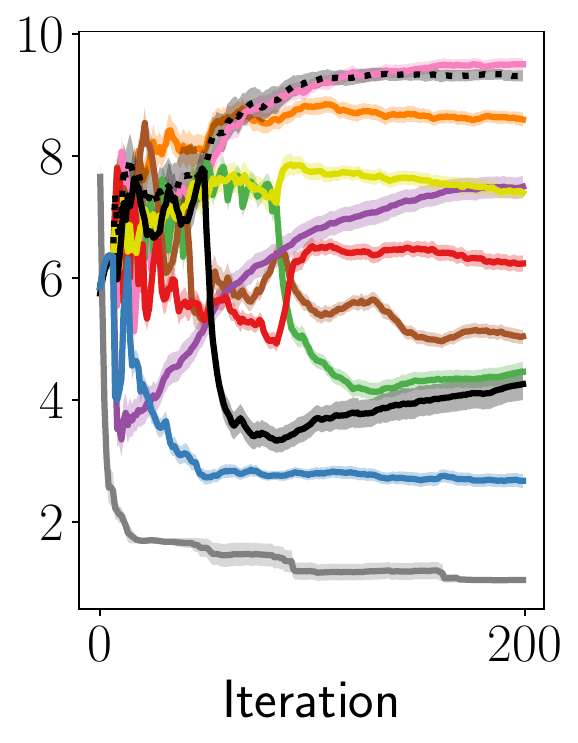}
    &  \hspace{-1em}\includegraphics[clip, trim=0.3cm 0cm 0cm 0.25cm, height=3.8cm]{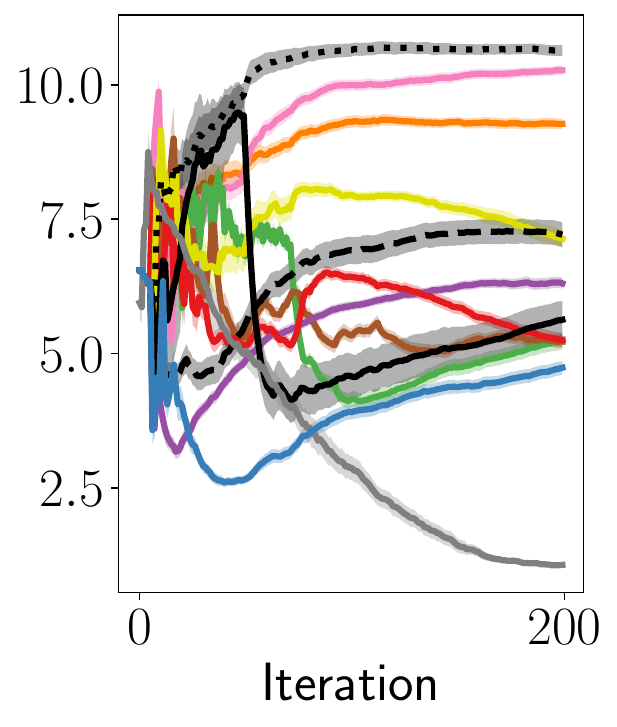}
    & \hspace{-1.5em} \includegraphics[clip, trim=0.3cm 0cm 0cm 0.25cm, height=3.8cm]{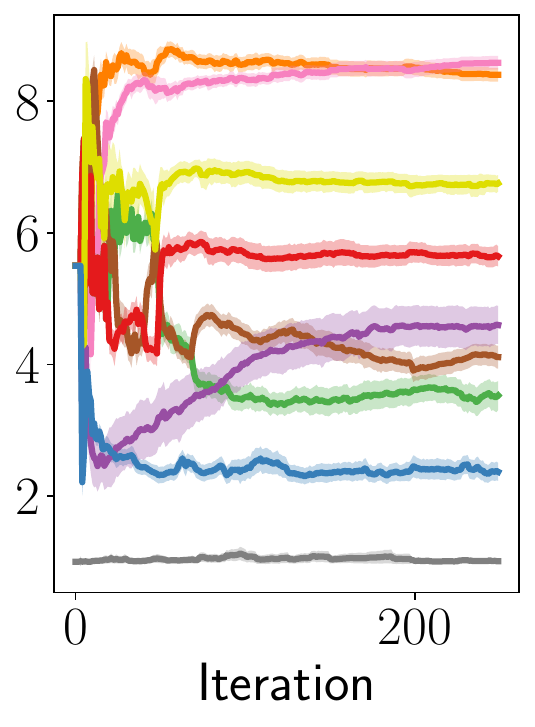}
    &\hspace{-1.4em} \includegraphics[clip, trim=0.3cm -4.5cm 0.0cm 0.0cm, height=3.8cm]{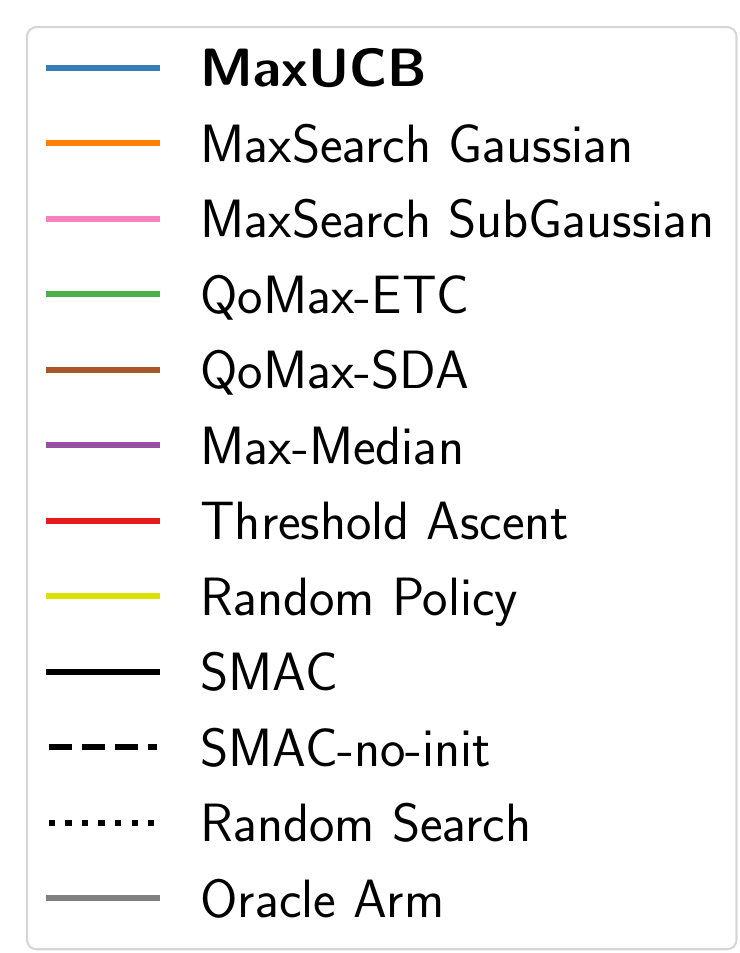}
    \end{tabular}
\caption{Average ranking of MKB algorithms on different benchmarks, lower is better. \SMAC{} and \randomsearch{} perform \combinedsearch{} across the joint space.}
\label{app:fig:average_ranking_part_1}
\end{figure*}

\textbf{A Few More Relevant Bandit Baselines.} We compare \OURALGO{} against \textit{Quantile UCB}~\citep{balef2024towards},  \textit{ER-UCB-N}~\citep{hu2021cascaded},  \textit{R-SR}~\citep{mussi2024best}, \textit{R-UCBE}~\citep{mussi2024best}, \SuccessiveHalving~\citep{karnin2013almost} and \textit{EXP3}~\citep{auer2002nonstochastic}. We report the averaged normalized loss over time in Figure~\ref{app:fig:average_norm_loss_part_2}, the average ranking in Figure~\ref{app:fig:average_ranking_part_2}. 

\begin{figure*}[htbp]
    \begin{tabular}{c c c c c}
    \scriptsize\tabrepo{[RS]}
    &\hspace{-0.5em}\scriptsize\tabreporaw{[SMAC]}
    & \hspace{-1em}\scriptsize\yahpogym{[SMAC]}
    &\hspace{-1.5em}\scriptsize\hebo{[HEBO]}
    & \\
    \includegraphics[clip, trim=0.2cm 0cm 0.3cm 0.25cm,height=3.8cm]{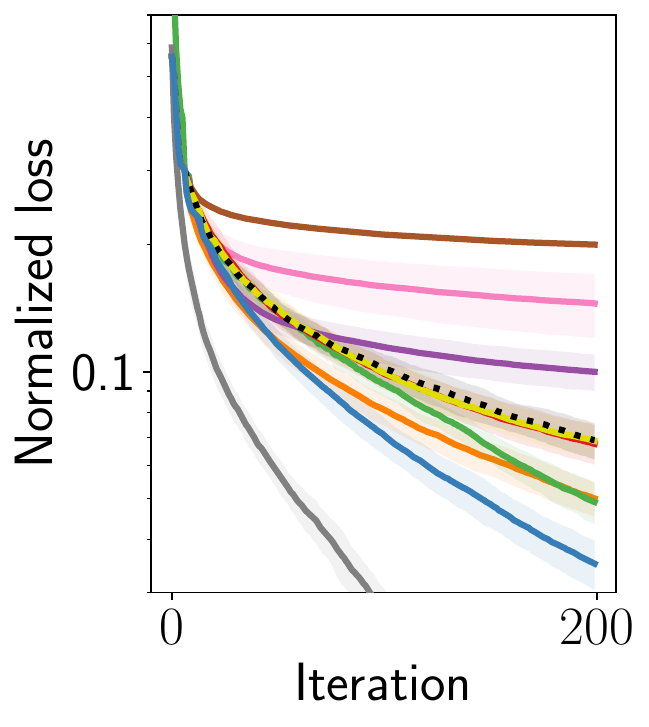}
    & \hspace{-1em} \includegraphics[clip, trim=0.3cm 0cm 0cm 0.25cm, height=3.8cm]{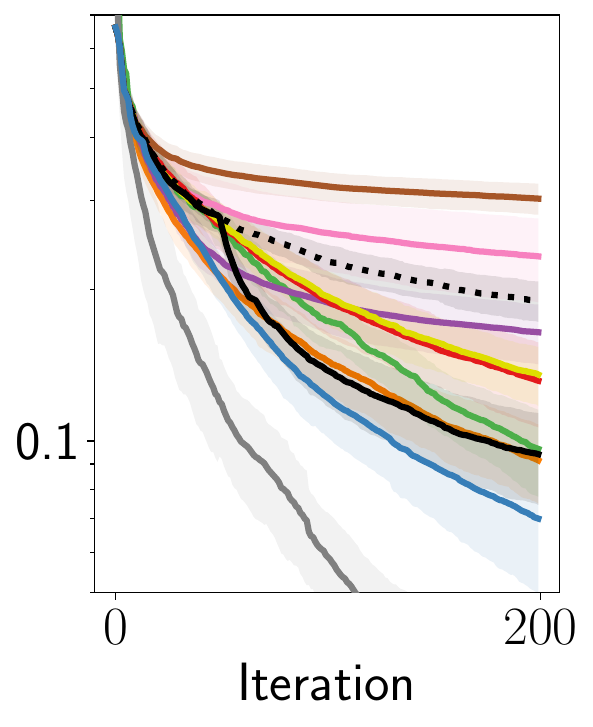}
    &  \hspace{-1em}\includegraphics[clip, trim=0.3cm 0cm 0cm 0.25cm, height=3.8cm]{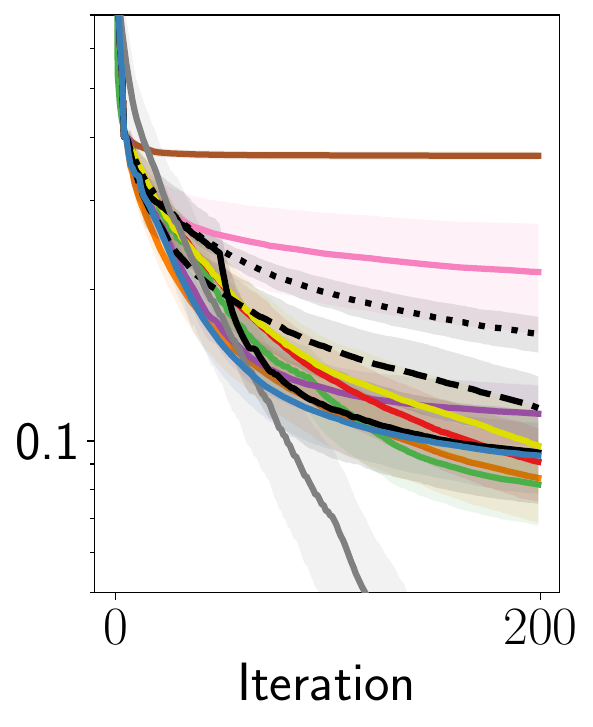}
    & \hspace{-1.5em} \includegraphics[clip, trim=0.3cm 0cm 0cm 0.25cm, height=3.8cm]{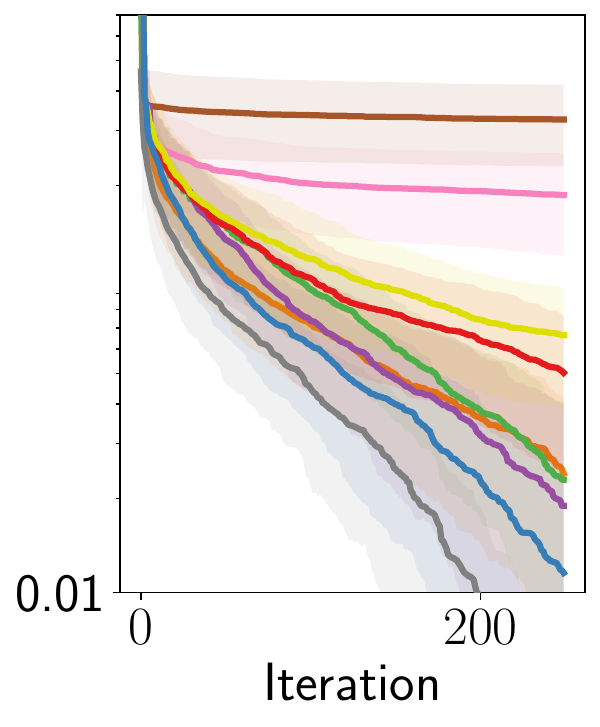}
    &\hspace{-1.4em} \includegraphics[clip, trim=0.3cm -4.5cm 0.0cm 0.0cm, height=3.8cm]{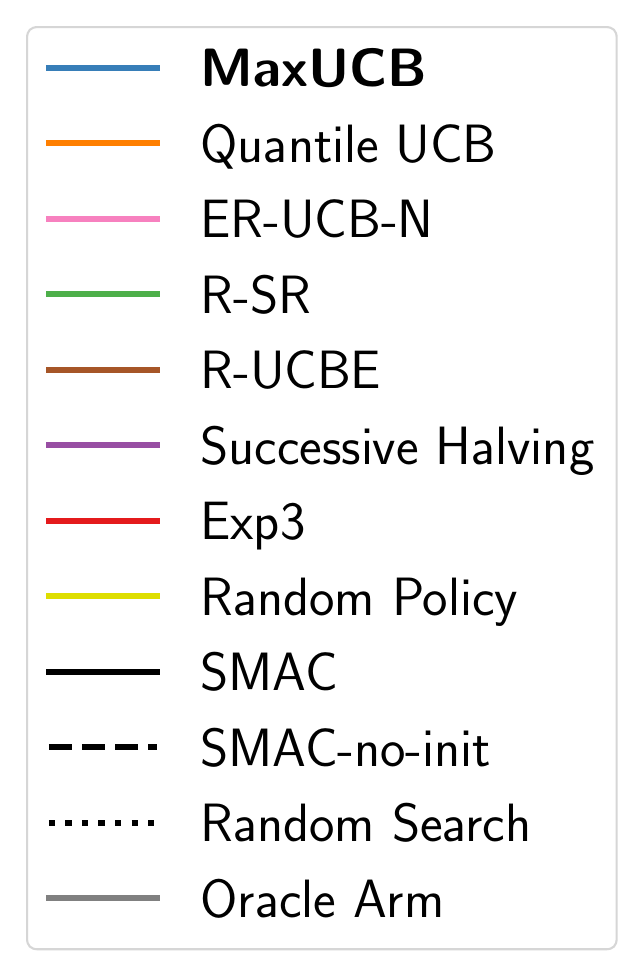}
    \end{tabular}
\caption{Average normalized loss of algorithms on different benchmarks, lower is better. \SMAC{} and \randomsearch{} perform \combinedsearch{} across the joint space.}
\label{app:fig:average_norm_loss_part_2}
\end{figure*}

\begin{figure*}[htbp]
    \begin{tabular}{c c c c c}
    \scriptsize\tabrepo{[RS]}
    &\hspace{-0.5em}\scriptsize\tabreporaw{[SMAC]}
    & \hspace{-1em}\scriptsize\yahpogym{[SMAC]}
    &\hspace{-1.5em}\scriptsize\hebo{[HEBO]}
    & \\
    \includegraphics[clip, trim=0.2cm 0cm 0.3cm 0.25cm,height=3.8cm]{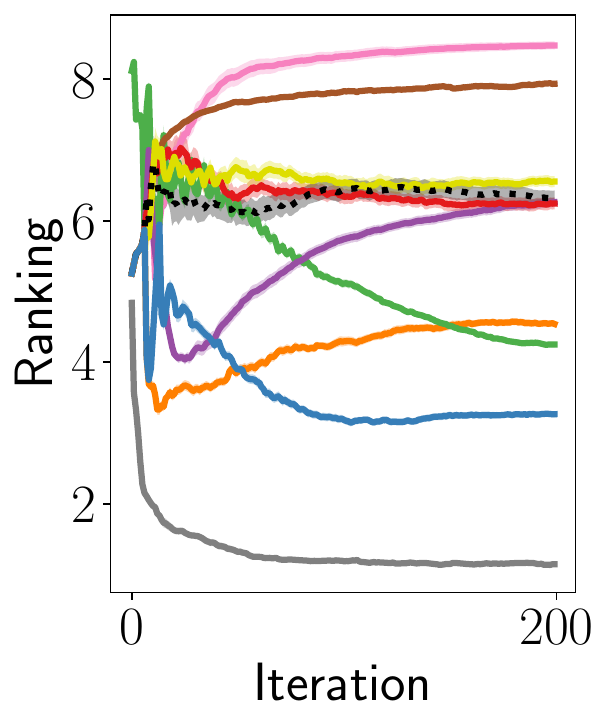}
    & \hspace{-1em} \includegraphics[clip, trim=0.3cm 0cm 0cm 0.25cm, height=3.8cm]{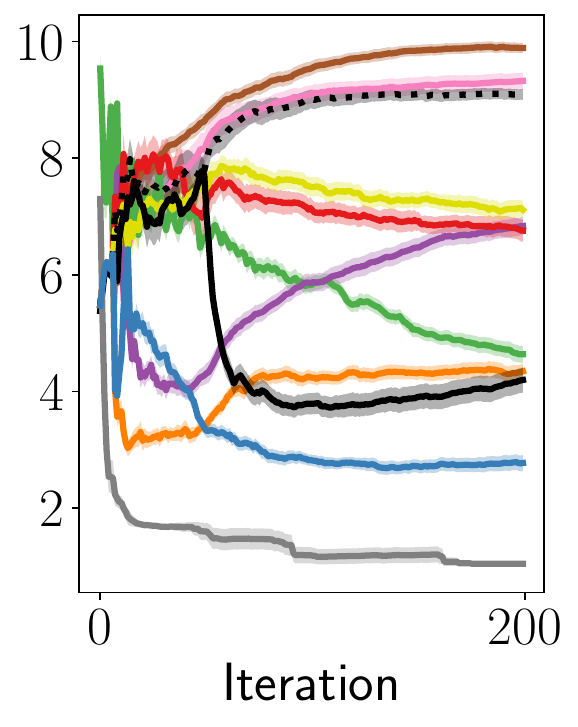}
    &  \hspace{-1em}\includegraphics[clip, trim=0.3cm 0cm 0cm 0.25cm, height=3.8cm]{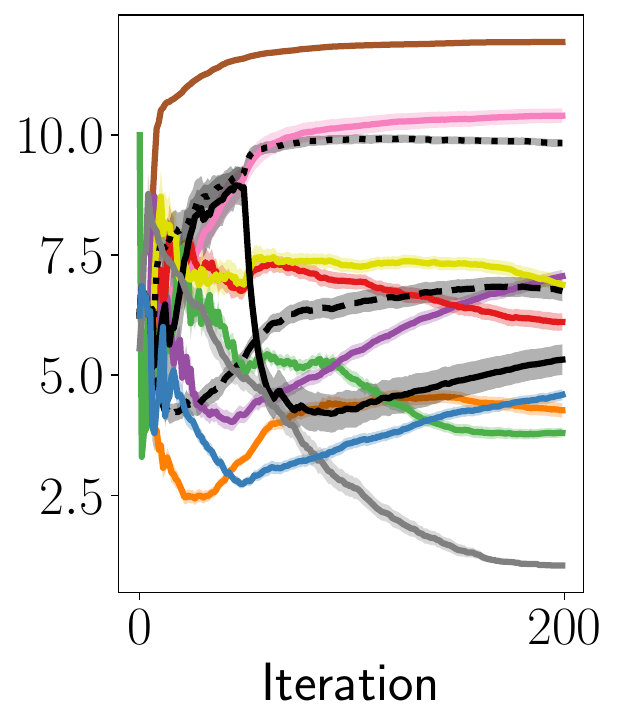}
    & \hspace{-1.5em} \includegraphics[clip, trim=0.3cm 0cm 0cm 0.25cm, height=3.8cm]{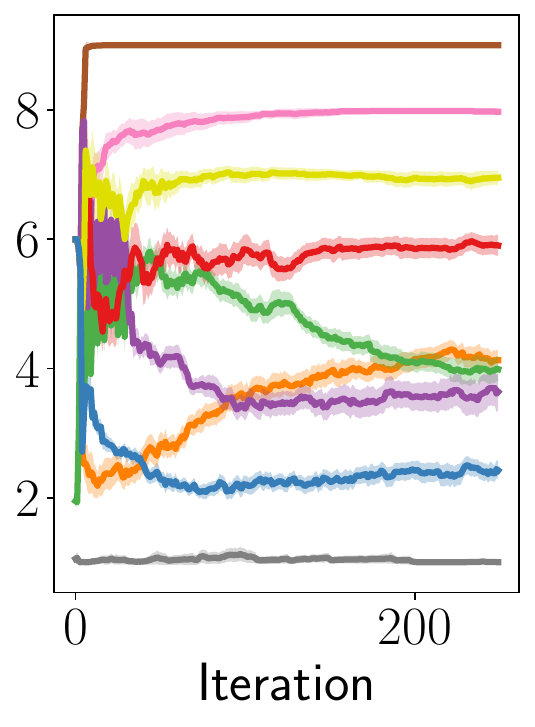}
    &\hspace{-1.4em} \includegraphics[clip, trim=0.3cm -4.5cm 0.0cm 0.0cm, height=3.8cm]{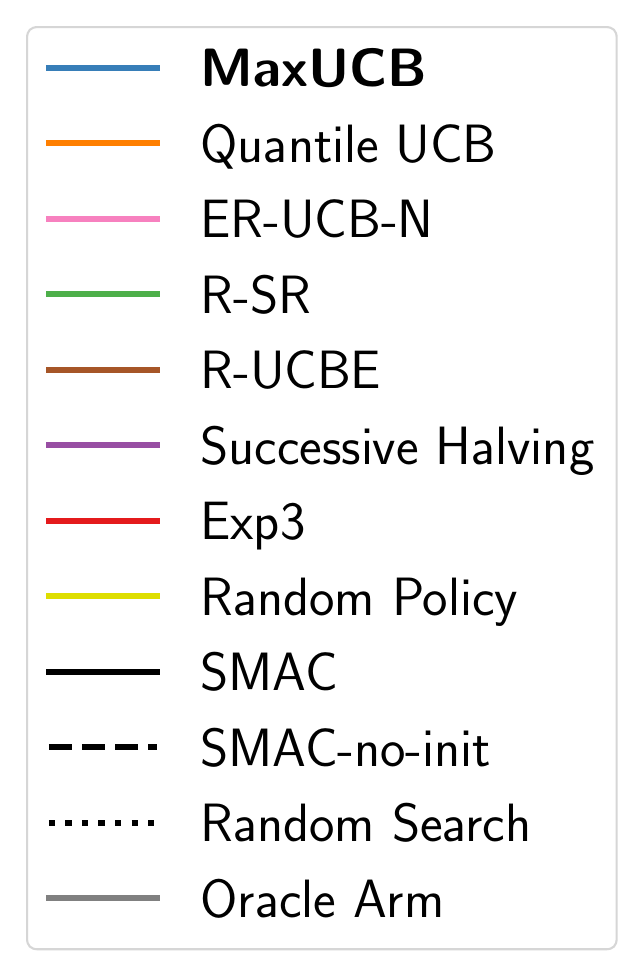}
    \end{tabular}
\caption{Average ranking of algorithms on different benchmarks, lower is better. \SMAC{} and \randomsearch{} perform \combinedsearch{} across the joint space.}
\label{app:fig:average_ranking_part_2}
\end{figure*}

\clearpage
\section{More Details on the Empirical Behaviour of MaxUCB}
\label{app:algorithm_behaviour}
Here, we provide further analysis of \OURALGO{}. Concretely, we study how often our algorithm pulls the optimal arm and compare it to the theoretical results. Furthermore, we evaluate an extension of \OURALGO{} to handle non-stationary rewards and finally study \OURALGO{} performance on common synthetic benchmarks used in the extreme bandit literature.

\subsection{The number of times each arm is pulled}
\label{app:number_pulls_arms}
Proposition~\ref{proposition:Regret_ExUCB} shows that the number of times the optimal arm is pulled can be viewed as a good metric for measuring the performance of algorithms. Figure \ref{app:fig:pulls_arms_TabRepo_RS}, \ref{app:fig:pulls_arms_TabRepoRaw},\ref{app:fig:pulls_arms_yahpo_gym},\ref{app:fig:pulls_arms_hebo} shows the average number of pulling arms on different benchmarks. They indicate that, on average, \OURALGO, \RisingBandits{}, and  \MaxMedian{} algorithms often choose the optimal arm. However, for \MaxMedian{}, the number of pulls of the optimal arm is either very close to $0$ or to $T$, leading to a non-robust performance, which has already been observed in \citet{baudry2022efficient} experiments. \UCB{} and \ERUCBS{} perform almost similarly.

\begin{figure}[htbp]
\centering
\includegraphics[clip, trim=0cm 0cm 11cm 0cm, height=6cm]{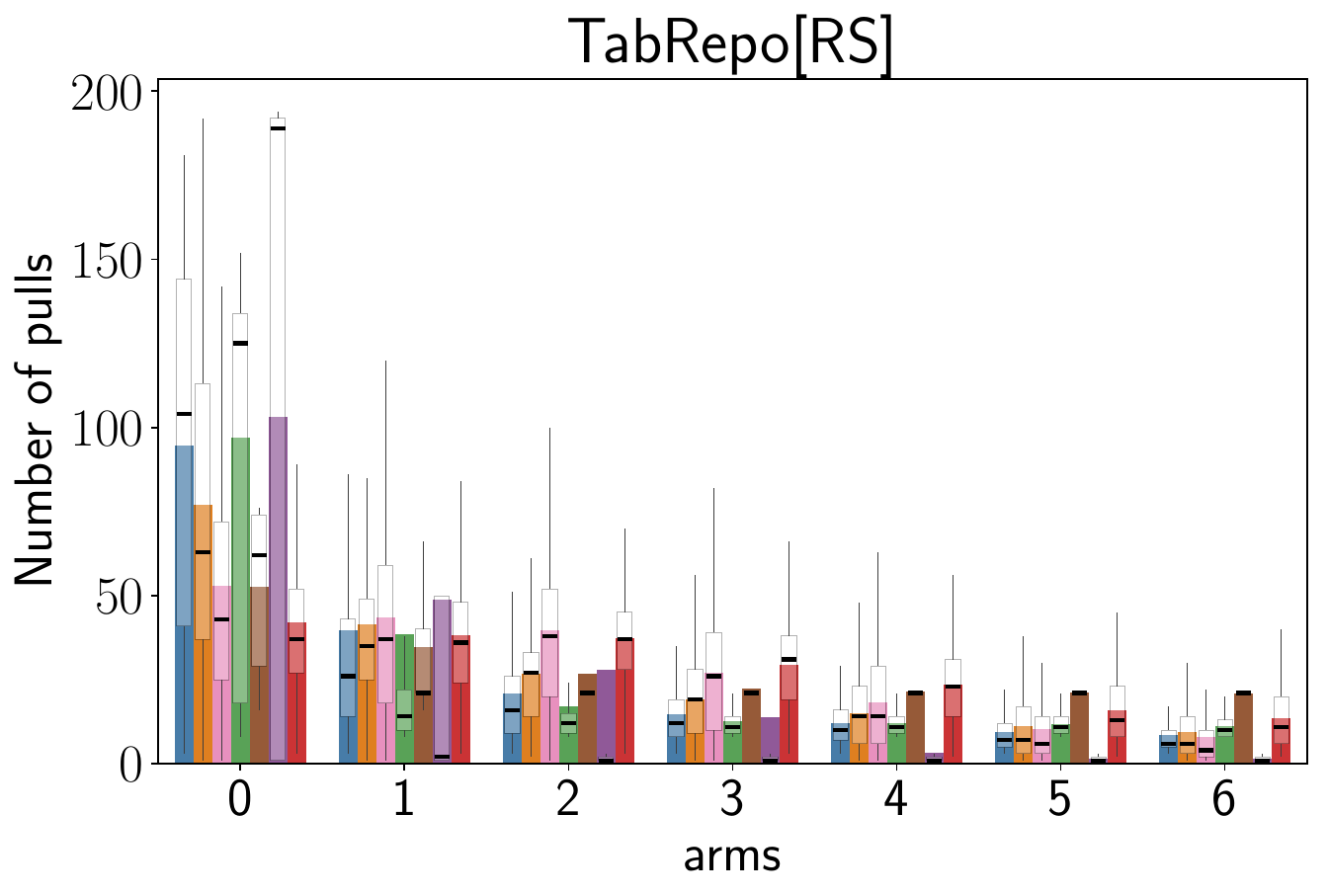}
\includegraphics[clip, trim=0cm 1cm 0cm 0.0cm, height=5.5cm]{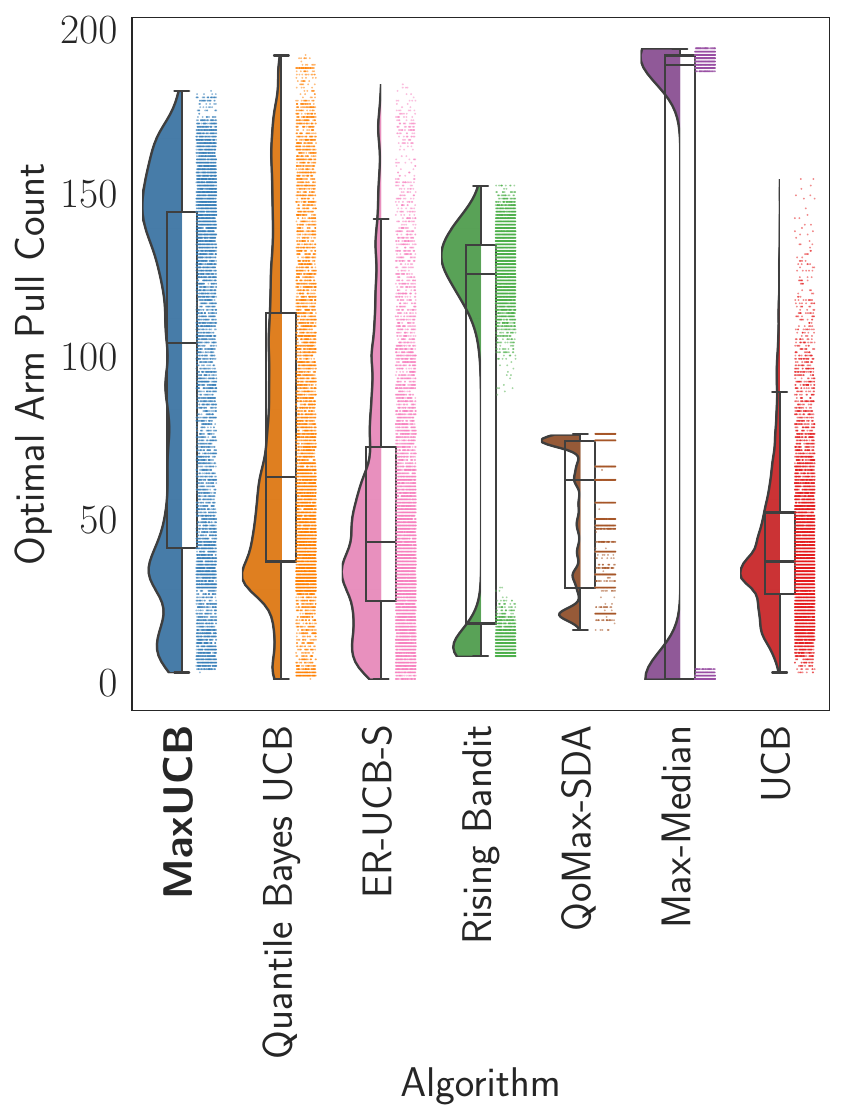}
\caption{ (Right)  The number of all arm pulls, with each bar graph showing the average and the error bars indicating additional statistical information.  (Left) The number of best arm pulls for different bandit algorithms.}
\label{app:fig:pulls_arms_TabRepo_RS}
\end{figure}

\begin{figure}[htbp]
\centering
\includegraphics[clip, trim=0cm 0cm 11cm 0cm, height=6cm]{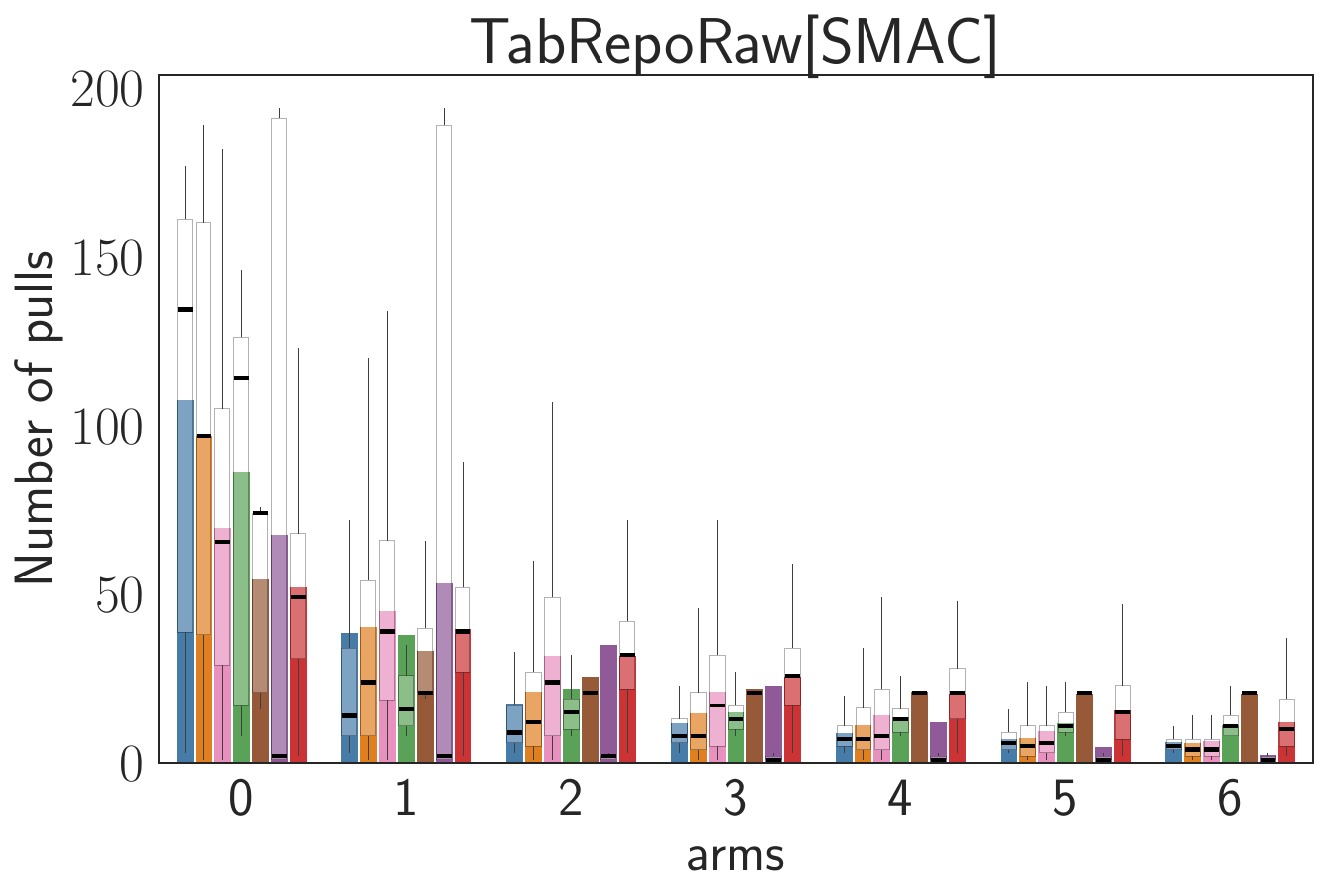}
\includegraphics[clip, trim=0cm 1cm 0cm 0.0cm, height=5.5cm]{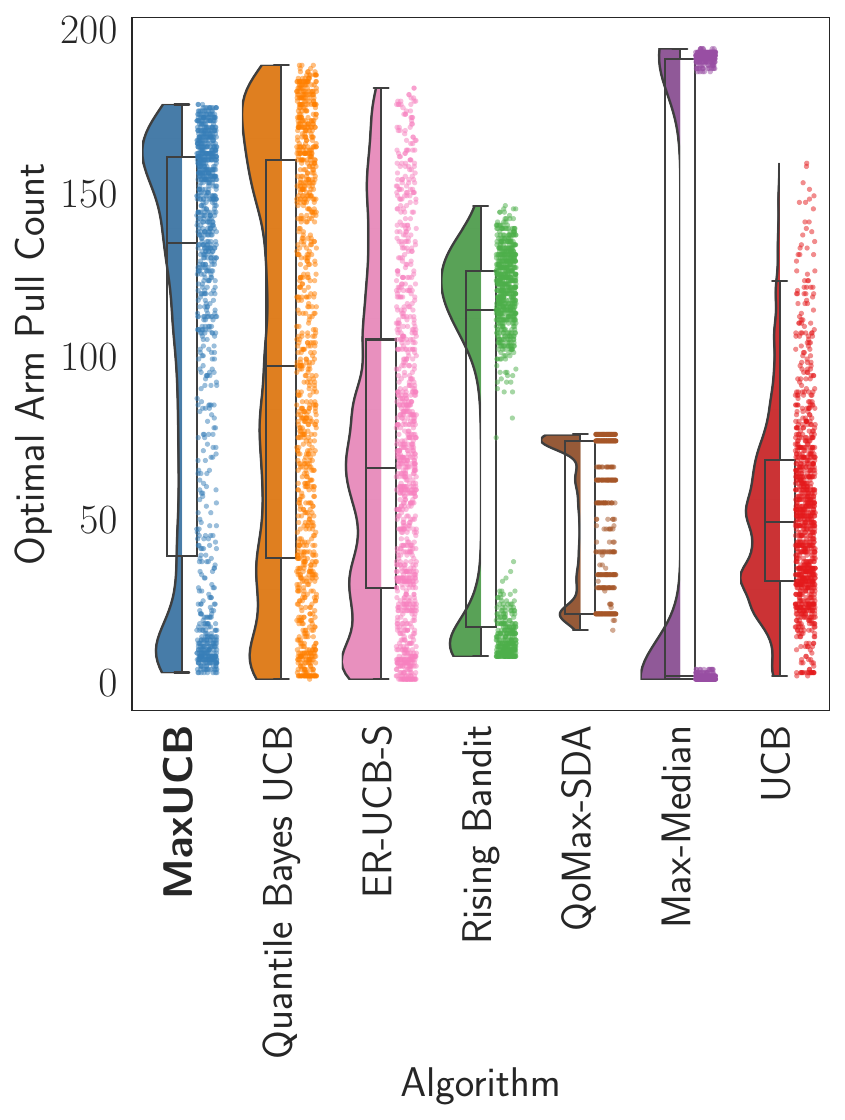}
\caption{ (Right)  The number of all arm pulls, with each bar graph showing the average and the error bars indicating additional statistical information.  (Left) The number of best arm pulls for different bandit algorithms.}
\label{app:fig:pulls_arms_TabRepoRaw}
\end{figure}

\begin{figure}[htbp]
\centering
\includegraphics[clip, trim=0cm 0cm 11cm 0cm, height=6cm]{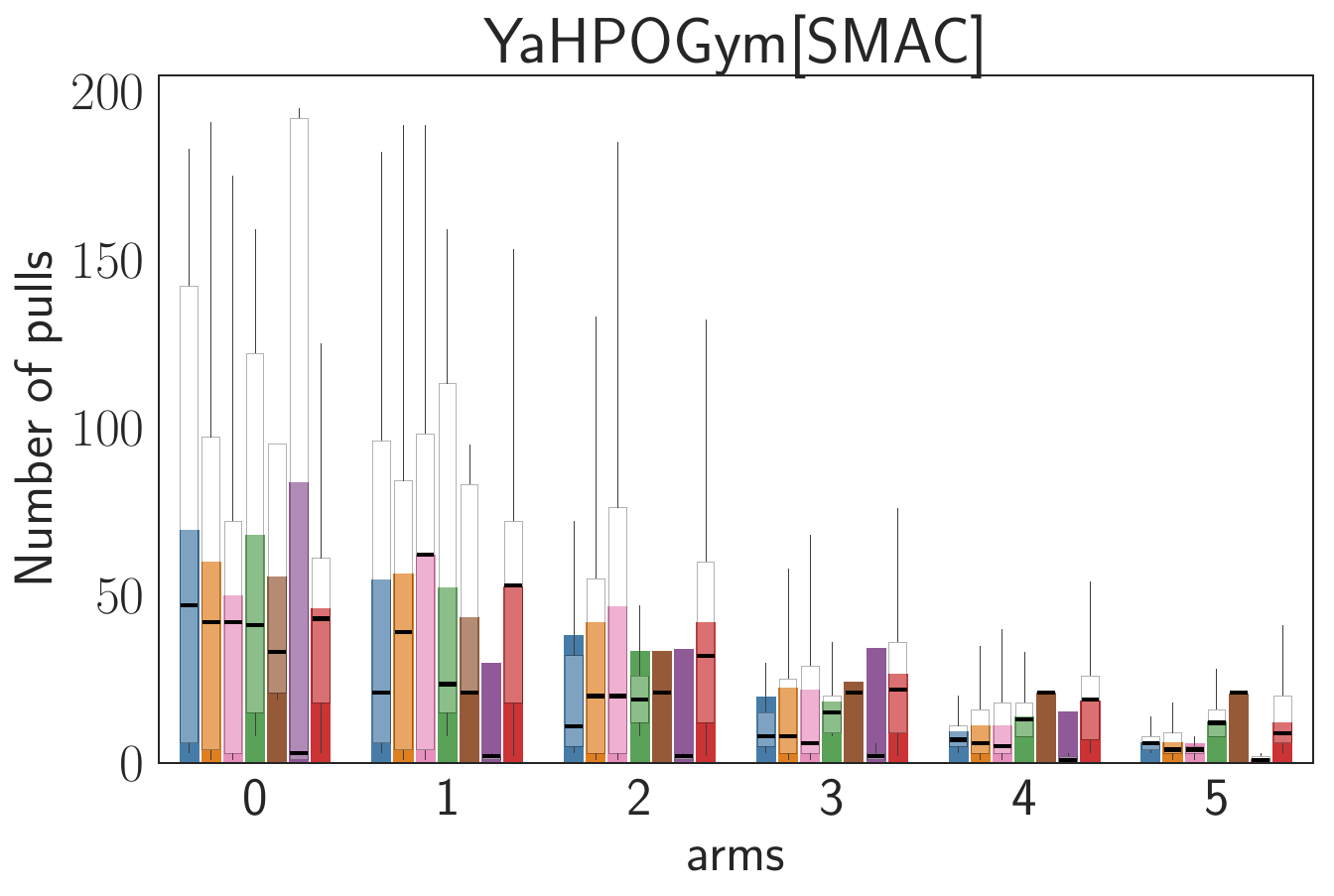}
\includegraphics[clip, trim=0cm 1cm 0cm 0.0cm, height=5.5cm]{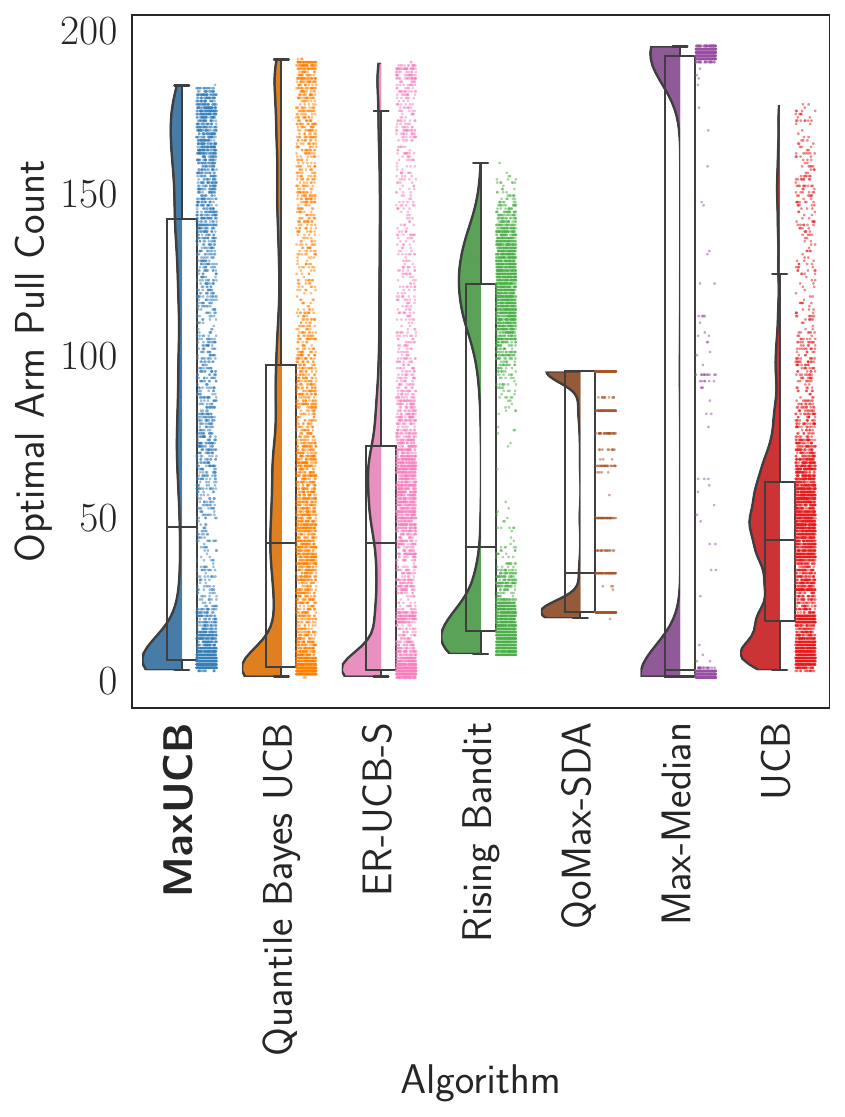}
\caption{ (Right)  The number of all arm pulls, with each bar graph showing the average and the error bars indicating additional statistical information.  (Left) The number of best arm pulls for different bandit algorithms.}
\label{app:fig:pulls_arms_yahpo_gym}
\end{figure}

\begin{figure}[htbp]
\centering
\includegraphics[clip, trim=0cm 0cm 11cm 0cm, height=6cm]{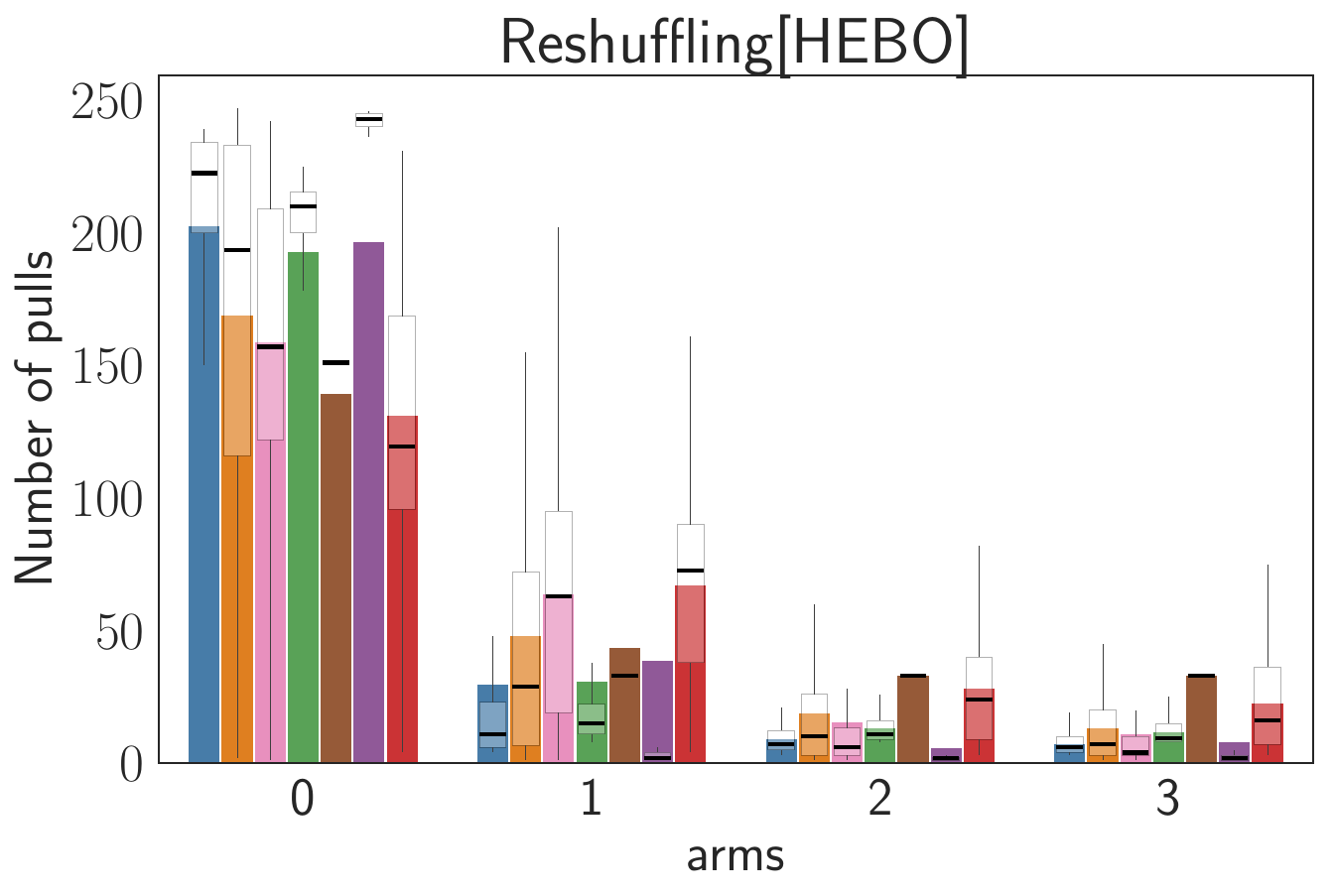}
\includegraphics[clip, trim=0cm 1cm 0cm 0.0cm, height=5.5cm]{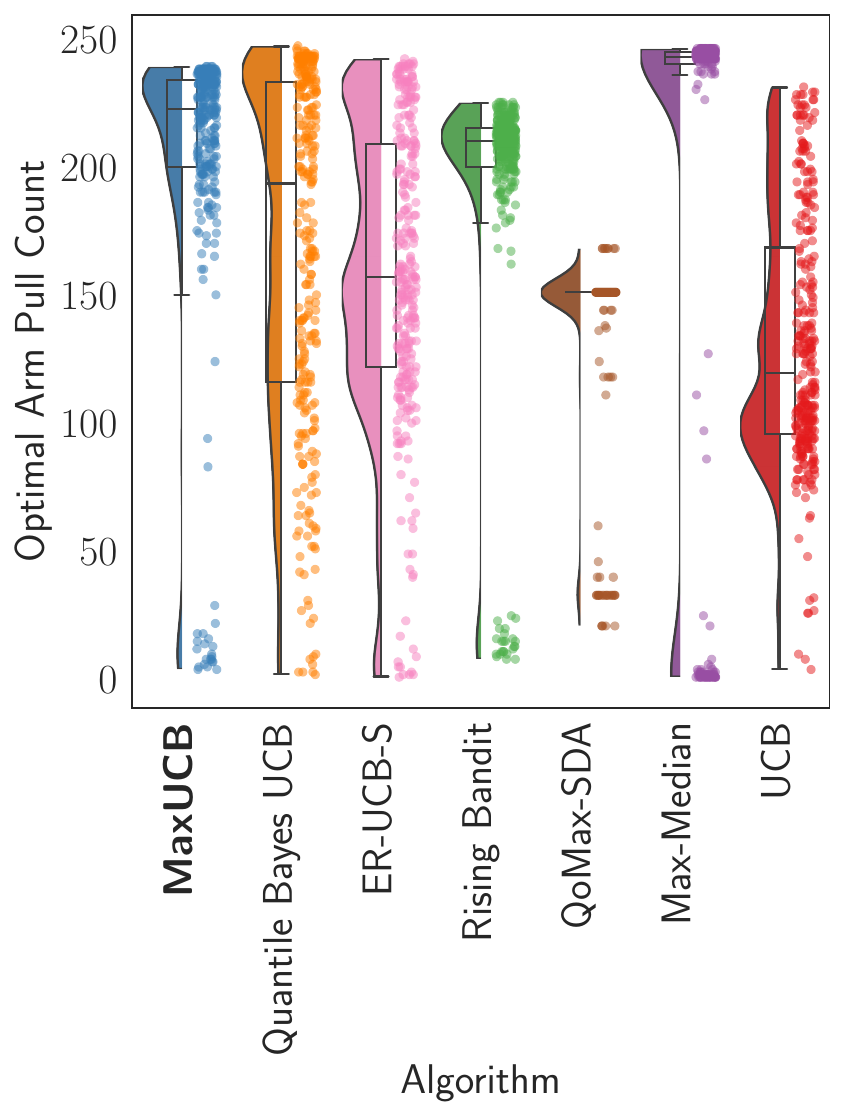}
\caption{ (Right)  The number of all arm pulls, with each bar graph showing the average and the error bars indicating additional statistical information.  (Left) The number of best arm pulls for different bandit algorithms.}
\label{app:fig:pulls_arms_hebo}
\end{figure}

\clearpage

\subsection{From theory to practice}
\label{app:sec:thoery_vs_reality}
To validate our theorem against practical outcomes, we applied our algorithm to all benchmarks and plotted the number of pulls for each arm, denoted as "Real Experiment." Additionally, we computed the upper bound on the number of pulls by using the empirical values of $L_1$ and $U_i$ and $\Delta_i$. Notably, we report the first term of \Eqref{eq:num_pulls_regret} since the second term is nearly constant across all arms according to calculation. The results demonstrate that although the empirical pull counts are much less than the theoretical bounds of \Eqref{eq:num_pulls_regret}, both follow a similar decreasing pattern as the rank of suboptimality increases.

\begin{figure}[ht]
\centering
\hspace*{-2.25cm}\includegraphics[height=4.5cm]{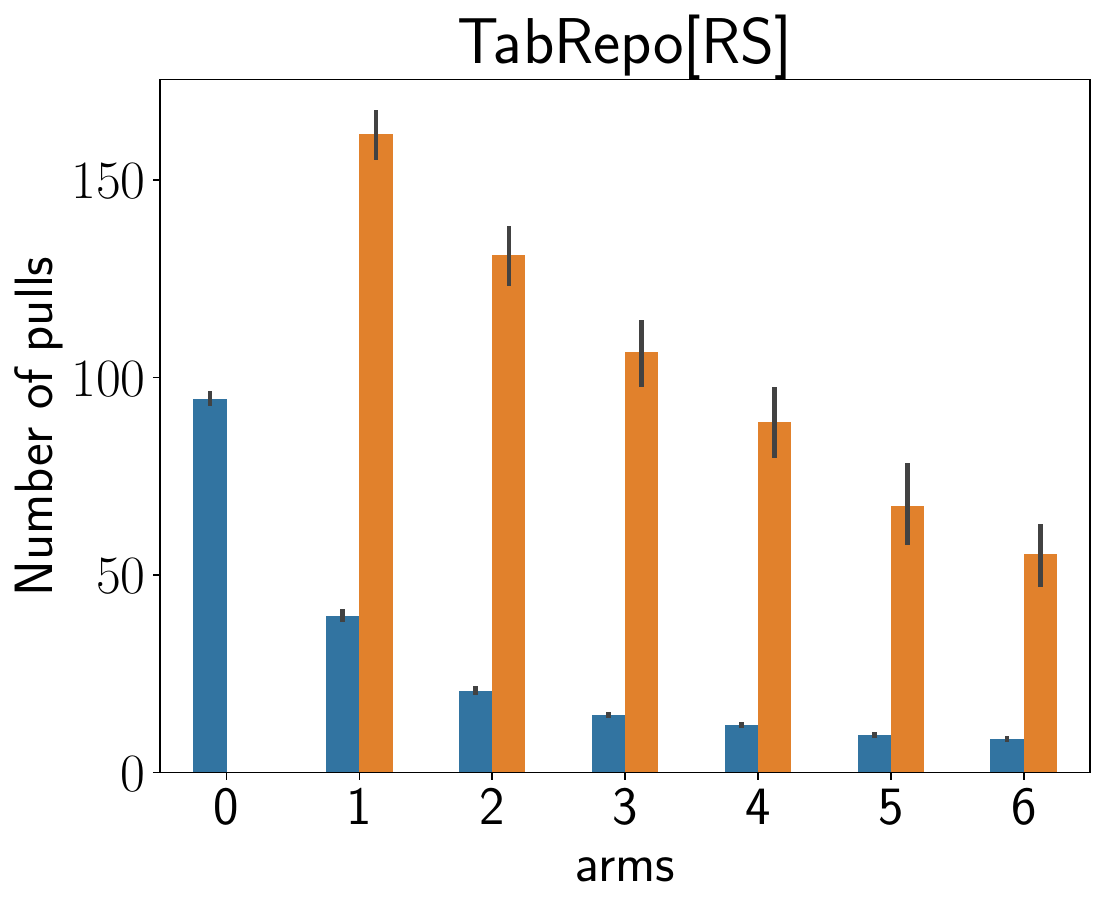}
\includegraphics[height=4.5cm]{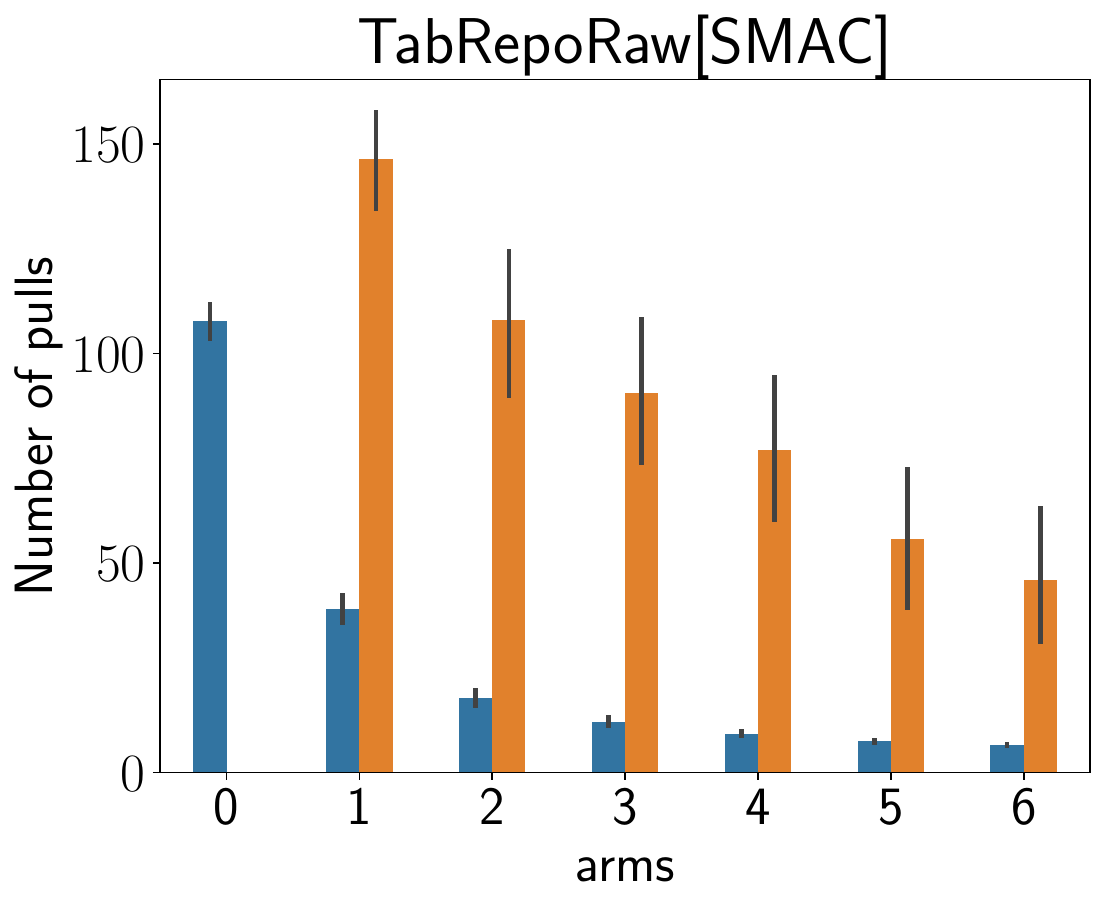}
\includegraphics[height=4.5cm]{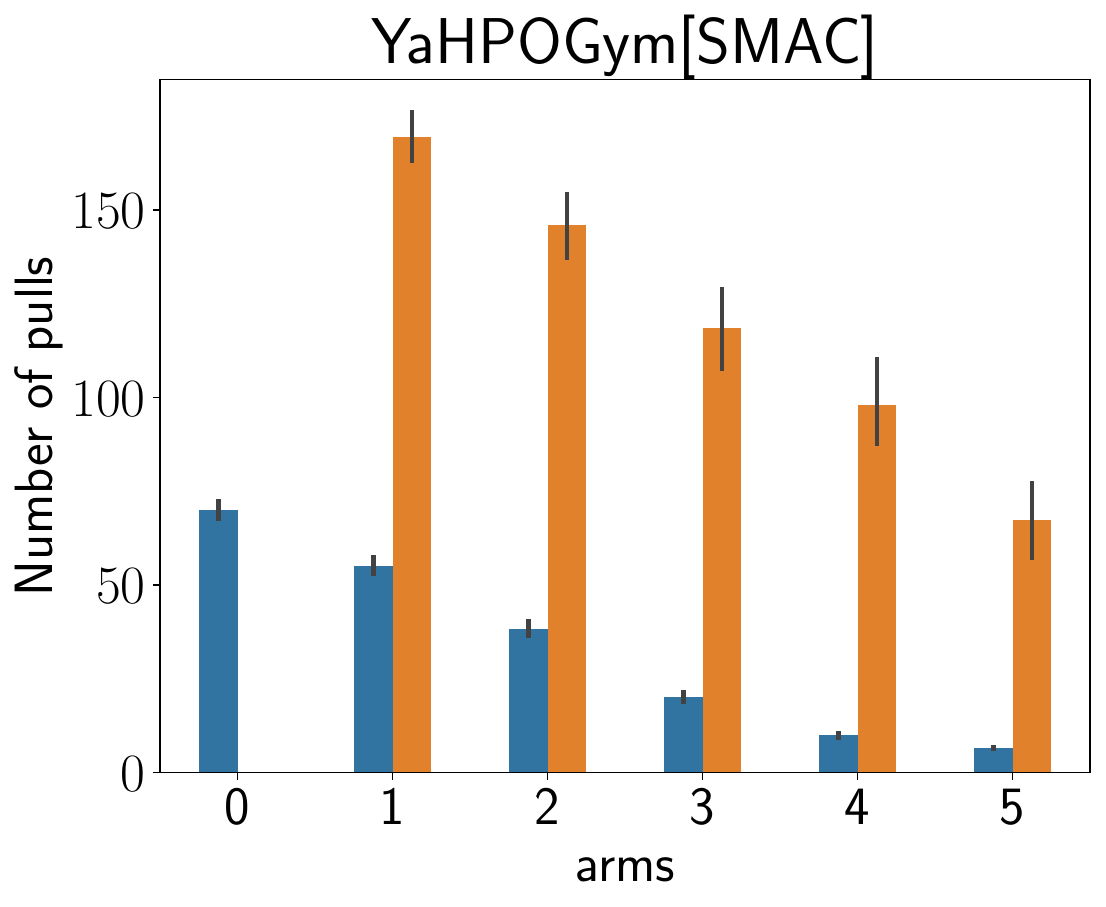}
\includegraphics[height=4.5cm]{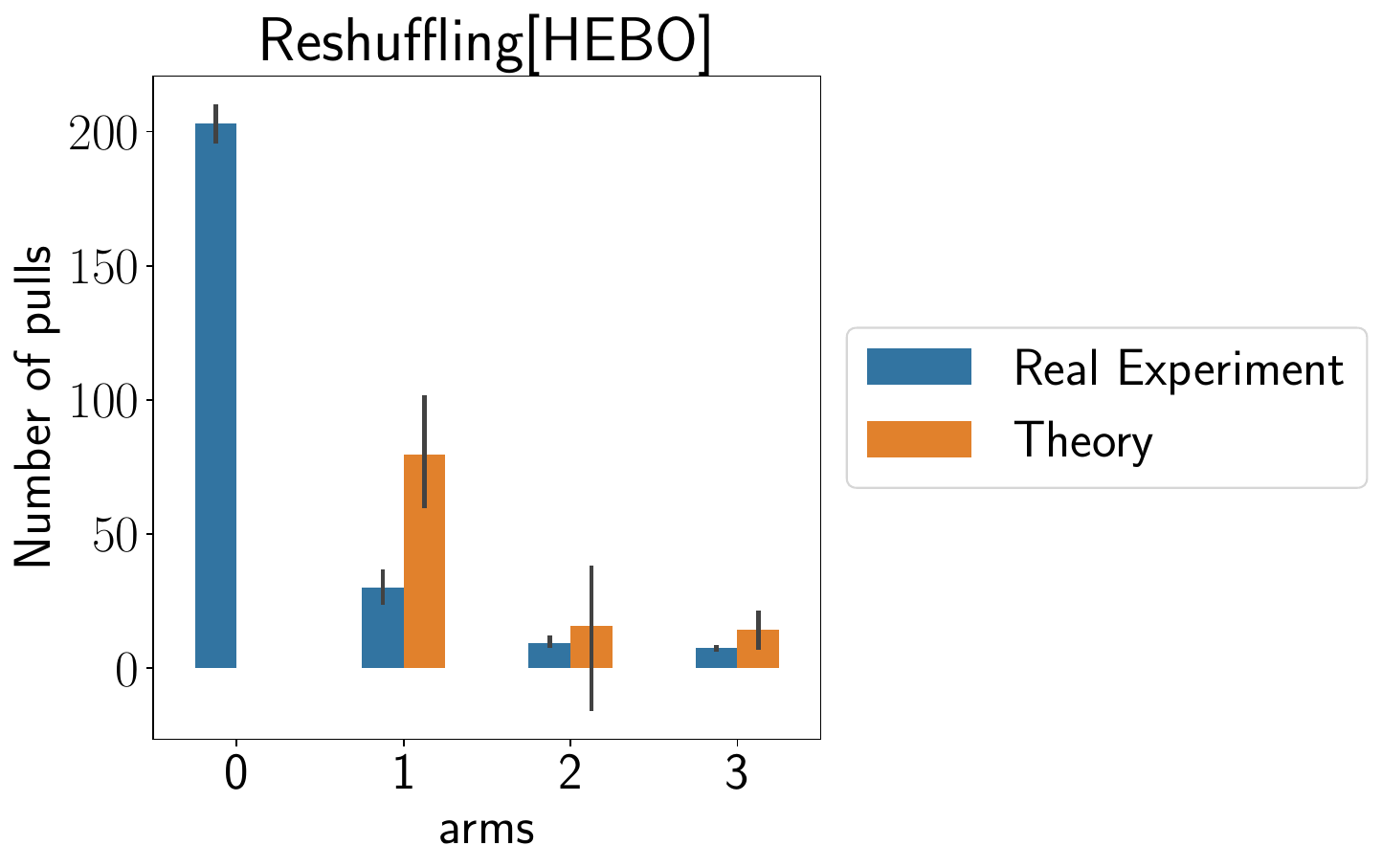}
\label{app:fig:thoery_vs_reality}
\caption{The number of pulls for each arm in our algorithm, labeled as "Real Experiment" and the theoretical values of this number}
\end{figure}

\subsection{Addressing Non-stationary Rewards}
\label{app:sec:burning_extUCB}
To handle non-stationary rewards, we pull each arm $C$ times without observing the rewards before running \OURALGO{}. This "burn-in" allows the Markov Chain to reach equilibrium, especially from a poor starting point. Algorithm \ref{alg:pseudocode_MaxUCB_Burnin} shows the adapted version of our algorithm.
Therefore, empirically, this allows \OURALGO{} to operate after a fixed exploration phase of all arms until the reward distribution is stationary.

We run Algorithm \ref{alg:pseudocode_MaxUCB_Burnin} with different parameters of $C \in \{5, 6, 7, 8\}$ using up to $48$ iterations corresponding to almost $25\%$ of the total budget. Figure \ref{app:fig:burning_heatmap} shows normalized loss per task where columns are sorted by the maximum change of the mean 
of the reward distributions of the optimal arm computed every $10$ HPO iterations (as an indicator of non-stationarity; shown at the top panel in Figure~\ref{app:fig:burning_heatmap}. 
Figure \ref{app:fig:burn_in_results} shows the average ranking and normalized loss over time for different values of the hyperparameter $C$.  

The initial burn-in improves final performance for the few tasks where we observe a high shift (right part of Figure~\ref{app:fig:burning_heatmap}) while the initial performance is worse across all tasks (as shown in Figure~\ref{app:fig:burn_in_results}). However, the results are not sensitive to the exact value of $C$. Overall, this naive solution can improve performance for some tasks at the cost of not using potentially valuable information obtained from initial exploration. Thus, optimally addressing non-stationary rewards could be a promising direction for future work.

\begin{algorithm}[htbp]
\scriptsize
\caption{\OURALGO-Burn-in}
\label{alg:pseudocode_MaxUCB_Burnin}
\begin{algorithmic}[1]
\Require  \maxucb{$\alpha$(exploration parameter), $C$ (burn-in rounds)}, $T$(time horizon), $K$(arms)
\Comment{Burn-in phase}
\For {$j \leq C$, for each arm $i \leq K$}
    \State Pull arm $i$
\EndFor
 \Comment{Initial phase}
\For {each arm $i \leq K$}
 \State Pull arm $i$ 
 \State  set $n_i \gets 1$, observe reward $r_{i,1}$ 
\EndFor
 \Comment{Main phase}
\For {$t=(CK + K + 1)$ to $T$}
    \For {each arm $i \leq K$}
    \State \parbox[t]{0.8\linewidth}{Update policy \maxucb{$U_i =\max{(r_{i,1},...,r_{i,n_i})} +  (\frac{ \alpha \log(t)}{n_i})^2 $ }}
    \EndFor
    \State Select arm $I_t = \argmax\limits_{i\leq K} U_i$ 
    \State  $n_{I_t} \gets n_{I_t} +  1$ 
    \State  Observe reward $r_{I_t, n_{I_t}}$
\EndFor
\end{algorithmic}
\end{algorithm}

\begin{figure}[htbp]
\begin{subfigure}{1.0\textwidth}
        \includegraphics[width=\textwidth]{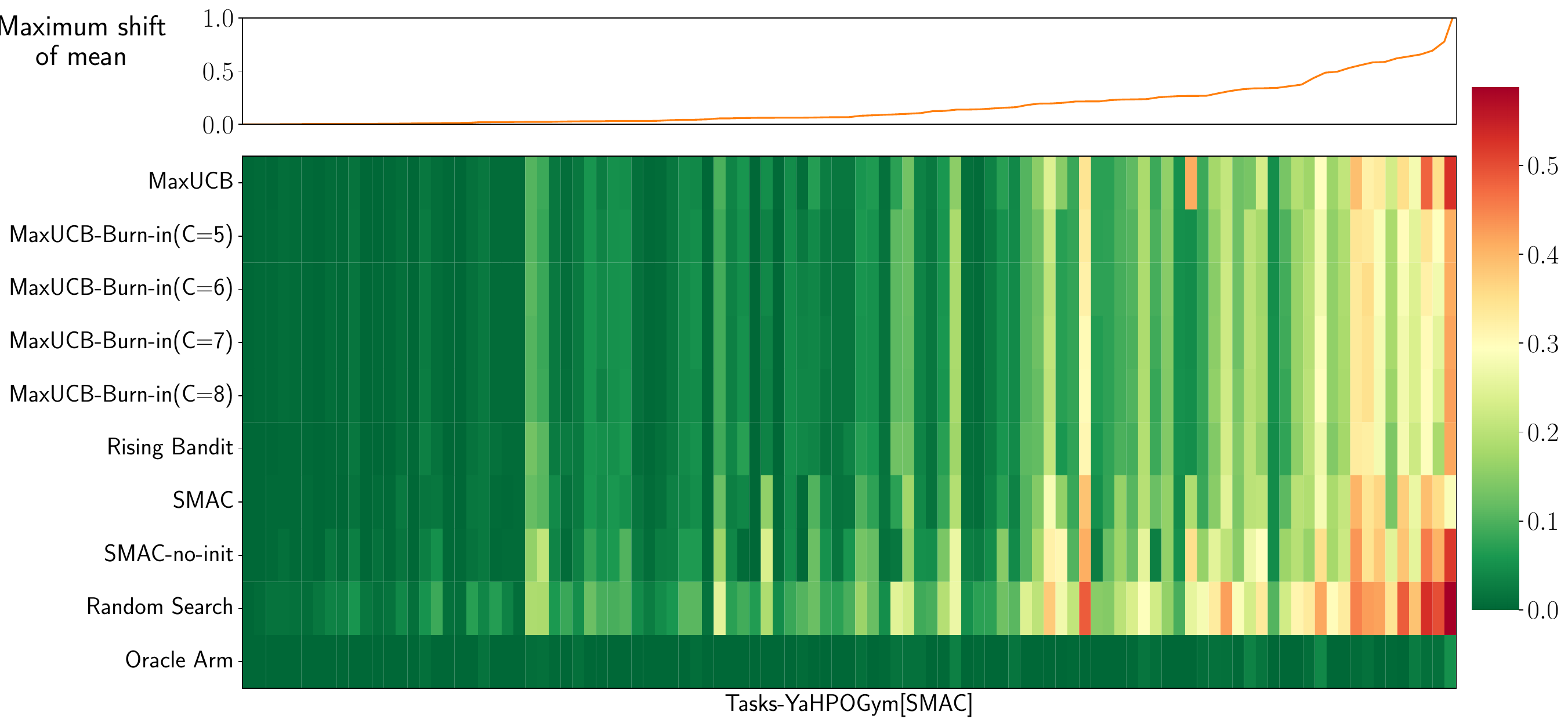}
\end{subfigure}
\caption{Heat map shows normalized loss per task, sorted by the distribution shift. }
\label{app:fig:burning_heatmap}
\end{figure}

\begin{figure*}[htbp]
\begin{tabular}{c c c}
\begin{subfigure}{0.35\textwidth}
        \includegraphics[clip, trim=0cm 0cm 0cm 0cm, height=4.5cm]{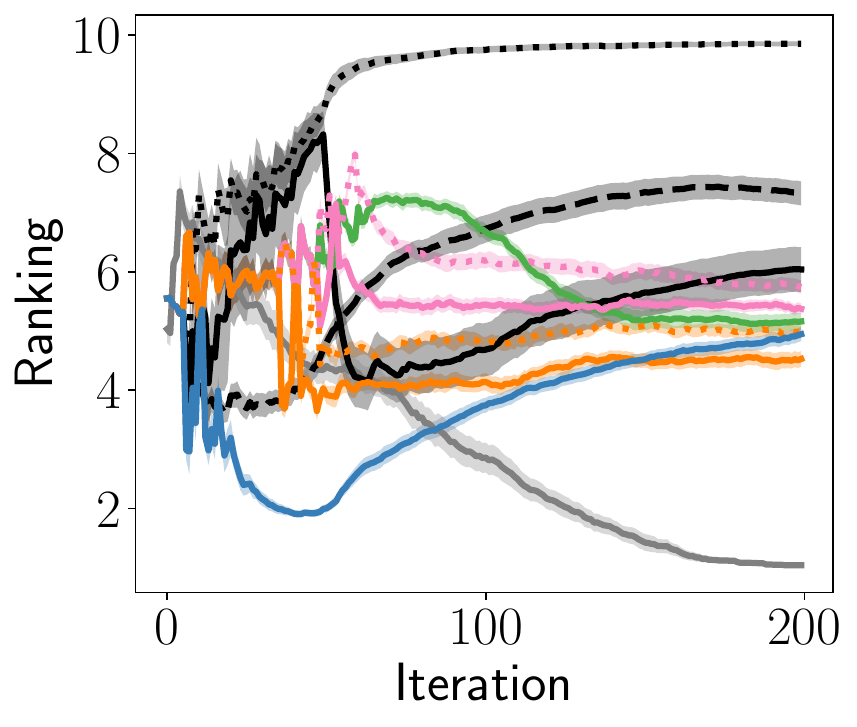}
\end{subfigure}
&     \begin{subfigure}{0.35\textwidth}
        \includegraphics[clip, trim=0cm 0cm  0cm 0cm, height=4.5cm]{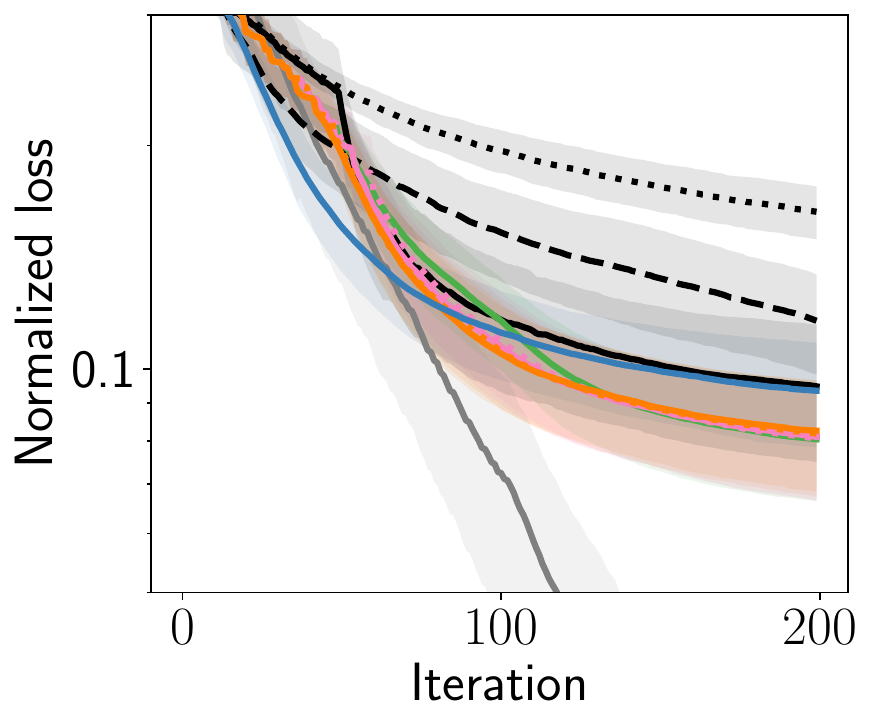}
\end{subfigure}
&     \begin{subfigure}{0.4\textwidth}
        \includegraphics[clip, trim=0cm -8cm 0.0cm 0cm, height=4cm]{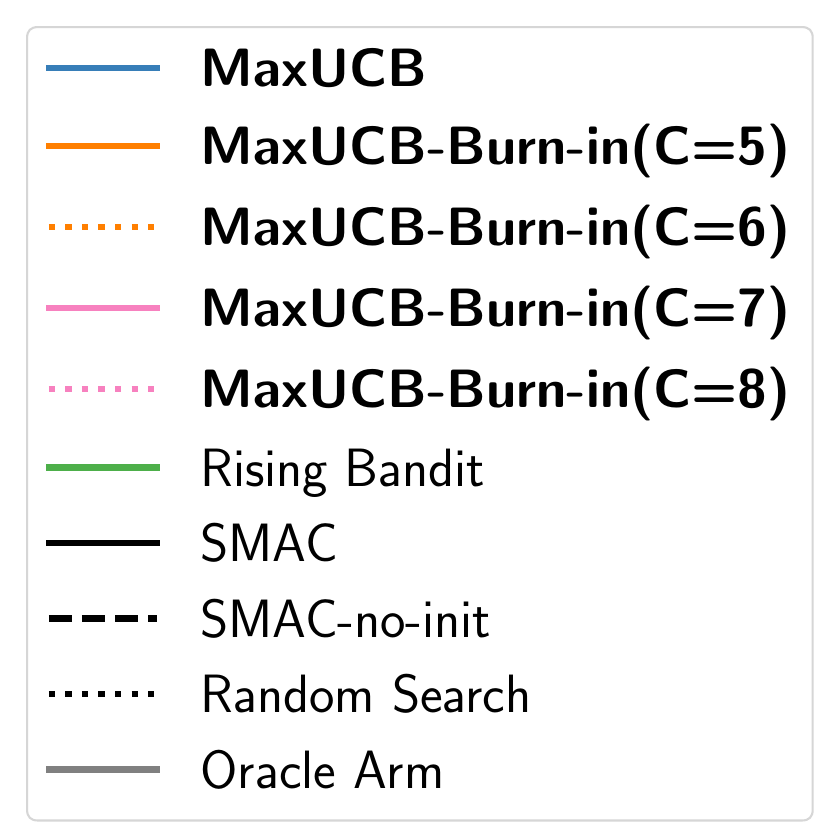}
\end{subfigure}%
\end{tabular}
\caption{Average rank and normalized loss of algorithms on \yahpogym[] benchmark, lower is better. }
\label{app:fig:burn_in_results}
\end{figure*}

\pagebreak[4]
\subsection{Toy examples from the extreme bandit's literature}
\label{app:sec:extreme_bandits_experiments}
In this section, we provide additional results on commonly used benchmark functions used in the extreme bandit literature. Concretely, we use a similar setup to \citep{baudry2022efficient} and report the following four tasks: 
\begin{enumerate}
    \item $K = 5$ Pareto distributions with tail parameters parameters $\lambda_k = [2.1, 2.3, 1.3, 1.1, 1.9]$. Results are shown in Figure~\ref{app:fig:xp_1}.
    \item $K = 10$ Exponential arms with a survival function $G_k(x) = e^{-\lambda_k x}$ with parameters $\lambda_k = [2.1, 2.4, 1.9, 1.3, 1.1, 2.9, 1.5, 2.2, 2.6, 1.4]$. Results are shown in Figure~\ref{app:fig:xp_3}.
     \item $K = 20$ Gaussian arms, with same mean $\mu_k = 1,\forall k$, and different variances {\scriptsize$ \sigma_k = [1.64, 2.29, 1.79, 2.67, 1.70, 1.36, 1.90, 2.19, 0.80, 0.12, 1.65, 1.19, 1.88, 0.89, 3.35, 1.5, 2.22, 3.03, 1.08, 0.48]$}. The dominant arm has a standard deviation of $3.35$. Results are shown in Figure~\ref{app:fig:xp_4}.
      \item (our toy example) $K = 5$ power distributions with domain parameter $[3, 4, 5, 5, 4]$ and shape parameter $[1.01, 1.01, 1.01, 1.1, 1]$. Results are shown in Figure~\ref{app:fig:xp_11}.
\end{enumerate}
For each task, we run $N = 1000$ independent repetitions for six time horizons $T \in \{ 50 100, 200, 500, 1000, 2000 \}$. We show the CDF of the rewards for each arm, the number of times the optimal arm was pulled, and the proxy empirical regret. Notably, the proxy empirical regret is introduced by \citet{baudry2022efficient} to overcome the issue of high variance in the maximum values of distributions.

\begin{figure}[hbp]
\centering
\includegraphics[height=4.5cm]{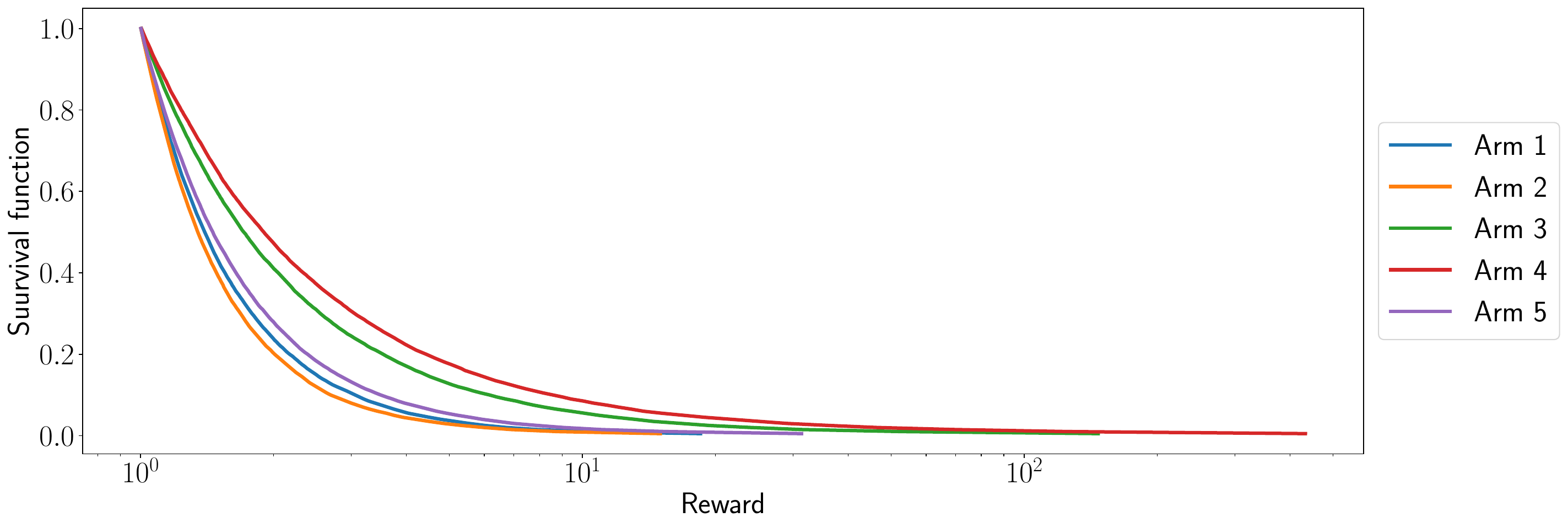}\\
\includegraphics[height=4cm]{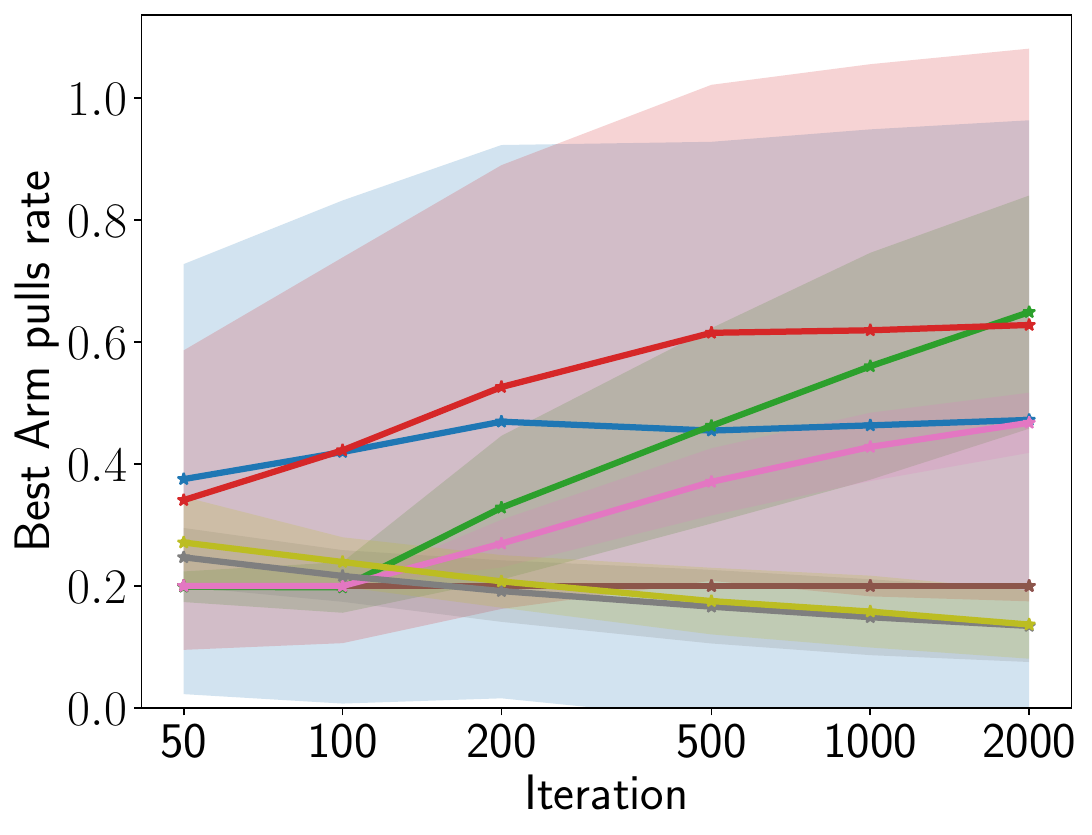}
\includegraphics[height=4cm]{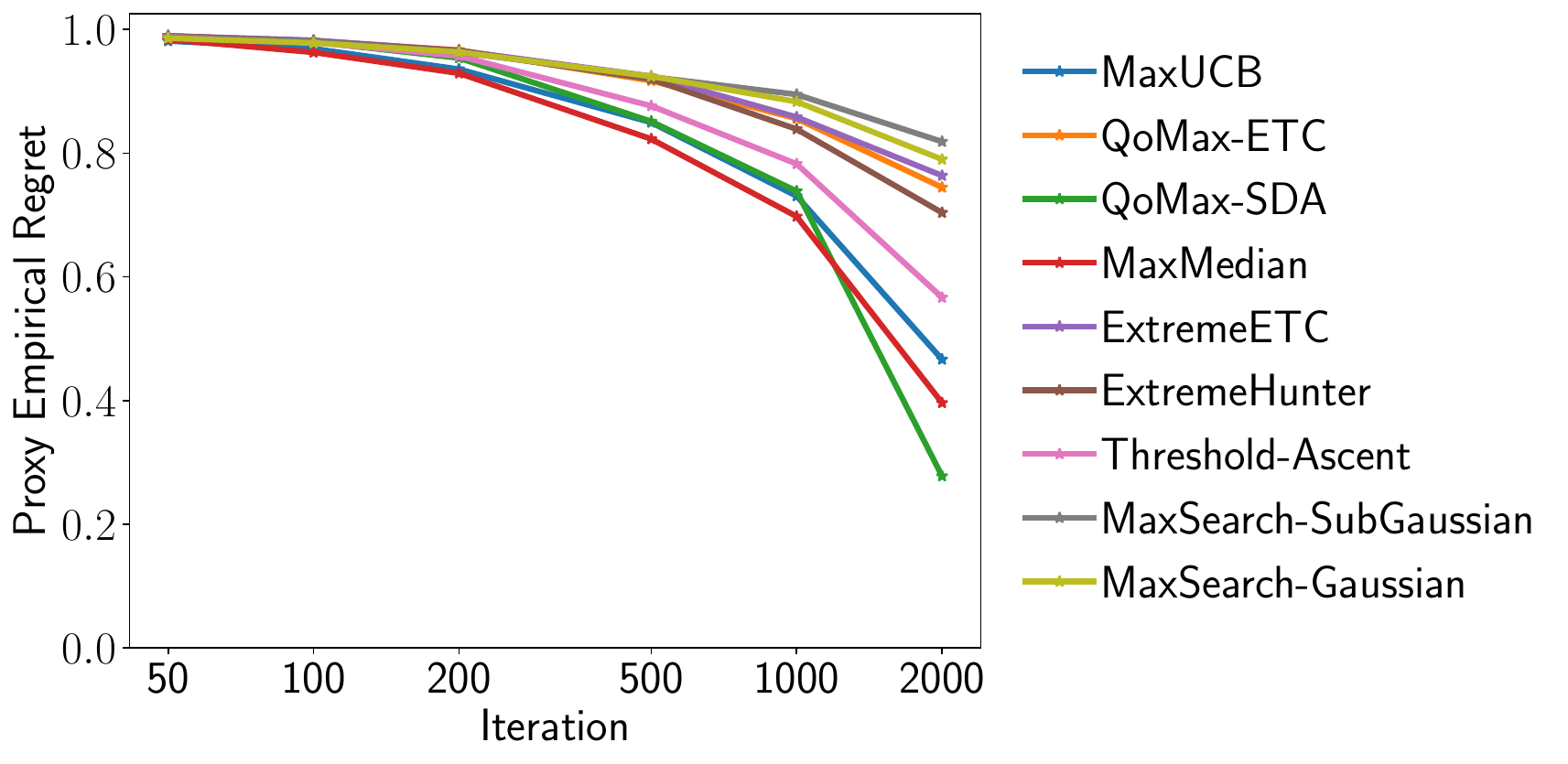}
\caption{Experiment 1: (Top) Survival function of distribution of each arm (left) Number of pulls of the optimal arm. (Right) Proxy Empirical Regret}
\label{app:fig:xp_1}
\end{figure}

\begin{figure}[tbp]
\centering
\includegraphics[height=4.5cm]{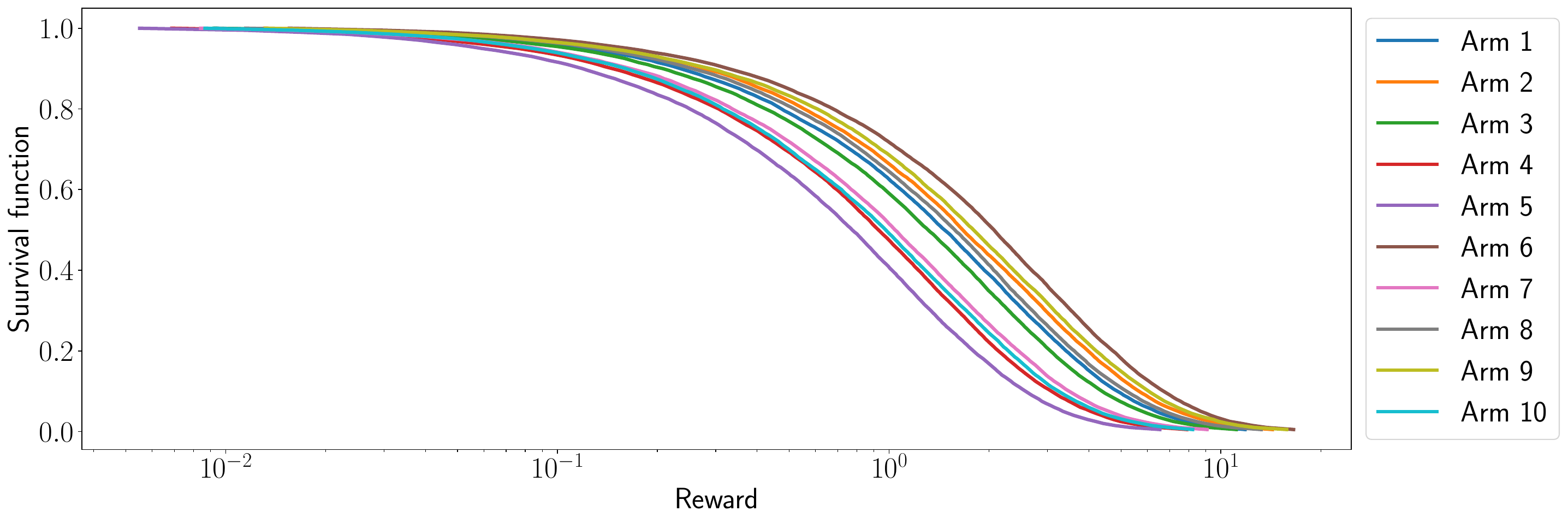}\\
\includegraphics[height=4cm]{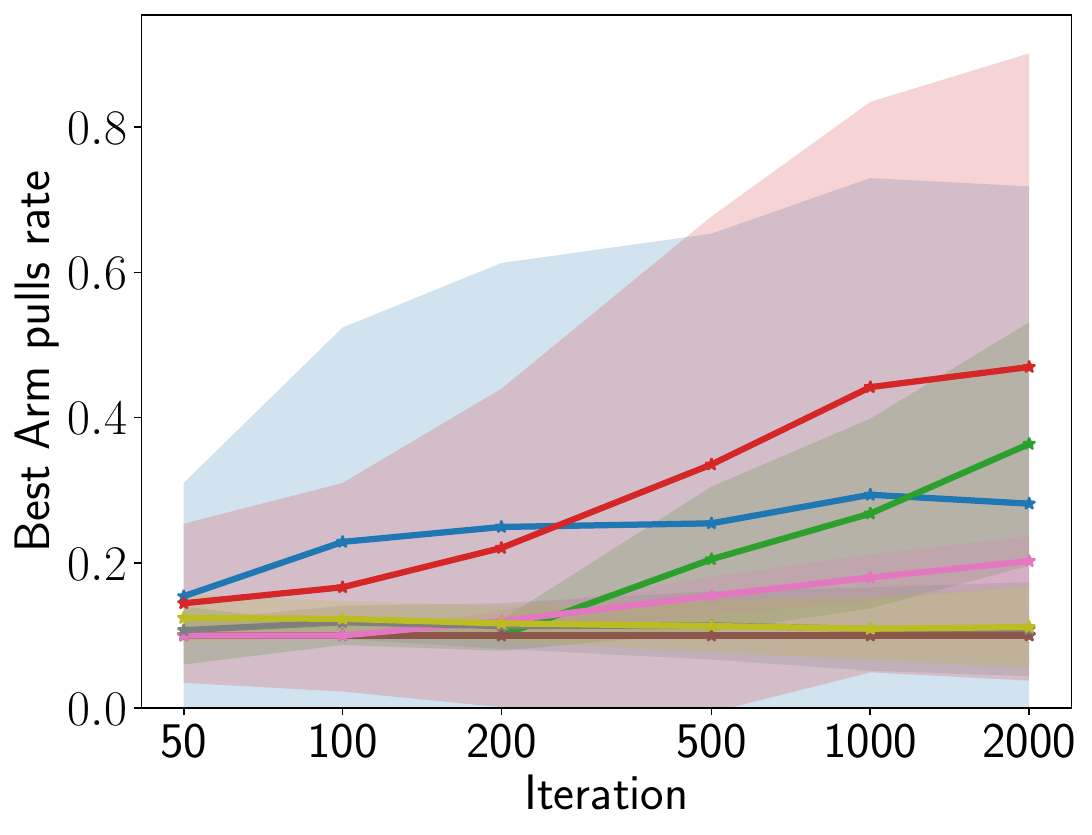}
\includegraphics[height=4cm]{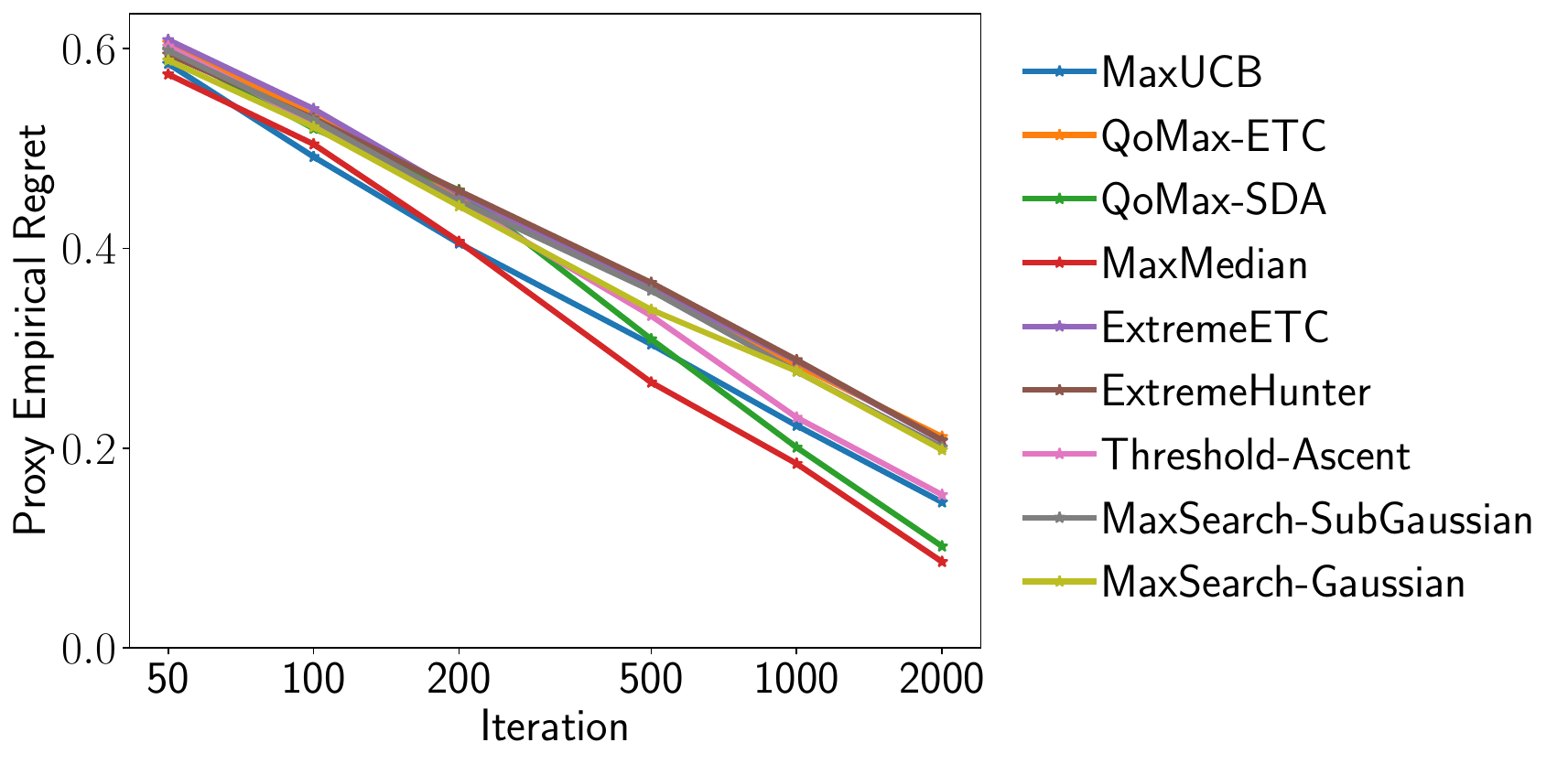}
\caption{Experiment 2: (Top) Survival function of distribution of each arm (left) Number of pulls of the optimal arm. (Right) Proxy Empirical Regret  }
\label{app:fig:xp_3}
\end{figure}

\begin{figure}[tbp]
\centering
\includegraphics[height=4.5cm]{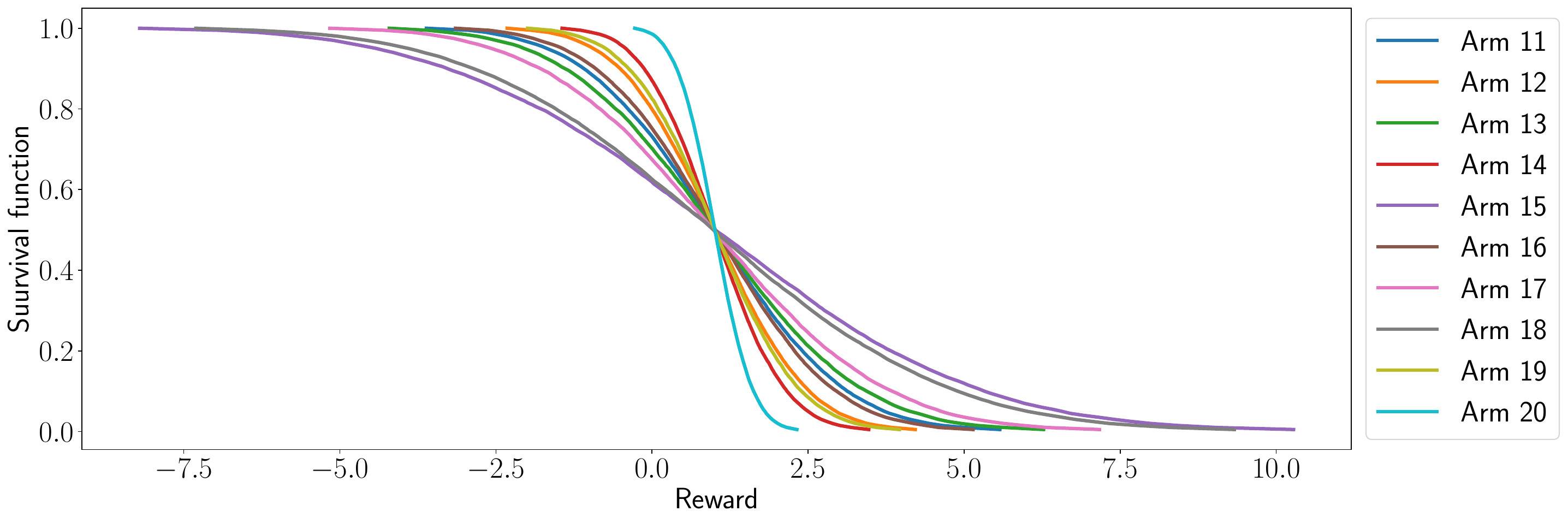}\\
\includegraphics[height=4cm]{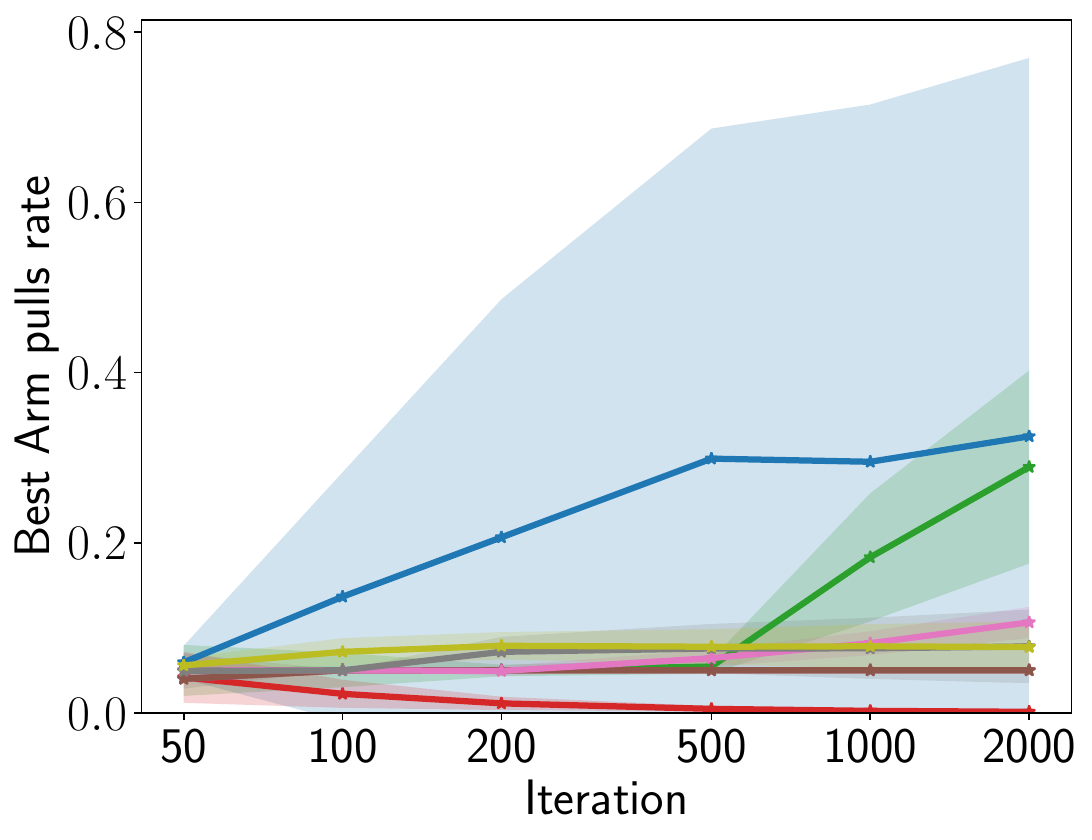}
\includegraphics[height=4cm]{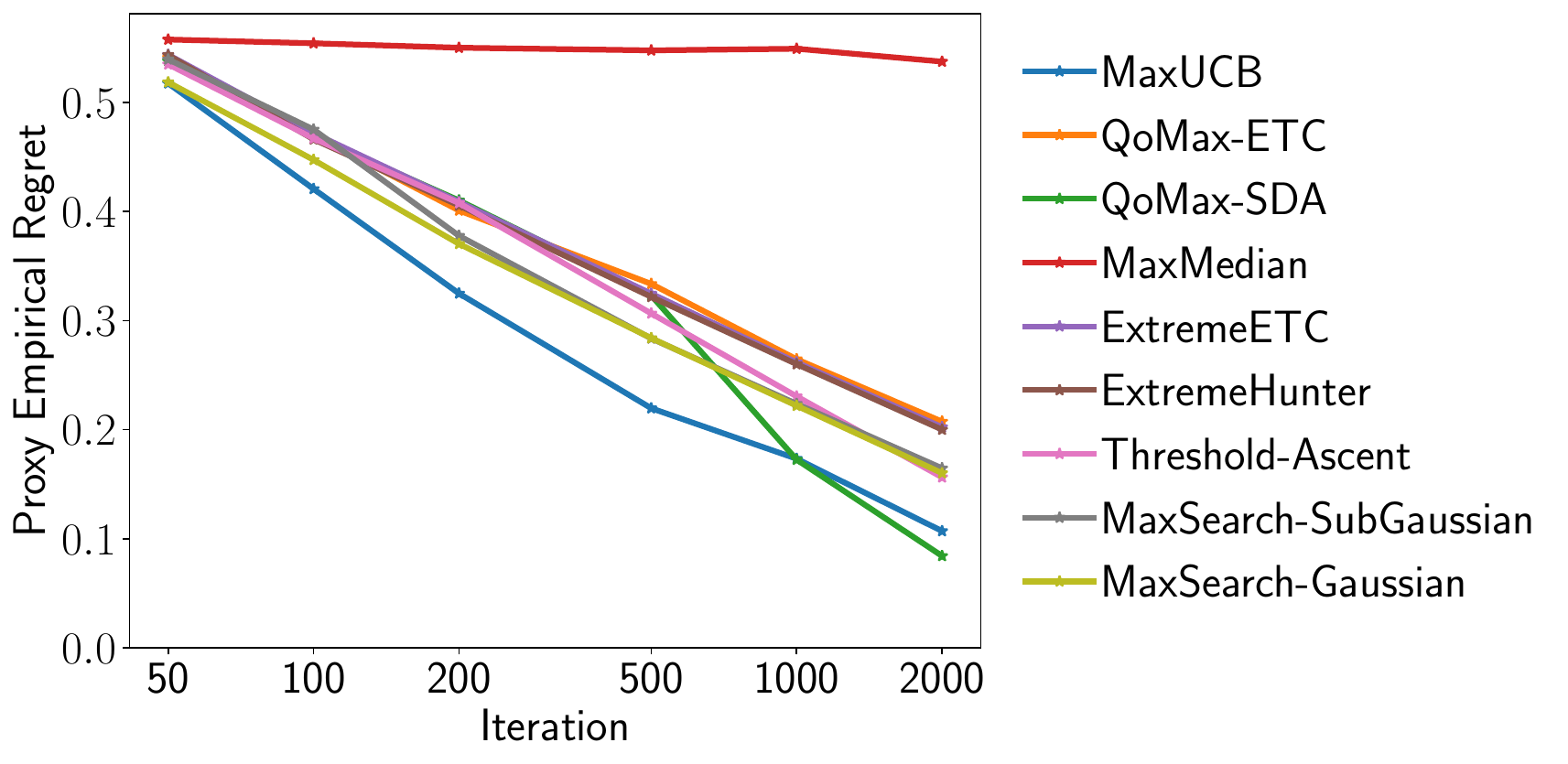}
\caption{Experiment 3: (Top) Survival function of distribution of some arms (left) Number of pulls of the optimal arm. (Right) Proxy Empirical Regret  }
\label{app:fig:xp_4}
\end{figure}

\begin{figure}[tbp]
\centering
\includegraphics[height=4.5cm]{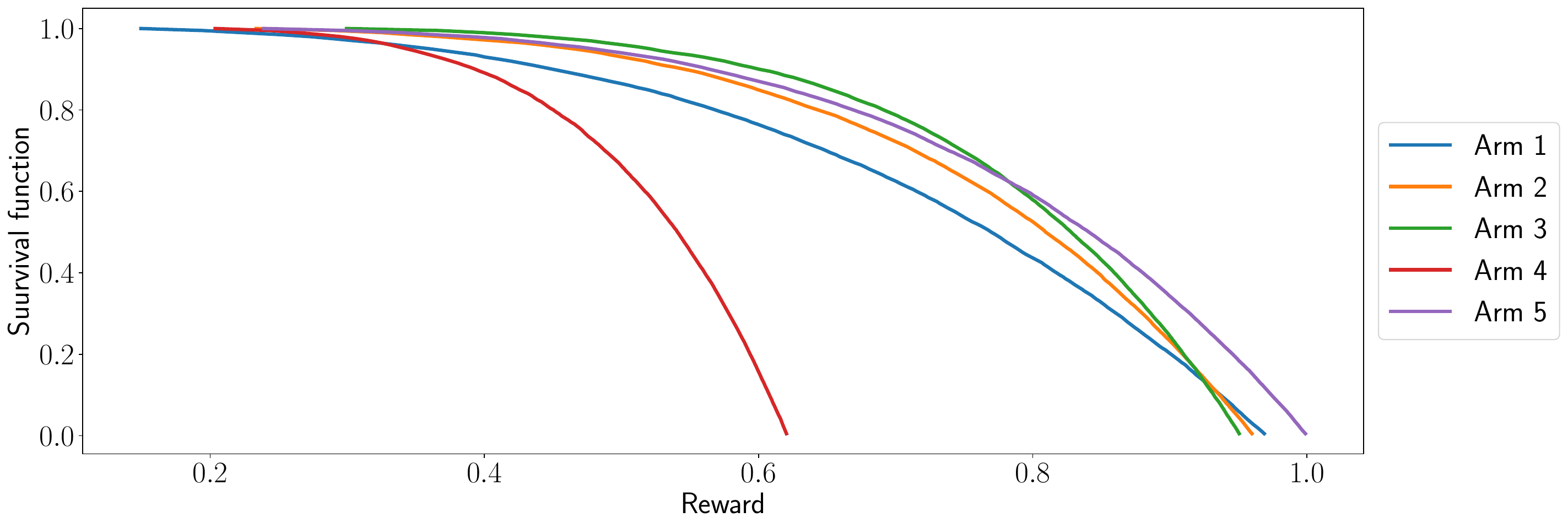}\\
\includegraphics[height=4.0cm]{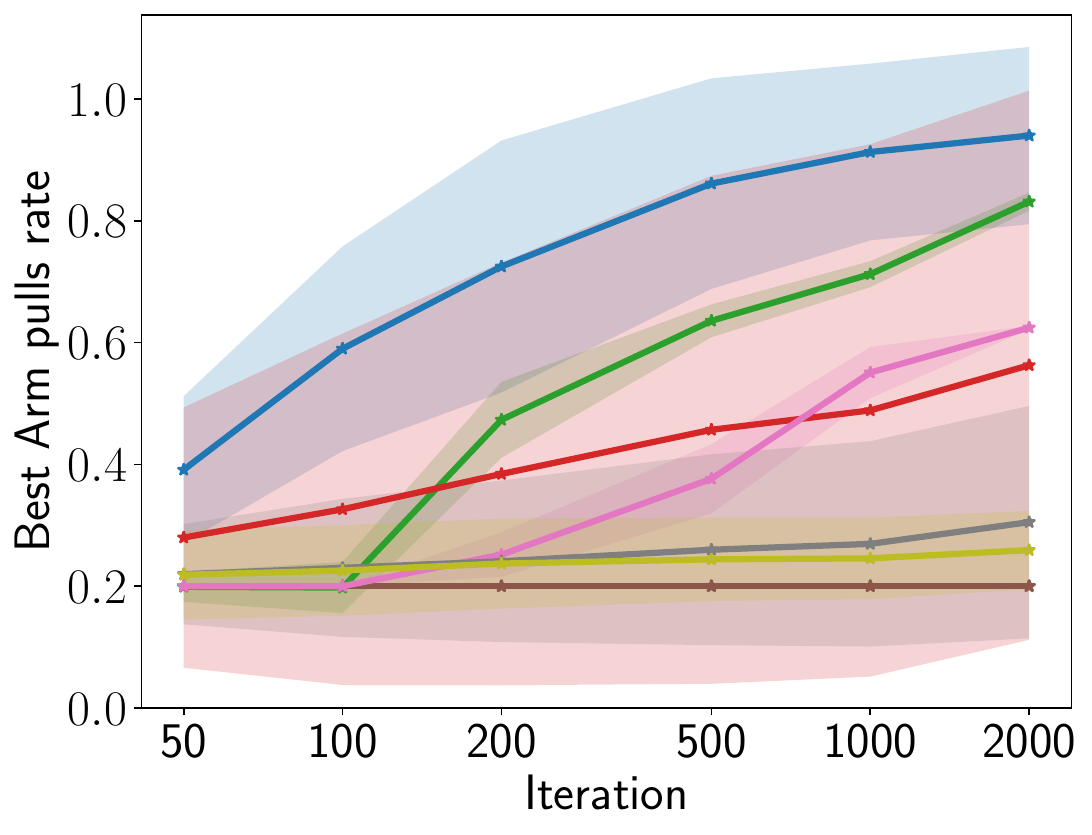}
\includegraphics[height=4.0cm]{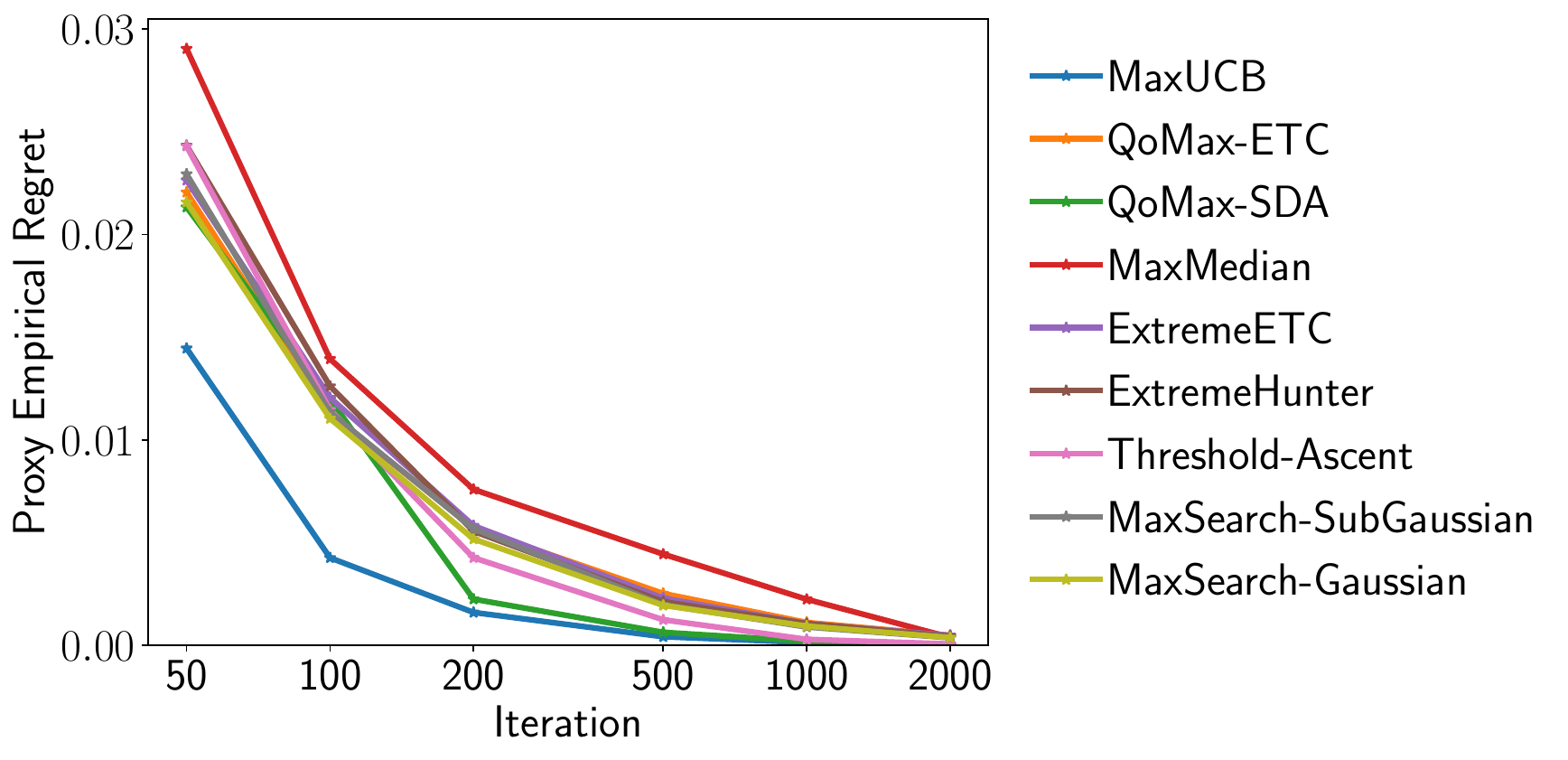}
\caption{Experiment 4: (Top) Survival function of distribution of each arm (left) Number of pulls of the optimal arm. (Right) Proxy Empirical Regret  }
\label{app:fig:xp_11}
\end{figure}
\vfill

\newpage
\subsection{Supernet selection in Few-Shot Neural Architecture Search}
\label{app:sec:supernet_experiments}
In one-shot NAS, a single supernet approximates all architectures. However, this estimation could be
inaccurate. To address this, few-shot NAS splits the supernet into smaller sub-supernets  \citep{hu2022generalizing, ly2024analyzing}. 
To further improve performance, \citep{hu2022generalizing} introduced supernet selection, which identifies the most promising sub-supernet and its optimal architecture using techniques such as \SuccessiveHalving{} . We aim to identify the best-performing architecture among several subspaces, each corresponding to a sub-supernet. This problem is analogous to the MKB problems, where each arm represents a sub-supernet.

\textbf{Dataset Preparation.} We use data provided in \citep{ly2024analyzing} for three benchmark datasets: \textit{CIFAR-10}, \textit{CIFAR-100}, and \textit{ImageNet16-120}. Each dataset's search space is split using $10$ different metrics. The splitting follows a binary tree structure with a depth of $3$, where operations are divided into two groups at each branch. This process results in $8$ sub-supernets per metric. Consequently, for each dataset, we have one full search space and $8$ sub-search spaces.

Each combination of dataset and splitting metric is treated as a separate task, yielding a total of $30$ tasks ($3$ datasets × $10$ metrics), each with $8$ arms. Following the setup of \citep{ly2024analyzing}, we randomly sample $600$ architectures from both the full search space and each sub-search space. This process is repeated $32$ times using different random seeds to ensure variability and robustness in the results.

\textbf{Analyzing the reward distribution.}
Figure~\ref{app:fig:SubSupernet_ecdf_reward_distributions} illustrates the empirical survival functions of the rewards and sub-optimality gaps for the benchmark. As shown, the distribution shape is similar to that of HPO tasks: both are bounded and left-skewed. However, the sub-optimality gap is considerably smaller than HPO tasks, suggesting that identifying the optimal arm is more challenging and may require additional iterations. In Figure~\ref{app:fig:SubSupernet_L_U_distributions}, we show values of $L$ and $U$ from \ourProposition{} for this benchmark.

\begin{figure}[hb!]
\centering
\begin{tabular}{c c c}
&Supernet
&\\
\includegraphics[width=0.3\linewidth]{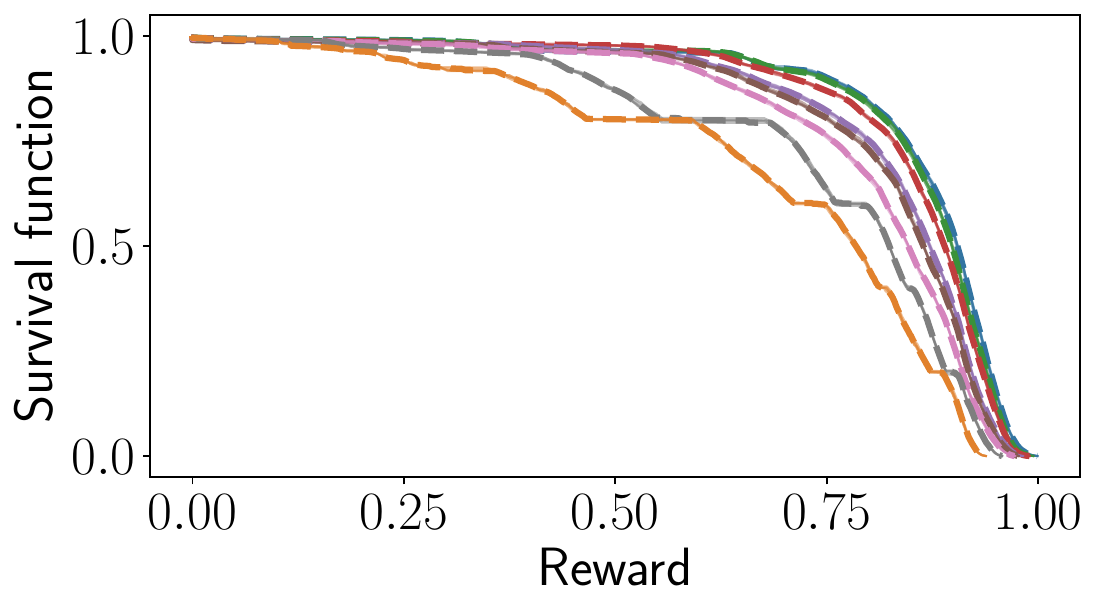}
&\includegraphics[width=0.3\linewidth]{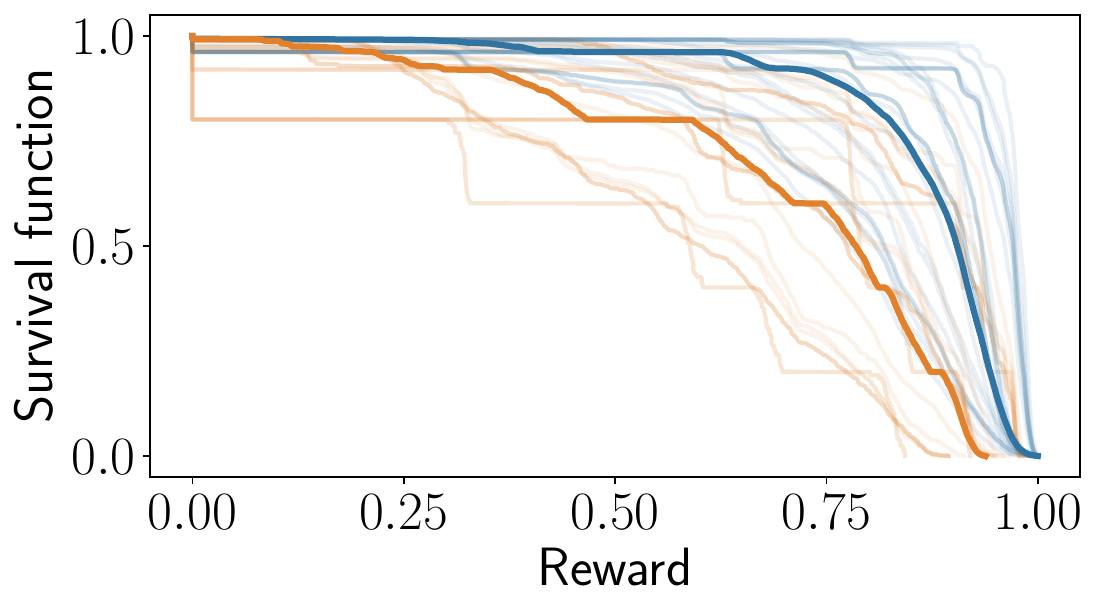}
&\includegraphics[width=0.3\linewidth]{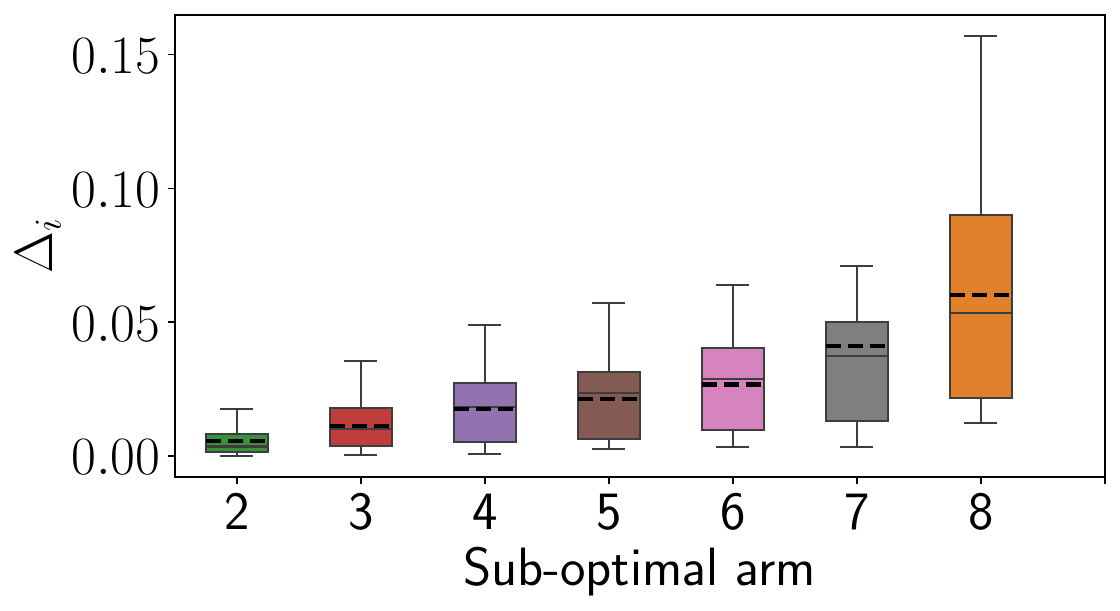}\\
\multicolumn{3}{c}{\includegraphics[width=1.0\linewidth]{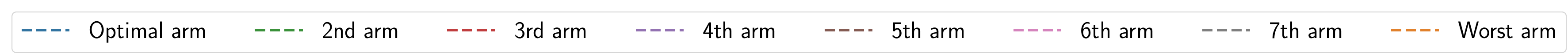}}
 \end{tabular}
\caption{(Left) The average empirical survival function of the reward (observed performances) per arm ranked per dataset.
We divided the reward sequence into five segments over the budget to show the distribution change over time.
Thin lines correspond to the survival function of different segments, visualizing the change over time. (Middle) The average ECDF per dataset for the best and worst arm with thin lines corresponding to individual datasets. (Right) The sub-optimality gap $\Delta_i$.}
\label{app:fig:SubSupernet_ecdf_reward_distributions}
\end{figure}

\begin{figure}[ht]
\centering
\textbf{Supernet}\\
\includegraphics[width=1.0\textwidth]{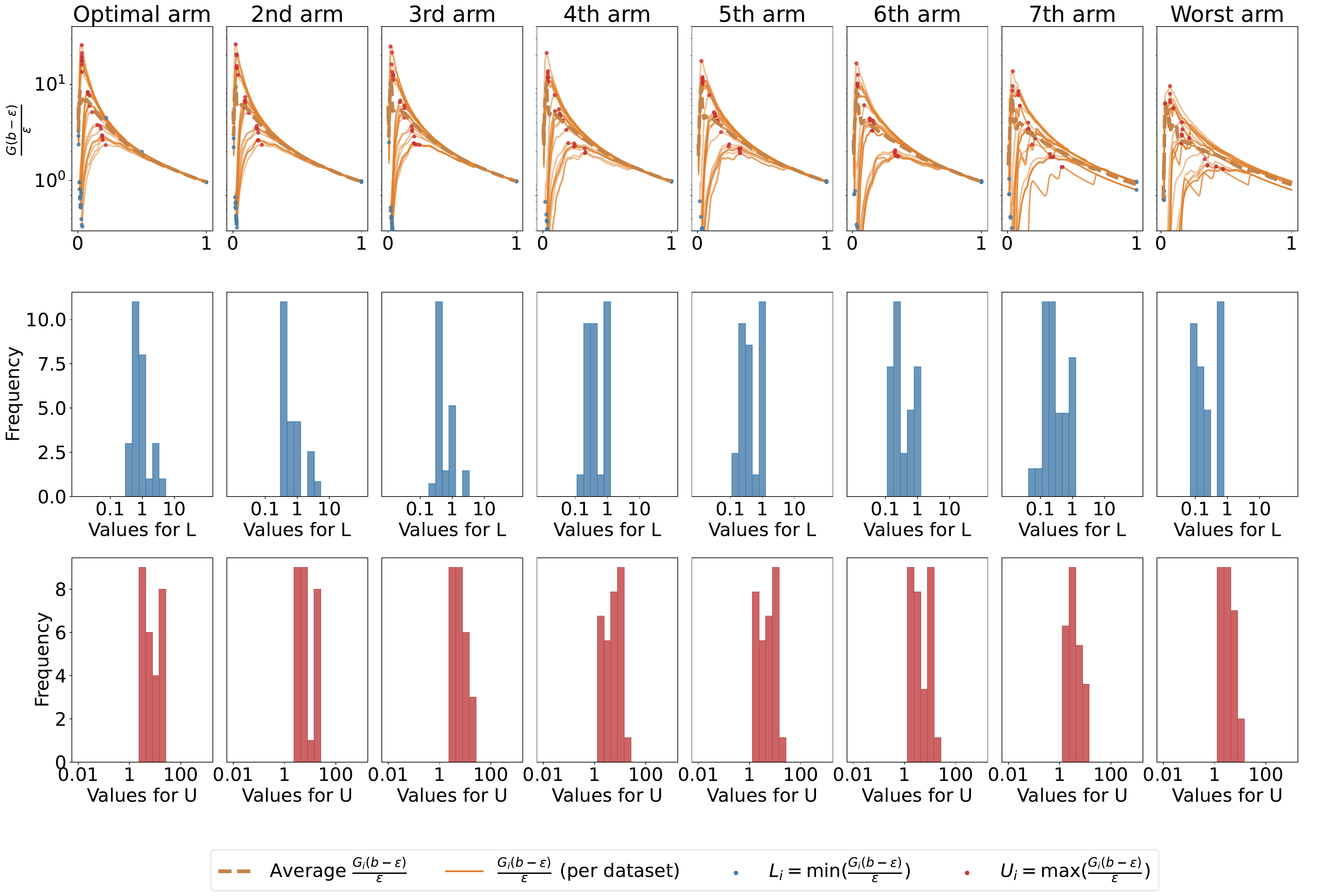}
\caption{ Arms are ordered by sub-optimality gap. (Top) Thin orange lines represent $\frac{G(b - \epsilon)}{\epsilon}$, while the blue and red points correspond to $L$ and $U$ for our empirical reward distributions (see \ourProposition{} for details). (Middle) Histogram of values for $L$. (Bottom) Histogram of values for $U$. }
\label{app:fig:SubSupernet_L_U_distributions}
\end{figure}

\textbf{Performance Analysis.}
Figure~\ref{app:fig:SubSupernet_results}  presents the average ranking and normalized loss of various bandit algorithms in this benchmark. As shown, \SuccessiveHalving{}, \MaxMedian{}, and \ERUCBS{} perform well with a small time budget but fail to explore sufficiently to identify the optimal arm. \RisingBandits{}, as a fixed-confidence best-arm identification method, struggles to find the optimal arm. In contrast, \OURALGO{} outperforms all other baselines, demonstrating better performance when searching the entire search space with a higher budget.

\begin{figure*}[htbp]
\begin{tabular}{c c c}
\begin{subfigure}{0.4\textwidth}
        \includegraphics[clip, trim=0cm 0cm 0cm 0cm, height=4.5cm]{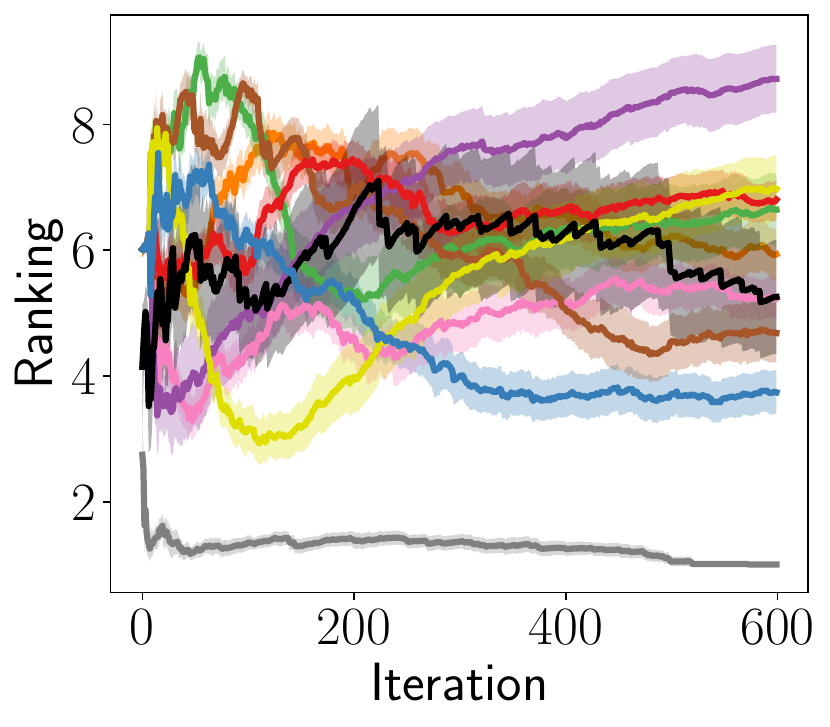}
\end{subfigure}
&     \begin{subfigure}{0.4\textwidth}
        \includegraphics[clip, trim=0cm 0cm  0cm 0cm, height=4.5cm]{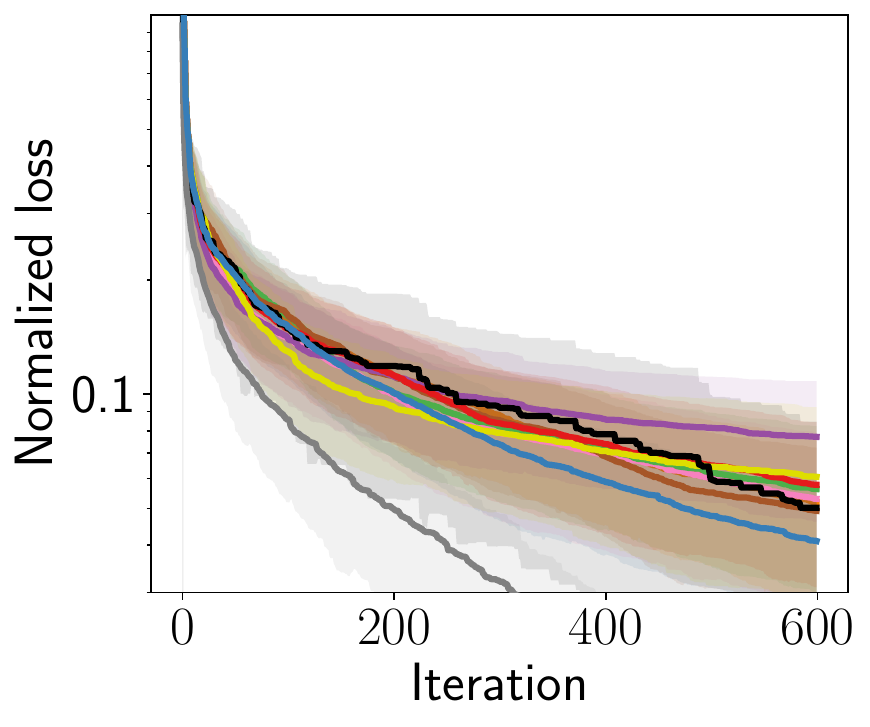}
\end{subfigure}
&     \begin{subfigure}{0.35\textwidth}
        \includegraphics[clip, trim=0cm -4cm 0.0cm 0cm, height=4cm]{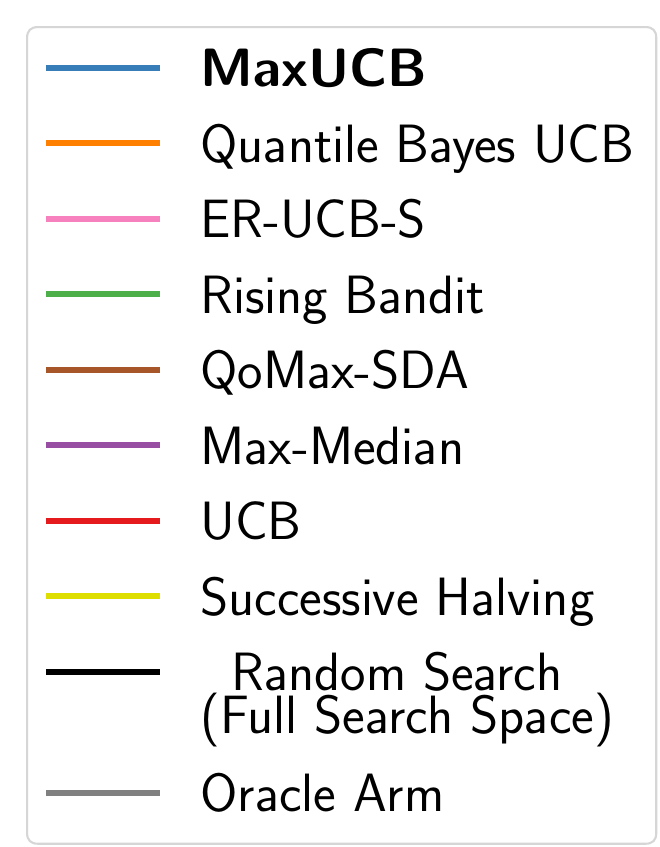}
\end{subfigure}%
\end{tabular}
\caption{Average rank and normalized loss of algorithms on \textit{Supernet Selection} benchmark, lower is better. }
\label{app:fig:SubSupernet_results}
\end{figure*}

\clearpage
\section*{NeurIPS Paper Checklist}
\label{app:checklist}
\begin{enumerate}

\item {\bf Claims}
    \item[] Question: Do the main claims made in the abstract and introduction accurately reflect the paper's contributions and scope?
    \item[] Answer: \answerYes{}
    \item[] Justification: We wrote the abstract and introduction to reflect the paper's contribution (a practical Bandit strategy for decomposed CASH) as detailed as possible and discuss it ability to generalize to other settings in the Conclusion (Section~\ref{Sec:Conclusion}) and in Appendix~\ref{app:theorem1_proof_extension}.

\item {\bf Limitations}
    \item[] Question: Does the paper discuss the limitations of the work performed by the authors?
    \item[] Answer: \answerYes{} 
    \item[] Justification: We provide a separate paragraph for the limitations of our methods in Section~\ref{Sec:Conclusion}. 

\item {\bf Theory assumptions and proofs}
    \item[] Question: For each theoretical result, does the paper provide the full set of assumptions and a complete (and correct) proof?
    \item[] Answer: \answerYes{} 
    \item[] Justification:  Yes, we provide all necessary assumptions for the theoretical results in Section~\ref{sec:dataanalysis} and detail proofs for theories in Appendix~\ref{app:proofs}. We also discuss the validity of our assumptions using on empirical data analysis in Section~\ref{sec:dataanalysis}.

    \item {\bf Experimental result reproducibility}
    \item[] Question: Does the paper fully disclose all the information needed to reproduce the main experimental results of the paper to the extent that it affects the main claims and/or conclusions of the paper (regardless of whether the code and data are provided or not)?
    \item[] Answer: \answerYes{} 
    \item[] Justification:  Yes, we provide detailed information about all the methods we used. This includes a description of the computing infrastructure in Footnote~\ref{footnote:implementation}, an ablation study for hyperparameter selection in Section~\ref{Sec:Numerical_Experiments} and Appendix~\ref{app:aplation_study}, as well as metrics and experimental details in Appendix~\ref{app:details_on_experiments}. 
    \item[] Guidelines:
    \begin{itemize}
        \item The answer NA means that the paper does not include experiments.
        \item If the paper includes experiments, a No answer to this question will not be perceived well by the reviewers: Making the paper reproducible is important, regardless of whether the code and data are provided or not.
        \item If the contribution is a dataset and/or model, the authors should describe the steps taken to make their results reproducible or verifiable. 
        \item Depending on the contribution, reproducibility can be accomplished in various ways. For example, if the contribution is a novel architecture, describing the architecture fully might suffice, or if the contribution is a specific model and empirical evaluation, it may be necessary to either make it possible for others to replicate the model with the same dataset, or provide access to the model. In general. releasing code and data is often one good way to accomplish this, but reproducibility can also be provided via detailed instructions for how to replicate the results, access to a hosted model (e.g., in the case of a large language model), releasing of a model checkpoint, or other means that are appropriate to the research performed.
        \item While NeurIPS does not require releasing code, the conference does require all submissions to provide some reasonable avenue for reproducibility, which may depend on the nature of the contribution. For example
        \begin{enumerate}
            \item If the contribution is primarily a new algorithm, the paper should make it clear how to reproduce that algorithm.
            \item If the contribution is primarily a new model architecture, the paper should describe the architecture clearly and fully.
            \item If the contribution is a new model (e.g., a large language model), then there should either be a way to access this model for reproducing the results or a way to reproduce the model (e.g., with an open-source dataset or instructions for how to construct the dataset).
            \item We recognize that reproducibility may be tricky in some cases, in which case authors are welcome to describe the particular way they provide for reproducibility. In the case of closed-source models, it may be that access to the model is limited in some way (e.g., to registered users), but it should be possible for other researchers to have some path to reproducing or verifying the results.checklist
        \end{enumerate}
    \end{itemize}

\item {\bf Open access to data and code}
    \item[] Question: Does the paper provide open access to the data and code, with sufficient instructions to faithfully reproduce the main experimental results, as described in supplemental material?
    \item[] Answer: \answerYes{}{} 
    \item[] Justification: We make our code and data available at \url{https://anonymous.4open.science/r/CASH_with_Bandits/README.md}.

\item {\bf Experimental setting/details}
    \item[] Question: Does the paper specify all the training and test details (e.g., data splits, hyperparameters, how they were chosen, type of optimizer, etc.) necessary to understand the results?
    \item[] Answer: \answerYes{} 
    \item[] Justification:  Yes, in addition to making code available, we provide detailed information about all the methods we used, including the choice of hyperparameters. This includes a description of the computing infrastructure in Footnote~\ref{footnote:implementation}, an ablation study for hyperparameter selection in Section~\ref{Sec:Numerical_Experiments} and Appendix~\ref{app:aplation_study}, as well as metrics and experimental details in Section~\ref{app:details_on_experiments}. 

\item {\bf Experiment statistical significance}
    \item[] Question: Does the paper report error bars suitably and correctly defined or other appropriate information about the statistical significance of the experiments?
    \item[] Answer: \answerYes{} 
    \item[] Justification: We report confidence intervals concerning the random seed after running the experiments multiple times (see Figure~\ref{fig:average_rank} and Table~\ref{tab:sign_test} in the main paper), and we discuss this further in the Appendix~\ref{app:details_on_experiments}.

\item {\bf Experiments compute resources}
    \item[] Question: For each experiment, does the paper provide sufficient information on the computer resources (type of compute workers, memory, time of execution) needed to reproduce the experiments?
    \item[] Answer: \answerYes{} 
    \item[] Justification: We provide a description of the computing infrastructure in Footnote~\ref{footnote:implementation}.
    
\item {\bf Code of ethics}
    \item[] Question: Does the research conducted in the paper conform, in every respect, with the NeurIPS Code of Ethics \url{https://neurips.cc/public/EthicsGuidelines}?
    \item[] Answer: \answerYes{} 
    \item[] Justification: Our work complies with the NeurIPS Code of Ethics. It involves no human subjects, sensitive data, or misuse risks, and uses only public datasets.

\item {\bf Broader impacts}
    \item[] Question: Does the paper discuss both potential positive societal impacts and negative societal impacts of the work performed?
    \item[] Answer: \answerNA{} 
    \item[] Justification: This paper focuses on advancing the field of Machine Learning without direct societal impacts or specific applications that would need to be discussed.
    
\item {\bf Safeguards}
    \item[] Question: Does the paper describe safeguards that have been put in place for responsible release of data or models that have a high risk for misuse (e.g., pretrained language models, image generators, or scraped datasets)?
    \item[] Answer: \answerNA{} 
    \item[] Justification: Our work poses no known risks of misuse and does not involve high-risk models or datasets.

\item {\bf Licenses for existing assets}
    \item[] Question: Are the creators or original owners of assets (e.g., code, data, models), used in the paper, properly credited and are the license and terms of use explicitly mentioned and properly respected?
    \item[] Answer: \answerYes{} 
    \item[] Justification: All external assets (datasets and code) used in our work are publicly available, properly cited, and used in accordance with their respective licenses.

\item {\bf New assets}
    \item[] Question: Are new assets introduced in the paper well documented, and is the documentation provided alongside the assets?
    \item[] Answer: \answerNA{} 
    \item[] Justification: This paper does not introduce any new assets.

\item {\bf Crowdsourcing and research with human subjects}
    \item[] Question: For crowdsourcing experiments and research with human subjects, does the paper include the full text of instructions given to participants and screenshots, if applicable, as well as details about compensation (if any)? 
    \item[] Answer: \answerNA{} 
    \item[] Justification: This work does not involve crowdsourcing or research with human subjects.

\item {\bf Institutional review board (IRB) approvals or equivalent for research with human subjects}
    \item[] Question: Does the paper describe potential risks incurred by study participants, whether such risks were disclosed to the subjects, and whether Institutional Review Board (IRB) approvals (or an equivalent approval/review based on the requirements of your country or institution) were obtained?
    \item[] Answer: \answerNA{}{} 
    \item[] Justification:  This research does not involve human subjects or require IRB approval.

\item {\bf Declaration of LLM usage}
    \item[] Question: Does the paper describe the usage of LLMs if it is an important, original, or non-standard component of the core methods in this research? Note that if the LLM is used only for writing, editing, or formatting purposes and does not impact the core methodology, scientific rigorousness, or originality of the research, declaration is not required.
    \item[] Answer: \answerNA{} 
    \item[] Justification: We only used LLMs for text polishing and coding assistance.

\end{enumerate}

\end{document}